\documentclass[11pt,twoside]{book}
\usepackage{subcaption}
\usepackage{project}
\usepackage{imakeidx}
\usepackage{wrapfig}
\usepackage{bm}
\usepackage[noend]{algorithmic}
\usepackage{algorithm}
\usepackage{natbib}
\addto\captionsenglish{}

\def\projectauthor{Sarath Sreedharan, Anagha Kulkarni, Subbarao Kambhampati}
\def\projecttitle{\textbf{Explainable Human-AI Interaction:\\ A Planning Perspective}}

\def\projectdepartment{Arizona State University}
\newcommand{\todo}{\textcolor{black}}%
\newcommand{\unsure}{\textcolor{black}}%
\newtheorem{definition}{Definition}
\newtheorem{proposition}{Proposition}
\usepackage{amsmath}
\DeclareMathOperator*{\argmax}{arg\,max}
\DeclareMathOperator*{\argmin}{arg\,min}

\renewcommand{\algorithmicrequire}{\textbf{Input:}}
\renewcommand{\algorithmicensure}{\textbf{Output:}}

\begin{document}
    \prefrontmatter
    \frontmatter
    \pagenumbering{roman} 
    \cleardoublepage
    \chapter*{Abstract}
From its inception, AI has had a rather ambivalent relationship with humans---swinging between their augmentation and replacement. Now, as AI technologies enter our everyday lives at an ever increasing pace, there is a greater need for AI systems to work synergistically with humans. One critical requirement for such synergistic human-AI interaction is that the AI systems be explainable to the humans in the loop. To do this effectively, AI agents need to go beyond planning with their own
models of the world, and take into account the mental model of the human in the loop. Drawing from several years of research in our lab, we will discuss how the AI agent can use these mental models to either conform to human expectations, or change those expectations through explanatory communication. While the main focus of the book is on cooperative scenarios, we will point out how the same mental models can be used for obfuscation and deception. Although the book is primarily driven by our own research in these areas, in every chapter, we will provide ample connections to relevant research from other groups. 
\mbox{}\linebreak
\\
\\
\noindent {\bf Keywords:} Human-aware AI Systems, Human-AI Interaction, Explainability, Interpretability, Human-aware planning, Obfuscation
        %\markboth{Abstract}{}
    \cleardoublepage
    \pdfbookmark{\contentsname}{toc}
    \renewcommand{\sectionmark}[1]{\markright{#1}}
    \addtolength{\parskip}{-\baselineskip}  
    \tableofcontents
    % \begingroup
    %     \let\clearpage\relax
    %     \vspace*{20pt}
    %     \listoftables
    %     \vspace*{20pt}
    %     \listoffigures
    %     \vspace*{20pt}
    % \endgroup
    \addtolength{\parskip}{\baselineskip}
    \renewcommand{\sectionmark}[1]{\markright{\thesection\ #1}}
    \mainmatter
     
%\blankpage

{
\chapter*{Preface}
\addcontentsline{toc}{chapter}{\protect\numberline{}{Preface}}
%\thispagestyle{plain}
%\markboth{PREFACE}{PREFACE}

%\noindent
% This is the Preface.

% \todo{
% TODO for review draft (end of Feb):}
% \begin{enumerate}
%     \item Add Bibliographic Remarks to Applications chapter -- Done
%     \item Add introduction paragraph to Vocabulary mismatch Chapter -- Done
%     \item Split Bibliographic remarks for Chapter 2 -- Done
%     \item Complete Bibliographic remarks for Chapter 4-- Done
%     \item Complete Bibliographic remarks for Chapter 9 -- Done
%     \item Remove or change figure 3.1 -- Done
%     \item \todo{Identify and remove any unresolved reference or citations.}
%     \item incorporate all the comments from first review from Rao --Done
%     \item Fix or modify introduction
%     \item Update the intro para for chapter 6 and change the title for chapter 6 -- Done
%     \item (Maybe) Separate out the discussion on online stuff and not make it central to the definition --Done
%     \item{Search for and remove any references to "the article" and replace with "the book"} -- Done (also looked for paper and just article)
%     \item{Remove any red parts--before finalizing the version to send}
%     \item{Put an empty conclusion chapter--and say to be written} -- Done
% \end{enumerate}

Artificial Intelligence (AI) systems that interact with us the way we interact with each other have long typified Hollywood’s image, whether you think of HAL in “2001: A Space Odyssey,” Samantha in “Her,” or Ava in “Ex Machina.” It thus might surprise people that making systems that interact, assist or collaborate with humans has never been high on the technical agenda.

From its beginning, AI has had a rather ambivalent relationship with humans. The biggest AI successes have come either at a distance from humans (think of the “Spirit” and “Opportunity” rovers navigating the Martian landscape) or in cold adversarial faceoffs (the Deep Blue defeating world chess champion Gary Kasparov, or AlphaGo besting Lee Sedol). In contrast to the magnetic pull of these “replace/defeat humans” ventures, the goal of designing AI systems that are human-aware, capable of interacting and collaborating with humans and engendering trust in them, has received much less attention.

More recently, as AI technologies started capturing our imaginations, there has been a conspicuous change — with “human” becoming the desirable adjective for AI systems. There are so many variations — human-centered, human-compatible, human-aware AI, etc. — that there is almost a need for a dictionary of terms. Some of this interest arose naturally from a desire to understand and regulate the impacts of AI technologies on people. 
%
%In previous columns, I've looked, for example, at bias in AI systems and the impact of AI-generated synthetic reality, such as deep fakes or "mind twins."
Of particular interest for us are the 
%This time, let us focus on the 
challenges and impacts of AI systems that continually interact with humans — as decision support systems, personal assistants, intelligent tutoring systems, robot helpers, social robots, AI conversational companions, etc.

To be aware of humans, and to interact with them fluently, an AI agent needs to exhibit social intelligence. Designing agents with social intelligence received little attention when AI development was focused on autonomy rather than coexistence. Its importance for humans cannot be overstated, however. After all, evolutionary theory shows that we developed our impressive brains not so much to run away from lions on the savanna but to get along with each other.

A cornerstone of social intelligence is the so-called “theory of mind” — the ability to model mental states of humans we interact with. Developmental psychologists have shown (with compelling experiments like the Sally-Anne test) that children, with the possible exception of those on the autism spectrum, develop this ability quite early.

Successful AI agents need to acquire, maintain and use such mental models to modulate their own actions. At a minimum, AI agents need approximations of humans’ task and goal models, as well as the human’s model of the AI agent’s task and goal models. The former will guide the agent to anticipate and manage the needs, desires and attention of humans in the loop (think of the prescient abilities of the character Radar on the TV series “M*A*S*H*”), and the latter allow it to act in ways that are interpretable to humans — by conforming to their mental models of it — and be ready to provide customized explanations when needed.

With the increasing use of AI-based decision support systems in many high-stakes areas, including health and criminal justice, the need for AI systems exhibiting interpretable or explainable behavior to humans has become quite critical. The European Union’s General Data Protection Regulation posits a right to contestable explanations for all machine decisions that affect humans (e.g., automated approval or denial of loan applications). While the simplest form of such explanations could well be a trace of the reasoning steps that lead to the decision, things get complex quickly once we recognize that an explanation is not a soliloquy and that the comprehensibility of an explanation depends crucially on the mental states of the receiver. After all, your physician gives one kind of explanation for her diagnosis to you and another, perhaps more technical one, to her colleagues.

Provision of explanations thus requires a shared vocabulary between AI systems and humans, and the ability to customize the explanation to the mental models of humans. This task becomes particularly challenging since many modern data-based decision-making systems develop their own internal representations that may not be directly translatable to human vocabulary. Some emerging methods for facilitating comprehensible explanations include explicitly having the machine learn to translate explanations based on its internal representations to an agreed-upon vocabulary.

AI systems interacting with humans will need to understand and leverage insights from human factors and psychology. Not doing so could lead to egregious miscalculations. Initial versions of Tesla’s auto-pilot self-driving assistant, for example, seemed to have been designed with the unrealistic expectation that human drivers can come back to full alertness and manually override when the self-driving system runs into unforeseen modes, leading to catastrophic failures. Similarly, the systems will need to provide an appropriate emotional response when interacting with humans (even though there is no evidence, as yet, that emotions improve an AI agent’s solitary performance). Multiple studies show that people do better at a task when computer interfaces show appropriate affect. Some have even hypothesized that part of the reason for the failure of Clippy, the old Microsoft Office assistant, was because it had a permanent smug smile when it appeared to help flustered users.

AI systems with social intelligence capabilities also produce their own set of ethical quandaries. After all, trust can be weaponized in far more insidious ways than a rampaging robot. The potential for manipulation is further amplified by our own very human tendency to anthropomorphize anything that shows even remotely human-like behavior. Joe Weizenbaum had to shut down Eliza, history’s first chatbot, when he found his staff pouring their hearts out to it; and scholars like Sherry Turkle continue to worry about the artificial intimacy such artifacts might engender. Ability to manipulate mental models can also allow AI agents to engage in lying or deception with humans, leading  to a form of “head fakes” that will make today’s deep fakes tame by comparison. While a certain level of “white lies” are seen as the glue for human social fabric, it is not clear whether we want AI agents to engage in them.

As AI systems increasingly become human-aware, even quotidian tools surrounding us will start gaining mental-modeling capabilities. This adaptivity can be both a boon and a bane. While we talked about the harms of our tendency to anthropomorphize AI artifacts that are not human-aware, equally insidious are the harms that can arise when we fail to recognize that what we see as a simple tool is actually mental-modeling us. Indeed, micro-targeting by social media can be understood as a weaponized version of such manipulation; people would be much more guarded with social media platforms if they realized that those platforms are actively profiling them.

Given the potential for misuse, we should aim to design AI systems that must understand human values, mental models and emotions, and yet not exploit them with intent to cause harm. In other words, they must be designed with an overarching goal of beneficence to us.

All this requires a meaningful collaboration between AI and humanities — including sociology, anthropology and behavioral psychology. Such interdisciplinary collaborations were the norm rather than the exception at the beginning of the AI field and are coming back into vogue. 

Formidable as this endeavor might be, it is worth pursuing. We should be proactively building a future where AI agents work along with us, rather than passively fretting about a dystopian one where they are indifferent or adversarial. By designing AI agents to be human-aware from the ground up, we can increase the chances of a future where such agents both collaborate and get along with us.     

%{\bf new--to expand} 
This book then is a step towards designing such a future. We focus in particular on recent research efforts on making AI systems explainable. Of particular interest are settings where the AI agents make a sequence of decisions in support of their objectives, and the humans in the loop get to observe the resulting behavior. We consider techniques for making this behavior explicable to the humans out of the box, or after an explanation from the AI agent. These explanations are modeled as reconciliations of the mental models the humans have of the AI agents' goals and capabilities. The central theme of many of these techniques is reasoning with the mental models of the humans in the loop. While much of our focus is on cooperative scenarios, we also discuss how the same techniques can be adapted to support lies and deception in adversarial scenarios. In addition to the formal frameworks and algorithms, this book also discusses several applications of these techniques in decision-support and human-robot interaction scernarios.  While we focus on the developments from our group, we provide ample context of related developments across several research groups and areas of AI currently focusing on the explainability of AI systems. 

\vspace*{2pc}
%\noindent\AUTHORS\\
%\noindent February 2021
%}

%\clearpage
    {
\chapter*{Acknowledgements}
\addcontentsline{toc}{chapter}{\protect\numberline{}{Acknowledgements}}

This book is largely the result of a strand of research conducted at the Yochan Lab at Arizona State University over the last five years. The authors would like to express their sincere thanks to multiple past and present members of the Yochan group as well as external collaborators for their help and role in the development of these ideas.

First and foremost, we would like to thank Tathagata Chakraborti, who was the driving force behind multiple topics and frameworks described in this paper. Tathagata could not actively take part in the writing of this book, but his imprints are there throughout the manuscript. 

Other Yochan group members who played a significant role in the development of the ideas described in this book include Sailik Sengupta (currently at Amazon Science), Sachin Grover and Yantian Zha. Siddhant Bhambri and Karthik Valmeekam read a draft of the book and gave comments. 
We would also like to thank Yu (Tony) Zhang and Satya Gautam Vadlamudi, who were post-doctoral scholars at Yochan during the initial stages of this work.  

%External 
David Smith (formerly of NASA AMES) and Hankz Hankui Zhuo (of Sun-Yat Sen University) have been frequent visitors to our group; David in particular is a co-author on multiple papers covered in this book. Matthias Scheutz of Tufts has been a long standing collaborator; an ONR MURI project with him a decade back was the original motivation for our interest in human-AI teaming. Our colleague Nancy Cooke, an expert in human-factors and human-human teaming, has been a co-investigator on several of the projects whose results are included in this book. 

%Special thanks are also due to Matthias Scheutz of Tufts and Nancy Cooke--joint research projects with whom were responsible for our initial interest in explainable human-AI interaction.

Other external collaborators of the work described here include Siddharth Srivastava of ASU, Christian Muise (formerly of IBM AI Research, currently at Queens University, Canada) and Sarah Keren of Technion.

Additionally, we have benefited from our conversations on Human-AI interaction (as well as encouragement) from multiple colleagues within and outside ASU, including: 
Dan Weld (U. Washington), 
Barbara Grosz (Harvard), 
Manuela Veloso (J.P. Morgan Research),
Pat Langley (ISLE), 
Ece Kamar (MSR),
Been Kim (Google),
David Aha (NRL), 
Eric Horvitz (Microsoft), and
Julie Shah (MIT).

%{\bf *OTHER NAMES--Sarath and Anagha?*}

Much of the research reported here has been supported over the years by generous support from multiple federal funding agencies. We would like to particularly thank the Office of Naval Research--and the program managers Behzad Kamgar-Parsi, Tom McKenna,  Jeff Morrison, Marc Steinberg and John Tangney for their sustained support. Thanks are also due to  Benjamin Knott formerly of AFOSR, Laura Steckman of AFOSR, and Purush Iyer of Army Research Labs for their support and encouragement. 

Parts of the material in this book is drawn from our refereed research papers published in several venues; we cite all the sources in the bibliography. Parts of the preface are drawn from a column  on human-aware AI systems that first appeared in The Hill. 
%personal thanks--optional

%Put an ack to The Hill? 

%{\bf  OPTIONAL
%Sreedharan would like to thank
%Kulkarni would like to thank.. 
%Kambhampati would also like to thank}
%\\
\medskip
\medskip
   \newcommand{\droptext}[1]{{\em #1}}

\chapter{Introduction}
\label{ch01}

Artificial Intelligence, the discipline many of us call our intellectual home, is suddenly having a rather huge cultural moment. It is hard to turn anywhere without running into mentions of AI technology and hype about its expected positive and negative societal impacts. AI has been compared  to fire {\em and} electricity, and commercial interest in the AI technologies has sky rocketed. 
%Companies and investors are pouring money into the field. 
Universities -- even high schools -- are rushing to start new degree programs or colleges dedicated to AI. Civil society organizations are scrambling to understand the impact of AI technology on humanity, and governments are competing to encourage or regulate AI research and deployment. 

There is considerable hand-wringing by pundits of all stripes on whether in the future, AI agents will get along with us or turn on us. Much is being written about the need to make AI technologies safe and delay the ``doomsday''. We believe that as AI researchers, we are not (and cannot be) passive observers.  It is {\em our} responsibility to design agents that can and will get along with us. Making such {\em human-aware} AI agents, however poses several foundational research challenges that go beyond simply adding user interfaces \textit{post facto}. In particular, human-aware AI systems need to be designed such that their behavior is \textit{explainable} to the humans interacting with them. 
This book describes some of the state-of-the-art approaches in making AI systems explainable. 

%These then are the challenges we focus on in this book. 

%We will see in this book that addressing these  challenges  
%also 
%broadens the scope of AI in fundamental ways. 

%\marginnote{Making such {\em human-aware} AI agents, however poses several foundational research challenges that go beyond simply adding user interfaces \textit{post facto}. I will argue that addressing these challenges  also broadens the scope of AI in fundamental ways. }[3cm]

%%Commenting the Old Lady :0(
%\begin{figure}[!tph]
%\centering
%\includegraphics[width=\columnwidth]{figures/chapter-1/old-woman-ai-no-fake.jpg}
%\caption{\em We should build a future where AI systems can be our quotidian partners}
%\end{figure}

%\section{The need for Human-Aware AI Systems}

%Our primary aim with this book is to call for an increased focus on 

\section{Humans \& AI Agents: An Ambivalent Relationship}

In this book we focus on human-aware AI systems---goal directed autonomous systems that are capable of effectively interacting, collaborating and teaming with humans.
%\footnote{In a way, it thus  follows in the footsteps of Barbara Grosz's AAAI Presidential Address \cite{barbara-presidential}, which talked about collaborative systems.} 
Although developing such systems seems like a rather self-evidently fruitful enterprise, and popular imaginations of AI, dating back to HAL,  almost always assume we already do have human-aware AI systems technology, little of the actual energies of the AI  research community have gone in this direction. 
%It is worth asking why

From its inception, AI has had a rather ambivalent relationship to humans---swinging between their augmentation and replacement. Most high profile achievements of AI have either been far away from the humans---think Spirit and Opportunity exploring Mars; or in a decidedly adversarial stance with humans, be it Deep Blue, AlphaGo or Libratus. Research into effective ways of making AI systems {\em interact, team and collaborate with  humans}  has received significantly less attention. It is perhaps no wonder that many lay people have fears about AI technology! 

This state of affairs is a bit puzzling given the rich history of early connections between AI and psychology. 
Part of the initial reluctance to  work on these issues  had to do with the worry that focusing on AI systems working with human might somehow dilute the grand goals of the AI enterprise, and might even lead to  temptations of  ``cheating,'' with most of the intelligent work being done by the humans in the loop. After all, prestidigitation has been a concern  since the original mechanical turk.  Indeed, much of the early work on human-in-the-loop AI systems mostly focused on using humans as a crutch for making up the limitations of the AI systems \citep{allen1994mixed}. In other words, early AI had humans be ``AI-aware'' (rather than AI be ``human-aware'').

Now, as AI systems are maturing with increasing capabilities, the concerns about them depending on humans as crutches are less severe. We would also argue that  focus on humans in the loop doesn't dilute the goals of AI enterprise, but in fact broadens them in multiple ways. After all, evolutionary theories tell us that humans may have developed the brains they have,  not so much to run away from the lions of the savanna or tigers of Bengal but rather to effectively cooperate and compete with each other.
\index{Sally Anne Test}%
Psychological tests such as the Sally Anne Test
%\footnote{https://en.wikipedia.org/wiki/Sally%E2%80%93Anne_test} 
\citep{sally-anne-test} demonstrate the importance of such social cognitive abilities in the development of collaboration abilities in children. 
%\fromrao{is this too strong?}

%\fromrao{Barbara Grosz talked about collaborative systems. Eric Horvitz. Cynthia}
\index{Intelligent Tutoring Systems}%
\index{Social Robotics}%
Some branches of AI, aimed at specific human-centric applications, such as intelligent tutoring systems \citep{kurt-its1}, and social robotics \citep{cynthia-book,cynthia-ros,scaz-tom}, did focus on the challenges of human-aware AI systems for a long time. It is crucial to note however that human-aware AI systems are needed in a much larger class of quotidian applications beyond those. These include human-aware AI assistants for many applications where humans continue to be at the steering wheel, but will need naturalistic assistance from AI systems---akin to what they can expect from a smart human secretary.
%Should there be a mention of future of work? 
\droptext{Increasingly, as AI systems become common-place, human-AI interaction will be the dominant form of human-computer interaction} \citep{weld-chi}.

%It is for this reason that 
For all these reasons and more, human-aware AI has started coming to the forefront of AI research of late.  Recent road maps for AI research, including the 2016 JASON report\footnote{https://fas.org/irp/agency/dod/jason/ai-dod.pdf} and the 2016 White House OSTP report\footnote{https://obamawhitehouse.archives.gov/sites/default/files/whitehouse\_files/\\
microsites/ostp/NSTC/national\_ai\_rd\_strategic\_plan.pdf} emphasize the need for research in human-aware AI systems. The 2019 White House list of strategic R\&D priorities for AI lists ``developing effective methods for human-AI collaboration'' at the top of the list of priorities\footnote{https://www.whitehouse.gov/wp-content/uploads/2019/06/National-AI-Research-and-Development-Strategic-Plan-2019-Update-June-2019.pdf}. Human-Aware AI was the special theme for the 2016 International Joint Conference on AI (with the tagline ``{\em why intentionally design a dystopian future and spend time being paranoid about it?}''); it has been a  special track at AAAI since 2018. 
%When I was the program chair of IJCAI 2016 (International Joint Conference on Artificial Intelligence), we chose human-aware AI to be the special theme of the conference (with the tagline ``{\em why intentionally design a dystopian future and spend time being paranoid about it?}).
%One of the thematic pillars of Partnership for AI is collaborations between people and AI systems. 

%a concern that people can levy is that the important part from HCI view is computer systems that are "learning" based instead of programming based

%task vs. team level intervention

%IJCAI theme etc. 

%the COBOT project does bring a way of complementing.. 

\begin{figure}[t]
\centering
\includegraphics[width=5in]{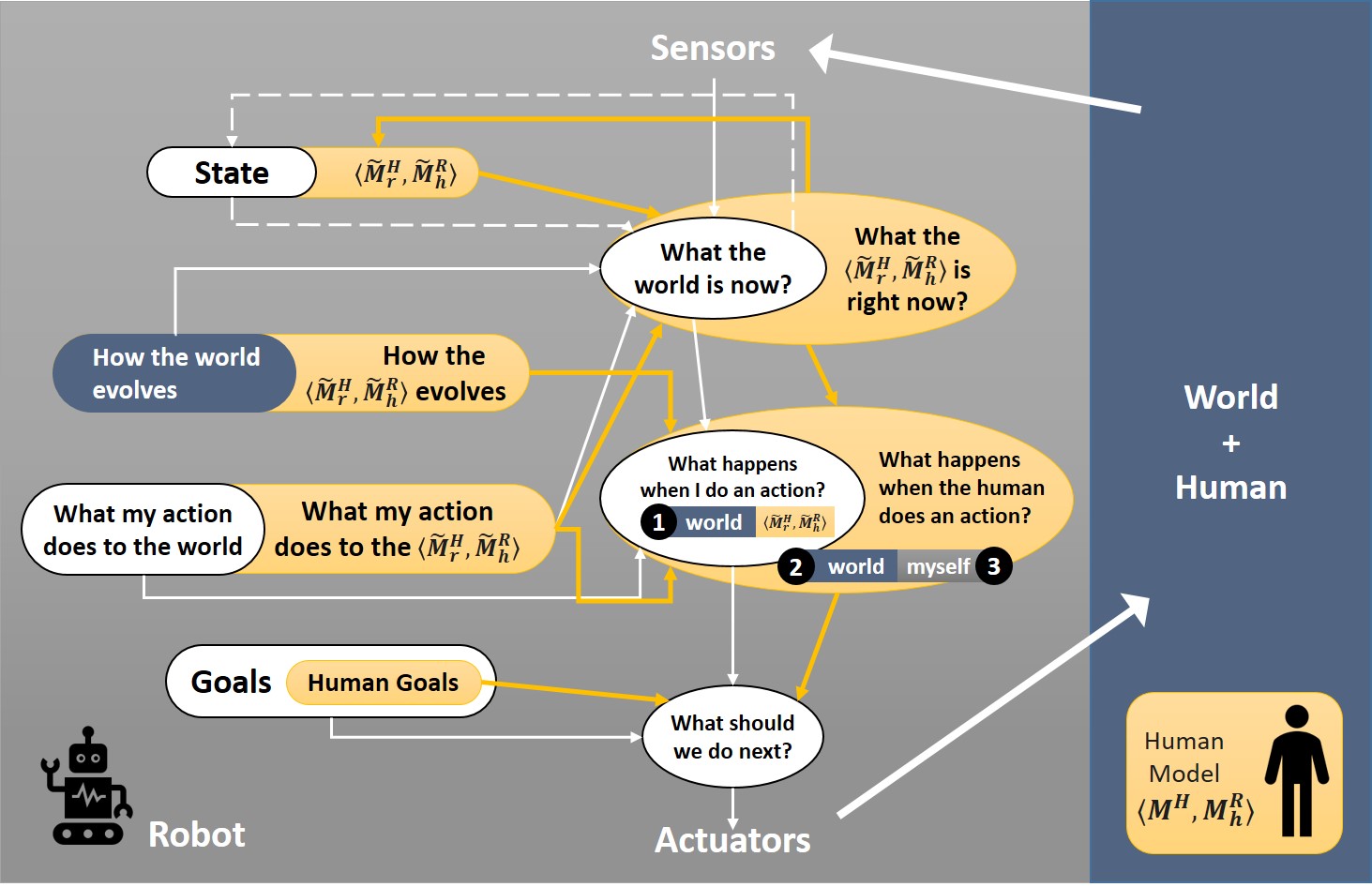}
%{newagent}
\caption{\em Architecture of an intelligent agent that takes human mental models into account. All portions in yellow are additions to the standard agent architecture, that are a result of the agent being human-aware. $\mathcal{M}^R_h$ is the mental model the human has of the AI agent's goals and capabilities and $\mathcal{M}^H_r$ is the (mental) model the  AI agent has of the human's goal and capabilities (see the section on Mental Models in Human-Aware AI)}
\label{agent}
\label{fig:newagent}
\end{figure}

%%%EXPLANATIONS
%%%%EXPLANATIONS 
%\section{Explainability in Humans and AI: A Broader View}
\section{Explanations in Humans}
\index{Explanation}%

Since our books is about explainable AI systems, it is useful to start with a broad overview of explainability and explanations in humans. 

%In this section, we will take a broader view of explainability and explanations, both in humans and as pursued by the various sub-fields of AI, with the motivation of properly situating the work described in this monograph in the context of the broader research on explainability. 
\subsection{When and Why do Humans expect explanations from each other?}

To understand the different use cases for explanations offered by  AI systems, it is useful to survey the different scenarios where humans ask for explanations from each other:

\begin{description}
    \item[When they  are confused and or surprised by the behavior of the other person]
    People expect explanations when the behavior from the other person is {\em not what they  expected}--and is thus \textit{inexplicable}. It is worth noting that this confusion at the other person's behavior is not predicated on that person's behavior being \textit{incorrect} or \textit{inoptimal}. We may well be confused/surprised when a toddler, for example, makes an optimal chess move. In other words, the need for explanations arises \textit{when the other person's behavior is not consistent with the  model we have of the other person}. Explanations are thus meant to \textit{reconcile} these misaligned expectations.

    % Our explanation as model reconcliation does exactly this!
    
    \item[When they  want to teach the other person] We offer explanations either to make the other person understand the real rationale behind a decision or to  convince the other person that our decision in this case is not a fluke. Explanatory dialog allows either party to correct the mental model of the other party. The explanations become useful in \textit{localizing the fault, if any,} in the other person's understanding our decision.

\end{description}

Note that the need for explanation is thus dependent on one person's model of the other person's capabilities/reasoning. Mental models thus play a crucial part in offering customized explanations. Indeed, a doctor explains her diagnostic decision  to her patient in one way and to her peers in a different (possibly more ``jargon-filled'' way), because she intuitively understands the levels of abstraction of the mental models they  have of diseases and diagnoses. 

It is also worth considering how the need for explanations \textit{reduces} over time. 
It naturally reduces as the mental models we have of the other agent gets better aligned with that other agent's capabilities, reducing any residual surprise at their behavior. This explains the ideal of ``wordless collaboration'' between long-time collaborators. 

The need for explanations is also modulated  by the trust we have in the other person. 
Trust can be viewed as our willingness to put ourselves in a position of vulnerability. When we trust the other person, even when their behavior/decision is inexplicable to  us, we might defer our demands for any explanations from that person. This explains why we ask fewer explanations from people we trust.

 %%ARE WE TALKING ABOUT THE interaction cycle too?????   

\subsection{How do Humans Exchange Explanations?}
\index{Tacit Explanations}%
\index{Explicit Explanations}%

We now turn to a  qualitative understanding of the ways in which humans  \textit{exchange} explanations, with a view to gleaning lessons for explanations in the human-AI interaction. While explanations might occur in multiple modalities, we differentiate two broad types: {\em Pointing explanations} and {\em Symbolic Explanations}. A given explanatory interaction might be interspersed with both types of explanations.

\begin{description}
\item[Pointing (Tacit) Explanations] These are the type of explanations where the main explanation consists of pointing to some specific features of the object that both the explainer and explainee can see. This type of explanations may well be the only feasible one to exchange when the agents share little beyond what they perceive in their immediate environment. Pointing explanations can get quite unwieldy (both in  terms of the communication bandwidth and the cognitive load for processing them)--especially when explaining sequential decisions (such as explaining why you took  an earlier flight than the one the other person expected)--as they will involve pointing to the relevant regions of the shared ``video history'' between the agents (or, more generally, \textit{space time signal tubes}).\index{Space Time Signal Tubes}%

\item[Symbolic (Explicit) Explanations]  These involve exchanging explanations in a symbolic vocabulary. Clearly, these require that the  explainer and explainee share a symbolic vocabulary to begin with. 
\end{description}

Typically, pointing explanations are used for tacit knowledge tasks, and symbolic explanations are used for explicit knowledge tasks. Interestingly, over time, people tend to develop symbolic vocabulary  even for exchanging explanations over tacit knowledge tasks. Consider, for example, terms such as \textit{pick and roll} in basket ball, that are used as a shorthand for a complex space time tube--even though the overall task is largely a tacit knowledge one.  

The preference for symbolic explanations is not merely because of their \textit{compactness}, but  also because they significantly  reduce the cognitive load on the receiver. This preference seems to exist despite the fact that the receiver may likely have to recreate in their own mind the ``video'' (space time signal tube) versions of symbolic explanations within their own minds. 

%\subsection{Why should Humans want explanations from AI systems?}

\subsection{(Why) Should AI Systems be Explainable?}
\index{Need for Explainability}%

Before we delve into explainability in AI systems, we have to address the view point that explainability from AI systems is largely unnecessary. Some have said, for example, that AI systems--such as those underlying Facebook--routinely make millions of decisions/recommendations (of the ``{\em you might be interested in seeing these pages}'' variety) and no users ask for explanations. Some even take the view that since AI systems might eventually be ``more intelligent'' than any mere mortal human, requiring them to provide explanations for their (obviously correct) decisions unnecessarily hobbles them.  Despite this, there are multiple reasons why we want AI systems to be explainable

\begin{itemize}

\item Since humans do expect explanations from each other, a naturalistic human-AI interaction requires that AI systems be explainable

\item Contestability of decisions is an important part of ensuring that the decisions are seen to be fair and transparent, thus engendering trust in humans. If AI systems make high stakes decisions, they too need to be contestable. 

\item Since there is always an insurmountable gap between the true preferences of humans and the AI system's estimate of those preferences, an explanatory dialog allows for humans to ``teach'' the AI systems their true preferences in a demand-driven fashion. 

\item Given that AI systems--especially those that are trained purely from raw data--may have unanticipated failure modes, explanations for their decisions often help the humans get a better sense of these failure modes. As a case in point, recently, it has been reported that when some Facebook users saw a video of an African male in a quotidian situation, the system solicitously asked the users \textit{Do you like to see more primate videos?}  As egregious as this ``heads-up'' of the failure mode of the underlying system sounds, it can be more insidious when the system just silently acts on the decisions arrived at through those failure modes, such as, for example, inexplicably filling the user's feeds with a slight uptick of primate videos. Explanations are thus a way for us to catch failure modes of these alien intelligences that we are increasingly surrounded by. 

\end{itemize}

\section{Dimensions of Explainable AI systems}

Now that we have reviewed how explanations and explainability play a part in human-human interactions, we turn to explainability in human-AI interactions. We will specifically  look at the use cases for explanations in human-AI interaction, some desirable requirements on explanations, and an overview of the ongoing research on explainable AI systems. 

%\subsection{Use cases for Human-in-the-loop Explanations}

\subsection{Use cases for explanations in Human-AI Interaction} 
%This can possibly go into the HAAI part earlier

When humans are in the loop with AI systems, they might be playing a variety of different roles. Being interpretable to one human in one specific role may or may not translate to interpretability for other humans in other roles. Perhaps the most popular role considered in the explainable AI literature to-date is  humans as debuggers trying to flag and correct an AI system's behavior. In fact, much of the explainable machine learning research has been focused on this type of debugging role for humans.  Given that the humans in the loop here have invested themselves into debugging the system, they  are willing to {\em go into the land of the AI agents}, rather than expect them to come into theirs. This lowers the premium on comprehensibility of explanations. 

Next we can have humans as observers of the robot's behavior -- either in a peer to peer, or student/teacher setting. The observer might be a lay person who doesn't have access to the robot's model of the task; or an expert one who does. 

Finally, the human might be a collaborator--who actively takes part in completing a joint task with the robot. 

No matter the role of the human, one important issue is whether the interaction between the human and the robot is a one-off one (i.e., they only interact once in the context of that class of tasks) or a longitudinal one (where the human interacts with the same robot over extended periods). In this latter case, the robot can engender {\em trust} in the human through its behavior, which, in turn reduces the need for interpretability. In particular, the premium on interpretability of the behavior itself is reduced when the humans develop trust over the capabilities and general beneficence of the robot. 

%Should we make connection to the $M^R_h$ and $M^H_r$ etc? 

\subsection{Requirements on Explanations}
\index{Requirements on Explanations}%
There are a variety of requirements that can be placed on explanations that an AI agent gives the human in the loop:

\begin{description}

\item[Comprehensibility:] The explanation should be comprehensible to the human in the loop. This not only means that it should be in terms that the human can understand, but should not pose undue cognitive load (i.e., expect unreasonable inferential capabilities). 

\item[Customization:] The explanations should be in a form and at a level that is accessible to the receiving party (explainee). 

\item[Communicability:] The explanation should be easy to exchange. For example, symbolic explanations are much easier than pointing explanations (especially in sequential decision problems when they have to point to space time signal tubes). 

\item[Soundness:] This is the guarantee from the AI agent that this explanation is really  the reason behind its decision. Such a guarantee also implicitly  implies that the agent will stand behind the explanation--and that the decision will change if the conditions underlying the explanation are falsified.  For example, if the explanation for a loan denial is that the applicant has no collateral, and the applicant then produces collateral, it is fairly expected that the loan denial decision will be reversed.

\item[Satisfctoriness:] This aims to measure how \textit{satisfied} the end user is with the explanation provided. While incomprehensible, uncustomized and poorly communicated explanations will likely be unsatisfatory to the user, it is also possible that the users are satisfied with misleading explanations that align well with what they want to hear. Indeed, making explanations satisfactory to the users is a slippery slope. Imagine the end-user explanations of the kind that the EU GDPR regulations require being provided by a system with a GPT-3 back-end generating plausible explanations that are likely to satisfy the users. This is why it is important for systems not to take an ``end to end'' machine learning approach and learn what kind of explanations make the end users happy. 

\end{description}
\subsection{Explanations as Studied in the AI Literature}
%{When and Why do AI systems give Explanations?}

Explanations have been studied in the context of AI systems long before the recent interest in the explainable AI. 
We will start by differentiating two classes of explanations: \textit{internal explanations} that the system develops to help its own reasoning, and \textit{external explanations} that the agent offers to other agents. 

\index{Internal Explanations}%
\index{Self Explanations}%
Internal explanations have been used in AI systems to guide their search (e.g. explanation-based or dependency directed backtracking), or to focus their learning. The whole area of explanation-based learning--which attempts to focus the system's learning element on the parts of the scenarios that are relevant to the decision at hand (thus providing feature relevance assessment). While much of the work in explanation-based search and learning have been done in the context of explicit symbolic models,  self explanations can also be useful for systems learning their own representations \cite{yantian-self-explanation}.

\index{External Explanations}%
External explanations may  be given by AI systems to other automated agents/AI systems (as is the case in autonomous and multi-agent systems research), or to other \textit{human} agents. There are however significant differences that arise based on whether the other agent is human or automated. In particular, issues such as cognitive load and inferential capacity in parsing the explanation play an important role when the other agent is a human, but not so much if it is automated. In particular, the work on certificates and proofs of optimality of the decision, which give gigabits of provenance information to support the decision, may work fine for automated agents but not human agents. 

Historically, (external)  explanations (to human agents) have been considered in the context of AI systems for as long as they have been deployed AI systems. There is a rich tradition of explanations in the context of expert systems. In all these cases, the form of the explanation does depend on (a) the role the human plays (debugger vs. lay observer) and (b) whether the task is an explicit knowledge one or a tacit knowledge one. 

%Should there be more connection to the tribes section?

\subsection{Explainable AI: The Landscape \& The Tribes}
\index{Explainable ML}%

The work in explainable AI systems can be classified into multiple dimensions. The first is whether what needs to be explained  is a single-decision (e.g. classification) task or a behavior resulting from a  sequence of decisions. Second dimension is whether the task is guided by explicit (verbalizable) knowledge or is a tacit task. The third is whether the interaction between the human and the AI agent is one-shot or iterative/longitudinal. Armed with these dimensions, we can discern several research ``tribes'' focusing on explainable AI systems:

\index{Shapley Values}%
\index{Saliency Maps}%
\paragraph{Explainable Classification} Most work on the so-called ``explainable ML'' focused on explainable classification. The classification tasks may be ``tacit'' in that the machine learns its own potentially inscrutable features/representations. 
A prominent subclass here is image recognition/classification tasks based on deep learning approaches. In these cases, the default communication between the AI agents and humans will be over the shared substrate of the (spatial) image itself. Explanations here thus amount to ``saliency annotations'' over the image--showing which pixels/regions in the image have played a significant part in the final classification decision.  
Of course, there are other classification tasks--such as loan approval or fraudulent transaction detection--where human-specified features, rather than pixels/signals are used as the input to the classifier. Here explanations can be in terms of the relative  importance of various features (e.g. shapley  values). 

\paragraph{Explainable Behavior} Here we are interested in sequential decision settings, such as human-robot or human-AI interactions, where humans and AI agents work in tandem to achieve certain goals. Depending on the setting, the human might be a passive observer/monitor, or an active collaborator. The objective here is for the AI agent (e.g. the robot) to exhibit behavior that is interpretable to the human in the loop. A large part of this work focuses on tacit interactions between the humans and robots (e.g. robots moving in ways that avoid colliding with the human, robots signaling which point they plan to go to etc.). The interactions are tacit in that the human has no shared vocabulary with the robot  other than the observed behavior of the robot. Here explainability or interpretability  typically  depend on the robot's ability to exhibit a behavior that helps the human understand its ``intentions''. Concepts that have been explored in this context include  {\em explicability}--the behavior being in conformance with human's expectation of the robot, {\em predictability}--the behavior being predictable over small time periods (e.g. next few actions), and {\em legibility}--the behavior signaling the goals of the robot. 

Of course, not all human-robot interactions have to be tacit; in many cases the humans and robots might share explicit knowledge and vocabulary about their collaborative task. In such cases, the robot can also ensure interpretability  through exchange of {\em explanations} in the shared symbolic vocabulary.  

Much of this monograph focuses on explainable behavior, in the context of explicit knowledge tasks. These are the scenarios where an AI agent's ability to plan its behavior up front interacts synergistically with  its desire to be interpretable to the human in the loop. We will focus first on scenarios where the human and the agent share a common vocabulary,  but  may have differing task models, and discuss how to ensure explicability or provide explanations in that scenario. Towards the end, we will also consider the more general scenario where even the vocabulary is not common--with the AI agent using its own internal representations to guide its planning and decision-making, and discuss how explanations can still be provided in human's vocabularly. 
%we will also consider cases where the vocabulary is different 

%Changing the placement of the explanations in humans and AI

\section{Our Perspective on Human-Aware and Explainable AI Agents} 
\index{Human-Aware AI}%

%naturalistic decision-making cites

In this section, we give a brief summary  of the the broad perspective taken in this book in designing human-aware and explainable AI systems. This will be followed in the next section by  the overview of the book itself. 

\subsection{How do we make AI agents Human-Aware?}

When two humans collaborate to solve a task, both of them will develop approximate models of the goals and capabilities of each other (the so called ``theory of mind''), and use them to support fluid team performance. AI agents interacting with humans -- be they embodied or virtual -- will also need to take this implicit mental modeling into account. This certainly poses several research challenges. Indeed, it can be argued that acquiring and reasoning with such models changes almost every aspect of the architecture of an intelligent agent. As an illustration, consider the architecture of an intelligent agent that takes human mental models into account shown in Figure~\ref{fig:newagent}. Clearly most parts of the agent architecture -- including state estimation, estimation of the evolution of the world, projection of its own actions, as well as the task of using all this knowledge to decide what course of action the agent should take -- are all critically impacted by the need to take human mental models into account. This in turn gives rise to many fundamental research challenges. %In \cite{cognitive-robot-teaming} we attempt to provide a survey of these challenges. 
%
%Rather than list the challenges again here, 
In this book, we will use the research in our lab to illustrate some of these challenges as well as our attempts to address them.
%\footnote{A longer bibliography of work related to human-aware AI from other research groups can be found at \url{http://rakaposhi.eas.asu.edu/cse591} as part of a graduate seminar at ASU on the topic.}  
Our work has focused on the challenges of human-aware AI in the context of human-robot interaction scenarios \citep{ppap}, as well as human decision support scenarios \citep{sengupta2017radar}. Figure~\ref{fig:testbeds} shows some of the test beds and micro-worlds we have used in our ongoing work. 

\begin{figure*}
\centering
\includegraphics[width=\textwidth]{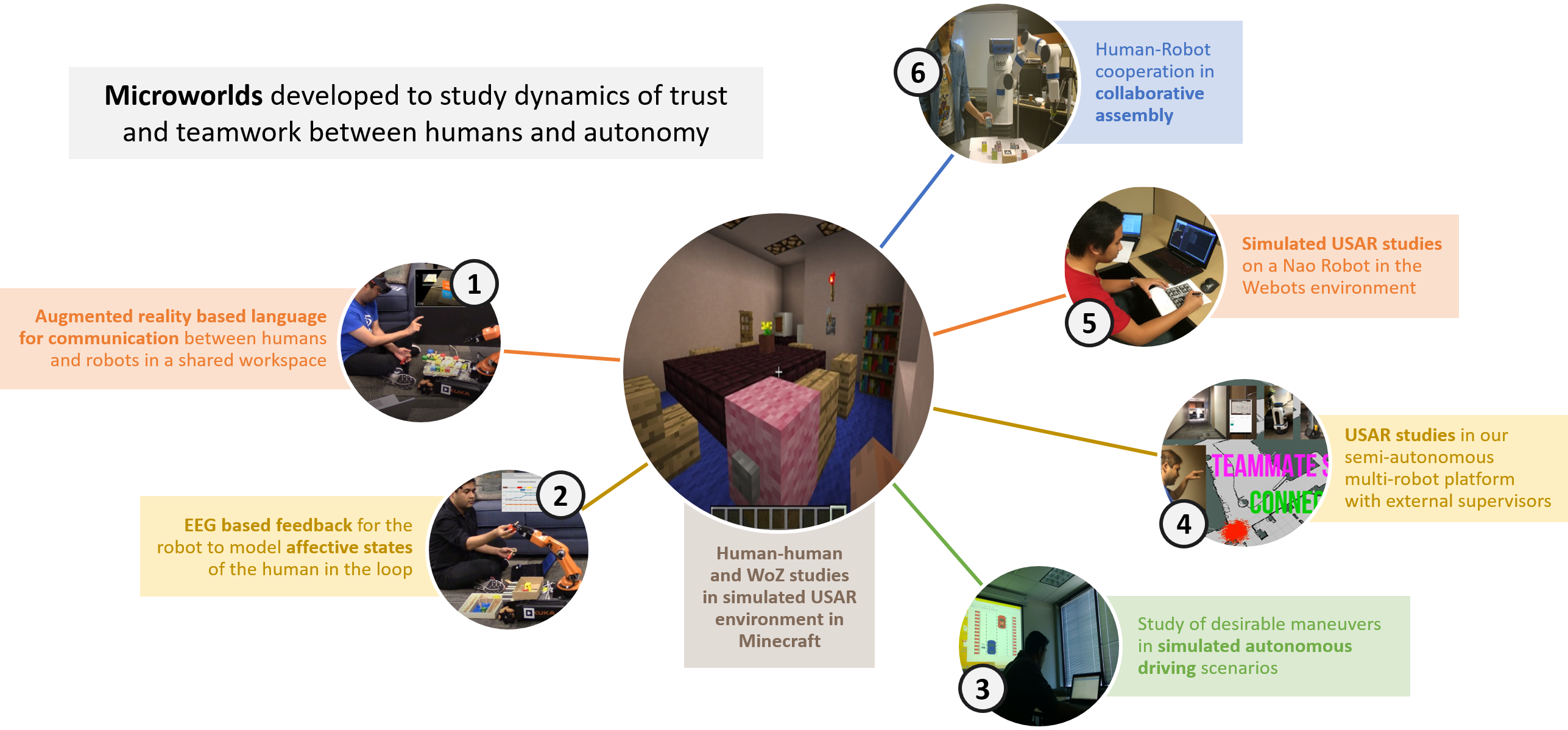}
\caption{\em Test beds developed to study the dynamics of trust and teamwork between autonomous agents and their human teammates.}
\label{fig:testbeds}
\end{figure*}

%\section{Mental Models in Human-Aware AI}
\subsection{Mental Models in Explainable AI Systems}
\index{Mental Models}%

In our research, we address the following central question in designing human-aware AI systems: {\em What does it take for an AI agent to show explainable behavior in the presence of humans?} Broadly put, our answer is this: {\em
To synthesize explainable behavior, AI
agents need to go beyond planning with their own
models of the world, and take into account the
mental model of the human in the loop. The
mental model here is not just the goals and
capabilities of the human in the loop, but
includes the human’s model of the AI agent’s
goals/capabilities.}

%In order for the AI agents to show behavior that makes sense to the human, they need to  go beyond planning with their own models of the world, and take into account the mental model of the human in the loop. The mental model here is not just the goals and capabilities of the humans in the loop, but includes the human's model of the AI agent's goals/capabilities. 
%show the new intelligent agent picture saying how everything changes?

\begin{figure}
\begin{center}
\begin{tabular}{c|c}
\includegraphics[width=2.5in]{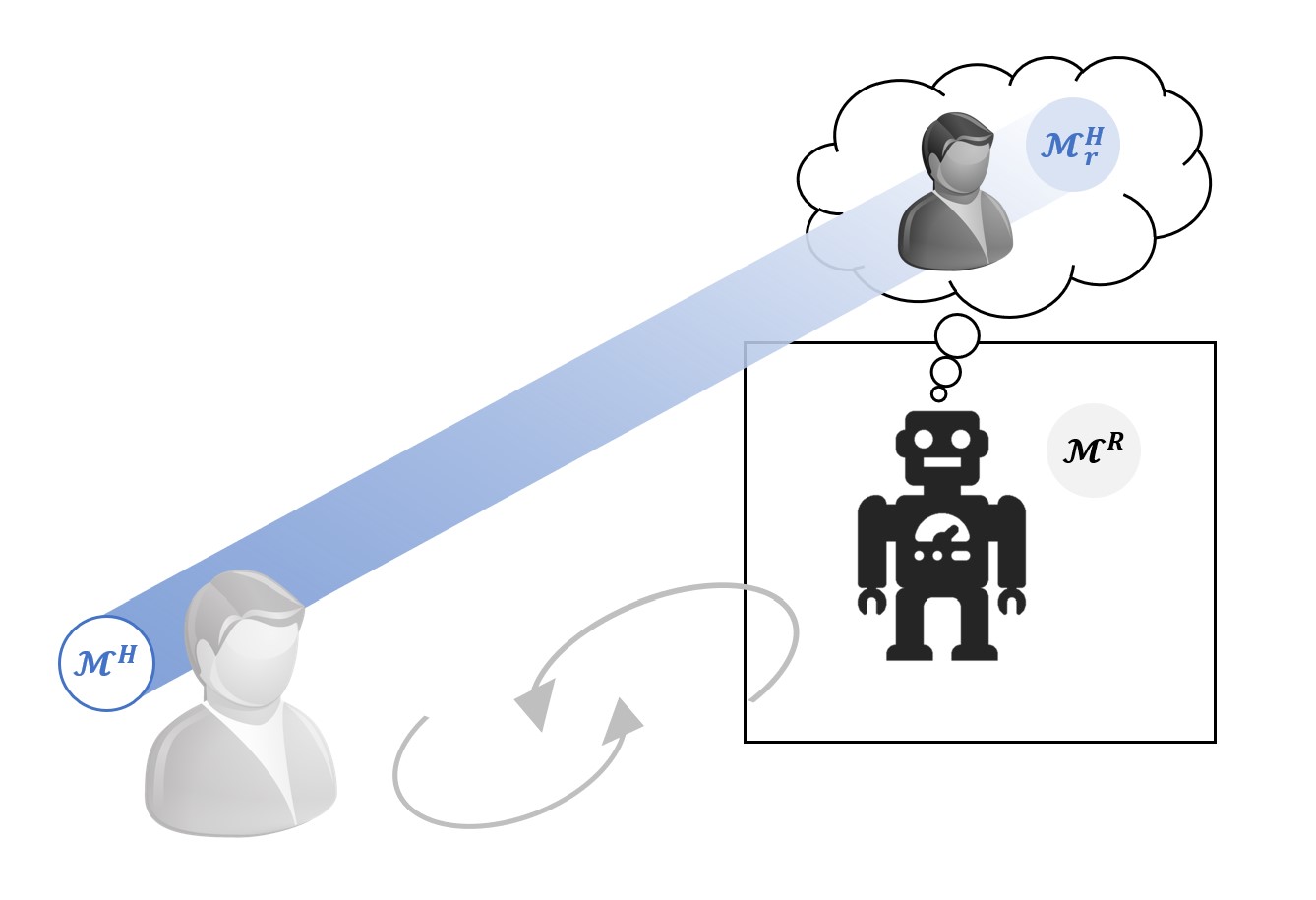} &
\includegraphics[width=2.5in]{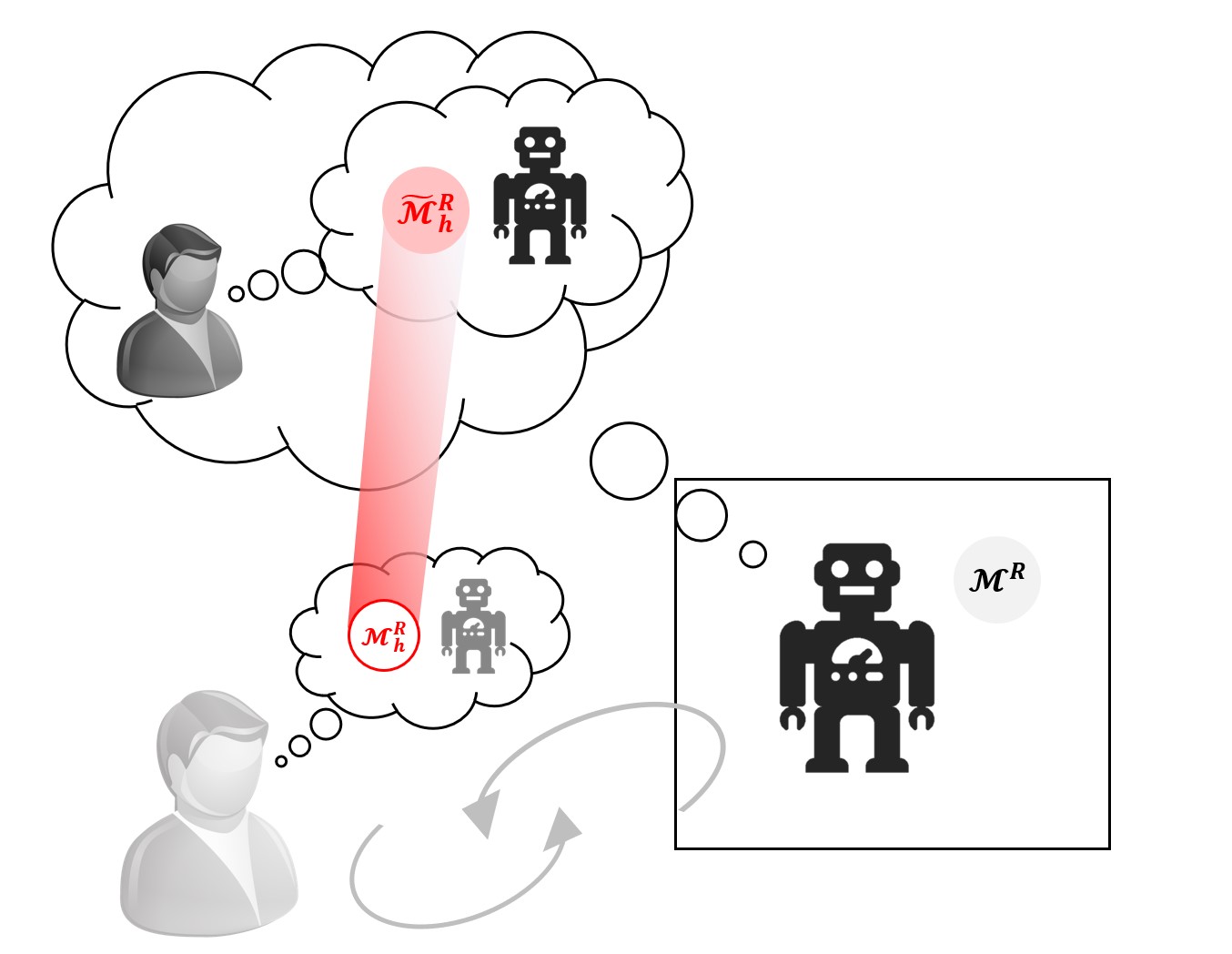}
\end{tabular}
\end{center}
\caption{\em Use of different mental models in synthesizing explainable behavior. (Left) The AI system can use its estimation of human's mental model, $\mathcal{M}^H_r$, to take into account the goals and capabilities of the human thus providing appropriate help to them. (Right) The AI system can use its estimation of human's mental model of its capabilities $\mathcal{M}^R_h$ to exhibit explicable behavior and to provide explanations when needed.}
\label{model-figure}
%\end{wrapfigure}
\end{figure}

%In particular, it is useful to distinguish between these models: 
Let $\mathcal{M}^R$ and $\mathcal{M}^H$ correspond to the actual goal/capability models of the AI agent and human. To support collaboration, the AI agent needs an approximation of $\mathcal{M}^H$, we will call it $\widetilde{\mathcal{M}}^H_r$, to take into account the goals and capabilities of the human. The AI agent also needs to recognize that the human will have a model of its goals/capabilities $\mathcal{M}^R_h$, and needs an approximation of this, denoted $\widetilde{\mathcal{M}}^R_h$. It is important to note that while $\mathcal{M}^R$ and $\mathcal{M}^H$ are intended to be ``executable models,'' in that courses of action consistent with them are in fact executable by the corresponding agent--robot or human, $\mathcal{M}^R_h$ and $\mathcal{M}^H_r$ are models of ``{\em expectation}'' by the other agent, and thus may not actually be executable by the first agent.   All phases of the ``sense--plan--act'' cycle of an intelligent agent will have to change appropriately to track the impact on these models (as shown in Figure~\ref{fig:newagent}). 

Of particular interest to us in this book is the fact that synthesizing explainable behavior  becomes a challenge of supporting planning in the context of these multiple models. In particular, we shall see that the AI agent uses $\mathcal{M}^H_r$ to anticipate the human behavior and provide appropriate assistance (or at least get out of the way), while the agent uses $\widetilde{\mathcal{M}}^R_h$, its estimate of $\mathcal{M}^R_h$, to understand human's expectation on its behavior, and use that understanding to either {\em conform} to the human expectation or actively {\em communicate} with the human to induce them to change $\mathcal{M}^R_h$. We discuss the conformance aspect in terms of generating {\em explicable behavior}. For the model communication part, either the communication can be done {\em implicitly}--which is discussed in terms of generating {\em legible behavior}, or can be done {\em explicitly}--with communication actions. This latter part is what we view as the ``explanation'' process, and address multiple challenges involved in generating such explanations.  Finally, while much of the book is focused on cooperative and non-adversarial scenarios, the mental model framework in Figure~\ref{model-figure} can also be used by the agent to obfuscate its behavior or provide deceptive communication. We also discuss how such selective obfuscation and deception is facilitated. 

%deception

%Mention that M^R_h etc are not executable. 

%In the following, we will look at some specific issues and capabilities provided by such human-aware AI agents. 
%A note on the model representation: In much of our work, we have used relational precondition-effect models. We believe however that our frameworks can be readily adapted to other model representations; e.g. \cite{modelfree}.

%There is the learning angle.. 
%should mention the AAMAS/IROS papers on serendipity and resource contention

%{\bf Rao: Sarath--the chapter overview can come here.. }

\section{Overview of this Book}

Through the rest of the book, we will look at some of the central challenges related to human-aware planning that arise due to and can be addressed through working with the human's mental model. In particular, we will ground our discussions of the topics within the context of using such models to generate either interpretable or deceptive behavior. The book is structured as follows:

\paragraph*{Chapter \ref{ch02}} In this chapter, we will focus on formally defining the goal-directed deterministic planning formalisms and the associated notations that we will be using to study and ground the technical discussions throughout this book. We will also define the three main interpretability metrics; namely, {\em Explicability}, {\em Legibility}, and {\em Predictability}. We will be revisiting these three measures in the following chapters. We will see how the different methods discussed throughout the book relate to these measures, and in various cases could be understood as being designed to optimize, at the very least a variant of these measures.

\paragraph*{Chapter \ref{ch03}} We next focus on one of these measures, namely, explicability, and see how we could allow the robot to choose plans that maximize explicability scores. In particular, we will look at two main paradigms to generate such {\em explicable plans}. First, we consider a model-based method called {\em reconciliation search} that will use the given human model along with a distance function (which could potentially be learned) to generate robot plans with high explicability scores. Then we will look at a model-free paradigm, where we use feedback from users to learn a proxy for the human model in the form of a labeling model and use that simpler model to drive the explicable planning. We will also look at how one could use environment design to allow the agents to generate plans with higher explicability scores. Within the design framework, we will also look at the evolution of plan explicability within the context of longitudinal interactions.

\paragraph*{Chapter \ref{ch04}} We next turn our focus onto legibility. We look at how the robot can reduce the human observer's uncertainty over its goals and plans, particularly when the observer has imperfect observations of its activities. The robot can reduce the observer's uncertainty about its goals or plans by implicitly communicating information through its behavior i.e. by acting legibly. We formulate this problem as a controlled observability planning problem, where in the robot can choose specific actions that allow it to modulate the human's belief over the set of candidate robot goals or plans. Further we will also discuss the relationship between plan legibility as discussed in this framework with the notion of plan predictability defined in the literature.

\paragraph*{Chapter \ref{ch05}} In this chapter, we return to the problem of maximizing explicability and investigate an alternate strategy, namely explanations. Under this strategy, rather than letting the robot choose possibly suboptimal plans that may better align with human expectations, the robot can follow its optimal plans and provide explanations as to why the plans are in fact optimal in its own model. The explanation in this context tries to resolve any incorrect beliefs the human may hold about the robot and thus helps to reconcile the differences between the human's mental model and the robot's model. In the chapter, we will look at some specific types of model reconciliation explanations and a model space search algorithm to generate such explanation. We will also consider some approximations for such explanations and discuss some user studies that have been performed to validate such explanations.

\paragraph{Chapter \ref{ch06}} In this chapter, we continue our discussion on model reconciliation and focus on one specific assumption made in the earlier chapter, namely that the human's mental model may be known. We will study how one could still generate model reconciliation explanations when this assumption may not be met. In particular, we will consider three cases of incrementally decreasing access to human mental model. We will start by considering cases where an incomplete version of the model may be available, then we will look at scenarios where we can potentially learn a model proxy from human data and finally we will see the kind of explanatory queries we can field by just assuming a prototypical model for the human.

\paragraph{Chapter \ref{ch_balance}} In this chapter, we return back to the problem of explicability and look at how one could combine the two strategies discussed in earlier chapters for addressing explicability; namely, explicable planning and explanation. In fact we will see how one could let the robot trade-off the cost of explaining the plan against the overhead of choosing a potentially suboptimal but explicable plan and will refer to this process as {\em balanced planning}. In the chapter, we will look at some particular classes of balanced planning and introduce a concept of {\em self-explaining plans}. We will also introduce a compilation based method to generate such plans and show how the concept of balancing communication and behavior selection can be extended to other interpretability measures. 

\paragraph{Chapter \ref{ch07}} In this chapter, we look at yet another assumption made in the generation of explanations, namely, that there exists a shared vocabulary between the human and the robot through which it can communicate. Such assumptions may be hard to meet in cases where the agent may be using learned and/or inscrutable models. Instead we will look at how we could learn post-hoc representations of the agent's model using concepts specified by the user and use this representation for model reconciliation. We will see how such concepts could be operationalized by learning classifiers for each concepts. In this chapter, we will also look at how we could calculate confidence over such explanations and how these methods could be used when we are limited to noisy classifiers for each concept. Further, we will discuss how we could acquire new concepts, when the originally specified set of concepts may not be enough to generate an adequate representation of the model.

\paragraph{Chapter \ref{ch08}} In this chapter, we turn our attention to adversarial settings, where the robot may have to act in obfuscatory manner to minimize the leakage of sensitive information. In particular, we look at settings where the adversary has partial observability of the robot's activities. In such settings, we show how the robot can leverage the adversary's noisy sensor model to hide information about its goals and plans. We also discuss an approach which maintains goal obfuscation even when the adversary is capable of diagnosing the goal obfuscation algorithm with different inputs to glean additional information. Further, we discuss a general setting where both adversarial as well as a cooperative observers with partial observations of robot's activities may exist. We show that in such a case, the robot has to balance the amount of information hidden from the adversarial observer with the amount of information shared with the cooperative observer. We will also look at the use of deceptive communication in the form of lies. We will show how we can leverage the tools of model reconciliation to generate lies and we will also discuss how such lies could in fact be used to benefit the human-robot team.

\paragraph{Chapter \ref{ch09}} In this chapte, we provide a discussion of some of the applications based on the approaches discussed in the book. We will look at two classes of applications. In the first, we look at decision support systems and in particular the methods related to explanations could be used to help decision-makers better understand the decisions being proposed by the user. In the second case, we will look at a specific application designed to help a user with transcribing a declarative model for the task. This application helps in identifying mistakes in current version of the model by reconciling against an empty user model.

\todo{\paragraph{Chapter \ref{ch10}} Finally we will close the book with a quick discussion on some of the most challenging open problems in Human-Aware AI.}

Although this monograph is informed primarily by our group's work, in each chapter, we provide bibliographic remarks at the end connecting the material discussed to other related works.

\clearpage

                % various environments like theorem,
    \chapter{Measures of Interpretability}
\label{ch02}
\index{measures}%

%\todo{Connect to earlier chapter}
\unsure{This chapter will act as the introduction to the technical discussions in the book. We will start by establishing some of the basic notations that we will use, including the definitions of deterministic goal-directed planning problems, incomplete planning models, sensor models, etc. 
%We would welcome the reader to do a quick perusal of these definitions, even if they are familiar with the basic formalisms as the notation we will be using in this book may differ from ones the reader may be familiar with. We will not be repeating these definitions in the latter chapters. 
With the basic notations in place, we will then focus on establishing the three main interpretability measures in human-aware planning; namely, {\em Explicability, Legibility, and Predictability}. We will revisit two of these measures (i.e., explicability and legibility) and discuss methods to boost these measures throughout the later chapters.}

\section{Planning Models}
\index{Planning Model}%
\index{STRIPS}%
\index{Execution Semantics}%
\index{Transition Function}%
\index{Goal-Directed Planning Model}%
\index{Deterministic Planning Model}%
\index{Preconditions}%
\index{Effects}%
\index{State Fluents@Fluents}%
\index{Action Cost}%
\index{Plan Validity}%
\index{Plan Optimality}%
\index{Goals}%

We will be using goal-oriented STRIPS planning models to represent the planning problems used in our discussions. However, the ideas discussed in this book apply to any of the popular planning representations.
Under this representation scheme, a planning model (sometimes also referred to as a planning problem) can be represented by the tuple $\mathcal{M} = \langle F, A, I, G, C \rangle$, where the elements correspond to
\begin{itemize}
    \item $F$ - A set of propositional fluents that define the space of possible task states. Each state corresponds to a specific instantiation of the propositions. We will denote the set of states by $S$. When required we will uniquely identify each state by the subset of fluents which are true in the given state. \unsure{ This representation scheme implicitly encodes the fact that any proposition from F not present in the set representation of the state is false in the underlying state.}
    \item $A$ - The set of actions that are available to the agent. Under this representation scheme, each action $a_i \in A$ is described by a tuple of the form 
    \[a_i = \langle  \textrm{pre}(a_i),  \textrm{adds}(a_i),  \textrm{dels}(a_i)\rangle,\] where
    \begin{itemize}
        \item $\textrm{pre}(a_i)$ - The preconditions for executing the action. For most of the discussion, we will follow the STRIPS execution semantics, wherein an action is only allowed to execute in a state where the preconditions are `met'. 
        \unsure{In general, preconditions can be any logical formula over the propositional fluents provided in $F$. Moreover, we would say an action precondition is met in a given state if the logical formula holds in that state (remember a state here correspond to a specific instantiation of fluents) %set of true propositional facts and a set of false propositional facts).
        } 
        We will mostly consider cases where the precondition is captured as a conjunctive formula over a subset of state fluents, which we can equivalently represent as a set over these fluents.
        \unsure{This representation allows us to test whether a precondition holds by  checking if the set of fluents that are part of the precondition is a subset of the fluents that are true in the given state.}
        Thus the action $a_i$ is executable in a state $s_k$, if $\textrm{pre}(a_i) \subseteq s_k$.
        \item $\textrm{adds}(a_i)/\textrm{dels}(a_i)$ - The add and delete effects of the action $a_i$ together captures the effects of executing the action in a state. The add effects represent the set of state fluents that will be made true by the action and the delete effects capture the state fluents that will be turned false. Thus executing an action $a_i$ in state $s_j$ results in a state $s_k = (s_j \setminus \textrm{dels}(a_i)) \cup \textrm{adds}(a_i)$
    \end{itemize}
\item $I$ - The initial state from which the agent starts. 
\item $G$ - The goal that the agent is trying to achieve. Usually the goal is considered to be a partially specified state. That is, the agent is particularly interested in achieving specific values for certain state fluents and that the values of other state fluents do not matter. Thus any state where the specified goal fluent values are met are considered to be valid goal states.
\item $C$ - The cost function ($C: A \rightarrow \mathbb{R}_{>0}$) that specifies the cost of executing a particular action.
\end{itemize}

Given such a planning model, the solution takes the form of a plan, which is a sequence of actions. A plan $\pi = \langle a_1, ..., a_k\rangle$ is said to be a valid plan for a planning model $\mathcal{M}$, if executing the sequence of the plan results in a state that satisfies the goal. Given the cost function, we can also define the cost of a plan,  
$C(\pi) = \sum_{a_i \in \pi} C(a_i)$. 
%We will generally assume the cost of invalid plans to be $\infty$.
A plan, $\pi$, is said to be optimal if there exists no valid plan that costs less than $\pi$. 
\unsure{We will use the term behavior to refer to the observed state action sequence generated by the execution of a given plan.}
\todo{For models where all actions have unit cost, we will generally skip $C$ from the model definition}.
We will use the superscript `$*$' to refer to optimal plans. Further, we will use the modifier `~$\bar{}$~' to refer to prefixes of a plan and `$+$' operator to refer to concatenation of action sequences. \todo{In cases where we are comparing the cost of the same plan over different models, we will overload the cost functions $C$ to take the model as an argument, i.e., we will use $C(\pi, \mathcal{M}_1)$ to denote the cost of plan $\pi$ in the model $\mathcal{M}_1$, while $C(\pi, \mathcal{M}_2)$ denotes its cost in $\mathcal{M}_1$. Additionally, we will use the notation $C_{\mathcal{M}}^*$ to denote the cost of an optimal plan in the model $\mathcal{M}$.}

Since most of the discussion in this book will rely on reasoning about the properties of such plan in different models, we will use a transition function $\delta$ to capture the effect of executing the plan under a given model, such that the execution of an action $a_i$ in state $s_j$ for a model $\mathcal{M}$, where $\delta(s_j, a, \mathcal{M})$, gives the state that results from the execution of action $a$ in accordance with the model $\mathcal{M}$.

%\[
% \delta(s_j, a, \mathcal{M}) = s_j \setminus \textrm{dels}_{a_i} \cup \textrm{adds}_{a_i}
% \]
% where, the preconditions and effects of an action are as defined for the model $\mathcal{M}$.

\paragraph{Incomplete Models} 
\index{Annotated Model}%
\index{Incomplete Model}%
In this book, we will also be dealing with scenarios where the model may not be completely known. In particular, we will consider cases where the specification may be incomplete insofar as we do not know every part of the model with absolute certainty, instead we will allow information on parts of the model that may be possible. To represent such models, we can follow the conventions of an annotated model, wherein the model definition is quite similar to STRIPS one, except that each action $a_i$ is now defined as $a_i = \langle \textrm{pre}(a_i), \textrm{poss\_prec}_{a_i},
\textrm{adds}(a_i),
\widetilde{\textrm{pre}}(a_i), \widetilde{\textrm{dels}}(a_i), \widetilde{\textrm{dels}}(a_i)\rangle$, where `$\widetilde{\textrm{pre}}$', `$\widetilde{\textrm{adds}}$' and `$\widetilde{\textrm{dels}}$' represent the possible preconditions, adds and deletes of an action. If a fluent $f$ is part of such a possible list, say a possible precondition, then we are, in effect, saying that we need to consider two possible versions of action $a_i$; one where it has a precondition $f$ and one where it does not. If the model in total contains $k$ possible model components (including preconditions and effects), then it is in fact representing $2^k$ possible models (this set of possible models are sometimes referred to as the \textit{completion set} of an annotated model).

\paragraph{Sensor Models}
\index{Sensor Model}%
\index{Observation Function}%
\index{Observation Tokens}%
\unsure{Most of the settings discussed in the book contain multiple agents and require one of the agents to observe and make sense of the actions performed by the other. We will refer to the former agent as the observer and the other as the actor (though these roles need not be fixed in a given problem). In many cases, the observer may not have perfect visibility of the actor's activities.} In particular, we will consider cases where the observer's sensor model may not be able to distinguish between multiple different activities performed by the actor. To capture cases like this, we will use the tuple, $Obs = \langle \Omega, O \rangle$ to specify the sensor model of the observer, where the elements correspond to:
\begin{itemize}
\item $\Omega$ - A set of observation tokens that are distinguishable by the observer. We will use $\omega$ to denote an observation token, and $\langle \omega \rangle$ to denote a sequence of tokens.
\item $O$ - An observation function ($O: A \times S \rightarrow \Omega$) maps an action performed and the next state reached by the actor to an observation token in $\Omega$. This function captures any limitations present in the observer's sensor model. If the agent cannot distinguish between multiple different activities, we say that the agent has partial observability. Given a sequence of tokens, $\langle \omega \rangle$, a plan $\pi$ is consistent with this sequence if and only if the observation at any step could have been generated by the corresponding plan step (denoted as $\langle \omega \rangle \models \pi$).
\end{itemize} 

\paragraph{Models in Play in Human-Aware Planning Problems}
\index{Mental Models}%
Throughout most of this book, we will use the tuple $\mathcal{M}^R = \langle F^R, A^R, I^R, G^R, C^R\rangle$ to capture the robot's planning model that it uses to plan its actions, $\mathcal{M}^H = \langle F^H, A^H, I^H, G^H, C^H\rangle$ the model human may use to capture their own capabilities and plan their behavior and $\mathcal{M}^R_h = \langle F^R_h, A^R_h, I^R_h, G^R_h, C^R_h\rangle$ the mental model the human maintains of the robot. 
In cases where we want to differentiate between the different definitions of the same action under different models, we will use the specific model name in the superscript to differentiate them. For example, we will refer to the definition of action $a_i$ in the robot's original model as $a_i^{\mathcal{M}^R}$, while the human's expectation of this model will be captured as $a_i^{\mathcal{M}^R_h}$. In cases where the human is maintaining an explicit set of possible robot models then we will use $\mathbb{M}^R_h$ to represent this set.
\unsure{In this book, we will be using the term robot to refer to the autonomous agent that will be acting in the world and interacting with the human. Though, by no means are the methods discussed in the book are limited to physically embodied agents. In fact, in Chapter \ref{ch09}, we will see many examples of software agents capable of using the same methods.
Much of the discussions in this book will be focused on cases where the human teammate is merely an observer, and thus in terms of human models we will be focusing on $\mathcal{M}^R_h$. For the sensor models, since the human will be the observer, we will denote the human sensor model by the tuple, $Obs^H_r = \langle \Omega^H_r, O^H_r \rangle$. }

%With respect to the sensor models, we will be mostly focusing on settings where the human has partial observability of the robot's actions denoted by the sensor model, $Obs^H_r = \langle \Omega^H_r, O^H_r \rangle$. Since we will be considering settings where the robot has full observability of its own behavior, we will not be dealing with $Obs^R$.

% \section{Running Example}

% To demonstrate different types of interpretable behaviors we will be using a grid-based Urban Search and Rescue domain. This domain involves an autonomous robot (internal agent) operating in a disaster area, and a human commander (external agent) overlooking the robot's behavior. The grid represents the floor map of the floor where the robot is operating. The human's mental model of the floorplan (which is the older floorplan before the disaster) may be different from the updated floorplan available to the robot. In some situations, even if the human commander has access to the updated floorplan, she may not have full observability of the robot's activities and may have a belief over the robot's true state based on the observation tokens. In this domain, the robot has three types of actions: (1) moving from one cell to an adjacent cell in the grid (2) surveying rooms and (3) identifying victims. 

\section{Modes of Interpretable Behavior}
\index{Modes of Interpretable Behavior}%
\index{Explicability}%
\index{Explicability!Score}%
\index{Legibility}%
\index{Legibility!Score}%
\index{Predictability}%
\index{Predictability!Score}%
\unsure{ The term interpretable behavior has been used to cover a wide variety of behaviors. However, one unifying thread that runs throughout the different views is the fact that they are all strategies designed to handle the asymmetry between what the human knows about the robot (in terms of its model or the plan it is following) and the robot's model or plans. A robot's actions may be uninterpretable 
when it does not conform to the expectations or predictions 
engendered by the human's model of the agent.
In this chapter, we will look at three notions of interpretability; namely {\em \textbf{explicability, legibility, and predictability}}. Each of these interpretability types captures a different aspect of the asymmetry. For each type, we will define a scoring function, that maps a given behavior and the human's beliefs about the robot to a score. Thus a robot that desires to exhibit a certain type of interpretable behavior can do so by selecting behaviors that maximize the given score. Figure \ref{fig:fig}, present some sample behavior illustrating the various interpretability measures.}
\begin{figure*}
    \centering
    \begin{subfigure}[b]{0.4\textwidth}
	    \centering
        \includegraphics[width=0.4\textwidth]{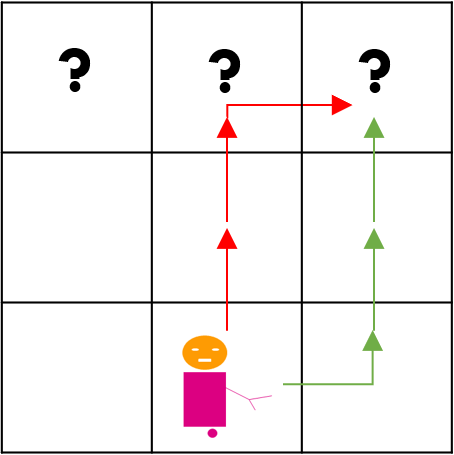}
        \caption{Plan legibility / transparency.}
        \label{fig:1}
    \end{subfigure}
    \hfill %add desired spacing between images, e. g. ~, \quad, \qquad, \hfill etc. 
    \begin{subfigure}[b]{0.4\textwidth}
	    \centering
        \includegraphics[width=0.4\textwidth]{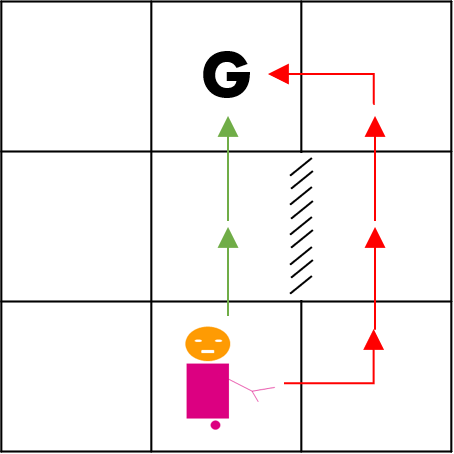}
        \caption{Plan explicability.}
        \label{fig:2}
    \end{subfigure}
    \hfill %add desired spacing between images, e. g. ~, \quad, \qquad, \hfill etc. 
    \begin{subfigure}[b]{0.4\textwidth}
	    \centering
        \includegraphics[width=0.4\textwidth]{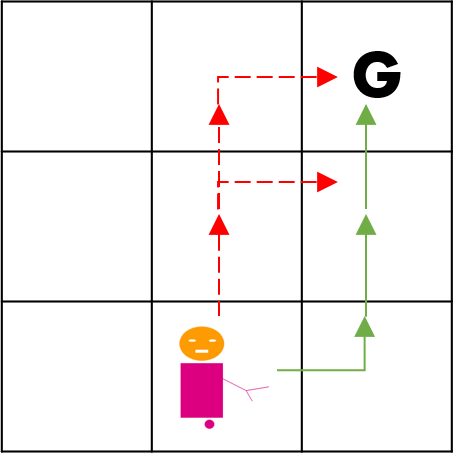}
        \caption{Plan predictability.}
        \label{fig:3}
    \end{subfigure}
\caption{
A simple illustration of the differences between plan explicability, legibility and predictability.
In this Gridworld, the robot can travel across cells, but cannot go backwards.
Figure \ref{fig:1} illustrates a legible plan (\textcolor{green}{green}) in the presence of 3 possible goals of the robot, marked with {\bf ?}s.
The \textcolor{red}{red} plan is not legible since all three goals are likely in its initial stages.
Figure \ref{fig:2} illustrates an explicable plan (\textcolor{green}{green}) which goes straight to the goal {\bf G} as we would expect. 
The \textcolor{red}{red} plan may be more favorable to the robot due to its internal constraints (the arm sticking out might hit the wall), but is inexplicable (i.e. sub-optimal) in the observer's model. 
Finally, Figure \ref{fig:3} illustrates a predictable plan (\textcolor{green}{green})
since there is only one possible plan after it performs the first action.
The \textcolor{red}{red} plans fail to disambiguate among two possible completions of the plan.
Note that all the plans shown in Figure \ref{fig:3} are explicable (optimal in the observer's model) but only one of them is predictable -- i.e. explicable plans may not be predictable.
Similarly, in Figure \ref{fig:2}, the \textcolor{red}{red} plan is predictable after the first action (even though not optimal, since there is only one likely completion) but not explicable -- i.e. predictable plans may not be explicable.
Without a prefix in Figure \ref{fig:2}, the \textcolor{green}{green} plan is the only predictable plan.
}
 	\label{fig:fig}
\end{figure*}

% \paragraph{Modes of Operation}
% One interesting factor we need to take into account while defining interpretable behavior is what part of the plan is being analyzed by the observer. Are they watching the plan being executed on the fly or are they trying to analyze the entire plan? We will refer to the former as online settings, while the latter as offline or post-hoc settings. One of the reasons why we want to differentiate between these two modes is the fact that many of the interpretability scores we are interested are inherently non-monotonic. This means behavior that may appear interpretable at the beginning may end up being uninterpretable in the context of the final plan and similarly initially uninterpretable behavior could lead to an overall interpretable plan. This means online settings where one of the goals could be to maintain high interpretability at every step may require different strategies from those where we care about ensuring the entire plan is interpretable. A related setting that doesn't quite fit into either paradigms, but still is studied in the literature, are settings where human is interested in making sure that the behavior is interpretable after K steps, we will refer to such settings as {\em K-step interpretability} settings. 
%\todo{observer -> what portion -> during/after ->online/off-line -> t-step. Non-monotonic}
\subsection{Explicability}

\unsure{The notion of explicability is related to the human's understanding of the robot's model and how well it explains or aligns with the behavior generated by the robot.}

\begin{quote}\em
Explicability measures how close a plan is to the expectations of the observer. 
\end{quote}

To formally define explicability, we need to define the notion of distance between the robot's plan and that expected by the human observer. Such distance measure could be based on a number of different aspects, including the cost of the given plan or even on syntactic features like actions used. For now, we will assume, we have access to a function $\mathcal{D}$ that takes two plans and a model and returns a positive number that reflects the distance between the plans. \index{Plan Distance}%
\unsure{From the semantics of explicability, the farther the plan is from one of the `expected' plans, the more unexpected it is to the observer}.
Thus, the \textbf{explicability score} for a robot's plan can be written as,
\[
E(\pi, \mathcal{M}^R_h) \propto \textrm{max}_{\pi'  \in \Pi^{\mathcal{M}^R_h}} -1*\mathcal{D}(\pi,\pi', \mathcal{M}^R_h)
\]
where, $\Pi^{\mathcal{M}^R_h}$ are the set of plans expected by the human. \unsure{That is, the explicability score of a robot's plan is negatively related to the minimum possible plan distance from the human's set of expected plans. Throughout the book we will mostly leave the exact function mapping distances to explicability score undefined. Instead we will assume that this mapping could be any monotonic increasing function over the term $-1*\mathcal{D}(\pi,\pi', \mathcal{M}^R_h)$.}
Thus the objective of explicable planning becomes,
\\
\[\textrm{Find:}~ \pi\]
%\[\textrm{So as to maximize}~    E(\pi, \mathcal{M}^R_h)\]
\[\max_{\pi \in \Pi^{\mathcal{M}^R}}~    E(\pi, \mathcal{M}^R_h)\]
\\
\unsure{where $\Pi^{\mathcal{M}^R}$ is the set of plans valid for the model $\mathcal{M}^R$. We will be returning to this score multiple times as it represents one of the key desirable properties of a human-aware robot -- namely the ability to generate plans that the human can ``interpret'', insofar that it aligns with her expectations about the robot.}
\subsection{Legibility}

The robot may be capable of performing multiple different tasks in a given environment. In such environments, the human observer may sometimes have ambiguity over the current task being pursued by the robot. In such cases, the robot can make its behavior interpretable to the human by implicitly communicating its task  through its behavior. This brings us to another type of interpretable behavior, namely, legible behavior.
\todo{In general, legible behavior which allows the robot to convey some information implicitly. In this setting, the inherent assumption is that the human observer maintains a set of possible hypotheses about the robot model but is unaware of the true model.} In a nutshell,

\begin{quote}\em
Legibility reduces observer's ambiguity over possible models the robot is using.
\end{quote}

\unsure{Let $\mathbb{M}^R_h$ be the set of possible models that the human thinks the robot may have. Here the robot's objective is to execute a plan that reduces the human observer's ambiguity over the possible tasks under consideration. Therefore, the \textbf{legibility score} of a plan $\pi$ can be defined as follows:}

\unsure{
\[\mathcal{L}(\pi, \mathbb{M}^R_h, \mathcal{M}^R) \propto \frac{1}{f_{amb}^{\mathcal{M}^R}(\{\mathcal{M}^j ~|~ \mathcal{M}^j \in \mathbb{M}^R_h \land \delta(I^j, \pi, \mathcal{M}^j) \models G^j \})} \]}

\todo{
Where $f_{amb}^{\mathcal{M}^R}$ is a measure of ambiguity over the set of models consistent with the behavior observed. There are multiple ways such an ambiguity function could be instantiated including, cardinality of the set, the similarity of the remaining models, and in the probabilistic case, the probability that the human would attach to the models that meet some specific criteria (for example, it shares some parameter with the true robot model). Thus the objective of legible planning thus becomes}

\[\textrm{Find:}~ \pi \]
\[\max_{\pi \in \Pi^{\mathcal{M}^R}}~ \mathcal{L}(\pi, \mathbb{M}^R_h)~\]
\\

In Chapter \ref{ch04}, we will see specific cases of legibility namely, goal legibility and plan legibility, defined in settings where the human has partial observability of the robot's activities. Due to the human's partial observability of the robot's activities, the human may suffer from uncertainty over the robot's true goal, given a set of candidate goals (goal legibility) or robot's true plan given a set of possible plans towards a goal (plan legibility). \unsure{Moreover, for scenarios with partial observability, we will be defining legibility over observation sequence generated by the plan as opposed to over the original plans.}

\subsection{Predictability}

In the case of explicability, as long as the robot's behavior conforms to one of the human's expected plans, the behavior is explicable to the human. However, predictability goes beyond explicability, in that, the behavior has to reduce the number of possible plan completions in the human's mental model given the robot's task. In a nutshell, 

\begin{quote}\em
Plan predictability reduces ambiguity over possible plan completions, given a plan prefix.
\end{quote}

That is, a predictable behavior is one that allows the human observer to guess or anticipate the robot's actions towards its goal. We can define the \textbf{predictability score}, i.e., predictability score of a plan prefix, $\bar{\pi}$, given the human's mental model, $\mathcal{M}^R_h$, as follows:

\[\mathcal{P}(\bar{\pi}, \mathcal{M}^R_h) \propto \frac{1}{|~\{ \tilde{\pi} ~|~ 
\tilde{\pi} \in \textrm{compl}_{\mathcal{M}^R_h}(\bar{\pi}) \} ~|} \]

Here, $\textrm{compl}_\mathcal{M}(\cdot)$ denotes the set of possible completions with respect to model $\mathcal{M}$, i.e., $\forall \pi \in \{\textrm{compl}_\mathcal{M}(\cdot)\}, \delta(I, \pi, \mathcal{M}) \models G$. The predictability score is higher with less number of possible completions for a given prefix in the human's mental model. Further, the objective in the problem of predictable planning is to select an action that reduces the number of possible completions given the prefix and the robot's goal. \unsure{Note that, this measure is being defined in an online setting, that is the robot has executed a partial prefix and is in the process of completing the execution in a predictable fashion.} Therefore, the problem of predictable planning with respect to a given prefix $\bar{\pi}$ is defined as follows:
\[\textrm{Find:}~ a~ \] 
\[\max_{a \in A^R}~ \mathcal{P}(\bar{\pi} + a, \mathcal{M}^R_h)~\]

\unsure{ While predictability is a popular measure for interpretability, in this particular book we will not be spending much time investigating this particular measure.}

\section{Communication to Improve Interpretability}
\index{Communication to Improve Interpretability}%

One of the core reasons for the uninterpretability of any agent behavior is the asymmetry between the agent's knowledge and what the human observer knows about the agent. Explicability and legibility stem from the human's misunderstanding or lack of knowledge about the agent's model and predictability deals with the human's ignorance about the specific plan being followed by the agent. This means that in addition to adjusting the agent behavior, another strategy the agent could adopt is to inform the human about its model or the current plan. 

\subsection{Communicating Model Information} 
\label{model-space}
\index{Parameterized Model Representation}
\unsure{ Among the communication strategies, we will focus primarily on communicating model information. Such communication requires the ability to decompose the complete model into meaningful components and be able to reason about the effect of receiving information about that component in the human model. We will do this by assuming that the planning models can be represented a set of model features or parameters that can be meaningfully communicated to the human.} STRIPS like models gives us a natural  parameterization scheme of the form
 $\Gamma : \mathcal{M} \mapsto s$ represents any planning problem $\mathcal{M} = \langle  F, A, I, G, C\rangle$ as a state $s \subseteq F$ as follows --

\vspace{-10pt}
\begin{align*}
    \tau(f) &= 
\begin{cases}
    \textit{init-has-f} & \text{ if } f \in I,\\
    \textit{goal-has-f} & \text{ if } f \in G,\\
    \textit{a-has-precondition-f} & \text{ if } f \in \textit{pre}(a),~a \in A\\
    \textit{a-has-add-effect-f} & \text{ if } f \in \textit{adds}(a),~a \in A\\
    \textit{a-has-del-effect-f} & \text{ if } f \in \textit{dels}(a),~a \in A\\
    \textit{a-has-cost-f} & \text{ if } f = C(a),~a \in A\\
\end{cases} \nonumber \\[1ex]
    \Gamma(\mathcal{M}) =~& \big\{\tau(f)~|~\forall f \in I \cup G \cup \nonumber \\ 
    & \bigcup_{a\in A}\{f'~|~\forall f' \in \{C(a)\} \cup \textrm{pre}(a) \cup \textrm{adds}(a) \cup \textrm{dels}(a)\}\big\}
\end{align*}

Now the agent can communicate to the user in terms of these model features, wherein they can assert that certain features are either part of the model or not. If $\mu$ is a set such messages, then let $\mathcal{M}^R_h + \mu$ be the updated human model that incorporates these messages. If $I$ is an interpretability score (say explicability or legibility) and $\pi$ the current plan, then the goal of the communication becomes to identify a set of messages $\mu$ such that
\[ I(\pi, \mathcal{M}^R_h + \mu) > I(\pi, \mathcal{M}^R_h)\]
\unsure{The discussion in Chapter \ref{ch05} focuses on the cases where the robot's plan is fixed and it tries to identify the set of messages that best communicate this plan.} There may be cases where the most desirable approach for the agent is to not only communicate but also adapt its behavior to reduce uninterpretability and communication overhead. We will refer to such approaches as balanced planning and will dive deeper into such methods in Chapter \ref{ch06}.

Apart from the knowledge asymmetry, another reason for uninterpretability could be the difference in the computational capabilities of the human and the agent. We could also use communication in these cases to help improve the human's ability to comprehend the plan. \unsure{ While we won't look at methods for generating such explanatory messages in detail, we will briefly discuss them in Chapter \ref{ch10}}.

% Talk about knowledge asymmetry
% Talk about communication
% Talk model information difference
% Talk about information
% Talk about communication to bridge inferential gap
% \subsection{Explanation}
\section{Other Considerations in Interpretable Planning}
In this section, we will look at some of the other factors that have been studied in the context of interpretable behavior generation and human-aware planning in general.
\paragraph{Modes of Operation}
\index{Modes Of Operation}%
\index{Online Measures}%
\index{Offline Measures}%
\index{K-Step Interpretability}%
One additional factor we can take into account while describing interpretable behavior is what part of the plan is being analyzed by the human observer. Are they watching the plan being executed on the fly or are they trying to analyze the entire plan? We will refer to the former as online setting and the latter as offline or post-hoc setting. In the previous sections, while we defined predictability in an online setting both legibility and explicability were defined for an offline setting. This is not to imply that legibility and explicability are not useful in online settings (though the usefulness of offline predictability is debatable in fully observable settings, we will touch on the topic in Chapter \ref{ch04}). The online formulations are useful when the robot's reward/cost functions are tied to the interpretability score at every execution time step. For online variations of the scores, the score for each prefix may be defined in regard to the expected completions of the prefix.
%In such cases, methods designed to optimize for online interpretability scores may be preferred over the ones that try to optimize for the overall score, as such strategies may involve the agent taking steps or actions that may be considered unintepretable with respect to the current score. 
A related setting that does not fit into either paradigms, is one where the human is interested in interpretable behavior after the execution of first $k$ steps by the agent, we will refer to such settings as {\em k-step interpretability} settings. 
\unsure{These settings are not the same as the online planning setting, as they allow for cases where the human would not pose any penalty for uninterpretable behavior for the first $k-1$ steps.}
%\todo{We can define a k-step version of all three interpretability scores}.

\paragraph{Motion vs Task Planning}
The notions of interpretable behaviors have been studied in both motion planning  and task planning. From the algorithmic perspective, this is simply differentiated in usual terms -- e.g. continuous versus discrete state variables. 
However, the notion of interpretability in task planning engenders additional challenges. This is because a reasonable human's mental model of the robot for motion planning can be assumed to be one that prefers straight-line plans and thus need not be modeled explicitly (or needs to be acquired or learned).
For task planning in general, this is less straightforward. \unsure{In Chapters \ref{ch03} and \ref{ch06}, we will see how we could learn such human models.}

\paragraph{Computational Capabilities of the Human}
\index{Human Inferential Capabilities}%
\unsure{One aspect of decision-making we haven't modeled explicitly here is the computational/inferential capabilities of the robot and the human. The robot's computation capabilities determine the kind of behaviors they can exhibit and the human's computational capabilities determine the kind of behavior they expect from the robot. Unfortunately, most current works in interpretable behavior generation generally tend to assume they are dealing with perfectly rational humans and agents (with few exceptions such as the user of noisy-rational models). Developing and using more sophisticated models of bounded rationality that better simulate the human and the robot's reasoning capabilities remain an exciting future direction for interpretable behavior generation.}

\paragraph{Behavior versus Plans}
\index{Behavior}%
Our discussion has mostly been confined to analysis of behaviors -- i.e. one particular observed instantiation of a plan or policy.
In particular, a plan -- which can be seen as a set of constraints on behavior -- engenders a {\em candidate set} of behaviors some of which may have certain interpretable properties while others may not.
\unsure{However, this also means that an algorithm that generates plan with a specific interpretable property, can also do so for a particular behavior 
it models, since in the worst case a behavior is also a plan that has 
a singular candidate completion.}
A general treatment of a plan can be very useful in the offline setting -- e.g. in decision-support
where human decision-makers are deliberating over possible plans with the support from an automated planner.
Unfortunately, interpretability of such plans has received 
very little attention beyond explanation generation. 

\section{Generalizing Interpretability Measures}
\index{Generalilzed Interpretability}%
\index{Model-Following Behavior}%
\index{Model-Communication Behavior}%

% talk about the two modes of behavior
In this chapter and in the rest of the book, we focus one three interpretability measures that were identified by previous works as capturing some desirable behaviors in human-robot interaction settings. Note that these are specific instances of larger behavioral patterns the robot could exhibit in such scenarios. Another possible way to categorize agent behaviors would be to organize them based on how they influence or use the human mental models. In the case of settings with humans as observers, this gives rise to two broad categories, namely;

\textbf{Model-Communication Behaviors:} Model-communication involves molding the human's mental models through implicit (i.e tacit) or explicit communication, to allow the agent to achieve their objectives. Examples of model-communication behaviors include legible behavior generation where the agent is trying to implicitly communicate some model information and the model-reconciliation explanation where the agent is directly communicating some model information so that the robot behavior appears more explicable.

\textbf{Model-Following Behaviors:} This strategy involves taking the current models of the human and generating behavior that conforms to current human expectations or exploits it in the robot's favor. The examples of this strategy we have seen so far include explicability and predictability. 

One could see the agent engaging cyclically in model-communication and model-following behaviors, wherein the agent may choose to mold the user's expectations to a point where its behavior may be better received by the observer.
We could go one step further and get away from behavior categorization altogether, and simply define a single human-aware planning problem that generates these individual interpretable behaviors in different scenarios. In fact, \cite{bayesian} defines a generalized human-aware planning problem as follows

\begin{definition}
\index{Generalized Human-Aware Planning Problem@G-HAP}
A {\bf Generalized human-aware planning problem or G-HAP} is a tuple $\Pi_{\mathcal{H}} = \langle \mathcal{M}^R, \mathbb{M}^R_h, P^0_h,$ $P_{\ell}, C_{\mathcal{H}}\rangle$, where $\mathcal{M}^R$ as always is the robot model, $ \mathbb{M}^R_h$ is the set of models human maintains of the robot,  $P^0_h$ is the human's initial prior over the models in the hypothesis set $\mathbb{M}^R_h$, $P_{\ell}$ are the set of inference models the human associates with each possible model and $C_{\mathcal{H}}$ is a generalized cost function that depends on the exact objective of the agent. 
\end{definition}

Here $\mathbb{M}^R_h$ also includes special model denoted as $\mathcal{M}^0$. $\mathcal{M}^0$ corresponds to a high entropy model, that is meant to capture the fact that the human's current set of explicit models cannot account for the observed robot behavior. The formulation assumes the human is a Bayesian reasoner who reasoning about the robot behavior given their beliefs about the robot model (and previous observations about the robot behavior).

This problem formulation generates explicable behavior as discussed in this book if (a) $\mathbb{M}^R_h = \{\mathcal{M}^R_h, \mathcal{M}^0\}$,  (b) the objective of the planning problem is to select behavior that maximizes the posterior probability associated with model $\mathcal{M}^R_h$ and (c) if, in human's expectation, the likelihood of the robot selecting a plan for  $\mathcal{M}^R_h$ (i.e.  $\mathcal{M}^R_h$) is inversely proportional to a distance from a set of expected plans. \citep{bayesian} specifically shows this for cases where the distance measure corresponds to cost difference and the expected plans correspond to optimal plans.

Similarly, the formulation generates legible behaviors when the prior on $\mathcal{M}^0$ is zero and if the objective is to select behavior that maximizes the cumulative posterior over all the models that share the model parameter the robot is trying to communicate (say a goal).  Finally, predictable behaviors are generated when $\mathcal{M}^0$ has zero prior, the $\mathbb{M}^R_h = \{\mathcal{M}^R_h, \mathcal{M}^0\}$ and the objective of the planning problem is to choose behavioral prefixes, whose completions have a high posterior likelihood.

\section{Bibliographic Remarks}
We would refer the reader to look at \cite{geffner2013concise} for a detailed discussion of the planning formalism we will be focusing on. In most cases, we can leverage any appropriate planning algorithm (satisficing or optimal depending on the setting) that applies to this formalism. In cases, where we require the use of a specifically designed search or planning algorithm, we will generally include a pseudo code in the chapter. 

In terms of the interpretability measures, the original work to introduce explicability was \cite{exp-yu}, though it used a learned labeling model as a proxy to measure the explicability score. The measure was further developed in \cite{explicable-anagha} to take into account the user's model and a learned distance function. Parallelly, works \cite{explain} and \cite{balancing} have looked at applying communication based strategies to address inexplicability but primarily in the context of when the distance function is based on cost and expected plan consists of plans that are optimal in the robot model.

With respect to legibility and predictability, initial works were done in the context of motion-planning problems by \cite{dragan2013legibility} and \cite{Dragan17}. Legibility was later extended to task planning settings by \cite{macnally2018action} and further into partially observable domains by \cite{unified-anagha}. A {\em k-step} version of predictability was introduced in \cite{fisac2018generating}. The noisy-rational model as a stand-in for the human's inferential capability was used by  \cite{dragan2013legibility}, \cite{Dragan17}, and \cite{fisac2018generating}. Even outside the interpretability literature, there is evidence to suggest that such models can be an effective way to capture the inferential process of humans \cite{baker2014modeling}. \cite{landscape} provides a good overview of the entire space of works done in this area and provides a categorization of existing individual works done in this space into these three core scores.

The communication paradigm discussed in the chapter is based on work done in \cite{explain}. Such model parameterization was later extended to MDPs in \cite{modelfree}. Even outside of explanations, the use of representing models has been investigated in the context of other applications that require a search over a space of plausible models, like in the case of \cite{bryce2016maintaining}.

%\paragraph{Communication}

%<Also talk about connection to design>

\clearpage

                % example, etc.
    \chapter{Explicable Behavior Generation}
\label{ch03}
\unsure{In chapter 2, among other things we defined the notion of explicability of a plan and laid out an informal description of explicable planning. In this chapter, we will take a closer look at explicability and discuss some practical methods to facilitate explicable planning. This would include discussion on both planning algorithms specifically designed for generating explicable behavior and how one could design/update the task to make the generation of explicable plans easier.}

\unsure{In the earlier chapter, we discussed how the generation of explicable plans requires the robot to simultaneously reason with both models $\mathcal{M}^R$ and $\mathcal{M}^R_h$. This is because the robot needs to select feasible and possibly low-cost plans from $\mathcal{M}^R$, that align with or are close to plans expected by the human (as computed from $\mathcal{M}^R_h$). This implies the robot needs to have access to $\mathcal{M}^R_h$ (at least an approximation of it in some form). An immediate question the reader could ask is if $\mathcal{M}^R_h$ is available to the robot, why is $\mathcal{M}^R$ required in the plan generation process at all? This is necessary since the human's mental model might entail plans that are not even feasible for the robot or are prohibitively expensive, and can thus at best serve as a guide, and not as an oracle, to the explicable plan generation process.}

\unsure{In fact, the critical consideration in the computation of explicable behavior would be the form in which the human's mental model, $\mathcal{M}^R_h$, is accessible to the robot. We present two settings in this chapter, one where the robot directly has access to $\mathcal{M}^R_h$, and another where the robot learns an approximation of it, $\widetilde{\mathcal{M}}^R_h$.}

In the first setting, we hypothesize that a prespecified plan distance measure %\cite{srivastava2007domain,nguyen-partialp-2012} 
can quantify the distance between the robot plan $\pi_{\mathcal{M}^R}$ and the expected plans $\pi_{\mathcal{M}^R_h}$ from $\mathcal{M}^R_h$. 
\unsure{There are many scenarios in which the system could have direct access to the human's mental model}. 
In domains such as household or factory floor scenarios, there is generally a clear expectation of how a task should be performed. In such cases, $\mathcal{M}^R_h$ can be constructed following the norms or protocols that are relevant to that domain. Most deployed products make use of inbuilt models of user expectations in some form or the other. 
\unsure{In domains like urban search and rescue, the human and the robot may start with the same model and they may diverge over time as the robot acquires more information owing to the use of more advanced sensors.
With $\mathcal{M}^R_h$ in place, we can then use it along with $\mathcal{M}^R$ to generate plans that are closer to the plans expected by the human. Section \ref{section:plan-dist} presents a specific heuristic search method for generating such plans and the basic approach is illustrated in Figure \ref{fig:architecture1}.
}
%Building such models, of course, require interactions with users of that domain. 
%With $\mathcal{M}^R_h$ in place, we can then use it to generate expected plans. Then we compute the distance between plans from $\mathcal{M}^R$ and $\mathcal{M}^R_h$. The candidate robot plans are assessed in terms of their explicability by human users in the domain. These assessments are then mapped to the precomputed distances to derive an explicability distance. Then the plan generation process uses explicability distance as a heuristic to guide the search. This approach is illustrated in Figure \ref{fig:architecture1} and is discussed in Section \ref{section:plan-dist}.

\begin{figure*}%[th]
\centering 
\includegraphics[width=\columnwidth]{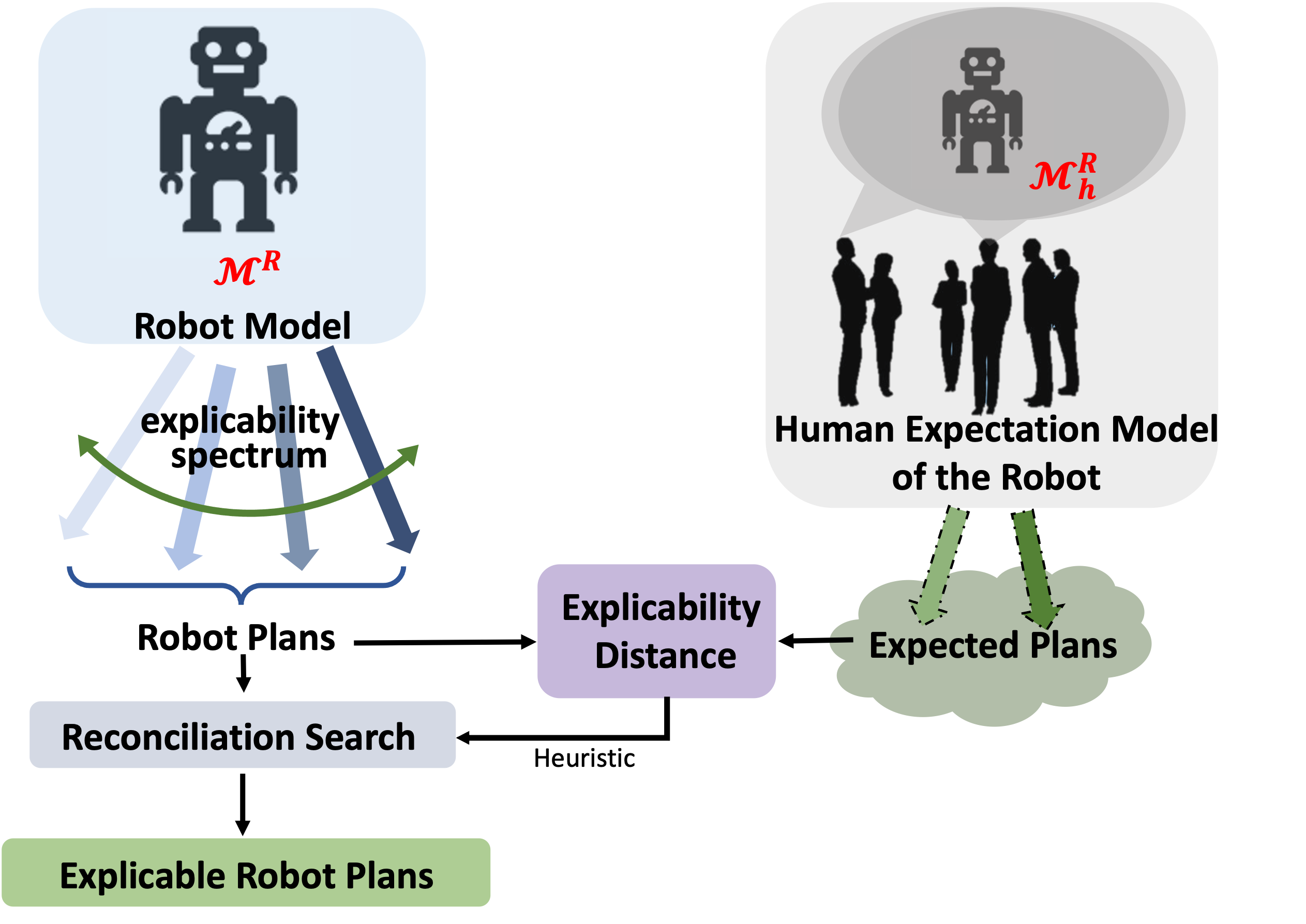}
\caption{Schematic diagram of the model based explicable planning.}
\label{fig:architecture1}
\end{figure*}

In the second setting, we learn an approximation of the human's mental model represented as $\widetilde{\mathcal{M}}^R_h$. This model is learned based on an underlying assumption that humans tend to associate tasks/sub-goals with actions in a given plan. %\cite{vallacher1987people,csibra2007obsessed}. 
Thus the approximate model, $\widetilde{\mathcal{M}^R_h}$, takes the form of a labeling function that maps plan steps to specific labels. This learned model is then used as a heuristic to generate explicable plans. This approach is illustrated in Figure \ref{fig:architecture2} and is discussed in Section \ref{section:model-free}.

\begin{figure*}%[th!]
\centering 
\includegraphics[width=\columnwidth]{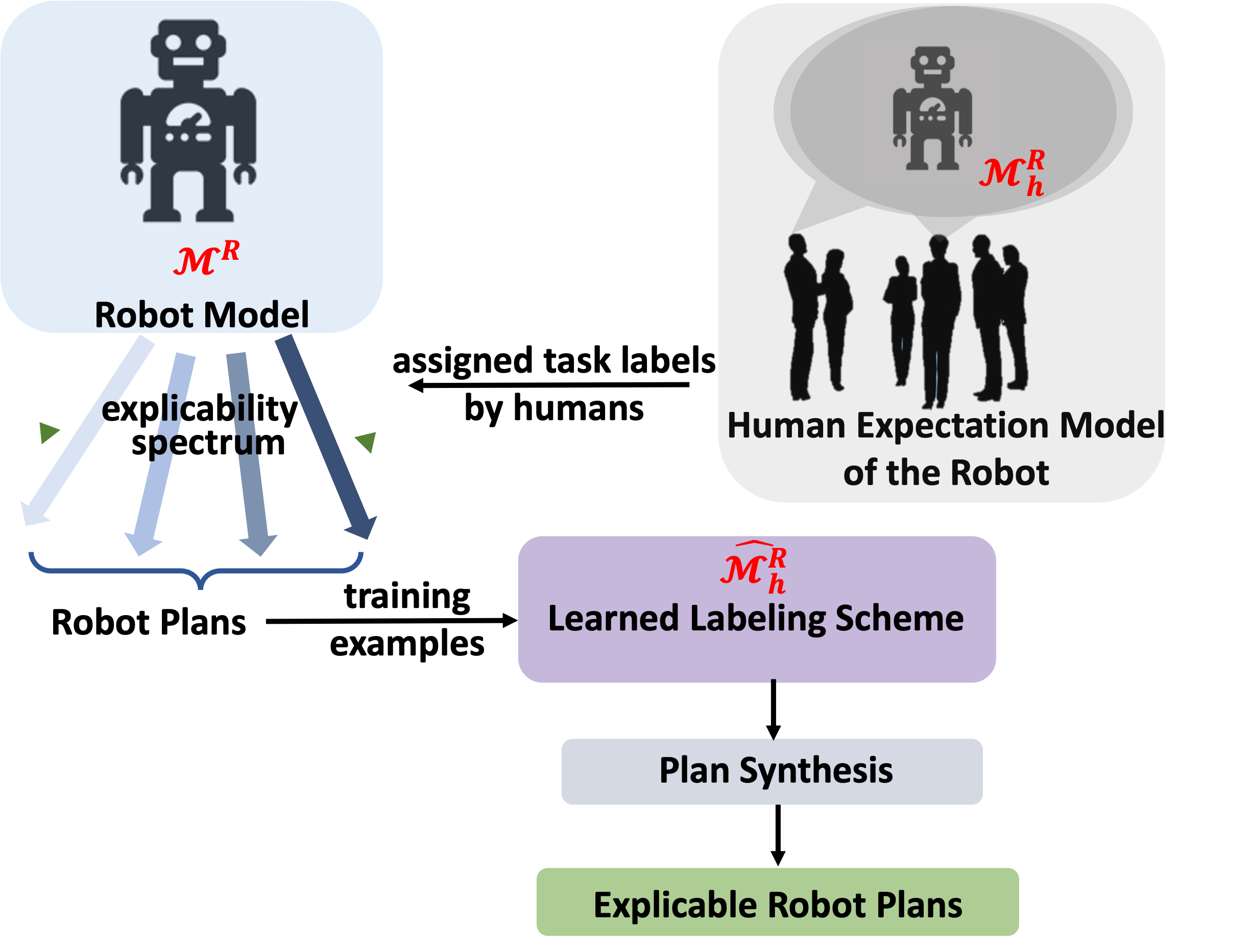}
\caption{\unsure{Schematic diagram of model-free explicable planning: A learned labeling model is used as a proxy for the actual human model. This model is used to compute the heuristic for explicable plan generation.}}
\label{fig:architecture2}
\end{figure*}

Finally, we will conclude this chapter by looking at methods that can help improve the feasibility of these behaviors. The plan expected by the human observer may be prohibitively expensive and/or infeasible for the robot. In such cases, it may be wise to redesign the environment itself in which the human and the robot are operating, especially when tasks need to be performed repeatedly in the environment. In Section \ref{section:design}, we discuss how the environment design techniques can be leveraged to improve the explicability of the robot's behaviors, and also explore the longitudinal impact on explicable behaviors in a repeated interaction setting. 

\section{Explicable Planning Problem}
\index{Explicable Planning}%
As a reminder from Chapter \ref{ch02}, the problem of explicable planning arises when the robot model differs from the human's expectation of it. Let $\pi_{\mathcal{M}^R}$ be the robot's plan solution to the planning problem,  $\mathcal{M}^R= \langle F, A^R, I^R, G^R, C^R \rangle $; whereas, $\Pi_{\mathcal{M}^R_h}$ be the set of plans that the human expects given her mental model of robot's model, $\mathcal{M}^R_h= \langle F, A^R_h, I^R_h, G^R_h, C^R_h \rangle $. The differences in the human's mental model can lead to different plan solutions. 

Thus explicable planning consists of the robot coming up with plans that are close to what the human expects (measured using a distance function $\mathcal{D}$). 

\begin{definition}
\index{Explicable Planning!Problem}%
The \textbf{explicable planning problem} is defined as a tuple $\mathcal{P}_{Exp} = \langle \mathcal{M}^R, \mathcal{M}^R_h,  \mathcal{D}\rangle $, where, $\mathcal{D}$ is the distance function that the human uses to compute the distance between her expected plan and the robot's plan.
%\begin{itemize}
%\item $\mathcal{M}^R$ is the robot's planning model 
%\item $\mathcal{M}^R_h$ is the human's mental model of the robot's planning model
%\item $\mathcal{D}$ is the distance function that the human uses to compute the distance between her expected plan and the robot's plan. 
%\end{itemize}
\end{definition}

%Note that actual explicability score ($E$) will be defined with respect to the distance function, in particular the distance to the closest expected plan in the human model (say, by using a decreasing monotonic function on the distance). 

In general, explicable planning is a multi-objective planning problem where the robot is trying to choose a behavior that tries to both maximize its explicability score (i.e. minimize the distance to the expected plan) and minimize its cost. To keep the planning formulation simple, we will assume that we can reduce the planning objective to a linear combination of the cost of the plan in the robot model and the distance of the plan from the expected human plan. 
Thus the solution to an explicable planning problem is an explicable plan that achieves the goal and minimizes the plan distances while also minimizing the cost of the plan.

\begin{definition}
\index{Maximally Explicable Plan}%
A \textbf{maximally explicable plan} is a plan, $\pi^{*}_{\mathcal{M}^R}$, starting at $I^R$ that achieves the goal $G^R$, such that, $
\displaystyle 
\argmin_{\pi_{\mathcal{M}^R}} C(\pi_{\mathcal{M}^R}) + \min_{\pi \in \Pi_{\mathcal{M}^R_h}} f_{\mathcal{IE}}(\mathcal{D}(\pi_{\mathcal{M}^R}, \pi, \mathcal{M}^R)) $. 
\label{exp-def}
\end{definition}

\unsure{Where $f_{\mathcal{IE}}$ is any increasing monotonic function over the distance $\mathcal{D}$ (the actual choice of which may depend on the specific optimization algorithm).
We will refer to the term $\min_{\pi \in \Pi_{\mathcal{M}^R_h}} f_{\mathcal{IE}}(\mathcal{D}(\pi_{\mathcal{M}^R}, \pi, \mathcal{M}^R))$ as the inexplicability score or simply $\mathcal{IE}$.
Readers would note that $\mathcal{IE}$ is the opposite of the explicability score defined in the previous chapter. 
In this chapter, we will mostly focus on planning algorithms that are trying to minimize the cost of plans generated, and as such focusing on inexplicability score, makes it easier to integrate the concept of explicabillity into the algorithms as we can treat $\mathcal{IE}$ as yet another cost term. Though other planning formulations say utility maximization formulation, may benefit from the use of explicability score.}

\section{Model-Based Explicable Planning}
\index{Explicable Planning!Model-Based}%

\label{section:plan-dist}

We will start by looking at explicable behavior generation in cases where the human's mental model, $\mathcal{M}^R_h$, and the distance function, $\mathcal{D}$ are known beforehand. 
Here we assume that the set of plans expected by the human correspond to the optimal plans in $\mathcal{M}^R_h$, i.e.,  $\Pi_{\mathcal{M}^R_h} = \{ \pi | C^R_h(\pi) = C^{*}_{\mathcal{M}^R_h}\}$. For each robot plan, the minimum distance (per the given distance function, $\mathcal{D}$) is computed with respect to this set of expected plans.  

\subsection{Plan Generation through Reconciliation Search}
\index{Explicable Planning!Reconciliation Search}%

\begin{algorithm}%[]
\caption{Reconciliation Search}
\label{search1}
\algorithmicrequire{~$\mathcal{P}_{Exp} = \langle \mathcal{M}^R, \mathcal{M}^R_h, \mathcal{D}\rangle$ and $\textrm{max\_cost}$} \\
\algorithmicensure{~$\Pi_{Exp}$}
\begin{algorithmic}[1]
\STATE $\Pi_{Exp} \gets \emptyset $ \qquad \COMMENT{\textcolor{blue}{Explicable plan set}} 
\STATE $\textit{open} \gets \emptyset$ \qquad \COMMENT{\textcolor{blue}{Open list} }
\STATE $\textit{closed} \gets \emptyset$ \qquad \COMMENT{\textcolor{blue}{Closed list} }
\STATE $\textit{open}$.insert$(I^R, 0, \inf)$ 
\WHILE {$\textit{open} \neq \emptyset$} 
\STATE  $n \gets \textit{open}$.remove() \qquad \COMMENT{\textcolor{blue}{Node with highest $h(\cdot)$} }
\IF {$n \models G^R$} 
\STATE $\Pi_{Exp}$.insert($\pi \text{ s.t. } \delta_{\mathcal{M}^R}(I^R, \pi) \models n $) 
\ENDIF 
\STATE $\textit{closed}$.insert($n$) 
\FOR {each $v \in  \textit{successors}(n)$}
\IF {$v \notin \textit{closed}$}
\IF {$g(n) + \textit{cost}(n, v) \leq \textrm{max\_cost}$}
\STATE $\textit{open}$.insert($v, g(n) + \textit{cost}(n, v), h(v)$)  \qquad %\Comment{\textcolor{blue}{Plan $\pi_v$ from $\mathbb{I}$ to $v$}}
\ENDIF
\ELSE
\IF {$h(n) < h(v)$}
\STATE $\textit{closed}$.remove($v$)
\STATE $\textit{open}$.insert($v, g(n) + \textit{cost}(n, v), h(v)$)
\ENDIF
\ENDIF
\ENDFOR
\ENDWHILE
\STATE \RETURN $\Pi_{Exp}$ 
\end{algorithmic}
\end{algorithm}

\unsure{Our primary objective here is to use the distance to the expected plans in $\mathcal{M}^R_h$ as a heuristic to come up with plans that are both valid for the model $\mathcal{M}^R$ and explicable. One obvious idea we could try is to first use $\mathcal{M}^R$ to come up with the expected plan set $\Pi^R_h$ and then during planning using the model  $\mathcal{M}^R$ (via a heuristic progression search), expand only the plan prefixes which have the minimal distance to one of the expected plans in the set $\Pi^R_h$ (we will assume that the inexplicability score here is equal to the distance and use the terms inexplicability score and distance interchangeably).}

\unsure{Unfortunately, this won't work in practice due to the non-monotonicity of the inexplicability score. As a partial plan grows, each new action may contribute either positively or negatively to the distance, thus making the final inexplicability score non-monotonic. Consider that the goal of an autonomous car is to park itself in a parking spot on its left side. The car takes the left turn, parks, and turns on its left indicator. Here the turning on of the left tail light after having parked is an inexplicable action. The first two actions are explicable to the human drivers and contribute positively to the overall explicability of the plan but the last action has a negative impact.} 

\unsure{Due to the non-monotonic nature of distance, we cannot stop the search process after finding the first solution. Consider the following case, if $\textit{d}_1$ is the distance of the current prefix, then a node may exist in the open list (set of unexpanded nodes) whose distance is greater than $\textit{d}_1$, which when expanded may result in a solution plan with distance less than $\textit{d}_1$. A greedy method that expands a node with the lowest distance score of the corresponding prefix at each step is not guaranteed to find an optimal explicable plan (one of the plans with the lowest inexplicability score) as its first solution.
One possible solution that has been studied in the literature has been to use a cost-bounded anytime greedy search algorithm called \emph{reconciliation search} that generates all the valid candidate plans up to a given cost bound, and then progressively searches for plans with lower inexplicability scores. 
The value of the heuristic $\textit{h(v)}$ in a particular state $\textit{v}$ encountered during the search is based entirely on the distance of the agent plan prefix $\pi_{\mathcal{M}^R}$ up to that state,\begin{math} 
~\textrm{h(v)}\  = \mathcal{D}(\pi_{\mathcal{M}^R},\pi^{\prime}_{\mathcal{M}^R_h}, \mathcal{M}^R_h) 
 \text{ s.t. } \delta_{\mathcal{M}^R}(I^R, \pi_{\mathcal{M}^R}) \models v 
\text{ and } \delta_{\mathcal{M}^R_h}(I^R, \pi^{\prime}_{\mathcal{M}^R_h}) \models v
\end{math}.}

\unsure{The approach is described in detail in Algorithm \ref{search1}. At each iteration of the algorithm, the plan prefix of the agent model is compared with the explicable trace $\pi^{\prime}_{\mathcal{M}^R_h}$ (these are the plans generated using $\mathcal{M}^R_h$ up to the current state in the search process) for the given problem. There are few choices one could consider in creating such prefixes including, considering prefixes from among the expected plans, optimal or even valid plans to the state in the human model.}
\unsure{The search algorithm then makes a locally optimal choice of states and continues till a solution is found. The search is not stopped after generating the first solution, but instead, it restarts from the initial state. The search continues to find all the valid candidate solutions within the given cost bound or until the state space is completely explored. }

\subsection{Possible Distance Functions}
\index{Plan Distance Functions}%
\unsure{In the above discussion, we mostly assumed we are given some specific distance function $\mathcal{D}$. A simple candidate for distance function could be the cost difference, but that presupposes that we know exactly the cost function of the human and that the explicability consideration of the human observer is strictly limited to cost difference. In general, it may be possible that the human is using more complex distance functions and we may want to learn the distance function directly from the human. In this case, we could consider learning a distance function that is a combination of various individual distance functions that can be directly computed from the human's model. Some of the candidate distance functions we could consider include measures like action-distance, causal-link distance, and state distance.}

\paragraph{Action distance}
\index{Plan Distance Functions!Action Distance}%

We denote the set of unique actions in a plan $\pi$ as $A(\pi) = \{a~|~a\in\pi\}$.
Given the action sets $A(\pi_{\mathcal{M}^R})$ and $A(\pi^*_{\mathcal{M}^R_h})$ of two plans $\pi_{\mathcal{M}^R}$ and $\pi^*_{\mathcal{M}^R_h}$ respectively, the action distance is, 
\begin{math} \label{eqn:eqn1}
\mathcal{D}_A(\pi_{\mathcal{M}^R}, \pi^*_{\mathcal{M}^R_h}) = 1 - \frac{\lvert A(\pi_{\mathcal{M}^R}) \cap A(\pi^*_{\mathcal{M}^R_h}) \rvert}{\lvert A(\pi_{\mathcal{M}^R}) \cup A(\pi^*_{\mathcal{M}^R_h}) \rvert}
\end{math}. 
Here, two plans are similar (and hence their distance measure is smaller) if they contain same actions. Note that it does not consider the ordering of actions. 

\paragraph{Causal link distance}
\index{Plan Distance Functions!Causal Link Distance}%

A causal link represents a tuple of the form $\langle a_i, p_i, a_{i+1} \rangle$, where $p_{i}$ is a predicate variable that is produced as an effect of action $a_i$ and used as a precondition for the next action $a_{i+1}$. The causal link distance measure is represented using the causal link sets $Cl(\pi_{\mathcal{M}^R})$ and $Cl(\pi^*_{\mathcal{M}^R_h})$,
\begin{math} \label{eqn:eqn2}
\mathcal{D}_C(\pi_{\mathcal{M}^R}, \pi^*_{\mathcal{M}^R_h}) = 1 - \frac{\lvert Cl(\pi_{\mathcal{M}^R}) \cap Cl(\pi^*_{\mathcal{M}^R_h}) \rvert}{\lvert Cl(\pi_{\mathcal{M}^R}) \cup Cl(\pi^*_{\mathcal{M}^R_h}) \rvert}
\end{math}.

\paragraph{State sequence distance}
\index{Plan Distance Functions!State Sequence Distance}%

This distance measure finds the difference between sequences of the states. Given two state sequences $(s^R_0, \ldots, s^R_n)$ and $(s^H_0, \ldots, s^H_{n^{\prime}})$ for $\pi_{\mathcal{M}^R}$ and $\pi^*_{\mathcal{M}^R_h}$ respectively, where $n \geq n^{\prime}$ are the lengths of the plans, the state sequence distance is, \begin{math} \label{eqn:eqn3}
\mathcal{D}_S(\pi_{\mathcal{M}^R}, \pi^*_{\mathcal{M}^R_h}) = \frac{1}{n}  \big[ \  \sum_{k=1}^{n^{\prime}} d(s_k^R, s_k^H) + n - n^{\prime} \big] \ \
\end{math}, where \begin{math} d(s_k^R, s_k^H) = 1 - \frac{\lvert s_k^{R} \cap s_k^{H} \rvert}{\lvert s_k^{R} \cup s_k^{H} \rvert} \end{math} represents the distance between two states (where $s_k^{R}$ is overloaded to denote the set of predicate variables in state $s_k^{R}$). The first term measures the normalized difference between states up to the end of the shortest plan, while the second term, in the absence of a state to compare to, assigns maximum difference possible.

As mentioned earlier, the actual distance used by the human observer could be some combination of such constituent distance function, so we will need to consider composite distance functions.

\begin{definition}
\index{Plan Distance Functions!Composite Distance}%
A \textbf{composite distance}, $\mathcal{D}_{comp}$ is a distance between pair of two plans $\langle \pi_{\mathcal{M}^R}, \pi_{\mathcal{M}^R_h} \rangle$, such that, $\mathcal{D}_{comp}(\pi_{\mathcal{M}^R}, \pi_{\mathcal{M}^R_h},\mathcal{M}^R_h) = || \mathcal{D}_A(\pi_{\mathcal{M}^R}, \pi_{\mathcal{M}^R_h})  + \mathcal{D}_C(\pi_{\mathcal{M}^R}, \pi_{\mathcal{M}^R_h}) + \mathcal{D}_S(\pi_{\mathcal{M}^R}, \pi_{\mathcal{M}^R_h}) || $.
\end{definition}

But for each robot plan we want to find the minimum distance with respect to the set of human's expected plans. We say a distance minimizing plan in the set of the expected plans is defined as follows:

\begin{definition}
\index{Distance Minimizing Plan}%
A \textbf{distance minimizing plan}, ${\pi}^{*}_{\mathcal{M}^R_h}$, is a plan in $\Pi_{\mathcal{M}^R_h}$, such that for a robot plan, $\pi_{\mathcal{M}^R}$, the composite distance is minimized, i.e., $\forall \pi \in \Pi_{\mathcal{M}^R_h}~\mathcal{D}_{comp}(\pi_{\mathcal{M}^R, \pi, \mathcal{M}^R_h}) \geq \mathcal{D}_{comp}(\pi_{\mathcal{M}^R, \pi^{*}_{\mathcal{M}^R_h}, \mathcal{M}^R_h})$.
\end{definition}

Our overall objective is to learn an inexplicability score from the individual distances. One way to do it would be to first collect scores human participants and then try to learn a mapping to human specified scores from the individual plan distances between a robot plan and corresponding distance minimizing plan in the set of expected plans.
To that end, we define a explicability feature vector as follows:

\begin{definition}
\label{def:4}
\index{Explicability Feature Vector}%
An \textbf{explicability feature vector}, $\mathbb{F}$, is a three-dimensional vector, which is given with respect to a distance minimizing plan pair, $\langle \pi_{\mathcal{M}^R}, \pi^{*}_{\mathcal{M}^R_h} \rangle$, such that, \begin{math}
\mathbb{F}= \langle \mathcal{D}(\pi_{\mathcal{M}^R}, \pi^{*}_{\mathcal{M}^R_h}), \mathcal{D}_C(\pi_{\mathcal{M}^R}, \pi^{*}_{\mathcal{M}^R_h}), \mathcal{D}_S(\pi_{\mathcal{M}^R}, \pi^{*}_{\mathcal{M}^R_h}) \rangle^T 
\end{math}.
\end{definition}

This allows us to learn an explicability distance function, $\hat{\mathcal{D}}_\textit{Exp}(\pi_{\mathcal{M}^R}~/~\pi^{*}_{\mathcal{M}^R_h})$, which is essentially a regression function, \textit{f}, that fits the three plan distances to the total plan scores, with $b$ as the parameter vector, and $\mathbb{F}$ as the explicability feature vector, such that, 
\unsure{\begin{math}
\hat{\mathcal{D}}_{\textit{Exp}}(\pi_{\mathcal{M}^R}~/~\pi^{*}_{\mathcal{M}^R_h}) \approx \textit{f} (\mathbb{F}, b)
\end{math}}. Therefore, a regression model can be trained to learn the explicability assessment (total plan scores) of the users by mapping this assessment to the explicability feature vector which consists of plan distances for corresponding plans. 
\unsure{Since we are unaware of the distance measure, we can never guess the right plan against which the human may be comparing the current behavior. The distance minimizing plan is a stand-in for this unknown plan. As such any distance function learned from features derived from this plan will be an approximation of the actual distance. Though in many cases such approximation acts as a reasonable guide for behavior generation.}

\section{Model-Free Explicable Planning}
\label{section:model-free}
\index{Explicable Planning!Model-Free}%

\unsure{If the robot does not have access to the human's mental model, then an approximation of it can be learned. This approach shows that it is not necessary to build a full-fledged planning model of the human's mental model, rather it is enough to learn a simple labeling function. Here the underlying hypothesis is that humans tend to associate abstract tasks or sub-goals to actions in a plan. If the human-in-the-loop can associate any domain-specific label to an action in the plan then the action is assumed to be explicable, otherwise, the action is considered inexplicable. Such a labeling scheme can be learned from training examples annotated by humans. This learned model can then be used as a heuristic function in the planning process. This approach is illustrated in  Figure \ref{fig:architecture2}.}

\subsection{Problem Formulation}

\unsure{In this case, since the $\mathcal{M}^R_h$ is not known beforehand, the distance function $\mathcal{D}(\cdot, \cdot, \cdot)$ in Definition 2 is approximated using a learning method. Here the terms $\mathcal{D}(\pi_{\mathcal{M}^R}, \pi_{\mathcal{M}^R_h}, \mathcal{M}^R_h)$ and by extension the inexplicability scores are measured in terms of the labels the human would associate with the plan.
Thus the method expects a set of task labels $T = \{T_1, T_2, \ldots, T_M \}$  be provided for each domain. Depending on how the human view the role of the action in the plan, each action may be associated with multiple labels. The set of action labels for explicability is the power set of task labels: $\mathbb{L} = 2^T$. When an action label includes multiple task labels, the action is interpreted as contributing to multiple tasks; when an action label is an empty set, the action is interpreted as inexplicable.}

\unsure{For a given plan $\pi_{\mathcal{M}^R}$, we can define the inexplicability score of the plan as follows
$\mathcal{IE}(\pi_{\mathcal{M}^R}, \mathcal{M}^R_h)= f \circ \mathcal{L}^*(\pi_{\mathcal{M}^R})$, where $f$ is a domain-independent function that takes plan labels as input, and $\mathcal{L}^*$ is the labeling scheme of the human for agent plans based on $\mathcal{M}^R_h$. A possible candidate for $f$ could be to compute the ratio between the number of actions with empty action labels and the number of all actions.}  

\subsection{Learning Approach}
\index{Labeling Models for Explicability}%
\unsure{Now the question is, for a new plan how do we predict the labels human would associate with it? As mentioned earlier, one could use a training set consisting of plans annotated with labels and train a sequence to sequence model.
Here, possible models, one could use include sequential neural networks (including Recurrent Neural Networks \citep{goodfellow2016deep} variants and transformers) and probabilistic models like Hidden Markov Models \citep{russell2002artificial} or Conditional Random Fields \citep{lafferty2001conditional}. Here each step in the sequence corresponds to a step in the plan execution and the possible features that could be associated with each step include features like the action description, the state fluents obtained after executing each action in the sequence $\langle a_0, ..., a_n \rangle$ from the initial state, etc.  One could also potentially include other features like motion level information, like gripper position, orientation, etc.
We will use the notation $\mathcal{L}'$ to denote the learned labeling model and to differentiate it from the true labeling model $\mathcal{L}^*$.
}

\subsection{Plan Generation}
\unsure{Once we have learned a labeling model, we can then use it to drive the behavior generation.}

\subsubsection{Plan Selection}
\unsure{The most straightforward method is to perform plan selection on a set of candidate plans which can simply be a set of plans that are within a certain cost bound of the optimal plan. Candidate plans can also be generated to be diverse with respect to various plan distances. For each plan, the features of the actions are extracted as we discussed earlier. Then the trained model $\mathcal{L}'$ can be used to produce the labels for the actions in the plan and the inexplicability score $\mathcal{IE}$ can then be computed as given by the mappings above. The inexplicability score can then be used to choose a plan that is least inexplicable.}

\begin{algorithm}%[]
\caption{Synthesis of Explicable Plan}
\label{model_free_exp_search}
\algorithmicrequire{ $\mathcal{M}^R,\mathcal{L}'$ }\newline
\algorithmicensure{ $\pi^*_{\mathcal{M}^R}$}
\begin{algorithmic}[1]
\STATE $\textit{open} \leftarrow \emptyset$
\STATE $\textit{open}$.push$(I^R)$
\WHILE {$\textit{open}\neq\emptyset$}
\STATE  $s\leftarrow\textit{open}$.remove()
\IF {$G^R$ is reached}
\STATE \textbf{return} $s$.plan \COMMENT{\textcolor{blue}{the plan that leads to $s$ from $I^R$)}}
\ENDIF
\STATE Compute all possible next states $N$ from $s$
\FOR {$n \in N$}
\STATE Compute the relaxed plan $\pi_{RELAX}$ for $n$
\STATE Concatenate $s$.plan with $\pi_{RELAX}$ as $\pi^{prime}$ \COMMENT{\textcolor{blue}{$s$.plan contains plan features and $\pi_{RELAX}$  only contains action descriptions}}
\STATE Compute and add other relevant features
\STATE Compute $L_{\pi^R} = \mathcal{L}'(\pi^{prime})$
\STATE Compute $h = f(IE, h_{cost})$ \COMMENT{\textcolor{blue}{$f$ is a combination function; $h_{cost}$ is the relaxed planning heuristic}}
\ENDFOR
\IF {$h(n^*) < h^*$ \COMMENT{\textcolor{blue}{$n^* \in N$ with minimum $h$}}}
\STATE $\textit{open}$.clear()
\STATE $\textit{open}.$push($n^*$); $h^* = h(n^*)$ \COMMENT{\textcolor{blue}{$h^*$ is initially MAX}}
\ELSE
\STATE Push all $n \in N$ into $\textit{open}$
\ENDIF
\ENDWHILE
\end{algorithmic}
\end{algorithm}

\subsubsection{Plan Synthesis}
A more efficient way is to incorporate these measures as heuristics into the planning process. Algorithm \ref{model_free_exp_search} presents a simple modified version of Enforced Hill Climbing \citep{hoffmann2005ignoring}, that can make use of such learned heuristics. To compute the heuristic value given a planning state, one can use the relaxed planning graph to construct the remaining planning steps. However, since relaxed planning does not ensure a valid plan, one can only use action descriptions as plan features for actions that are beyond the current planning state when estimating the inexplicability measures. These estimates are then combined with the relaxed planning heuristic (which only considers plan cost) to guide the search.

\section{Environment Design for Explicability} 
\label{section:design}
\index{Design for Explicability}%
The environment in which the robot is operating may not always be conducive to explicable behavior. As a result, making its behavior explicable may be prohibitively expensive for the robot.  Additionally, certain behaviors that are explicable with respect to the human's mental model may not be feasible for the robot. Fortunately, in highly structured settings, where the robot is expected to solve repetitive tasks (like in warehouses, factories, restaurants, etc.), it might be feasible to \emph{redesign the environment} in a way that improves explicability of the robot's behavior, given a set of tasks. This brings us to the notion of environment design which involves redesigning the environment to maximize (or minimize) some objective for the robot. Thus, environment design can be used to boost the explicability of the robot's behavior, especially in settings that require solving repetitive tasks and a one-time environment design cost to boost explicable behavior might be preferable over the repetitive cost overhead of explicable behavior borne by the robot. 

However, environment design alone may not be a panacea for explicability. For one, the design could be quite expensive, not only in terms of making the required environment changes but also in terms of limiting the capabilities of the robot. Moreover, in many cases, there may not be a single set of design modifications that will work for a given set of tasks. For instance, designing a robot with wheels for efficient navigation on the floor will not optimize the robot's motion up a stairwell. 
\unsure{This means, to achieve truly effective synergy with autonomous robots in a shared space, we need a synthesis of environment design and human-aware behavior generation. In this section, we will talk about how we could trade off the one-time (but potentially expensive) design changes, against repetitive costs borne by the robot to exhibit explicable behavior.} 

\begin{figure*}%[th!]
\centering
\begin{subfigure}{\columnwidth} 
\centering
\includegraphics[width=\columnwidth]{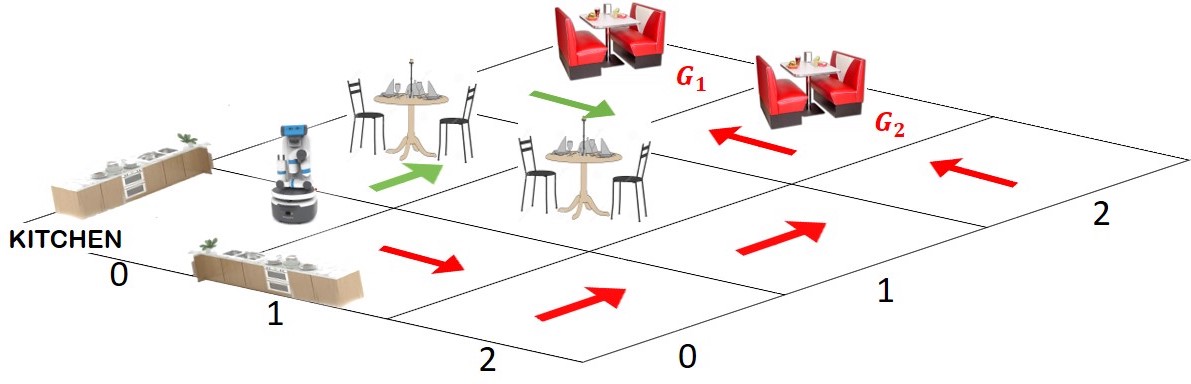}
\caption{Explicable behavior is costlier without design.}
\label{fig:3-ex1}
\end{subfigure} 
\hfill
\begin{subfigure}{\columnwidth}
\centering
\includegraphics[width=\columnwidth]{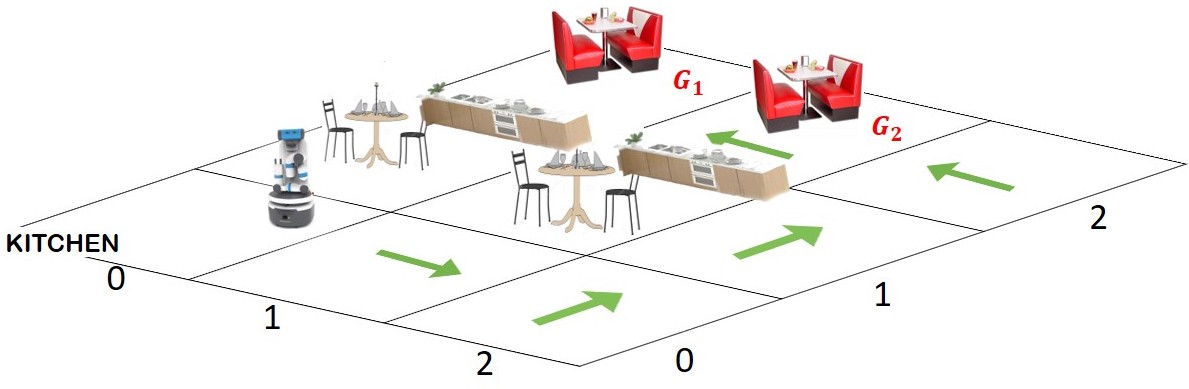}
\caption{Optimal behavior is explicable with design.}
\label{3-fig:ex2}
\end{subfigure}
\caption{
% A scenario demonstrating <-- removed this to fit in one line!
Use of environment design to improve the explicability of a robot's behavior in a shared environment.
}
\label{3-fig:example2}
\end{figure*}

\subsubsection{Motivating Example}

Consider a restaurant with a robot server (Figure \ref{fig:ex1}). 
Let $G_1$ and $G_2$ represent the robot's possible goals of serving the two booths: it travels between the kitchen and the two booths. 
The observers consist of customers at the restaurant. 
Given the position of the kitchen, the observers may have expectations on the route taken by the robot. However, unbeknownst to the observers, the robot can not traverse between the two tables and can only take the route around the tables. Therefore, the path marked in red is the cheapest path for the robot but the observers expect the robot to take the path marked in green in Figure \ref{fig:ex1}. 

In this environment, there is no way for the robot to behave as per the human's expectations. Applying environment design provides us with alternatives. For example, the designer could choose to build two barriers as shown in Figure \ref{3-fig:ex2}. With these barriers in place, the humans would expect the robot to follow the path highlighted in green. However, whether it is preferable to perform environment modifications or to bear the impact of inexplicable behavior depends on the cost of changing the environment versus the cost of inexplicability caused by the behavior.

\subsection{Problem Setting}

\label{sec:background}

We again consider an explicability problem of the  $\mathcal{P}_{Exp}= \langle \mathcal{M}^R, \mathcal{M}^R_h, \mathcal{D}_{\mathcal{M}^R_h}\rangle$. In this subsection we will again consider cases where the human mental model is given upfront. 
Let the expected set of plans in the human model. $\Pi_{\mathcal{M}^R_h}$, be the set of plans optimal in $\mathcal{M}^R_h$. 
The robot plans to minimize the inexplicability score in the human's mental model.
We will use the notation $\Pi^*_{\mathcal{IE}(\cdot, \mathcal{M}^R_h, \delta_{\mathcal{M}^R_h})}$ (in the absence of the parameter $\pi^R$) to refer to the set of plans in the robot's model with the lowest inexplicability score, and $\mathcal{IE}_{min}(\mathcal{P}_{Exp})$ to represent the lowest inexplicability score associated with the set. Further, let $f_{Exp}$ be the decision function used by the explicable robot: $f_{Exp}(\mathcal{P}_{Exp})$ represents the cheapest plan that minimizes the inexplicability score, i.e. $f_{Exp}(\mathcal{P}_{Exp}) \in \Pi^*_{\mathcal{IE}(\cdot, \mathcal{M}^R_h, \delta_{\mathcal{M}^R_h})}$ and $\lnot\exists \pi': \pi' \in \Pi^*_{\mathcal{IE}(\cdot, \mathcal{M}^R_h, \delta_{\mathcal{M}^R_h})}$ such that $c^R(\pi') < c^R(f_{Exp}(\mathcal{P}_{Exp}))$.

\subsubsection{Environment Design}
\index{Environment Design}
\unsure{Before we delve more into the problem of environment design to boost explicability, lets take a general look at {\em environment design problems}.}
An environment design problem takes as input the initial environment configuration along with a set of available modifications and computes a subset of modifications that can be applied to the initial environment to derive a new environment in which a desired objective is optimized. 

We consider ${\mathcal{M}^R}^0 = \langle F^0, {A^R}^0, {I^R}^0, {G^R}^0, {C^R}^0 \rangle$ as the initial environment
and $\rho^R$ as the set of valid configurations of that environment: ${\mathcal{M}^R}^0 \in \rho^R$. Let $\mathcal{O}$ be some metric that needs to be optimized with environment design, i.e a planning model with lower value for $\mathcal{O}$ is preferred. 
A design problem is a tuple $\langle {\mathcal{M}^R}^0, \Delta,  \Lambda^R, C, \mathcal{O} \rangle$ where, $\Delta$ is the set of all modifications, $\Lambda^R: \rho^R \times 2^\Delta \rightarrow \rho^R$ is the model transition function that specifies the resulting model after applying a subset of modifications to the existing model, $C: \Delta \rightarrow \mathbb{R}$ is the cost function that maps each design choice to its cost. The modifications are independent of each other and their costs are additive. We will overload the notation and use $C$ as the cost function for a subset of modifications as well, i.e. $C(\xi) = \sum_{\xi_i \in \xi} C(\xi)$. 

The set of possible modifications could include modifications to the state space, action preconditions, action effects, action costs, initial state and goal. In general, the space of design modifications, which are an input to our system, may also involve modifications to the robot itself (since the robot is part of the environment that is being modified).
An optimal solution to a design problem identifies the subset of design modifications, $\xi$, that minimizes the following objective  
consisting of the cost of modifications and the metric $\mathcal{O}$: ~$\min \mathcal{O}(\Lambda^R({\mathcal{M}^R}^0, \xi)),~ C(\xi)$.

\subsection{Framework for Design for Explicability}
\label{Ch03:Design}
\index{Design for Explicability!Framework}%

In this framework, we not only discuss the problem of environment design with respect to explicability but also in the context of \textbf{(1)} a set of tasks that the robot has to perform in the environment, and \textbf{(2)} over the lifetime of the tasks i.e. the time horizon over which the robot is expected to repeat the execution of the given set of tasks. 
These considerations add an additional dimension to the environment design problem since the design will have lasting effects
on the robot's behavior. In the following, we will first introduce the design problem for a single explicable planning problem, then extend it to a set of explicable planning problems and lastly extend it over a time horizon.

\subsubsection{Design for a Single Explicable Problem}

In the design problem for explicability, the inexplicability score becomes the metric that we want to optimize for. That is we want to find an environment design such that the inexplicability score is reduced in the new environment.  This problem can be defined as follows:

\begin{definition}
\index{Design for Explicability!Problem}%
The \textbf{design problem for explicability} is a tuple, 
\[\mathcal{DP}_{Exp} = \langle \mathcal{P}^0_{Exp}, \Delta, \Lambda_{Exp}, C, \mathcal{IE}_{min} \rangle,\] 
where:
\begin{itemize}
\item $\mathcal{P}^0_{Exp} \in \rho_{Exp}$ is the initial configuration of the explicable planning problem, where $\rho_{Exp}$ represents the set of valid configurations for $\mathcal{P}_{Exp}$. 
\item $\Delta$ is the set of available design modifications. The space of all possible modifications is the power-set $2^\Delta$.
\item $\Lambda_{Exp}: \rho_{Exp} \times 2^\Delta \rightarrow \rho_{Exp}$ is the transition function over the explicable planning problem, which gives an updated problem after applying the modifications.  
\item $C$ is the additive cost associated with each design in $\Delta$.
\item $\mathcal{IE}_{min}:\rho_{Exp} \rightarrow \mathbb{R}$ is the minimum possible inexplicability score in a configuration, i.e. the inexplicability score associated with the most explicable plan.
%for a set of (possibly) updated explicable problems.
\end{itemize}
\end{definition}

With respect to our motivating example in Figure \ref{fig:3-ex1}, $\mathcal{DP}_{Exp}$ is the problem of designing the environment to improve the robot's explicability given its task of serving every new customer at a booth (say $G_1$) only once. The optimal solution to $\mathcal{DP}_{Exp}$ involves finding a configuration which minimizes the minimum inexplicability score.
We also need to take into account an additional optimization metric which is the effect of design modifications on the robot's plan cost. That is, we need to examine to what extent the decrease in inexplicability is coming at the robot's expense. For instance,
if you confine the robot to a cage so that it cannot move, its behavior becomes completely and trivially explicable, but the cost of achieving its goals goes to infinity. 
%consider a cheap design modification which reduces the least possible inexplicability over the explicable planning problem; however it may not be a desirable if it results in the robot having to bear an exorbitant plan cost.
%however, now the robot has to bear an exorbitant plan cost to perform tasks in comparison to its previous optimal plans or explicable plans. 
%Thus the optimal solution to \emph{DPExp} can be optimized as follows:

\begin{definition}
\index{Design for Explicability!Optimal Solution}
An \textbf{optimal solution} to $\mathcal{DP}_{Exp}$, is a subset of modifications $\xi^*$ that minimizes the following: 
\begin{align}
\min~ \mathcal{IE}_{min}(~\mathcal{P}^*_{Exp}),~ C(\xi^*), c^R(f_{Exp}(\mathcal{P}^*_{Exp})) 
\label{3-eq:sol}
\end{align}
where $\mathcal{P}^*_{Exp} = \Lambda_{Exp}(\mathcal{P}^0_{Exp}, \xi^*)$ is the final modified explicable planning problem, 
$\mathcal{IE}_{min}(\cdot)$ represents the minimum possible inexplicability score for a given configuration, 
$C(\xi^*)$ denotes the cost of the design modifications and $c^R(f_{Exp}(\mathcal{P}^*_{Exp}))$ is the cost of the cheapest most explicable plan in a configuration.
\end{definition}

\subsubsection{Design for Multiple Explicable Problems}
\index{Design for Explicability!Multi-task Design}
%These are important considerations in any environment design problem, even though largely overlooked in existing literature.
%In addition, 
%in our specific case of a cooperative setting, these 
%Our cooperative setting gives rise to even more intriguing challenges in the design of environments and 
%its effect on the robot's behavior.
%But rather than just replacing the design objective with inexplicability score, we would also want to extend the design framework along two important dimensions. Namely  1) the need to support multiple problems 2) the need to consider the problem lifetime (i.e the time over which the robot is expected to perform the tasks). Arguably these are important consideration in any environment design problems, but the fact that we are considering an inherently cooperative setting allows us to provide novel strategies to tackle challenges arising these considerations.
We will now show how $\mathcal{DP}_{Exp}$ evolves when there are multiple explicable planning problems in the environment that the robot needs to solve. When there are multiple tasks there may not exist a single set of design modifications that may benefit all the tasks. In such cases, a solution might involve performing design modifications that benefit some subset of the tasks while allowing the robot to act explicably with respect to the remaining set of tasks. Let there be $k$ explicable planning problems, given by the set $\mathbf{P}_{Exp} = \{ \langle \mathcal{M}^R(0), \mathcal{M}^R_h(0), \delta_{\mathcal{M}^R_h(0)}  \rangle, \ldots, \langle \mathcal{M}^R(k), \mathcal{M}^R_h(k), \delta_{\mathcal{M}^R_h(k)} \rangle\}$, with a categorical probability distribution $\mathbb{P}$ over the problems. We use $\mathcal{P}_{Exp}(i) \in \mathbf{P}_{Exp}$ to denote the $i^{th}$ explicable planning problem. These $k$ explicable problems may differ in terms of their initial state and goal conditions.
%We will use $\rho_{Exp}$ to represent the set of valid configurations for the explicable planning problems. 
Now the design problem can be defined as:
\begin{align}
\mathcal{DP}_{Exp,\mathbb{P}} = \langle \mathbf{P}^0_{Exp}, \mathbb{P}, \Delta, \Lambda_{Exp}, C, \mathcal{IE}_{min,\mathbb{P}}  \rangle,
\end{align}
where $\mathbf{P}^0_{Exp}$, is the set of planning tasks in the initial environment configuration, $\mathcal{IE}_{min,\mathbb{P}}$ is a function that computes the minimum possible inexplicability score in a given environment configuration by taking the expectation over the minimum inexplicability score for each explicable planning problem, 
i.e., $\mathcal{IE}_{min,\mathbb{P}}(\mathbf{P}_{Exp}) = \mathbb{E}[\mathcal{IE}_{min}(\mathcal{P}_{Exp})]$, where $\mathcal{P}_{Exp} \sim \mathbb{P}$. With respect to our running example, $\mathcal{DP}_{Exp,\mathbb{P}}$ is the problem of designing the environment 
%to improve robot's explicability 
given the robot's task of serving every new customer only once at either of the booths ($G_1$, $G_2$) with probability given by  $\mathbb{P}$. 

The solution to $\mathcal{DP}_{Exp,\mathbb{P}}$ has to take into account the distribution over the set of explicable planning problems. Therefore the optimal solution is given by: 
\begin{align}
\min~ ~\mathcal{IE}_{min, \mathcal{D}}(~\mathbf{P}^*_{Exp}),~ C(\xi^*),~
\mathbb{E}[c^R(f_{Exp}(\mathcal{P}^*_{Exp}) )]
\end{align}

\noindent where $\mathcal{P}^*_{Exp} \sim \mathbb{P}$. 
A valid configuration minimizes the minimum possible inexplicability score, which involves 1) expectation over minimum inexplicability scores for each explicable planning problem; 2) the cost of the design modifications (these modifications are applied to each explicable planning problem); and 3) the expectation over the cheapest most explicable plan for each explicable planning problem.

\subsubsection{Longitudinal Impact on Explicable Behavior}
\index{Design for Explicability!Longitudinal Setting}%
\index{Explicability!Longitudinal Setting}%
\index{Explicability!Discounting}%

The process of applying design modifications to an environment makes more sense if the tasks are going to be performed repeatedly in the presence of a human (i.e. the robot does not have to bear the cost of being explicable repeatedly).
This has quite a different temporal characteristic in comparison to that of execution of one-time explicable behavior. For instance, design changes are associated with a one-time cost (i.e. the cost of applying those changes in the environment). On the other hand, if we are relying on the robot to execute explicable plans at the cost of foregoing optimal plans, then it needs to bear this cost multiple times in the presence of a human over the time horizon. 

\begin{figure*}%[th!]
\centering
\includegraphics[scale=0.5]{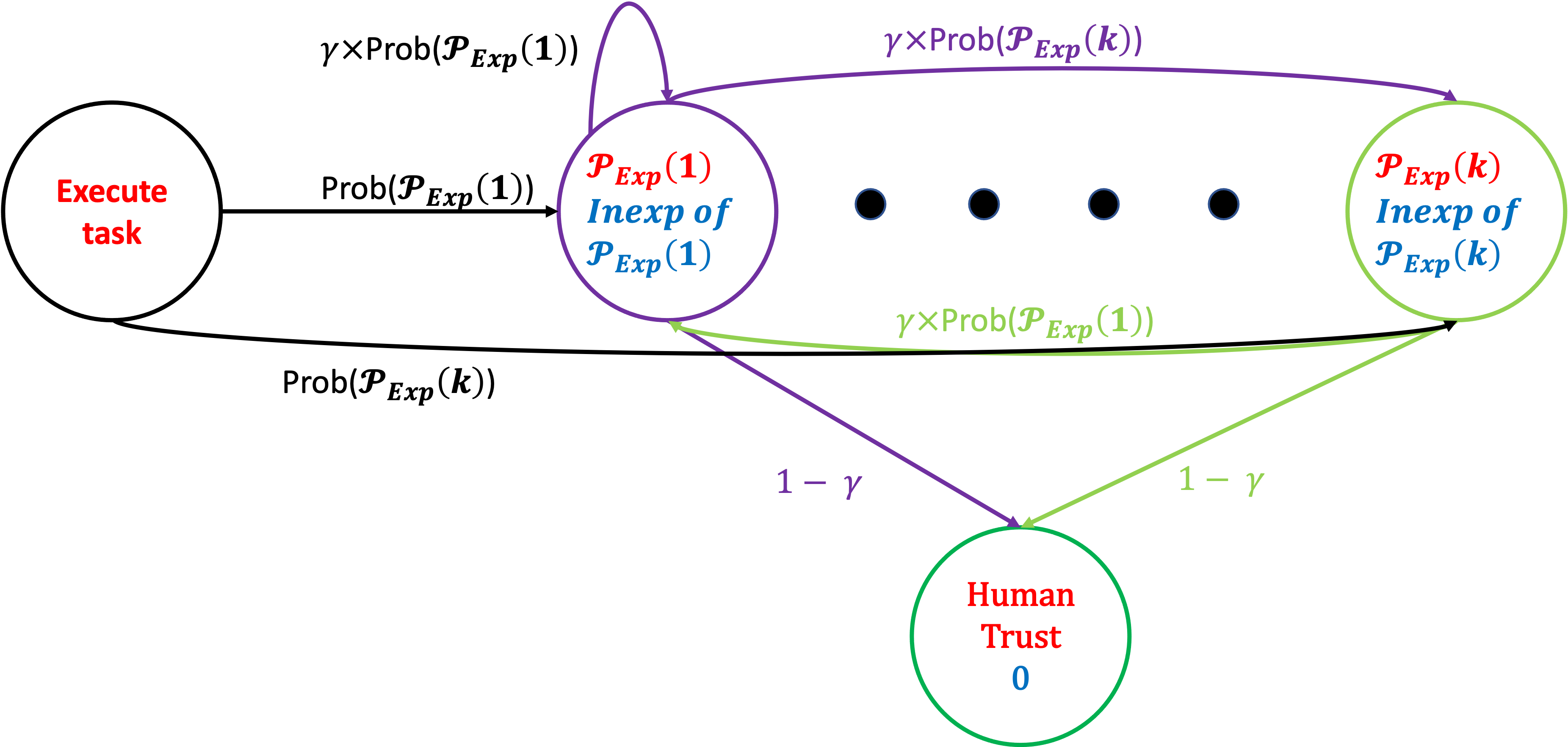}
\caption{Illustration of longitudinal impact on explicability. \emph{Prob} determines the probability associated with executing each task in $\mathbf{P}_{Exp}$. For each task, the reward is determined by the inexplicability score of that task. The probability of achieving this reward is determined by $\gamma$ $\times$  probability of executing that task. Additionally, with a probability $(1 - \gamma)$ the human
ignores the inexplicability of a task and the associated reward is given by an inexplicability score of 0.}
\label{3-fig:ex3}
\end{figure*} 

We will use a discrete time formulation where the design problem is associated with a time horizon $\mathcal{T}$. At each time step, one of the $k$ explicable planning problems is chosen. Now the design problem can be defined as:
\begin{align}
\mathcal{DP}_{Exp,\mathbb{P}, \mathcal{T}} = \langle \mathbf{P}^0_{Exp}, \mathbb{P}, \Delta, \Lambda_{Exp}, C, \mathcal{IE}_{min,\mathbb{P}}, \mathcal{T} \rangle
\end{align}

In our running example, $\mathcal{DP}_{Exp,\mathbb{P}, \mathcal{T}}$ is the problem of designing the environment given the robot's task of serving the same customer at either of the booths with a distribution $\mathbb{P}$ over a horizon $\mathcal{T}$.

%In the past literature, the explicable behavior has been studied with respect to a single interaction with a human over a given task \cite{exp-yu,explicable-anagha}. However, we consider a time horizon, $\mathcal{T} > 1$, over which the robot's interaction with the human may be repeated multiple times for the same task. 
\unsure{Note than earlier discussions on explicable behavior generation focused on cases where there was only a single interaction between the human and the robot. 
However, in this section we consider a time horizon, $\mathcal{T} > 1$, over which the robot's interaction with the human may be repeated multiple times for the same task.}
This means the human's expectations about the task can evolve over time. This may not be a problem if the robot's behavior aligns perfectly with the human's expectations. Although, if the robot's plan for a given task is associated with a non-zero inexplicability score, then the human is likely to be more surprised the very first time she notices the inexplicable behavior than she would be if she noticed the inexplicable behavior subsequent times.
% The second time the robot performs the same task, the human may get used to the inexplicability of the robot's behavior and may expect the robot to perform the same inexplicable behavior. 
As the task is performed over and over, the amount of surprise associated with the inexplicable behavior starts decreasing. In fact, there is a probability that the human may ignore the inexplicability of the robot's behavior after sufficient repetitions of the task. We incorporate this intuition by using discounting. 
%longitudinal influence of repetition on the inexplicability score. 
%We calculate the longitudinal aspect of inexplicability score by introducing a very intuitive form of discounting. 

Figure \ref{3-fig:ex3} illustrates the Markov reward process to represent the dynamics of this system. Let $(1 - \gamma)$ denote the probability that the human will ignore the inexplicability of the robot's plan, i.e, the reward will have inexplicability score 0. $\gamma$ times the probability of executing a task represents the probability that the reward will have the minimum inexplicability score associated with that task. Assuming $\gamma < 1$, the minimum possible inexplicability score for a set of explicable planning problems is:
\begin{align}
f_\mathcal{T}(\mathcal{IE}_{min,\mathbb{P}}(\mathbf{P}_{Exp})) 
&= \mathcal{IE}_{min,\mathbb{P}}(\mathbf{P}_{Exp}) \nonumber \\
& \ + \gamma * \mathcal{IE}_{min,\mathbb{P}}(\mathbf{P}_{Exp}) + \ldots + \nonumber \\
& \ \ \gamma^{T-1} * \mathcal{IE}_{min,\mathbb{P}}(\mathbf{P}_{Exp}) \nonumber \\
f_\mathcal{T}(\mathcal{IE}_{min,\mathbb{P}}(\mathbf{P}_{Exp})) 
&= \frac{1-\gamma^\mathcal{T}}{1-\gamma} * \mathcal{IE}_{min,\mathbb{P}}(\mathbf{P}_{Exp}) 
\end{align}

Thus the optimal solution to $\mathcal{DP}_{Exp,\mathbb{P}, \mathcal{T}}$ is given by: 
\begin{align}
\min~ ~f_\mathcal{T}(\mathcal{IE}_{min, \mathbb{P}}(\mathbf{P}^*_{Exp})),~ C(\xi^*),
\nonumber \\ 
~ \mathbb{E}[c^R(f_{Exp}(\mathcal{P}^*_{Exp}))] * \mathcal{T}
\end{align}
\noindent where, $\mathcal{P}^*_{Exp} \sim \mathbb{P}$. The optimal solution is a valid configuration that minimizes 1) the minimum possible inexplicability over the set of explicable planning problems given the human's tolerance to inexplicable behavior; 2) one-time cost of the design modifications; and 
3) the expectation over the cheapest most explicable plan for each explicable planning problem given a time horizon.
Note that, since the design cost is not discounted and we always make the design changes before the task is solved, there is never a reason to delay the design execution to future steps in the horizon. Instead it can be executed before the first time step.

\subsection{Search for Optimal Design}
\unsure{We can find the optimal solution for $\mathcal{DP}_{Exp,\mathbb{P}, \mathcal{T}}$, by performing a breadth-first search over the space of environment configurations that are achievable from the initial configuration through the application of the given set of modifications. The performance of the search depends on the number of designs available. By choosing appropriate design strategies, significant scale up can be attained.
Each search node is a valid environment configuration and the possible actions are the applicable designs. For simplicity, we convert the multi-objective optimization in Equation \ref{3-eq:sol} into a single objective as a linear combination of each term associated with a coefficients $\alpha$, $\beta$, and $\kappa$, respectively. The value of each node is decided by the aforementioned objective function. For each node, it is straightforward to calculate the design modification cost. However, in order to calculate the minimum inexplicability score and the robot's plan cost, we have to generate a plan that minimizes the inexplicability score for each explicable planning problem in that environment configuration. 
In this setting one could use the method discussed in Section \ref{section:plan-dist}.
Since the distances here are specifically limited to cost differences, we can also use a compilation based method, that compiles the problem of generating the explicable plan to a classical planning problem. We will discuss this compilation in detail in Chapter \ref{ch06}.
Essentially, the search has two loops: the outer loop which explores all valid environment configurations, and the inner loop which performs search in a valid environment configuration to find a plan that minimizes the inexplicability score. At the end of the search, the node with best value is chosen, and the corresponding set of design modifications, $\xi^*$, is output.} 

\unsure{One way to optimize the search over the space of environment configurations is to only consider the designs that are relevant to the actions in the optimal robot plans ($\Pi^*_{\mathbf{M}^R}$) and those in the human's expected plans ($\Pi^*_{\mathbf{M}^R_h}$) given the set of tasks. This can be implemented as a pruning strategy that prunes out designs that are not relevant to the actions.}

\subsection{Demonstration of Environment Design for Explicability} 
\index{Design for Explicability!Demonstration}%

We use our running example from Figure \ref{fig:ex1} to demonstrate how the design problem evolves. We constructed a domain where the robot had 3 actions: \textit{pick-up} and \textit{put-down} to serve the items on a tray and \textit{move} to navigate between the kitchen and the booths. Some grid cells are blocked due to the tables and the robot cannot pass through these: cell(0, 1) and cell(1, 1). Therefore, the following passages are blocked: cell(0, 0)-cell(0, 1), cell(0, 1)-cell(0, 2), cell(0, 1)-cell(1, 1), cell(1, 0)-cell(1, 1), cell(1, 1)-cell(1, 2), cell(1, 1)-cell(2, 1). We considered 6 designs, each consisting of putting a barrier at one of the 6 passages to indicate the inaccessibility to the human (i.e. the design space has $2^6$ possibilities). 

For the following parameters: $\alpha = 1$, $\beta = 30$, $\kappa = 0.25$ and $\gamma = 0.9$, 
\unsure{let us consider the problem of identifying design for three settings}: (a) single explicable problem for $\mathcal{T} = 1$, (b) multiple explicable problems for $\mathcal{T} = 1$, and (c) multiple explicable problems for $\mathcal{T} = 10$. As mentioned earlier, we will use the set of plans optimal in the human model as the expected plan set and the difference between cost of the plans as the distance function. \unsure{Let us concretize the three settings as follows}, (a) involves serving a new customer at a booth (say $G_1$) only once, (b) involves serving a new customer only once at either of the booths with equal probability and (c) involves serving each customer at most 10 times at either of the booths with equal probability. We found that for settings (a) and (b) no design was chosen. This is because these settings are over a single time step and the cost of installing design modifications in the environment is higher than the amount of inexplicability caused by the robot ($\beta > \alpha$). On the other hand, for setting (c), the algorithm generated the design in Figure \ref{3-fig:ex2}, which makes the robot's roundabout path completely explicable to the customers.  

\section{Bibliographic Remarks}
The first paper to introduce explicability was \cite{exp-yu} that introduced the model-free explicability. The paper used CRF \cite{lafferty2001conditional} to learn the labeling model and a modified version of the FF planner \citep{hoffmann2001ff} to generate the plan. The paper also performed some limited user studies to test whether such models can capture people's measure of plan explicability. The model-based explanation was introduced by \cite{explicable-anagha} and used a modified version of \cite{helmert2006fast} for the reconciliation search and used a learned distance function which used features described in Definition \ref{def:4}. The specific distance functions considered was originally discussed and used in the context of diverse planning (cf. \citep{srivastava2007domain,nguyen-partialp-2012}) where these distance functions are used to generate plans of differing characteristics. The paper also used data collected from participants to learn the distance function. The design for explicability was proposed by \cite{kulkarni2020designing} and followed the other design works that have been considered within the context of planning (cf. \citep{keren2014goal,keren2016privacy,keren2018strong}). In general, most of these works can be seen as special cases of the more general environment design work studied in papers like \cite{zhang2009general}. The specific paper \cite{kulkarni2020designing} used a compilation based method to generate the explicable plan as the paper looked at a cost difference based distance. The compilation follows the methods discussed in \cite{exact}, but removes the use of explanatory actions. We will discuss the base compilation in more detail in Chapter \ref{ch_balance}.
\clearpage

                % algorithm sample
    \chapter{Legible Behavior}
\label{ch04}

In this chapter, the discussion will focus on another type of interpretable behavior, namely legibility. 
The notion of legibility allows the robot to implicitly communicate information about its goals, plans (or model, in general) to a human observer. 
For instance, consider a human robot cohabitation scenario consisting of a multi-tasking robot with varied capabilities that is capable of performing a multitude of tasks in an environment. 
In such a scenario, it is crucial for the robot to aid the human's goal or plan recognition process, as the human observer may not always know the robot's intentions or objectives beforehand. 
Hence, in such cases, it may be useful for the robot to communicate information that the human is unaware of. As the better off the human is at identifying the robot's goals or plans accurately, the better off is the overall team performance. However, explicit communication of objectives might not always be suitable. For instance, the \textit{what, when} and \textit{how} of explicit communication may require additional thought. Further, several other aspects like cost of communication (in terms of resources or time), delay in communication (communications signals may take time to reach the human), feasibility of communication (broken or unavailable sensors), etc., may also need to be considered. On the other hand, the robot can simply synthesize a behavior that implicitly communicates the necessary information to the human observer. 

We will discuss the notion of legibility from the perspective of an offline setting where the observer has partial observability of the robot's actions. That is, the robot has to synthesize legible behavior in a setting where the human observer has access to the observations emitted from the entire execution trace of the robot and these observations do not always reveal the robot's exact action or state to the human. As the human is trying to infer the robot's goals or plans, she operates in a belief space due to the partial observability of the robot's activities. The robot has to modulate the human's observability and choose actions such that the ambiguity over specific goals or plans is reduced. 
We refer to this as the controlled observability planning problem (COPP). 
In the upcoming sections, we will see how the COPP formulation can be used to synthesize goal legible behavior as well as plan legible behavior in an offline setting with partial observability of the robot's activities. 

\section{Controlled Observability Planning Problem}
\label{4-sec:COPP}

\unsure{In this framework, we consider two agents: a robot and the human observer. The robot has full observability of its activities. However, the observer only has partial observability of the robot's activities. The observer is aware of the robot's planning model and receives observations emitted as a side effect of the robot's execution. This framework supports an offline setting, and therefore the observer only receives the observations after the robot has finished executing the entire plan.}

In this setting, the robot has a set of candidate goals, inclusive of its true goal. The candidate goals are the set of possible goals that the robot may achieve in the given environment. The observer is aware of the robot's candidate goal set but is unaware of the robot's true goal. We now introduce a general COPP framework that will be used to define the goal legibility and plan legibility problems in the upcoming sections. 

\unsure{
\begin{definition}
\index{COPP}%
A \textbf{controlled observability planning problem} is a tuple, $\mathcal{P}_{CO} = \langle \mathcal{M}^R, \mathcal{G}, \Omega, \mathcal{O} \rangle$, where,
\begin{itemize}
\item $\mathcal{M}^R$ is the planning model of the robot. 
\item $\mathcal{G} = \{G_1 \cup G_2 \ldots \cup G_{n-1} \cup G^R\}$ is a set of candidate goal conditions, each defined by subsets of fluent instantiations, where $G^R$ is the true goal of the robot.
\item $\Omega = \{o_i | i = 1, \ldots, m \}$ is a set of $m$ observations that can be emitted as a result of the action taken and the state transition. 
\item $\mathcal{O} : (A \times \mathcal{S}) \rightarrow \Omega $ is a many-to-one observation function associated with the human which maps the action taken and the next state reached to an observation in $\Omega$. That is to say, the observations are deterministic, each $\langle a, s' \rangle$ pair is associated with a single observation but multiple pairs can be mapped to the same observation. 
\end{itemize}
\end{definition}}

\unsure{Note that here the human's expectation of the robot (i.e. $\mathcal{M}^R_h$) is defined in a distributed fashion. In particular, this setting assumes that the expectation matches the real model in all aspects but the goals. That is the human is unaware of the true goal and only has access to the set of potential goals.
Moreover the human is a noisy observer (as defined by the observation model $\mathcal{O}$ and the set of observations $\Omega$). 
To simplify the setting we will focus on the cases where the human is aware of the observation model, i.e., from the observation they could potentially guess the possible state action pairs that could have generated it.
%The observer has access to $\mathcal{P_{CO}}$, but is unaware of the true goal of the robot. The observation function can be seen as a sensor model. The observer receives the partial observations of the robot's execution based on the mapping of this sensor model. That is, for every action taken by the robot and an associated state transition, the observer receives an observation. However, a single observation may be consistent with multiple action-state pairs because of the many-to-one formulation of $\mathcal{O}$. 
Therefore, the human observer operates in the belief space. The robot takes the belief space of the observer into account during its planning process, so as to modulate what the human believes the current possible states could be.} 

\subsection{Human's Belief Space}
\index{Human Belief Space}%

The human may use its observations of the robot's activity to maintain a \emph{belief state}. A belief state is simply a set of possible states consistent with the observations. We use $\hat{s}$ as a notational aid to denote a state that is a member of the belief state.

\begin{definition}
\label{4-dfn:belief}
The \textbf{initial belief}, $b_0$, induced by observation, $o_0$ is defined as, $b_0 = \{ \hat{s}_0 ~|~  \mathcal{O}(\emptyset, s_0) = o_0 \land \mathcal{O}(\emptyset, \hat{s}_0) = o_0 \}$.
\end{definition}

Whenever a new action is taken by the robot, and the state transition occurs, the human's belief updates as follows:

\begin{definition} A \textbf{belief update}, $b_{i+1}$ for belief $b_{i}$ is defined as, $b_{i+1} = update(b_i,  o_{i+1}) = \{ \hat{s}_{i+1}~|~\exists \hat{a},\  \delta(\hat{s}_{i}, \hat{a}) \models \hat{s}_{i+1} \land \hat{s}_{i} \in b_i \land \mathcal{O}(\hat{a}, \hat{s}_{i+1}) = o_{i+1}  \}$.
\end{definition}

A sequence of belief updates gives us the observer's belief sequence that is consistent with the observation sequence emitted by the robot.

\begin{definition} A \textbf{belief sequence} induced by a plan p starting at state $s_0$, BS(p, $s_0$), is defined as a sequence of beliefs $\langle b_o, b_1, \ldots, b_n \rangle$ such that there exist $o_0, o_1, o_2, \ldots, o_n \in \Omega$ where,
\begin{itemize}
\item $o_i = \mathcal{O}(a_i, s_i)$
\item $b_{i+1} = update(b_i, o_{i+1})$
\end{itemize}
\end{definition}

The set of plans that are consistent with the belief sequence of a given plan are called as belief plan set. 

\begin{definition} 
\label{4-def:bps}
\index{Human Belief Space!Belief Plan Set}%
A \textbf{belief plan set}, BPS(p, $s_0$) = $\{ p_1, \ldots, p_n \}$, induced by a plan $p$ starting at $s_0$, is a set of plans that are formed by causally consistent chaining of state sequences in $BS(p, s_0)$, i.e., BPS(p, $s_0$) = $\{ \langle \hat{s}_0, \hat{a}_1, \hat{s}_1, \ldots, \hat{s}_n \rangle ~|~ \forall~ \hat{a}_j, ~\hat{s}_{j-1} \models pre(\hat{a}_j) ~\wedge~ \hat{s}_{j-1} \in b_{j-1} ~\wedge~ \hat{s}_{j} \models \hat{s}_{j-1} \cup adds(\hat{a}_j) \setminus dels(\hat{a}_j) ~\wedge~ \hat{s}_{j} \in b_{j} \}$.
\end{definition}

The robot's objective is to generate a belief sequence in human's belief space, such that, the last belief in the sequence satisfies certain desired conditions. 

\subsection{Computing Solutions to COPP variants}
\index{COPP!Variants}%
Now we present a common algorithm template that can be used to compute plans for the COPP problem variants.

\subsubsection{Algorithm for Plan Computation}
\label{4-subsec:algo}
\index{COPP!Algorithm}%

\unsure{In this algorithm, each search node is represented by a belief estimate, which is a $\Delta$-sized subset of the belief state. A search node, $b_{\Delta}$, consists of state $s$ and a set of possible states, $\tilde{b}$ (cardinality of $\tilde{b}$ is given by $\Delta -1$, where $\Delta$ is a counter). The successor node uses the $s$ in $b_{\Delta}$ to generate the successor state, and $b_{\Delta}$ to create the next belief state. There are two loops in the algorithm: the outer and the inner loop. The outer loop maintains the cardinality of $b_{\Delta}$ by incrementing the value of $\Delta$ in each iteration, such that value of $\Delta$ ranges from $1, 2, \ldots, |\mathcal{S}|$.  In the inner loop, a heuristic-guided forward search (for instance, GBFS\citep{russell2002artificial}) can be used to search over space of belief states of cardinality less than or equal to $\Delta$. These loops ensure the complete exploration of the belief space. The algorithm terminates either when a plan is found or after running the outer loop for $|\mathcal{S}|$ iterations. The outer loop ensures that all the paths in the search space are explored.} 

\begin{proposition}
The algorithm necessarily terminates in finite number of $|\mathcal{S}|$ iterations, such that, the following conditions hold: 

\noindent (\textbf{Completeness}) The algorithm explores the complete solution space of $\mathcal{P_{CO}}$, that is, if there exists a $\pi_{\mathcal{P_{CO}}}$ that correctly solves $\mathcal{P_{CO}}$, it will be found. 

\noindent (\textbf{Soundness}) The plan, $\pi_{\mathcal{P_{CO}}}$, found by the algorithm correctly solves $\mathcal{P_{CO}}$ as ensured by the corresponding goal-test.
\end{proposition}

\unsure{We have not yet defined the goal-test and the notion of validity of a solution for a COPP problem. We will be encoding the notion of legibility we wish to pursue as part of the goal-test. We will do so by ensuring that the legibility score for the corresponding plan is above certain threshold. We will look at some specific legibility scores we can use as we ground this specific algorithmic framework for various use cases.}
%The goal tests of the COPP problem variants ensure that the solutions are correct. 

\begin{algorithm}[tbp!]
\raggedright
\caption{COOP Solution Plan Algorithm}
\label{4-procedure:COPP-algo}
\algorithmicrequire{ \ $\mathcal{P_{CO}} = \langle D, \mathcal{G}, \Omega, \mathcal{O} \rangle$} \\
\algorithmicensure{\ Solution $\pi_{\mathcal{P_{CO}}}$, observation sequence, $O_{\mathcal{P_{CO}}}$}
\begin{algorithmic}[1] 
\STATE Initialize $\textit{open}$, $\textit{closed}$, $\textit{unopened}$ lists and the counter $\Delta \gets 1 $
\STATE $b_{\Delta} \gets \{ I \}$ ; $b_0 \gets \{ \mathcal{O}(\emptyset, I) \}$ \COMMENT{\textcolor{blue}{Initialize initial search node, initial belief}} 
\STATE $\textit{open}.push(\langle b_{\Delta}, b_0 \rangle, priority = 0)$
\WHILE{$\Delta \leqslant |\mathcal{S}|$} \label{4-line:outer}
\WHILE{$\textit{open} \neq \emptyset$ } \label{4-line:inner}
\STATE $b_{\Delta}, b, h(b_{\Delta}) \gets \textit{open}$.pop() 
\IF{$ |b_{\Delta}| > \Delta $}
\STATE $\textit{unopened}$.push$(\langle b_{\Delta}, b \rangle, h(b_{\Delta}))$; \textbf{continue}
\ENDIF
\STATE $\textit{closed} \gets \textit{closed} \cup b_{\Delta} $ 
\IF{ $\langle b_{\Delta}, b \rangle \models$ \textsc{GOAL-TEST}($\mathcal{G}$)} \label{4-line:goal} 
\STATE $\textbf{return}~\pi_{\mathcal{P_{CO}}}, O_{\mathcal{P_{CO}}}$
\ENDIF
\FOR{$s^{\prime} \in \textit{successors}(s)$}
\STATE $o \gets \mathcal{O}(a, s^{\prime})$ 
\STATE $b^{\prime} \gets$ Belief-Generation($b, o$)
\STATE $b_{\Delta}^{\prime} = \langle s^{\prime}, \tilde{b}^{\prime} \rangle$ \COMMENT{\textcolor{blue}{$\tilde{b}^{\prime}$ of size $\Delta$-1}} 
\STATE $h(b_{\Delta}^{\prime}) \gets$ \textsc{HEURISTIC-FUNCTION}$(b_{\Delta}^{\prime}, b^{\prime})$ \label{4-line:heuristic} 
\STATE add $b_{\Delta}^{\prime}$ to $\textit{open}$ if not in $\textit{closed}$
\ENDFOR
\ENDWHILE
\STATE $\Delta \gets \Delta + 1$ 
\STATE copy items of $\textit{unopened}$ to $\textit{open}$, empty $\textit{unopened}$
\ENDWHILE 
\STATE
\STATE \textbf{procedure} Belief-Generation($b$, $o$)
\STATE $b^{\prime} \gets \{ \}$ 
\FOR{$\hat{s} \in b$}
\FOR{$\hat{a} \in A$}
\IF{$\mathcal{O}(\hat{a}, \delta(\hat{s}, \hat{a})) = o $}
\STATE $b^{\prime} \gets  b^{\prime} \cup \delta(\hat{s}, \hat{a}) $ 
\ENDIF
\ENDFOR
\ENDFOR
\STATE \textbf{return} $b^{\prime}$
\end{algorithmic}
\end{algorithm}

\subsubsection{Optimization} 
\index{COPP!ALgorithm!Optimization}%
In order to speed up the search process, we perform an optimization on the aforementioned algorithm. For each search node, $b_{\Delta}$, apart from the approximate belief estimate, we maintain the full belief update $b$ consistent with a path to $s$. The approximate belief update $b_{\Delta}$ can be generated by choosing $\Delta$-sized combinations of states from the complete belief. For example, when $\Delta = 1$, $b_{\Delta}$ only consists of the state $s$ but still maintains full belief update $b$, when $\Delta = 2$, $b_{\Delta}$ consists of a new combination of approximate belief of size 2 derived from the maintained full belief. When $\Delta = 1$, because of the check for duplicate states in the closed list, only one path to the search node is explored. Therefore, the use of $\Delta$ allows the search process to explore multiple paths leading to a particular search node. The complete $b$ helps in finding the problem variant solutions faster at lower $\Delta$ values. We present the details of the optimization in Algorithm \ref{4-procedure:COPP-algo}. In the following sections, we show how we customize the goal-test (line \ref{4-line:goal}) and the heuristic function (line \ref{4-line:heuristic}) to suit the needs of each of the COPP problem variants. 

\subsection{Variants of COPP}

\unsure{Within this framework, we will discuss the problem of goal legibility and plan legibility. With goal legibility, the objective of the robot is to convey \emph{at most} $j$ goals to the observer. With plan legibility, we look at a specific instance of COPP problem where $|\mathcal{G}| = 1$. Since we know that COPP problems require that $G^R$ be part of  $|\mathcal{G}|$ this means that in these cases the human knows the true goal of the robot. But even then the human observer may not know the exact plan being executed by the robot given their limited sensors. Thus the objective of the robot within plan legibility becomes to constrain the uncertainty of the observer to \textit{at most} $m$ similar plans to the goal. We will see both of these problems in detail in the following sections. The COPP framework can be also used in adversarial environments. These variants will be covered in Chapter \ref{ch09}.}
\section{Goal Legibility}
\index{COPP!Goal Legibility}%

In this setting, the robot's objective is to convey some information about its candidate goal set to the human observer. This may involve communicating its true goal to human, by ensuring at most one goal is consistent with the observation sequence produced. Or in general, the robot may want to communicate a set of at most $j$ candidate goals inclusive of its true goal, to the human observer. Essentially, the point with goal legibility is to reduce the observer's ambiguity over the robot's possible goal set by narrowing it down to at most $j$ goals. \unsure{ Which mean in this scenario, for a given plan $\pi$ the corresponding legibility score is given as  $\mathcal{L}(\pi, \mathcal{P}_{CO}) \propto \frac{1}{|~\{G ~|~ G \in \mathcal{G} \land G \in BS(\pi, I)~|}$}. 

\begin{figure}[tbp!]\centering
    \begin{subfigure}[t]{\columnwidth}        \includegraphics[width=0.83\columnwidth]{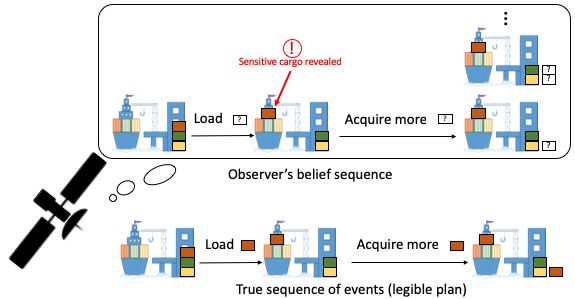}
        \caption{}
        \label{4-fig:ex1}
    \end{subfigure}
     ~ 
     \begin{subfigure}[t]{\columnwidth}
        \includegraphics[width=0.9\columnwidth]{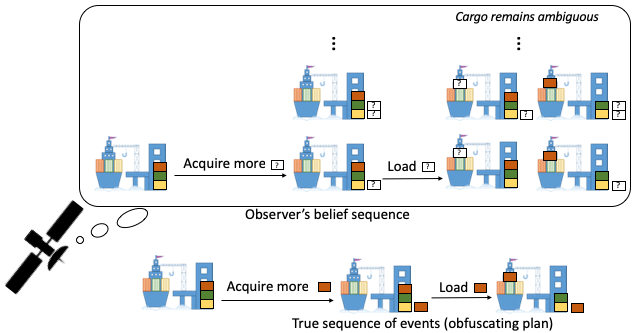}
        \caption{}
        \label{4-fig:ex2}
    \end{subfigure}
    \caption{\small{The differences in belief sequences induced by different plans for an observer with noisy sensors.}}
\label{4-fig:example2}
\end{figure}

\subsubsection{Example}

\unsure{Let's understand this problem with an example. Consider a situation where the robot is a port management agent and a human is the supervisor that has sensors or subordinates at the port who provide partial information about the nature of activity being carried out at the port (refer Figure \ref{4-fig:example2}). For instance, when a specific crate is loaded onto the ship, the observer finds out that something was loaded, but not the identity of the loaded crate. The observer knows the initial inventory at the port, but when new cargo is acquired by the port, the observer’s sensors reveal only that more cargo was received; they do not specify the numbers or identities of the received crates.
A legible plan for loading sensitive cargo (the red crate) and acquiring more cargo may first load the crate and then acquire more crates. This plan reveals the identity of the crate that was loaded based on the observers’ information about the remaining cargo in the port: the final belief state has a unique crate loaded on the ship even though it retains uncertainty about the new cargo in the port.
However, if the plan were to first acquire more cargo, the observer’s sensors are insufficient to determine which crate was loaded: the plan maintains ambiguity in the observer’s belief. This is reflected in the observer’s belief state sequence, where the last belief state includes states with all different types of crates in the ship. Although both plans have the same cost and effects for the dock, one conveys the activity being carried out while the other adds more uncertainty. The COPP framework allows the robot to select plans that may be legible in this manner. }

\subsubsection{Goal Legibility Problem}
\index{COPP!Goal Legibility!Problem}%

\unsure{As mentioned earlier, in goal legibility problem, the robot's aim is to take actions exclusive to the goal so as to help the observer in goal deduction.}  
%Here we generalize the notion of goal legibility with respect to $j$ number of goals.}

\begin{definition}
\index{Goal Legibility Planning Problem}%
A \textbf{goal legibility planning problem} is a $\mathcal{P}_{CO}$, where, $\mathcal{G} = \{ G^R \cup G_1 \cup \ldots \cup G_{n-1} \}$ is the set of $n$ goals where $G^R$ is the true goal of the robot, and $G_1, \ldots, G_{n-1}$ are confounding goals. 
\end{definition}

\unsure{Rather than optimizing directly for the legibility score, we will instead limit ourselves to solutions whose legibility score is above a certain threshold. In particular, we can generate a plan that conveys at most $j$ candidate goals inclusive of its true goal. Thus robot has to ensure that the observation sequence of a legible plan is consistent with \emph{at most} $j$ goals so as to limit the number of goals in the observer's final belief state.}

\begin{definition} 
\label{4-def:leg}
\index{J-Legible Plan}%
A plan, $\pi_j$, is a \textbf{j-legible plan}, if $\delta(I, \pi_j) \models G^R$ and the last belief, $b_n \in BS(\pi_j, I)$, satisfies the following, $|{G \in \mathcal{G}: \exists s \in b_n, s \models G}| \leqslant j$, where $1 \leqslant j \leqslant n$.
\end{definition}

%\subsubsection{Goal Diversity}
%\unsure{As mentioned earlier, cardinality is just one way to measure legibility. Another possibility could be to look at a goal diversity measure that establish the extent of diversity among the set of goals present in the observer's belief. For instance, a simple goal diversity measure could be the cardinality of the goals satisfied in the last belief state, which is the measure admitted by Definition \ref{4-def:leg}. However, other goal diversity measures can also be used to define the problems of goal legibility, and then \emph{j-legible} solutions can be written as \emph{at most j goal-diverse}.} 

\subsection{Computing Goal Legible Plans}

In the case of goal legibility, we run the algorithm with the robot's true goal.

\subsubsection{Goal test} 
In order to ensure that the computed plan is consistent with \textit{at most} $j$ true goals, we change our goal condition to additionally check whether at most $j-1$ confounding goals have been achieved in the last belief in $BS$, or it can be interpreted as at least $n-j$ goals are absent in the belief.

\subsubsection{Heuristic function} 
In this case, our objective is to avoid at least $n-j$ goals, that is be consistent with at most $j$ goals. We achieve this by minimizing the heuristic cost to the true goal from the robot's actual state and to the $j-1$ confounding goals from the robot's belief state. However, we maximize the heuristic cost to other $n-j$ goals in order to achieve at most $j$ goals in the last belief state. This is written as,
\begin{align}
h(s) = h_{G_A}(s) + h_{\mathcal{G}_{j-1}}(b) - h_{\mathcal{G}_{n-j}}(b) 
\end{align}

\section{Plan Legibility}
\index{COPP!Plan Legibility}%

\begin{figure}[t!]
\centering 
\includegraphics[width=\columnwidth]{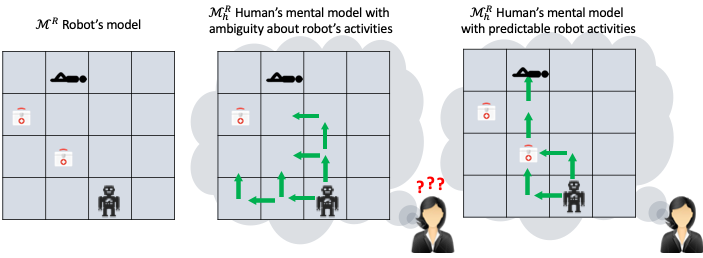}
\caption{Illustration of the impact of plan legibility on the observer's plan inference process.}
\label{4-fig:plan-leg}
\end{figure}

In this problem setting, the observer is aware of the robot's goal, since the candidate goal set consists of a single goal. Although, the observer is unaware of the robot's choice of plan to achieve that goal. In general, there can be multiple different plans that allow the robot to achieve a goal. This is further complicated by the fact that the observer has partial observability of the robot's activities. Therefore, in this problem setting the goal of the robot is to reduce the ambiguity over the possible plans to its goal. The robot can achieve this, by computing a plan whose observation sequence conforms to a set of plans that are similar in terms of their actions, thereby making it easy for the observer to guess the actions executed by the agent. 

For instance, in Figure \ref{4-fig:plan-leg}, the goal of the agent is to pick up a medkit and treat the victim on the disaster site. The observer has a noisy sensor model and it is not accurate enough to give the exact cell location of the robot. Therefore, the observer relies on how many steps the robot has moved to guess the robot's path of execution. In the initial state shown in the leftmost subfigure in Figure \ref{4-fig:plan-leg}, it can be seen that the robot has two choices of medkits to pick from. If it chooses the medkit as shown in the rightmost sub-figure, there are only two paths of length 2 towards the medkit and one path of length 2 towards the victim. In comparison, the middle sub-figure, shows the various different paths of length 4 that can lead to the other medkit, making it harder for the human to guess the robot's actions. Therefore, to reduce the ambiguity over its path, the robot chooses the medkit shown in the rightmost figure. This involves an observation sequence with two similar paths leading to the goal, making it easy for the human to guess some of the actions of the robot. Thus, the plan legibility problem can be solved by finding an observation sequence that is consistent with plans that are similar to each other.

\begin{definition}
\index{Plan Legibility Planning Problem}%
A \textbf{plan legibility planning problem} is a tuple, $\mathcal{P}_{PL} = \langle \mathcal{D}, \mathcal{G}_{PL}, \Omega, \mathcal{O} \rangle$, where, $|\mathcal{G}_{PL}| = 1$. 
\end{definition}

\subsection{Computing Plan Legible Plans}

The solution to a plan legibility problem is an \textit{$m$-similar} plan. An \textit{$m$-similar} plan is a plan whose observation sequence is consistent with \textit{at least} $m$ similar plans to the goal, such that, these plans are \textit{at most} $d$ distance away from each other. In order to compute an \emph{$m$-similar} plan, we need to keep track of the plans that are consistent with the observation sequence and reach the goal. To compute the diversity between all the pairs of plans consistent with the observation sequence, a plan distance measure like action distance, causal link distance, state sequence distance (discussed in Chapter \ref{ch03}) can be used. This approach can use any valid plan distance. We now define an $m$\emph{-similar} plan.  

\begin{definition} 
\index{d-distant plan pair}%
Two plans, $p_1, p_2$, are a \textbf{d-distant pair} with respect to distance function $\mathcal{D}$ if, $\mathcal{D}(p_1, p_2) = d$, where $\mathcal{D}$ is a diversity measure. 
\end{definition}

\unsure{We will be using this distance function as the basis of legibility score, but again we will focus on generating plans whose legibility score is above some prespecified threshold, by requiring that the possible plans are no farther than \emph{d-distance} away.}
%We use the belief plan set (BPS) introduced in Section \ref{4-sec:COPP} to define maximally d-distant pairs of plans associated with an observation sequence.

\begin{definition}
\index{Maximally d-distant plans}%
A BPS induced by plan p starting at $s_0$ is \textbf{maximally d-distant}, 
\[d_{max}(BPS(p, s_0))\textrm{ if }d = \displaystyle \max_{p1, p2 \in BPS(p, s_0)} \mathcal{D}(p1, p2)\].
\end{definition}

\begin{definition} 
\index{m-similar plans}%
A plan, $\pi_m$, is an \textbf{m-similar plan}, if for a given value of d and distance function $\mathcal{D}$, $d_{max}(BPS(\pi_m, I)) \leq d$, $|BPS(\pi_m, I)| \geq m$, where $m \geq 2$ and every plan in $BPS(\pi_m, I)$ achieves the goal in $\mathcal{G}_{PL}$.
\end{definition}

In order to generate a solution to the plan legibility problem, we use the algorithm presented in Algorithm \ref{4-procedure:COPP-algo}. Again, the goal test and heuristic function are customized to ensure that there are at least $m$ similar plans to the true goal that are consistent with the observation sequence and the maximum distance between these plans is \textit{at most} $d$. 

\subsubsection{Goal test} 
To ensure the plans in $BPS$, induced by an $m$\textit{-similar} plan, can achieve the goal in $\mathcal{G}_{PL}$, we check whether at least $m$ plans are reaching the goal or not and whether the maximum distance between plans in $BPS$ is at most $d$. Also in order to ensure termination of the algorithm, there is a cost-bound given as input to the algorithm.

\subsubsection{Heuristic function}
Apart from minimizing the heuristic cost to the goal, the customized heuristic given below also minimizes the $d$ of $d_{max}(BPS(p, s_0))$ induced by plan $p$ starting at $s_0$. This decreases the maximum distance between the plan pairs in the BPS. This distance can be computed using a plan distance measure.
\begin{align}
h(s) = h_{G_A}(s) + d_{max}(BPS(p, s_0)) 
\end{align}

\subsubsection{Plan Legibility as Offline Predictability}
\index{Offline Predictability}%

\unsure{Plan legibility can be likened to the notion of offline predictability insofar that plan legible behaviors allows the robot to reduce the observer's ambiguity over the possible plans executed given a goal, which is also a property exhibited by predictable behaviors.} However, predictable behaviors are also inherently easy to anticipate, either globally or locally. Globally predictable behavior is a behavior that an observer would anticipate the robot to perform for a certain goal, versus locally predictable behavior is a behavior where given a plan prefix, the rest of the suffix towards a certain goal can be easily anticipated by the observer. This notion of predictability has been mostly explored in the motion planning community. However, plan legibility does not always lead to predictable behaviors. This is because in plan legibility, the emphasis is on making the robot's actions easy to guess given a corresponding observation sequence. However, the observation sequence in itself might not be globally or even locally predictable to the observer.

\section{Bibliographic Remarks}
In the motion planning and robotics community, legibility has been a well-studied topic  \citep{dragan2013legibility,Dragan-2013-7732,dragan2015effects,knepper2017implicit}. The original work on legibility used Bayesian formulation to generate legible robot motions in an online setting. A legible motion is one that ensures that the actual robot goal has the highest likelihood amongst the candidate goals.
%, while predictable motion is considered to be the one that has the highest probability given the robot's goal. 
The legible planning discussed here is formulated within the controlled observability planning framework (introduced in \cite{unified-anagha}), which generalizes the notion of legibility in terms of human's observability as well as in terms of the amount of information divulged (with at most $j$ goals, or plans that are at most $d$ distance away). This framework uses the notion of sensor models to capture the human's imperfect observations and to model the human's belief update \citep{geffner2013concise,bonet2014belief,keren2016privacy}. The goal and plan legible behaviors have been categorized into legible planning and predictable planning in a recent survey on interpretable behaviors \citep{landscape}.
\clearpage

                % shaded blocks in your pages
    %May be move it to the sty file

%\newtheorem{defn}{Definition}

\chapter{Explanation as Model Reconciliation}
\label{ch05}
\unsure{In this chapter, we revisit the explicability score and investigate an alternate strategy to improve the explicability of the robot behavior, namely explanations. Rather than force the robot to choose behaviors that are inherently explicable in the human model, here we will let the robot choose a behavior optimal in its model and use communication to address the central reason why the human is confused about the behavior in the first place, i.e., the model difference. That is, the robot will help the human understand why the behavior was performed, by choosing to reveal parts of its model that were previously unknown to the human. 
This would allow us to overcome one of the main shortcomings of the plan generation methods discussed in Chapter \ref{ch02}, namely that there might not exist a plan in the robot model that has a high explicability score. In this scenario, the explicability score of the plan is only limited by the agent's ability to effectively explain it. In this chapter, in addition to introducing the basic framework of explanation as model reconciliation under a certain set of assumptions, we will also look at several types of model reconciliation explanations, study some of their properties and consider some simple approximations. In the coming chapters, we will further extend the idea of explanations and look at ways of relaxing some of the assumptions made in this chapter.}
\section{Model-Reconciliation as Explanation}
\label{ch05:model-rec-sec}
\index{Model-Reconciliation@Explanation!Model-Reconciliation}%
As mentioned in chapter \ref{ch02}, the explicability score of a plan $\pi$ generated using the robot's model $\mathcal{M}^R$ is given by the expression
\[
E(\pi, \mathcal{M}^R_h) \propto \textrm{max}_{\pi'  \in \Pi^{\mathcal{M}^R_h}} -1* \mathcal{D}(\pi,\pi', \mathcal{M}^R_h)
\]
\unsure{Where $\mathcal{D}$ is the distance function, $\mathcal{M}^R_h$ the human's estimate of what the robot's model may be, and $\pi'$ the plan closest to $\pi$ in $\mathcal{M}^R_h$ (per the distance function $\mathcal{D}$).} 
%A clear case, where even when the robot chooses the best possible plan, may result in a inexplicable plan for the human is when the user's estimate of the model differs from the robot's own models. 
This means, even if the robot chooses a plan in its own model ($\mathcal{M}^R$), a human observer may find it inexplicable if their estimate of the robot model differs from $\mathcal{M}^R$. 
Thus a way to address the inexplicability of the plan would be to explain the difference between the robot's own model and human's estimation. This explanatory process is referred to as \textbf{{\em model reconciliation}}.% in the literature. 
 We will specifically focus on cases, where $\mathcal{D}$ is given by the cost difference and the expected set corresponds to plans optimal in the human's model. 
%This means costlier plans would have smaller explicability scores. 
\unsure{While this method assumes the human is a perfectly rational decision-maker (which isn't necessarily true), studies have shown this method to still result in useful explanations in many scenarios.
Of course we could replace the distance function used and expected plan set with more realistic choices and the basic algorithms presented here should stay mostly same.}
%This does not effect the algorithms described in this paper, 
%Since the computation of similarity is only invoked during the evaluation process of a particular node and the stopping criterion of the search, rather than the search process itself.
%Moreover we will also see how we can use learned model proxies the issue of having access to explicit expected set and
%and we will also look at how we could potentially relax this.

\noindent The explanation process captured by Model Reconciliation begins with the following question:

\vspace{5pt}
\noindent {\em $Q_1$: Why plan $\pi$?}

\vspace{5pt}
\unsure{\noindent An explanation here needs to ensure that both the explainer and the explainee agree that $\pi$ is the best decision that could have been made in the given problem. 
In the setting we are considering here that would mean providing
model artifacts to the explainee so that $\pi$ is now also optimal in the updated mental model and thus have high explicability score under our current set of assumptions (we will refer to this as the completeness property later).}

%\todo{For simplicity of notation, in this chapter we will use cost functions that takes model as an argument, i.e., we will use $C(\pi, \mathcal{M}_1)$ to denote the cost of plan $\pi$ in the model $\mathcal{M}_1$, while $C(\pi, \mathcal{M}_2)$ denotes its cost in $\mathcal{M}_1$. Additionally, we will use the notation $C_{\mathcal{M}}^*$ to denote the cost of an optimal plan in the model $\mathcal{M}$.}

\begin{definition}
\index{Model Reconciliation Problem@MRP}%
The {\bf model reconciliation problem (MRP)} 
is represented as a tuple $\langle \pi^*, \langle \mathcal{M}^R, \mathcal{M}^R_h\rangle \rangle$ where $C(\pi^*, \mathcal{M}^R) = C_{\mathcal{M}^R}^*$, i.e. the plan $\pi^*$ is the optimal plan in the robot model $\mathcal{M}^R$ but may not be so in the human mental model $\mathcal{M}^R_h$.
\end{definition}
%\vspace{10pt}
%\noindent{\em Definition} 

Now we can define an explanation as
%\vspace{5pt}
\begin{definition}
\label{expl_defn}
\index{Model-Reconciliation!Explanation}%
An {\bf explanation} is a solution to the model reconciliation problem in the form of (1) a model update $\mathcal{E}$ such that the (2) robot optimal plan is (3) also optimal in the updated mental model.

\begin{itemize}
\item $\mathcal{\widehat{M}}^R_h \longleftarrow \mathcal{M}^R_h + \mathcal{E}$; and
\item $C(\pi, \mathcal{M}^R) = C^*_{\mathcal{M}^R}$;
\item $C(\pi, \mathcal{\widehat{M}}^R_h) = C^*_{\mathcal{\widehat{M}}^R_h}$. 
\end{itemize}
\end{definition}

Now the problem of finding the explanation becomes a search over the set of model information that can be provided to the user to get an updated user-model of desired property. This in term can be visualized as the search over the space of possible human models that can result from providing information consistent with the robot model. We can facilitate such a search over the model space, by leveraging the model parameterization scheme specified in Chapter \ref{ch02}. Specifically such a scheme results in a state space of the form
%\vspace{-15pt}
\begin{align*}
\mathcal{F} =~& \{\textit{init-has-f}~|~\forall f \in F^R_h \cup F^R\} \cup\{\textit{goal-has-f}~|~ \forall f \in F^R_h \cup F^R\}\\[1ex]
& \bigcup_{a \in A^R_h \cup A^R}\{\textit{a-has-precondition-f}, \textit{a-has-add-effect-f}, \\[-2ex]
& ~~~~~~~~~~~~~~~~~~~~~~~\textit{a-has-del-effect-f}~|~ \forall f \in F^R_h \cup F^R\}\\[1ex]
& \cup \{\textit{a-has-cost-}C(a)~|~a \in A^R_h\} \cup \{\textit{a-has-cost-}C(a)~|~a \in A^R\}.
\end{align*}

Again the mapping function $\Gamma : \mathcal{M} \mapsto s$ can convert the given planning problem $\mathcal{M} = \langle  F, A, I, G, C\rangle$ as a state $s \subseteq \mathcal{F}$, as follows --

\vspace{-10pt}
\begin{align*}
    \tau(f) &= 
\begin{cases}
    \textit{init-has-f} & \text{ if } f \in I,\\
    \textit{goal-has-f} & \text{ if } f \in G,\\
    \textit{a-has-precondition-f} & \text{ if } f \in \textit{pre}(a),~a \in A\\
    \textit{a-has-add-effect-f} & \text{ if } f \in \textit{adds}(a),~a \in A\\
    \textit{a-has-del-effect-f} & \text{ if } f \in \textit{dels}^-(a),~a \in A\\
    \textit{a-has-cost-f} & \text{ if } f = C(a),~a \in A\\
\end{cases} \nonumber \\[1ex]
    \Gamma(\mathcal{M}) =~& \big\{\tau(f)~|~\forall f \in I \cup G \cup \nonumber \\ 
    & \bigcup_{a\in A}\{f'~|~\forall f' \in \{C(a)\} \cup \textrm{pre}(a) \cup \textrm{adds}(a) \cup \textrm{dels}(a)\}\big\}
\end{align*}

\begin{definition}
\index{Model-Reconciliation!Model-Space Search}%
The \textbf{model-space search problem} is specified as $\langle \mathcal{F}, \Lambda , \Gamma(\mathcal{M}_1), \Gamma(\mathcal{M}_2)\rangle$ 
with a new action set $\Lambda$ containing unit model change actions $\lambda\in\Lambda,\lambda : \mathcal{F} \rightarrow \mathcal{F}$ such that $| (s_1 \setminus s_2) \cup (s_2 \setminus s_1) | = 1$. The new transition or edit function is given by $\delta_{\mathcal{M}_1,\mathcal{M}_2}(s_1, \lambda) = s_2$ such that,\texttt{condition 1}: $s_2 \setminus s_1 \subseteq \Gamma(\mathcal{M}_2)$ and \texttt{condition 2}:  $s_1 \setminus s_2 \not\subseteq \Gamma(\mathcal{M}_2)$ are satisfied.
\end{definition}

This means that model change actions can only make a single change to a domain at a time, and {\em all these changes are consistent
with the model of the planner.}
The solution to a model-space search problem is given by a {\em set} of edit functions $\{\lambda_i\}$ that
transforms the model $\mathcal{M}_1$ to $\mathcal{M}_2$, i.e. $\delta_{\mathcal{M}_1,\mathcal{M}_2}(\Gamma(\mathcal{M}_1), \{\lambda_i\}) = \Gamma(\mathcal{M}_2)$.
An explanation can thus be cast as 
a solution to the model-space search problem $\langle  \mathcal{F}, \Lambda , \Gamma(\mathcal{M}^R_h), \Gamma(\widehat{\mathcal{M}})\rangle$
with the transition function $\delta_{\mathcal{M}^R_h, \mathcal{M}^R}$
such that Condition (3) mentioned in Definition \ref{expl_defn} is preserved.

\subsection{The Fetch Domain}
\label{subsec:fetch}
Let us look at an example domain to see how model reconciliation explanation could be used.
Consider the Fetch robot \footnote{\url{https://fetchrobotics.com/fetch-mobile-manipulator/}} whose design requires it to \texttt{tuck} 
its arms and lower its torso or \texttt{crouch} before moving.
This is not obvious to a human navigating it
and it may lead to an unbalanced base and toppling of the robot
if the human deems such actions as unnecessary. 
The move action for the robot is described in PDDL in the
following model snippet --
 
\begin{verbatim}
(:action move
:parameters     (?from ?to - location)
:precondition   (and (robot-at ?from) (hand-tucked) (crouched))
:effect         (and (robot-at ?to) (not (robot-at ?from))))
\end{verbatim}

\vspace{-5pt}
\begin{verbatim}
(:action tuck
:parameters     ()
:precondition   ()
:effect         (and (hand-tucked) (crouched)))
\end{verbatim}

\vspace{-5pt}
\begin{verbatim}
(:action crouch
:parameters     ()
:precondition   ()
:effect         (and (crouched)))
\end{verbatim}
 
\begin{figure*}
\centering
\includegraphics[width=\textwidth]{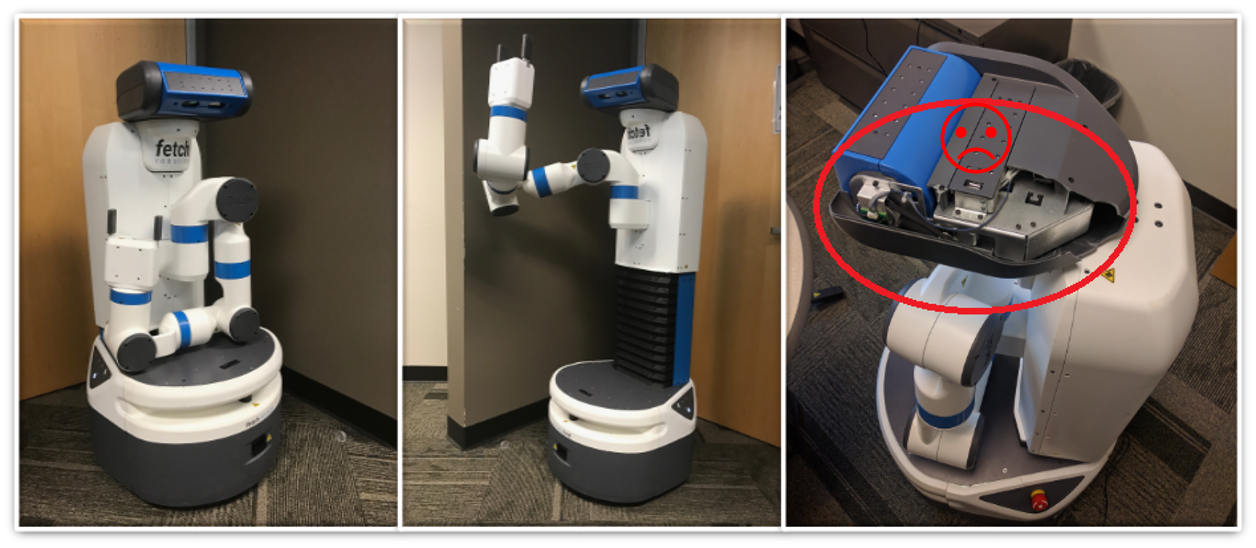}
\caption{The Fetch in the crouched position with arm tucked (left), torso raised and arm outstretched (middle) and the rather tragic consequences of a mistaken action model (right showing a fractured head from an accident).}
\label{fig:fetch}
\end{figure*}

Notice that the \texttt{tuck} action
also involves a lowering of torso so that the arm can rest on the base
once it is tucked in.\footnote{Fetch User Manual: \url{https://docs.fetchrobotics.com/}}
Now, consider a planning problem where the the robot needs to transport
a block from one location to another, with the following initial and
goal states -- 
% (here, considered identical for 
% both the robot and the human) --
 
\begin{verbatim}
(:init (block-at b1 loc1) (robot-at loc1) (hand-empty))  	          
(:goal (and (block-at b1 loc2)))
\end{verbatim}
 
An optimal plan for the robot involves a \texttt{tuck} action followed by a \texttt{move}:
 
\begin{verbatim}
pick-up b1 -> tuck -> move loc1 loc2 -> put-down b1
\end{verbatim}
 
The human, on the other hand, expects a much simpler model, 
as shown below.
%\footnote{This is actually a common problem with deploying any software to end users: generic user models
%are used to model the average user and these lack details and nuances
%of the system at hand that only experts would be aware of.} 
The \texttt{move} action does not have the preconditions for tucking the arm and lowering the torso, while \texttt{tuck} does not automatically lower the torso either.
 
\begin{verbatim}
(:action move
:parameters     (?from ?to - location)
:precondition   (and (robot-at ?from)
:effect         (and (robot-at ?to) (not (robot-at ?from))))
\end{verbatim}

\vspace{-5pt}
\begin{verbatim}
(:action tuck
:parameters     ()
:precondition   ()
:effect         (and (hand-tucked))
\end{verbatim}

\vspace{-5pt}
\begin{verbatim}
(:action crouch
:parameters     ()
:precondition   ()
:effect         (and (crouched)))
\end{verbatim}
 
The original plan is no longer optimal to the human
who can envisage better alternatives (a shorter plan without the 
extra {\tt tuck} action) in their mental model.
An explanation here is a model update that can mitigate this
disagreement --
 
\begin{verbatim}
Explanation >> MOVE_LOC1_LOC2-has-precondition-HAND-TUCKED
\end{verbatim}
 
This correction brings the mental model closer to the 
robot's ground truth and is necessary and sufficient to make the robot's plan optimal in the resultant domain so that the human cannot
envisage any better alternatives. 
This is the essence of the model reconciliation process.

\section{Explanation Generation}
Before we go on to specific methods we could use to identify the required explanation for a model reconciliation problem, let us consider the various desirable properties that characterize explanations in such settings.
 
\begin{itemize}
\item[P1.] \textbf{Completeness -} 
\index{Model-Reconciliation!Explanation!Completeness}%
Explanations of a plan should be able to be compared and contrasted against other alternatives, so that no better solution exists. We can enforce this property by requiring that in the updated human mental model the plan being explained is now optimal.
\begin{itemize}
\item {\em An explanation is complete iff $C(\pi, \mathcal{\widehat{M}}^R_h) = C^*_{\mathcal{\widehat{M}}^R_h}$.}
\end{itemize}
\item[P2.] \textbf{Conciseness -} 
\unsure{Explanation should be concise so that it is easily understandable to the explainee. 
Larger an explanation is, the harder it is for the human to process that information.
Thus we can use the length of the explanation as a useful proxy or first 
approximation for the complexity of an explanation.}
\item[P3.] \textbf{Monotonicity -} 
\index{Model-Reconciliation!Explanation!Monotonicity}%
This property ensures that remaining model differences cannot change the completeness of an explanation, i.e. all aspects of the model that were relevant to the plan have been reconciled. 
Thus, monotonicity of an explanation subsumes completeness and requires more detail.
\begin{itemize}
\item {\em An explanation is monotonic iff\\$C(\pi^*, \hat{\mathcal{M}}) = C^{*}_{\hat{\mathcal{M}}}$ $\forall\hat{\mathcal{M}} : ((\Gamma(\hat{\mathcal{M}}) \setminus \Gamma(\mathcal{M}^R_h)) \cup
(\Gamma(\mathcal{M}^R_h) \setminus
\Gamma(\hat{\mathcal{M}})
)
) 
\subset 
(\Gamma(\hat{\mathcal{M}}) \setminus \Gamma(\mathcal{M}^R_h)) \cup
(\Gamma(\mathcal{M}^R_h) \setminus \Gamma(\hat{\mathcal{M}}))
$.}
\end{itemize}
\unsure{Thus an updated model would satisfy monontonicity property if no additional information about the robot model would render the plan being explained suboptimal or invalid.}
%This is a very useful property to have. Doctors, for example, reveal different amount of detail to their patients as opposed to their peers. Further, the idea of completeness, i.e. withholding information on other model changes as long as they explain the observed plan, is also quite prevalent in how we deal with similar scenarios ourselves - e.g. progressing from Newtonian physics in high school to Einsteins Laws of Relativity in college.
\item[P4.] \textbf{Computability -} While conciseness deals with how easy it is for the explainee to understand an explanation, computability measures the ease of computing the explanation from the point of view of the planner.
\end{itemize}
 
% \begin{table}[tbp!]
% \small
% \centering
% \begin{tabular}{ r | c | c | c | c }
% \toprule
% {\bf Explanation Type} & {\bf R1} & {\bf R2}  & {\bf R3}  & {\bf R4} \\ \midrule
% Plan Patch Explanation  / VAL                 & \xmark & \cmark & \xmark & \cmark \\\hline
% Model Patch Explanation                       & \cmark & \xmark & \cmark & \cmark \\\hline
% Minimally Complete Explanation                & \cmark & \cmark & \xmark & \qmark \\\hline
% Minimally Monotonic Explanation               & \cmark & \cmark & \cmark & \qmark \\\hline
% (Approximate) Minimally Complete Explanation  & \xmark & \cmark & \xmark & \cmark \\
% \bottomrule
% \end{tabular}
% \caption{Requirements for different types of explanations.}
% \label{table-vs}
% \end{table}
We will now introduce different kinds of multi-model explanations that can participate in the model reconciliation process, provide examples, propose algorithms to compute them, and compare and contrast their respective properties. 
We note that the requirements outlined above are in fact often at odds with each other - an explanation that is very easy to compute may be very hard to comprehend.

\subsection{Explanation types}
\label{ch05:explanation_types}
\index{Model-Reconciliation!Explanation Types}%

A simple way to explain would be to provide the model differences pertaining to only the actions that are present in the plan being explained -- 

\begin{definition}
\index{Model-Reconciliation!PPE}%

 A {\bf plan patch explanation (PPE)} is given by --

\vspace{-10pt}
\begin{equation*}
\mathcal{E}^{PPE} = (\mathcal{F}^{\mathcal{M}^R}(\pi) \setminus \mathcal{F}^{\mathcal{M}^R_h}(\pi)) \cup  (\mathcal{F}^{\mathcal{M}^R_h}(\pi) \setminus \mathcal{F}^{\mathcal{M}^R}(\pi))
%\Delta_{\{\mathcal{M}^R, \mathcal{M}^R_h\}}\bigcup_{ \textrm{pre}(a) \cup \textrm{adds}(a) \cup \textrm{dels}(a) : a \in \pi} \tau(f)
\end{equation*}
\todo{where $\mathcal{F}^{\mathcal{M}}(\pi)$ gives the model parameters of $\mathcal{M}$ corresponding to all the actions in the plan $\pi$ (i.e., $\mathcal{F}^{\mathcal{M}}(\pi) = \bigcup_{ \{C(a)\} \cup \textrm{pre}(a) \cup \textrm{adds}(a) \cup \textrm{dels}(a) : a \in \pi} \tau(f)$).}
\end{definition}

Clearly, such an explanation is easy to compute and concise by focusing only on plan being explained. However, it may also contain information that need not have been revealed, while at the same time ignoring model differences elsewhere in $\mathcal{M}^R_h$ that could have contributed to the plan being suboptimal in it. Thus, it is not {\em complete}. 
%An adoption of VAL \cite{Fox,strathprints2550} to the multi-model setting will be, in fact, a subset of such PPEs, and suffer from the same limitations. 
On the other hand, an easy way to compute a complete explanation would be to provide the entire model difference to the human --

\begin{definition}
\index{Model-Reconciliation!MPE}%

 A {\bf model patch explanation (MPE)} is given by --

\begin{equation*}
\mathcal{E}^{MPE} =  (\Gamma(\mathcal{M}^R)\setminus\Gamma(\mathcal{M}^R_h)) \cup ( \Gamma(\mathcal{M}^R_h) \setminus \Gamma(\mathcal{M}^R))
\end{equation*}
\end{definition}

\unsure{This is also easy to compute but can be quite large and is hence far from being concise. 
Instead, we can try to minimize the size (and hence increase the comprehensibility) of explanations by searching in the space of models and thereby not exposing information that is not relevant to the plan being explained while still trying to satisfy as many requirements as we can. }

\begin{definition}
\index{Model-Reconciliation!MCE@Minimally Complete Explanation}%

A {\bf minimally complete explanation (MCE)} is the shortest possible explanation that is complete --

\begin{equation*}
\mathcal{E}^{MCE} = \argmin_{\mathcal{E}}| (\Gamma(\widehat{\mathcal{M}}) \setminus \Gamma(\mathcal{M}^R_h)) \cup (\Gamma(\mathcal{M}^R_h) \setminus \Gamma(\widehat{\mathcal{M}})) | \text{ with } C(\pi, \mathcal{\widehat{M}}^R_h) = C^*_{\mathcal{\widehat{M}}^R_h}
\end{equation*}
\end{definition}

The explanation provided before in the Fetch domain is indeed the smallest set of domain changes that may be made to make the given plan optimal in the updated action model, and is thus an example of a minimally complete explanation.

The optimality criterion happens to be relevant to both the cases where the human expectation is better, or worse, than the plan computed by the planner. This might be counter to intuition, since in the latter case one might expect that just establishing feasibility of a better plan would be enough. Unfortunately, this is not the case, as can be easily seen by creating counter-examples where other faulty parts of the human model might disprove the optimality of the plan. Which brings us to the proposition,
 
\vspace{10pt}
\begin{proposition}
If $C(\pi^*, \mathcal{M}^R_h) < \min_{\pi} C(\pi, \mathcal{M}^R_h)$, then ensuring feasibility of the plan in the modified planning problem, i.e. $\delta_{\mathcal{\widehat{M}}}(\widehat{I}, \pi^*) \models \widehat{G}$, is a necessary but \emph{not} a sufficient condition for $\widehat{\mathcal{M}} = \langle \widehat{F}, \widehat{A}, \widehat{I}, \widehat{G} \rangle$ to yield a valid explanation.
\end{proposition}

Note that a minimally complete explanation can be rendered invalid given further updates to the model. This can be easily demonstrated in our running example in the Fetch domain. Imagine that if, at some point, the human were to find out that the action \texttt{(move)} also has a precondition \texttt{(crouched)}, then the previous robot plan will no longer make sense to the human since now according to the human's faulty model (being unaware that the tucking action also lowers the robot's torso) the robot would need to do \emph{both} \texttt{tuck} and \texttt{crouch} actions before moving. Consider the following explanation in the Fetch domain instead --
 
\begin{verbatim}
Explanation >> TUCK-has-add-effect-CROUCHED
Explanation >> MOVE_LOC2_LOC1-has-precondition-CROUCHED
\end{verbatim}
 
This explanation does not reveal all model differences but at the same time ensures that the
plan remains optimal for this problem, irrespective of any other changes to the model, by accounting for all the relevant parts of the model that engendered the plan. It is also the smallest possible among all such explanations. 
% The requirement of monotonicity and minimality brings us to the notion of --

\begin{definition}
\index{Model-Reconciliation!MME@Minimally Monotonic Explanation}%
A {\bf minimally monotonic explanation (MME)} is the shortest  explanation that preserves both completeness and monotonicity --

\begin{equation*}
\mathcal{E}^{MME} = \argmin_{\mathcal{E}}|(\Gamma(\widehat{\mathcal{M}}) \setminus \Gamma(\mathcal{M}^R_h)) \cup (\Gamma(\mathcal{M}^R_h) \setminus \Gamma(\widehat{\mathcal{M}}))| \text{ with P1 \& P3 }
\end{equation*}
\end{definition}

An MCE or MME solution may not be unique to an MRP problem. This can happen when there are multiple model differences supporting the same causal links in the plan - a minimal explanation can get by (i.e. guarantee optimality in the modified model) by only exposing one of them to the human.
%Interestingly, we showed in \cite{hri2} how theoretically equivalent explanations
%are, in fact, sometimes interpreted differently by the explainee.

% \begin{prop}
% MCEs and MMEs are not unique, i.e. there might be multiple minimally complete and monotonic solutions to a given MRP.
% \end{prop}

Also it is easy to see that an MCE may not necessarily be part of an actual MME. This is illustrated in Figure~\ref{fetch-explanations}.

% Thus, we emphasize - 
 
\begin{proposition}
An MCE may not be a subset of an MME, but it is always smaller or equal in size, i.e. $|MCE|\leq|MME|$.
\end{proposition}

\begin{figure}
\includegraphics[width=\textwidth]{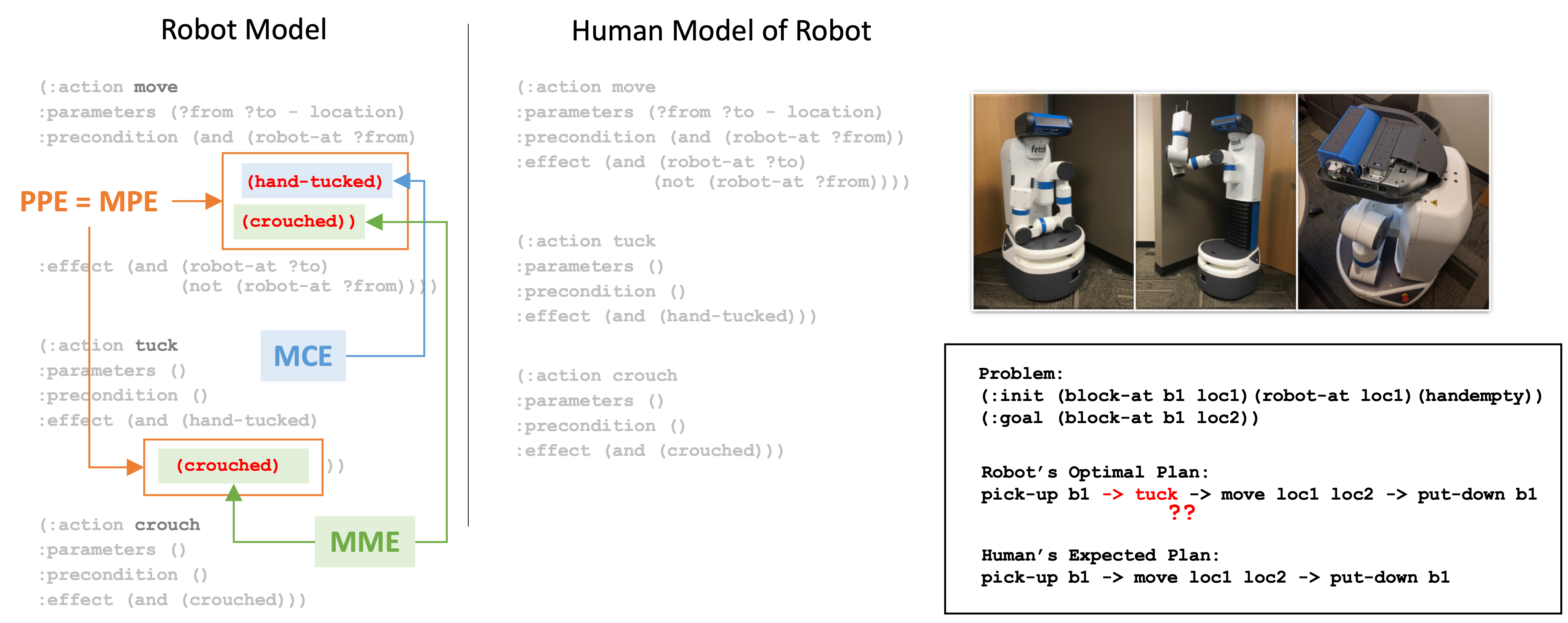}
\caption{Illustration of the different kinds of explanations in the Fetch domain. 
Here the PPE and MPE are equivalent (which is the worst case for the former) 
and both longer than the MCE or the MME. 
Also, the MCE is shorter than, and not a subset of the MME.}
\label{fetch-explanations}
\end{figure}

%\subsection{Properties of Explanation Discussed From Social Sciences}
\subsection{Desiderata for Explanations as  Discussed in Social Sciences}
\index{Explanation!Contrastive Explanations}%
\index{Explanation!Social Explanations}%
\index{Explanation!Selective Explanations}%
\unsure{Before, we delve into the actual algorithms for generating such explanations, lets take a quick look at how model reconciliation explanation connects to the wider literature on explanations. One of the key takeaways from many of the works dealing with this topic from social sciences has been the identification of three crucial properties for effective explanations, namely, explanations need to be {\em contrastive, social and selective}.}

\index{Contrastive Explanations}%
\index{Contrastive Explanations!Fact}%
\index{Contrastive Explanations!Foil}%
\index{Contrastive Explanations!Contrastive Query}%
\unsure{Contrastive explanations are explanations that take the form of answers to questions of the type ``Why P and not Q?'' Where P is the fact being explained and Q the foil or the alternate option raised by the explainee (i.e. the one who is asking for the explanation). Thus explanations for such question need to contrast the choice P over the raised foil. In the case of planning problems, a very common form of contrastive questions involve cases where P , or {\em the fact}, is the plan itself being proposed by the system and {\em the foil} is some alternate plan expected by the user. In model reconciliation, the explanations being provided as part of model reconciliation can be viewed as answering a contrastive query where the foil is implicitly specified, i.e., the foil could be any plan in the human model. Given the fact that these explanations establish optimality, the current plan must be better than or as good as any alternate plan the user could have come up with.}

\unsure{Social explanations are those that explicitly take into account the beliefs of the user. Model-reconciliation framework discussed here is clearly social in the sense that explanations are tailored to what the human believes to be the robot model. Selectivity of explanations deals with the fact that explanations need to only include relevant details (and in some sense minimal). As we will see below, the actual approaches to explanation generation focus on generating the minimal information needed to meet the required properties.}

\subsection{Model Space Search for Minimal Explanations}
\index{Model-Reconciliation!Explanation Generation}%

In the following, we will see how the state space provided by $\Gamma$ can be used in model-space search for computing MCEs and MMEs (computation of PPE and MPE follows directly from $\mathcal{M}^R$, $\mathcal{M}^R_h$ and $\pi^*$).

\subsubsection{Model Space Search for MCEs}
\index{Model-Reconciliation!Explanation Generation!MEGA}%

To compute MCEs, we employ A$^*$ search in the space of models, as shown in Algorithm \ref{algo1}. 
The algorithm is referred to as MEGA~-- {\bf M}ulti-model {\bf E}xplanation {\bf G}eneration {\bf A}lgorithm.
Given an MRP, we start off with the initial state $\Gamma(\mathcal{M}^R_h)$ derived from the human's expectation of a given planning problem $\mathcal{M}^R$, and modify it incrementally until we arrive at a planning problem $\widehat{\mathcal{M}}$ with $C(\pi^*, \widehat{\mathcal{M}}) = C^{*}_{\widehat{\mathcal{M}}}$, i.e. the given plan is explained. 
Note that the model changes are represented as a set, i.e. there is no sequentiality in the search problem. Also, we assign equal importance to all model corrections. 
We can easily capture differential importance of model updates by attaching costs to the edit actions $\lambda$ - the algorithm remains unchanged.
One could also employ a selection strategy for successor nodes to speed up search (by overloading the way the priority queue is popped) by first processing model changes that are relevant to actions in $\pi^*_R$ and $\pi_H$ before the rest. 

\begin{proposition}
\unsure{The successor selection strategy outlined in Algorithm \ref{algo1_selection} will never remove a valid solution from the search space.}
\end{proposition}
\begin{proof}[Proof Sketch]
Let $\mathcal{E}$ be the MCE for an MRP problem and let $\mathcal{E}'$ be any intermediate explanation found by our search such that $\mathcal{E}' \subset \mathcal{E}$, then the set $\mathcal{E} \setminus \mathcal{E}'$ must contain at least one $\lambda$ related to actions in the set $ \{a~|~a \in \pi^*_{R} \vee a \in \pi'\}$ (where $\pi'$ is the optimal plan for the model $\hat{\mathcal{M}}$ where $\delta_{\mathcal{M}^R_h,\mathcal{M}^R}(\Gamma(\mathcal{M}^R_h),\mathcal{E}') = \Gamma(\hat{\mathcal{M}}$). 
To see why this is true, consider an $\mathcal{E}'$ where $|\mathcal{E}'| = |\mathcal{E}| - 1$. If the action in  $\mathcal{E} \setminus \mathcal{E}'$ does not belong to either $\pi^*_R$ or $\pi'$ then it cannot improve the cost of $\pi^*_R$ in comparison to $\pi'$ and hence $\mathcal{E}$ cannot be the MCE. 
Similarly we can show that this relation will hold for any size of $\mathcal{E}'$. We can leverage this knowledge about  $\mathcal{E} \setminus \mathcal{E}'$ to create an admissible heuristic that considers only relevant changes.
% at any given point of time.% (by giving very large values to all other changes).
\end{proof}

\subsubsection{Model Space Search for MMEs}

As per definition, beyond the model obtained from the minimally monotonic explanation, 
there do not exist any models which are not explanations of the same MRP, 
while at the same time making as few changes to the original problem as possible. 
It follows that this is the largest set of changes that can be done on $\mathcal{M}^R$ and still find a model $\widehat{\mathcal{M}}$ where $C(\pi^*, \widehat{\mathcal{M}}) = C^{*}_{\widehat{\mathcal{M}}}$ - this property can be used in the search for MMEs.

\begin{proposition}
$\mathcal{E}^{MME} = \argmax_{\mathcal{E}}|(\Gamma(\widehat{\mathcal{M}}) \setminus \Gamma(\mathcal{M}^R)) \cup
(\Gamma(\mathcal{M}^R) \setminus \Gamma(\widehat{\mathcal{M}}))
|$ such that 
$\forall\hat{\mathcal{M}}~((\Gamma(\hat{\mathcal{M}}) \setminus \Gamma(\mathcal{M}^R)) \cup (\Gamma(\mathcal{M}^R) \setminus \Gamma(\hat{\mathcal{M}})))
\subseteq ((\Gamma(\widehat{\mathcal{M}}) \setminus \Gamma(\mathcal{M}^R))
\cup
(\Gamma(\mathcal{M}^R) \setminus \Gamma(\widehat{\mathcal{M}}))
)$ it is guaranteed to have $C(\pi^*, \hat{\mathcal{M}}) = C^{*}_{\hat{\mathcal{M}}}$.
\end{proposition}

This is similar to the model-space search for MCEs, 
but this time starting from the robot's model $\mathcal{M}^R$ instead. 
The goal here is to find the largest set of model changes for which the explicability criterion becomes invalid for the first time (due to either suboptimality or inexecutability). This 
requires a search over the entire model space (Algorithm \ref{algo2}). 
We can leverage Proposition 3 to reduce our search space. 
Starting from $\mathcal{M}^R$, given a set of model changes $\mathcal{E}$ where $\delta_{\mathcal{M}_R,\mathcal{M}_H}(\Gamma(\mathcal{M}^R), \mathcal{E}) = \Gamma(\widehat{\mathcal{M}})$ and $C(\pi^*,\widehat{\mathcal{M}}) > C^{*}_{\widehat{\mathcal{M}}}$, no superset of $\mathcal{E}$ can lead to an MME solution. 
In Algorithm \ref{algo2}, we keep track of such unhelpful model changes in the list h\_list. 
The variable $\mathcal{E}^{MME}$ keeps track of the current best list of model changes. Whenever the search finds a new set of model changes where $\pi^*$ is optimal and is larger than $\mathcal{E}^{MME}$, $\mathcal{E}^{MME}$ is updated with $\mathcal{E}$. The resulting MME is all the  possible model changes that did not appear in $\mathcal{E}^{MME}$.

Figure~\ref{ch05:picpic} contrasts MCE search with MME search. MCE search starts from $\mathcal{M}^R_h$, updates $\widehat{\mathcal{M}}$ towards $\mathcal{M}^{R}$ and returns the first node (indicated in orange) where $C(\pi^*, \widehat{\mathcal{M}}) = C^{*}_{\widehat{\mathcal{M}}}$.
MME search starts from $\mathcal{M}^{R}$ and moves towards $\mathcal{M}^R_h$. 
It finds the longest path (indicated in blue) where $C(\pi^*, \widehat{\mathcal{M}}) = C^{*}_{\widehat{\mathcal{M}}}$ for all $\widehat{\mathcal{M}}$ in the path. The MME (green) is the rest of the path towards $\mathcal{M}^R_h$.

\begin{figure}
\centering
\includegraphics[width=\textwidth]{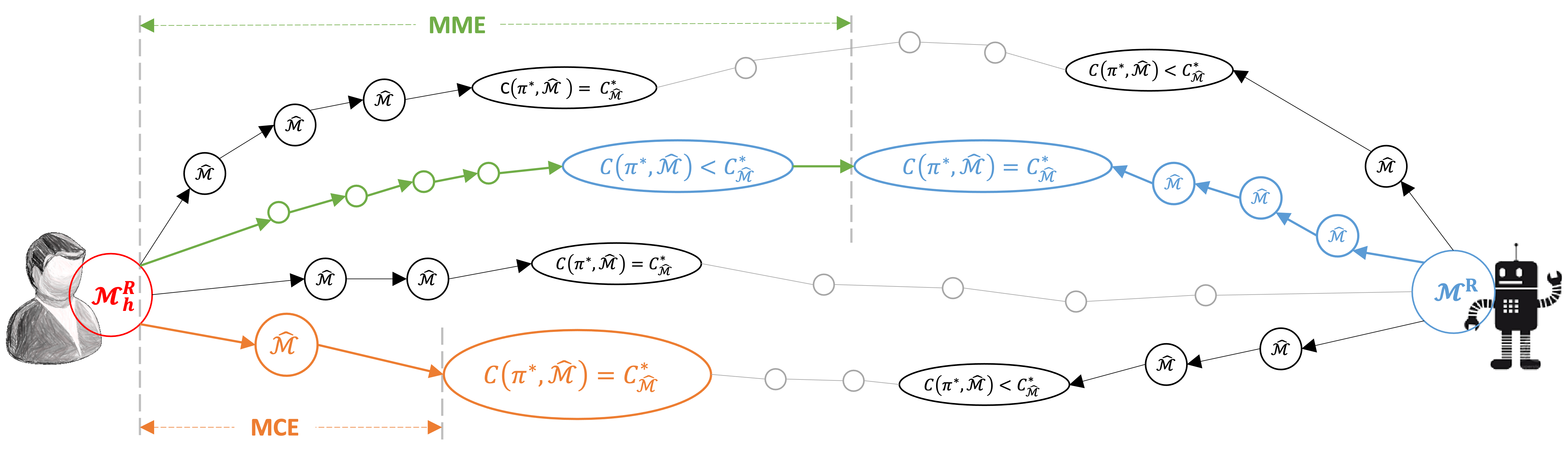}
\caption{Illustration contrasting MCE search with MME search. }
\label{ch05:picpic}
\end{figure}

% \todo{Add algorithms}
\begin{algorithm}%[]
\caption{Search for Minimally Complete Explanations}
\label{algo1}
\begin{algorithmic}[1]
\STATE{MCE-Search}{}
\STATE \emph{Input}: MRP $\langle \pi^*, \langle \mathcal{M}^R, \mathcal{M}^R_h\rangle\rangle$
\STATE \emph{Output}: Explanation $\mathcal{E}^{MCE}$
\vspace{2pt} 
\STATE \emph{Procedure}:  
\vspace{2pt} 
\STATE fringe $\leftarrow$ \texttt{Priority\_Queue()}
\STATE c\_list $\leftarrow$ \{\}  \COMMENT{\textcolor{blue}{Closed list}}
\STATE $\pi^*_R \leftarrow \pi^*$ \COMMENT{\textcolor{blue}{Optimal plan being explained}}
\STATE $\pi_H \leftarrow \pi$ such that $C(\pi, \mathcal{M}^R_h) = C^{*}_{\mathcal{M}^R_h}$ \COMMENT{\textcolor{blue}{Plan expected by human}}
\STATE $\text{fringe.push}(\langle \mathcal{M}^R_h, \{\}\rangle,~\text{priority} = 0)$
\WHILE{True}
\STATE $\langle \widehat{\mathcal{M}}, \mathcal{E} \rangle, c \leftarrow \text{fringe.pop}(\widehat{\mathcal{M}})$
\IF{$C(\pi^*_R, \widehat{\mathcal{M}}) = C^{*}_{\widehat{\mathcal{M}}}$} 
\STATE return $\mathcal{E}$   \COMMENT{\textcolor{blue}{Returns $\mathcal{E}$ if $\pi^*_R$ optimal in   $\widehat{\mathcal{M}}$}}
\ELSE
\STATE c\_list $\leftarrow$ c\_list $\cup~\widehat{\mathcal{M}}$
% \vspace{2pt} 
\FOR {$f \in \Gamma(\widehat{\mathcal{M}})~\setminus~\Gamma(\mathcal{M}^R)$  \COMMENT{\textcolor{blue}{Models that satisfy Condition 1}}}
\STATE $\lambda \leftarrow \langle 1, \{\widehat{\mathcal{M}}\}, \{\}, \{f\} \rangle$ \COMMENT{\textcolor{blue}{Removes f from $\widehat{\mathcal{M}}$}}
%\vspace{2pt} 
\IF{$\delta_{\mathcal{M}^R_h,\mathcal{M}^R}(\Gamma(\widehat{\mathcal{M}}), \lambda) \not\in \text{c\_list}$}
\STATE $\text{fringe.push}(\langle \delta_{\mathcal{M}^R_h,\mathcal{M}^R}(\Gamma(\widehat{\mathcal{M}}), \lambda),~\mathcal{E}~\cup~\lambda \rangle,~c + 1)$
\ENDIF
\ENDFOR

\FOR{$f \in \Gamma(\mathcal{M}^R)~\setminus~\Gamma(\widehat{\mathcal{M}})$ \COMMENT{\textcolor{blue}{Models that satisfy Condition 2}}}
\STATE $\lambda \leftarrow \langle 1, \{\widehat{\mathcal{M}}\}, \{f\}, \{\} \rangle$ \COMMENT{\textcolor{blue}{Adds f to $\widehat{\mathcal{M}}$}}
\IF{$\delta_{\mathcal{M}^R_h,\mathcal{M}^R}(\Gamma(\widehat{\mathcal{M}}), \lambda) \not\in \text{c\_list}$}
\STATE $\text{fringe.push}(\langle \delta_{\mathcal{M}^R_h,\mathcal{M}^R}(\Gamma(\widehat{\mathcal{M}}), \lambda),~\mathcal{E}~\cup~\lambda \rangle,~c + 1)$

\ENDIF
\ENDFOR
\ENDIF
\ENDWHILE
\end{algorithmic}
\end{algorithm}
\begin{algorithm}%[]
\caption{Selection strategy for identifying model updates relevant to the current plan}
\label{algo1_selection}
\begin{algorithmic}
\STATE {\em Procedure} Priority\_Queue.pop  $\hat{\mathcal{M}}$
\STATE $\text{candidates} \leftarrow \{\langle \langle \widehat{\mathcal{M}}, \mathcal{E} \rangle, c^*\rangle ~|~ c^* = \argmin_{c}\langle \langle \widehat{\mathcal{M}}, \mathcal{E} \rangle, c\rangle\}$
\STATE $\text{pruned\_list} \leftarrow \{\}$
\STATE $\pi_H \leftarrow \pi$ such that $C(\pi, \hat{\mathcal{M}}) = C^{*}_{\hat{\mathcal{M}}}$
\vspace{2pt} 
\FOR{$\langle \langle \widehat{\mathcal{M}}, \mathcal{E} \rangle, c\rangle \in$ candidates} 
\vspace{2pt} 
\IF{$\exists a \in \pi^*_R \cup \pi_H \text{ such that } \tau^{-1}((\Gamma(\widehat{\mathcal{M}})~\setminus~\Gamma(\hat{\mathcal{M}}))\cup (\Gamma(\hat{\mathcal{M}}) \setminus \Gamma(\widehat{\mathcal{M}}))) \in \{C(a)\} \cup \textrm{pre}(a) \cup \textrm{adds}(a) \cup \textrm{dels}(a)$ }
%\Comment{\textcolor{blue}{ Candidates relevant to $\pi^*_R$ or $\pi_H$}}}
\STATE 
$\text{pruned\_list} \leftarrow \text{pruned\_list}~\cup~\langle \langle \widehat{\mathcal{M}}, \mathcal{E} \rangle, c\rangle$
\ENDIF
\ENDFOR
\vspace{2pt} 
\IF{$\text{pruned\_list} = \phi$}
\STATE $\langle \widehat{\mathcal{M}}, \mathcal{E} \rangle, c \sim Unif(\text{candidate\_list})$
\ELSE
\STATE $\langle \widehat{\mathcal{M}}, \mathcal{E} \rangle, c \sim Unif(\text{pruned\_list})$
\ENDIF 

% \PROCEEDURE{Priority\_Queue.pop}{$\hat{\mathcal{M}}$}
% \ENDPROCEDURE
\end{algorithmic}
\end{algorithm}

\begin{algorithm}[tbp!]
%\scriptsize
\caption{Search for Minimally Monotonic Explanations}
\label{algo2}
\begin{algorithmic}[1]
  \STATE{MME-Search}{}
\STATE \emph{Input}: MRP $\langle \pi^*, \langle \mathcal{M}^R, \mathcal{M}^R_h\rangle\rangle$
\STATE \emph{Output}: Explanation $\mathcal{E}^{MME}$
\STATE \emph{Procedure}:  
\vspace{2pt} 
\STATE $\mathcal{E}^{MME} \leftarrow$ \{\}
\STATE fringe $\leftarrow$ \texttt{Priority\_Queue()}
\STATE c\_list $\leftarrow$ \{\} \COMMENT{\textcolor{blue}{Closed list}}
\STATE h\_list $\leftarrow$ \{\} \COMMENT{\textcolor{blue}{List of incorrect model changes}}
\STATE $\text{fringe.push}(\langle \mathcal{M}^R, \{\}\rangle,~\text{priority} = 0)$
\WHILE{fringe is not empty}
\STATE $\langle \widehat{\mathcal{M}}, \mathcal{E} \rangle, c \leftarrow \text{fringe.pop}(\widehat{\mathcal{M}})$
\IF{$C(\pi^*, \widehat{\mathcal{M}}) > C^{*}_{\widehat{\mathcal{M}}}$}
\STATE $\text{h\_list} \leftarrow \text{h\_list}~\cup~((\Gamma(\widehat{\mathcal{M}})~\setminus~\Gamma(\mathcal{M}^R)) \cup (\Gamma(\mathcal{M}^R)~\setminus~\Gamma(\widehat{\mathcal{M}})))$ \COMMENT{\textcolor{blue}{Updating h\_list }}
\ELSE
\STATE c\_list $\leftarrow$ c\_list $\cup~\widehat{\mathcal{M}}$
\FOR{$f \in \Gamma(\widehat{\mathcal{M}})~\setminus~\Gamma(\mathcal{M}^R_h)$  \COMMENT{\textcolor{blue}{Models that satisfy Condition 1}}}
\STATE $\lambda \leftarrow \langle 1, \{\widehat{\mathcal{M}}\}, \{\}, \{f\} \rangle$ \COMMENT{\textcolor{blue}{Removes $f$ from $\widehat{\mathcal{M}}$}}
\IF{$\delta_{\mathcal{M}^R,\mathcal{M}^R_h}(\Gamma(\widehat{\mathcal{M}}), \lambda) \not\in \text{c\_list} \newline \indent \indent \indent \textbf{~~~~and } \nexists S \text{ s.t. }((\Gamma(\widehat{\mathcal{M}})\setminus\Gamma(\mathcal{M}^R)) \cup
(\Gamma(\mathcal{M}^R) \setminus \Gamma(\widehat{\mathcal{M}}))
) \supseteq S \in \text{h\_list}$} %\Comment{\textcolor{blue}{Proposition 3}}
\STATE $\text{fringe.push}(\langle \delta_{\mathcal{M}^R,\mathcal{M}^R_h}(\Gamma(\widehat{\mathcal{M}}), \lambda),~\mathcal{E}~\cup~\lambda \rangle,~c + 1)$
\STATE $\mathcal{E}^{MME} \leftarrow \max_{|\cdot|}\{\mathcal{E}^{MME}, \mathcal{E}\}$
\ENDIF
\ENDFOR
\vspace{2pt} 
\FOR{$f \in \Gamma(\mathcal{M}^R_h)~\setminus~\Gamma(\widehat{\mathcal{M}})$  \COMMENT{\textcolor{blue}{Models that satisfy Condition 2}}}
\STATE $\lambda \leftarrow \langle 1, \{\widehat{\mathcal{M}}\}, \{f\}, \{\} \rangle$ \COMMENT{\textcolor{blue}{Adds $f$ from $\widehat{\mathcal{M}}$}}
\IF{$\delta_{\mathcal{M}^R,\mathcal{M}^R_h}(\Gamma(\widehat{\mathcal{M}}), \lambda) \not\in \text{c\_list} \newline \indent \indent \indent \textbf{~~~~and } \nexists S \text{ s.t. }((\Gamma(\widehat{\mathcal{M}})\setminus\Gamma(\mathcal{M}^R)) \cup (\Gamma(\mathcal{M}^R) \setminus \Gamma(\widehat{\mathcal{M}}))) \supseteq S \in \text{h\_list}$ %\COMMENT{\textcolor{blue}{Proposition 3}}
}
\STATE $\text{fringe.push}(\langle \delta_{\mathcal{M}^R,\mathcal{M}^R_h}(\Gamma(\widehat{\mathcal{M}}), \lambda),~\mathcal{E}~\cup~\lambda \rangle,~c + 1)$
\STATE $\mathcal{E}^{MME} \leftarrow \max_{|\cdot|}\{\mathcal{E}^{MME}, \mathcal{E}\}$
\ENDIF
\ENDFOR
\ENDIF
\ENDWHILE
\STATE $\mathcal{E}^{MME} \leftarrow  ((\Gamma(\widehat{\mathcal{M}})~\setminus~\Gamma(\mathcal{M}^R))\cup (\Gamma(\mathcal{M}^R) \setminus \Gamma(\widehat{\mathcal{M}}))) \setminus \mathcal{E}^{MME} $
\STATE return $\mathcal{E}^{MME}$
\end{algorithmic}
\end{algorithm}
\section{Approximate Explanations}
\index{Model-Reconciliation!Explanation!Approximations}%
In this section, we will discuss how some of the methods discussed in this section can be simplified to either 1) generate simpler explanation (in terms of the explanation size or more directed to a specific user query) and 2) reduce the computational overhead of the search.
\subsection{Explicit Contrastive Explanations}
\label{contr}
\index{Model-Reconciliation!Explanation!Contrastive}%
\index{Contrastive Explanations}%
As mentioned earlier, the explanations generated by model reconciliation can be viewed as an answer to an implicit contrastive query. Although this may lead to user being provided more information than required. This may be particularly unnecessary if the user's original expected set of plans is much smaller than the set of all optimal plans. In these cases, we could let the user directly specify their expected set of plans in the form of explicit foils. Thus explanation can focus on establishing how the current plan compares against their original expected set of plans. So this allows the system to not only provide less information, the system can now potentially also provide additional information that can allow the user to understand why the current plan is better than the plans they were expecting.% (we will see some examples of such information in Chapter \ref{ch07}). 
However in this case, the explanation could be a multi-step process where they raise additional queries (in the form of more foils) after each explanation.

Now if $\hat{\Pi}$ is the set of alternate foils raised by the human, the objective of a minimal explanation generation method would be

\begin{equation*}
\mathcal{E}^{contr} = \argmin_{\mathcal{E}}|(\Gamma(\widehat{\mathcal{M}}) \setminus \Gamma(\mathcal{M}^R_h)) \cup (\Gamma(\mathcal{M}^R_h) \setminus \Gamma(\widehat{\mathcal{M}}))| \text{ with } C(\pi^{*}, \widehat{\mathcal{M}}^R_h) \leq C(\pi', \widehat{\mathcal{M}}^R_h) ~\forall \pi' \in  \hat{\Pi}
\end{equation*}
Now we can modify the MCE search to use this new objective to identify the required explanation. Now the foil set $\hat{\Pi}$ may be either explicitly specified or may be implicitly specified in terms of constraints satisfied by plans in the set.
\subsection{Approximate MCE}
Both MCEs and MMEs may be hard to compute - in the worst case it involves a search over the entire space of model differences. 
Thus the biggest bottleneck here is the check for optimality of a plan given a new model. 
A check for necessary or sufficient conditions for optimality, without actually computing optimal plans can be used as a way to further prune the search tree. 

In the following section, we investigate an approximation to an MCE by employing a few simple proxies to the optimality test.
By doing this we lose the completeness guarantee but improve computability.
Specifically, we replace the equality test in line 12 of Algorithm \ref{algo1} by the following rules --

\begin{itemize}
\item[C(1)] $\delta_{\widehat{\mathcal{M}}}(\widehat{I}, \pi^*_R) \models \widehat{G}$; \textbf{and}
\item[C(2)] $C(\pi^*_R, \widehat{\mathcal{M}}) < C(\pi^*_R, \mathcal{M}^R_h)$ \textbf{or} $\delta_{\widehat{\mathcal{M}}}(\widehat{I}, \pi^*_H) \not\models \widehat{G}$; \textbf{and}
\item[C(3)] Each action contributes at least one causal link to $\pi^*_R$.
\end{itemize}
 
C(1) ensures that the plan $\pi^*_R$ originally computed is actually valid in the new model.
C(2) requires that this plan has either become better in the new model or at least that the human's expected plan $\pi^*_H$ has been disproved. Finally, C(3), ensures that for each action $a_i \in \pi^*_R$ there exists an effect $p$ that satisfies the precondition of at least one action $a_k$ (where $a_{i} \prec a_{k}$) and there exists no action $a_j$ (where $a_i \prec a_j \prec a_k)$ such that $ p \in \textrm{dels}(a_j)$.
Such explanations are only able to preserve local properties of a plan
and hence incomplete.
 
\begin{proposition}
C(3) is a necessary condition for optimality of $\pi^*$ in $\widehat{\mathcal{M}}$.
\end{proposition}

\vspace{5pt}
\begin{proof}[Proof Sketch]
Assume that for an optimal plan $\pi^*_R$, there exists an action $a_i$ where criterion (3) is not met. 
Now we can rewrite $\pi^*_R$ as $\pi'_R= \langle a_0, a_1, \ldots,a_{i-1},a_{i},a_{i+1},\ldots, a_n, a_{n+1} \rangle$, where $\textrm{pre}(a_0) = \phi$ and $\textrm{adds}(a_0) = \{I\}$ and $\textrm{pre}(a_{n+1}) = \{G\}$ and $\textrm{adds}(a_{n+1}) = \textrm{dels}(a_{n+1}) = \phi$. It is easy to see that $\delta_{\widehat{\mathcal{M}}}(\phi, \pi'_R)\models G$. 
Now let us consider a cheaper plan $\hat{\pi'_R}= \langle a_0, a_1, \ldots,a_{i-1},a_{i+1},\ldots, a_n, a_{n+1} \rangle$. Since $a_i$ does not contribute any causal links to the original plan $\pi^*_R$, we will also have $\delta_{\widehat{\mathcal{M}}}(\phi, \hat{\pi}'_R) \models G $. This contradicts our original assumption of $\pi^*_R$ being optimal, hence proved.
\end{proof}
\section{User Studies}
\index{Model-Reconciliation!User Studies}
\begin{figure}
\centering
\includegraphics[width=\columnwidth]{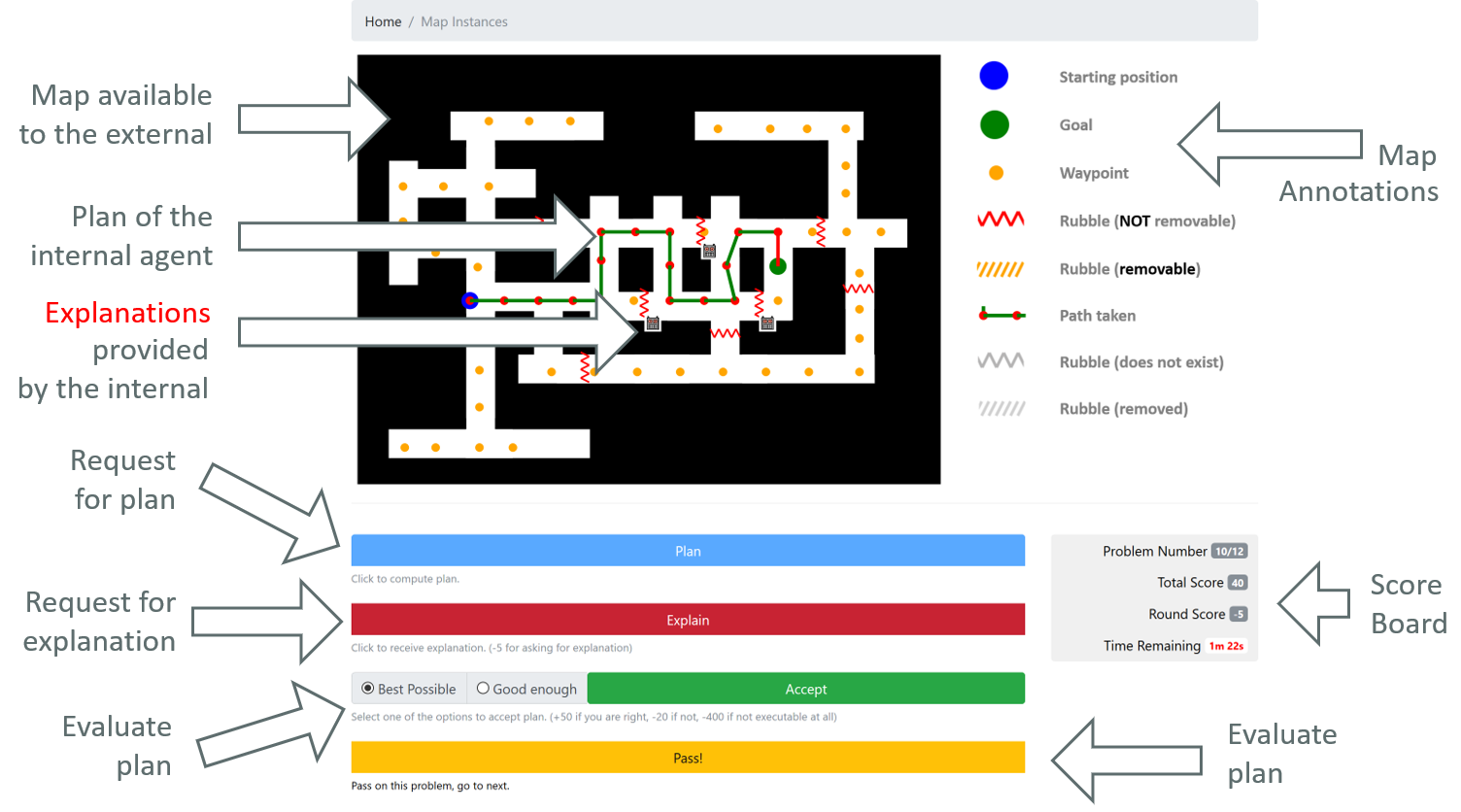}
\caption{Interface for a user study where participants assumed the role of the external commander and evaluated plans provided by the internal robot. They could request for plans and explanations to those plans (if not satisfied) and rate them
as optimal or suboptimal or (if unsatisfied) can chose to pass.}
\label{user-study}
\end{figure}
\begin{figure*}
\centering
\includegraphics[width=\textwidth]{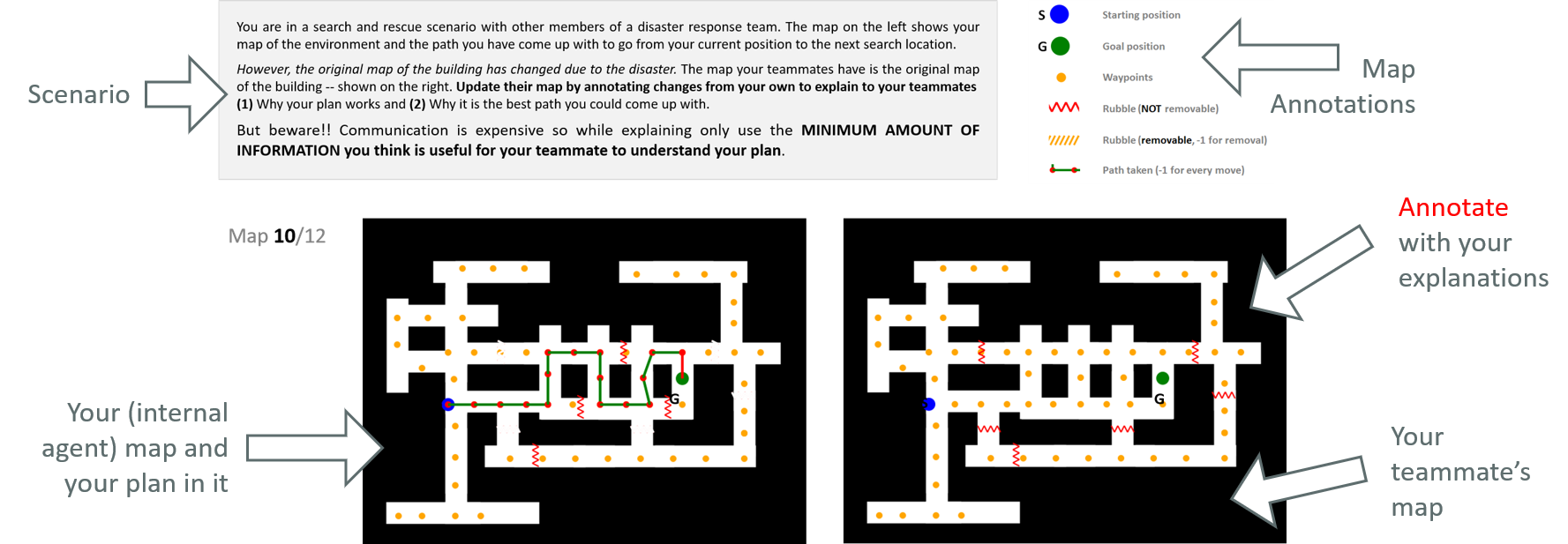}
\caption{Interface for Study-1: Participants assumed the role of the internal agent and
explained their plans to a teammate with a possibly different map of the world.
}
\label{interface1}
\end{figure*}
\unsure{Apart from the underlying formal theory and intuition for the methods discussed here, user studies have also been performed to explicitly test the effectiveness of these type of explanations. In one study, the participants were shown a simulated robot navigating  a floor in a building, and the floor map was made available to the user. They were then asked to evaluate whether the robot plans were valid and optimal. The participants had the option to query for explanations that revealed details about the floor map that were previously unknown to the user. Figure \ref{user-study} presents a screenshot of this interface along with a sample plan. The user study also had a reward structure that disincentivized participants from providing incorrect answers (in regards to whether the plans were valid and/or optimal). The study showed the majority of participants were correctly able to identify the validity and optimality of the plan based on the explanations. A subjective survey at the end of the study showed that the majority of participants found the explanations were easy to understand and useful.}

Additional studies were also performed that placed participants in the role of the explainer. In this study, they were asked to explain why a given plan was optimal. They were shown the original map and the map the explainee would see. After analyzing the explanations given by the participants, it was found that most of them perfectly matched up with many of the explanation types studied in this chapter.
Particularly, when the participants were asked to minimize the information to be provided the majority of explanation provided were MCEs. This points to the fact that in many cases, people do naturally calculate explanations of the type discussed in this chapter.

\section{Other Explanatory Methods}
In terms of explanations, this chapter focused on one specific type of explanation, namely explanation as model reconciliation. But this is not the only type of explanation that has been studied in the literature. In particular, we will look at two other types of explanations that are popular within the explainable planning community (generally referred to as XAIP), and in this section, we will look at how these explanation types could be viewed in the context of human-aware planning.

\paragraph{Explanations to Address Inferential Asymmetry.} 
\index{Explanation!Explanations to Address Inferential Asymmetry}%
The first type of explanation is one designed to address explanatory queries rising from differences in inferential capabilities of the human and the robot. Going back to the setting where the human is a pure observer, even if the human had access to the exact model, there is no guarantee that they would be able to correctly evaluate the validity or optimality of the plan being followed by the robot given potential inferential limitations of the human. So in addition to reconciling any model differences, the system may need to help address any confusion on the human's end stemming from such asymmetry in inferential capabilities. Some general strategies for addressing such considerations include,
\begin{enumerate}
    \item Allowing users to ask questions about plans: In particularly allowing for contrastive queries, in which the human raises specific alternatives they were expecting \cite{sreedharan2018hierarchical}. As mentioned earlier, such questions help expose the human's expectations about the robot and thus allowing the robot to provide a more directed explanation. Having access to the human's expected behavior, also means the robot could provide additional information that helps contrast the given decision against the alternatives the human had in mind (say by tracing the execution of the current plan and the alternative).
    \item Simplifying the problem: The next strategy one could employ may be to simplify the problem being explained. 
\todo{One of the obvious strategy we can employ here is to use state abstractions, we will look at this strategy in more detail in Chapter \ref{ch06} though in a slightly different context.}
%We have already seen one of these strategies, namely the use of state abstractions. 
But this is one of many problem simplification strategies we could employ. Other strategies include the use of local approximation (also discussed in Chapter \ref{ch07} (Section \ref{ch07:local_approx})), decomposing the problem into smaller subproblems (for example, focusing on the achievement of a subgoal instead of the true goal as done in the case of \cite{sreedharan2019can}) and even simplifying the class of the reasoning/planning problem. For example, in many cases it may suffice to focus on a determinization or a qualitative non-deterministic representation of a problem to address many of the explanatory challenges related to stochastic planning problems.
\item Providing explanatory witness: In many cases  while the explanation may be in regards to a space of possible solutions, it may be helpful to provide the human with specific example to how the current solution may satisfy a property that is in question. For example, if the human were to raise a partial specification of an alternative, say she asks why the robot didn't choose a particular action at a state, then the robot could point to a specific example plan that includes an action and point out that this alternate plan is costlier than the current plan (this is the strategy followed by \cite{krarup2019model}). This doesn't necessarily explain why all plans that include that action may be worse off than the current plan, but providing an example could help the user better understand the plan in question.
\end{enumerate}
Note that all the above strategies are generally based  on computational intuition on how to reduce computational overhead of a problem (many of these methods have their roots on heuristics for planning). On the other hand, a systematic framework to address such asymmetry would require us to correctly model and capture the human inferential capabilities, so that the agent can predict the inferential overhead placed on the human by a specific explanation approach. Unfortunately, we are still far from achieving such accurate modeling of the computational capabilities of a bounded rational agent let alone a human.

\paragraph{Plan and Policy Summary}
\index{Explanation!Policy Summary}%
Another category of work that we haven't had a chance to discuss in the book are the ones related to plan and policy summaries. Particularly in cases, where the plan may need to be approved by the human before it is executed. In such cases, the robot would need to communicate its plans to the human before it can be executed. Several works have looked at the problem of effectively summarizing a generated course of action. Many of these works focus on stochastic planning settings, where the solution concept takes the form of policies which could in general be harder to understand than traditional sequential plans. Some general strategies used in such policy summarization method include, the use of abstractions (for example \cite{topin2019generation} use state abstraction, while \cite{sreedharan2020tldr} use temporal abstraction) and selecting some representative state-action pairs, so the human can reconstruct the rest of the policy (for example \cite{amir2018highlights}, \cite{offra_summ}). Generally most works in this latter direction assume that the human correctly understands the task at hand and thus the reconstruction can be done correctly by the human without the aid of any model reconciliation. One may also need to address the problem of vocabulary mismatch in the context of policy communication, wherein the policy may need to be translated first into terms the user can understand (as done by the work \cite{hayes}).
% A lot ofwork in summarizing plans and policies

\paragraph{Other Popular Forms of Explanation Categorization.} As final note on other forms of explanation studied in XAIP, a popular way to categorize explanatory information with in the literature is to organize them based on the type of explanatory questions they may be addressing. \cite{danmaga} provides a very comprehensive list of explanatory questions that could be relevant to planning. This work not only includes question related to specific plans that are generated by the agent, but also questions related to the planning process as a whole. For example, the paper considers questions like why the agent decided to replan at a particular point during the plan execution.

\section{Bibliographic Remarks}
The idea of explanation as model reconciliation was first introduced in the paper \cite{explain}. 
The paper specifically looked at the application of the methods for STRIPS style models.
The user studies for evaluating the explanations were presented in the paper \cite{chakraborti2019plan}.
The paper looked at three types of model-reconciliation explanations, namely MCE, MPE and PPE. In the case of PPE, a slightly modified version of PPE was considered that also takes into account one of the possible optimal plans in the human model (thereby ensuring the explanation is still contrastive). 
\todo{\cite{sreedharan2021foundations} also presents a more unified view of the various earlier works done in the area of model-reconciliation explanation}.
While most works in model reconciliation explanations have focused on deterministic STRIPS-style planning models, researchers have also started looking at other planning formalisms. For example, \cite{modelfree} looked at model reconciliation in the context of MDPs and  \cite{logical} looked at model reconciliation in the context of logical programs. 

\index{LIME}%
There are even parallel works on explanation from fields like machine learning, that could be understood as a special case of model reconciliation. A particularly significant example being LIME \citep{ribeiro2016should}. In this work, the authors propose a model-agnostic method that take a model, and the specific instance to be explained and generates an interpretable linear model that locally approximates the dynamics of the current instance. In this case the final explanation takes the form of details about the linear model and the human's model is taken to be empty. One important factor to note here is the model being transferred to the user is not the original model but rather a local approximation of it. Moreover, the model is post-hoc expressed in terms of human understandable features. These strategies are also applicable in the context of model reconciliation explanation in the sequential decision making settings, and we will discuss methods that employ such methods in chapter \ref{ch07}.
\todo{Some examples of explanation works that use abstraction of the model and/or plan include \cite{zahavy2016graying,topin2019generation}, and \cite{sreedharan2020tldr}. Another method popular in explaining the utility of a specific action in plan involve providing causal chain contributed by that action to the overall plan \citep{veloso1992learning,seegebarth2012making,bercher2014plan}. Another technique popular in explaining multi-objective planning problem is to provide conflicts between the different objectives \cite{eiflernew}. For more works done in this direction, readers can refer to the survey in \cite{xaip-landscape}. \cite{miller2019explanation} presents a concise introduction into various works from disciplines like social sciences, psychology and philosophy that have looked at explanations.
}
\clearpage

                % bold math and math variables
    %\newtheorem{defn}{Definition}

\chapter{Acquiring Mental Models for Explanations}
\label{ch06}
The previous chapter sketches out some of the central ideas behind generating explanation as model reconciliation, but it does so while making some strong assumptions. Particularly, the setting assumes that the human's model of the robot is known exactly upfront. In this chapter, we will look at how we can relax this assumption and see how we can perform model reconciliation in scenarios where the robot has progressively less information about the human mental model.
%We will start by looking at cases where the human mental model may not be known upfront. 
We will start by investigating how the robot can perform model reconciliation with incomplete model information. Next we will look at cases where the robot doesn't have a human mental model but can collect feedback from users.\footnote{Note that the model we need for explicability and explanation, $\mathcal{M}^R_h$ can't just be learned in general from past behavior traces of the humans themselves, since those give us information only about $\mathcal{M}^H$ and not $\mathcal{M}^R_h$. While there may be self-centered humans who consider the AI robot's model to be same as theirs, more generally, these models will be different.} We will see how we can use such feedback to learn simple labeling models that suffice to generate explanations. 
We will also look at generating model reconciliation explanations by assuming the human has a simpler mental model, specifically one that is an abstraction of the original model and also see how this method can help reduce inferential burden on the human.
%In terms of relaxing assumptions regarding users inferential capabilities, we will consider how we could use abstraction to simplify the explanations. 
%Finally, we will also look at how explanations can be effectively combined with explicable plan generation. Thereby incorporating the considerations for explanation generation into the agents decision-making process. 
Throughout this chapter, we will focus on generating MCE though most of the methods discussed here could also be extended to MME and contrastive versions of the explanations.

\section{The Urban Search and Reconnaissance Domain}
\label{subsubsec:usar}
\index{USAR@Urban Search and Rescue}%
We will concretize the discussion in this section in a typical Urban Search and Reconnaissance (USAR) 
domain
where a remote robot is put into disaster response operation often controlled partly or fully by an external human commander, as shown in Figure \ref{fig:testbeds:usar}.
The robot's job is to scout areas that may be otherwise harmful to humans and report on its surroundings as instructed by the external supervisor. 
The scenario can also have other internal agents (humans or robots) with whom the robot needs to coordinate. 
\begin{figure}
\centering
\includegraphics[width=0.6\textwidth]{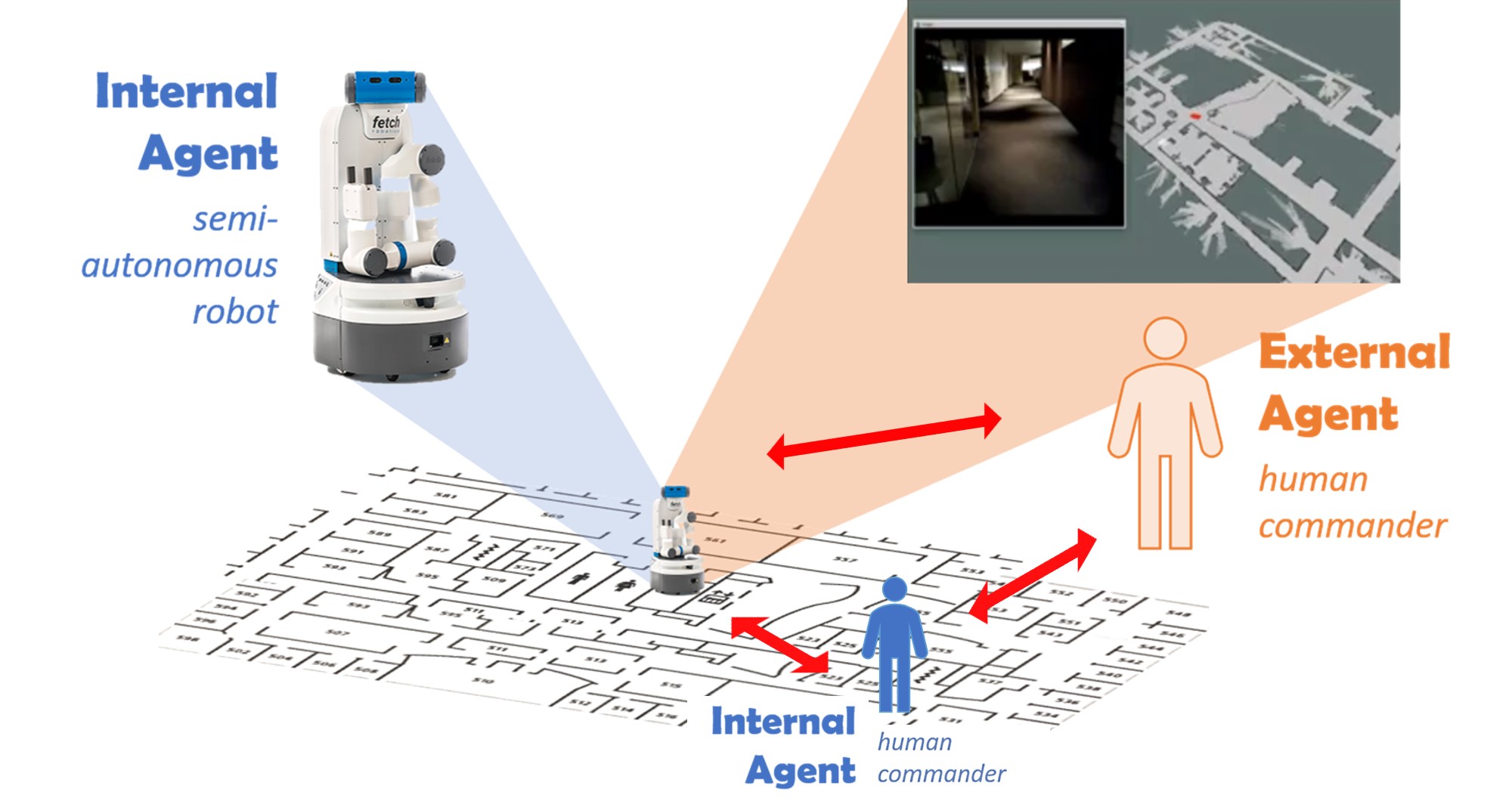}
\caption{A typical USAR domain with an internal robot and an external commander.}
\label{fig:testbeds:usar}
\end{figure}

Here, even though all agents start off with the same model --
i.e. the blueprint of the building -- their models diverge
as the internal agent interacts with the scene. 
Due to the disaster, new paths may have opened up due to 
collapsed walls or old paths may no longer be available due 
to rubble. This means that plans that are valid and optimal
in the robot's (current) model may not make sense to the external
commander. 
In the scenario in Figure \ref{fig:mega},
the robot is tasked to go from its current
location marked blue to conduct reconnaissance in the location
marked orange. The green path is most optimal in its current
model but this is blocked in the external's mental model
while the expected plan in the mental model is no longer
possible due to rubble. Without removing rubble in the blocked 
paths, the robot can instead communicate that the path at 
the bottom is no longer blocked. 
This explanation preserves the validity and optimality of its plan 
in the updated model (even though further differences exist).

\begin{figure}
\centering
\includegraphics[width=\textwidth]{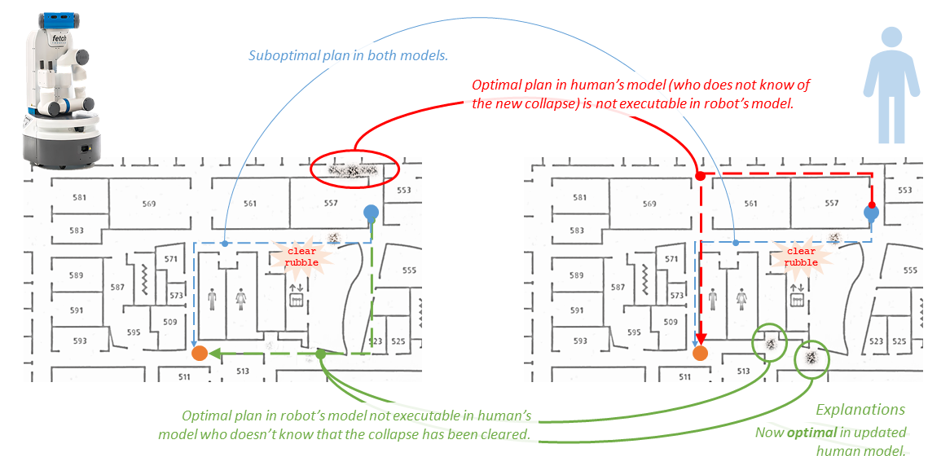}
\caption{Model differences in the USAR domain.}
\label{fig:mega}
\end{figure}

\section{Model Uncertainty}
\label{ch06:model_unc}
\index{Model-Reconciliation!Model Uncertainty}%
We will start by considering cases where the human's mental model may not be known exactly upfront. There may be parts of the human model the robot may know exactly, while it may be uncertain about others. We will leverage the annotated model representation to capture such incomplete models. Under this representation, in addition to the standard preconditions and effects associated with actions, 
the model includes {\em possible} preconditions and effects which may 
or may not be realized in practice.

\begin{definition}
\index{Annotated Model}%
An {\bf annotated model} 
is the tuple $\mathbb{M} = \langle  F,\mathbb{A}, \mathbb{I}, \mathbb{G} , C\rangle$ where $F$ is a finite set of fluents that define a state $s \subseteq F$, and $\mathbb{A}$ is a finite set of annotated actions, $\mathbb{I} = \langle I^0, I^+ \rangle,~\mathbb{G} = \langle G^0, G^+\rangle;$ the annotated initial and goal states ( such that $~I^0, G^0, I^+, G^+ \subseteq F$) and $C$ the cost function. 
Action $a \in~\mathbb{A}$ is a tuple $\langle \textrm{pre}(a), \widetilde{\textrm{pre}}(a), \textrm{adds}(a), \textrm{dels}(a), \widetilde{\textrm{adds}}(a), \widetilde{\textrm{dels}}(a)\rangle$ where in addition to $\textrm{pre}(a), \textrm{adds}(a), \textrm{dels}(a) \subseteq F$, i.e., the  known preconditions and add/delete effects each action also contains 
{\em possible preconditions} $\widetilde{\textrm{pre}}(a) \subseteq F$ containing propositions that it \emph{might} need as preconditions, and {\em possible add (delete) effects} $~\widetilde{\textrm{adds}}(a),\widetilde{\textrm{dels}}(a),\subseteq F$) containing propositions that it \emph{might} add (delete, respectively) after execution.
$I^0, G^0$ (and $I^+, G^+$) are the known (and possible) parts of the initial and goal states.
\end{definition}

Each possible condition $f \in \widetilde{\textrm{pre}}(a) \cup \widetilde{\textrm{adds}} (a) \cup \widetilde{\textrm{dels}}(a)$ has an associated probability $p(f)$ denoting how likely it is to be a known condition in the ground truth model -- i.e. $p(f)$ measures the confidence with which that condition has been learned. The sets of known and possible conditions of a model $\mathcal{M}$ are denoted by $\mathbb{S}_k({\mathcal{M}})$ and $\mathbb{S}_p({\mathcal{M}})$ respectively.
An {\em instantiation} of an annotated model $\mathbb{M}$ is a classical planning model where a subset of the possible conditions have been realized, and is thus given by the tuple $inst(\mathbb{M}) = \langle F, A, I, G \rangle$, initial and goal states $\mathcal{I} = \mathcal{I}^0 \cup \chi;~\chi \subseteq \mathcal{I}^+$ and $\mathcal{G} = \mathcal{G}^0 \cup \chi;~\chi \subseteq \mathcal{G}^+$ respectively, and action $A \ni a = \langle \textrm{pre}(a) \leftarrow \textrm{pre}(a) \cup \chi;~\chi \subseteq \widetilde{\textrm{pre}}(a), \textrm{adds}(a) \leftarrow \textrm{adds}(a) \cup \chi;~\chi \subseteq \widetilde{\textrm{adds}}(a), \textrm{dels}(a) \leftarrow \textrm{dels}(a) \cup \chi;~\chi \subseteq \widetilde{\textrm{dels}}(a)\rangle$. 
Clearly, given an annotated model with $k$ possible conditions, there may be $2^k$ such instantiations, which forms its {\em completion set}. \index{Annotated Model!Completion Set}%
\index{Annotated Model!Instantiation}%

\begin{definition}
\index{Annotated Model!Instantiation!Likelihood}%
{\bf Likelihood $P_{\ell}$} of instantiation $inst(\mathbb{M})$ of an annotated model $\mathbb{M}$ is:

\vspace{-10pt}
\begin{equation*}
P_{\ell}(inst(\mathbb{M})) = \prod_{f \in \mathbb{S}_p(\mathbb{M}) \wedge \mathbb{S}_k(inst(\mathbb{M}))} p(f) ~~~\times~~~ \prod_{f \in \mathbb{S}_p(\mathbb{M}) \setminus \mathbb{S}_k(inst(\mathbb{M}))} (1-p(f))
\end{equation*}
\end{definition} 

As discussed before, such models turn out to be especially useful for the representation 
of human (mental) models learned from observations, where uncertainty after the learning process can be represented in terms of model annotations. 
%We thus represent the uncertainty of the mental model of the human in the loop in terms of an annotated model.
Let $\mathbb{M}^R_h$ be the culmination of a model learning process and $\{\mathcal{M}^R_{h_i}\}_{i}$ be the completion set of $\mathbb{M}^R_h$.
One of these models is the actual ground truth (i.e. the human's real mental model). We refer to this as $g(\mathbb{M}^R_h)$.
We will explore now how this representation will allow us to compute
{\em conformant explanations} that can explain with respect to all possible mental models
and {\em conditional explanations} that engage the explainee in dialogue to 
minimize the size of the completion set to compute shorter 
explanations.
\subsubsection{Conformant Explanations}
\index{Conformant Explanations}%

\unsure{In this situation, the robot can try to compute MCEs for each possible configuration. 
However, this can result in situations where the explanations computed for individual models independently are not consistent across all possible target domains. 
Thus, in the case of model uncertainty, such an approach cannot guarantee that the resulting explanation will be acceptable.} 

\unsure{Instead, the objective is to find an explanation such that $\forall i~ \pi^*_{\widehat{\mathcal{M}}^R_{h_i}} \equiv \pi^*_{\mathcal{M}^R}$ 
(as shown in Figure~\ref{multi-mmp}). 
This is a single set of model updates that makes the given plan optimal (and hence explained) in all the updated models.
At first glance, it appears that such an approach, even though desirable, might turn out to be prohibitively expensive especially since solving for a {\em single} MCE involves search in the model space where each search node is an optimal planning problem. 
However, it turns out that the same search strategy can be employed here as well by representing the human mental model as an {\em annotated} model. The optimality condition for MCE now becomes --}

%\begin{itemize}
\[C(\pi, g(\mathbb{M}^R_h)) = C^*_{g(\mathbb{M}^R_h)}\]
%\end{itemize}

\begin{figure}[tb!]
\centering
\includegraphics[width=0.65\textwidth]{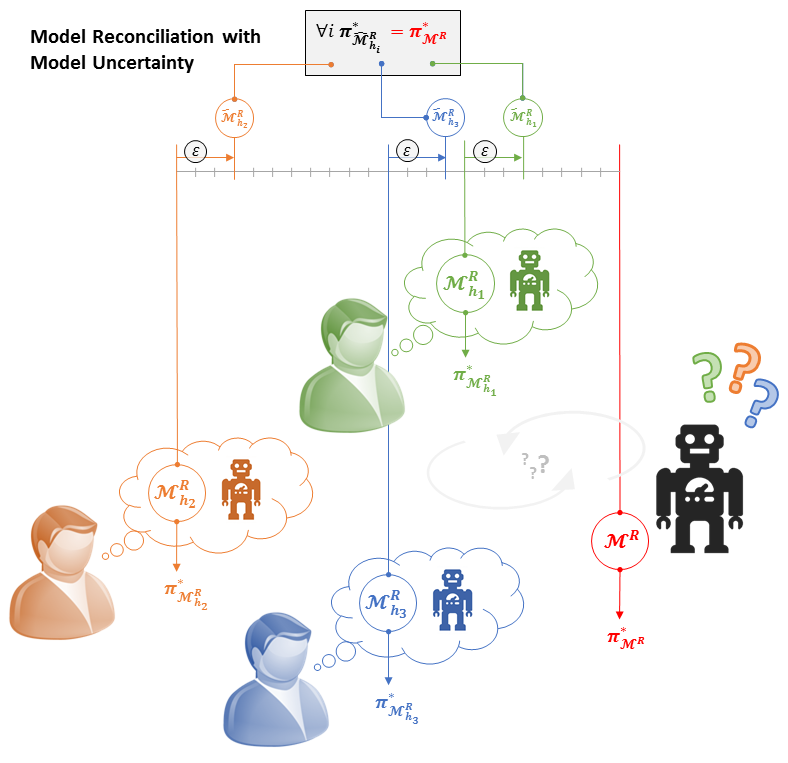}
\caption{Model reconciliation in the presence of model uncertainty or multiple explainees.}
\label{multi-mmp}
\end{figure}

%This is hard to achieve since it is not known which is the actual mental model of the human. So we want to preserve the optimality criterion for all (or as many) instantiations of the incomplete estimation of the mental model.
%Keeping this in mind, 
\unsure{We define {\em robustness} of an explanation for an incomplete mental model as the probability mass of models where it is a valid explanation.}

\vspace{10pt}
\begin{definition}
\index{Explanation!Robustness}%
{\bf Robustness} of an explanation $\mathcal{E}$ is given by --

\vspace{-10pt}
\begin{equation*}
R(\mathcal{E}) = \sum_{inst(\mathcal{\widehat{M}}^R_h) \text{ s.t. } C(\pi, inst(\mathcal{\widehat{M}}^R_h)) = C^*_{inst(\mathcal{\widehat{M}}^R_h)}} P_{\ell}(inst(\mathcal{\widehat{M}}^R_h))
\end{equation*} 
\end{definition}

% \subsection{Demonstration}
% \label{xaip:unc:demo}

\vspace{10pt}
\begin{definition}
 A {\bf conformant explanation} is such that $R(\mathcal{E}) = 1$.
\end{definition}
\unsure{Which is equivalent to saying that conformant explanation ensures that the given plan is explained in all the models in the completion set of the human model. 
Let's look at an example.} 
Consider again the USAR domain (Figure \ref{usar-uncertain}), the robot is now at P1 (blue) and needs to collect data from P5. While the commander understands the goal, she is under the false impression that the paths from P1 to P9 and P4 to P5 are unusable (red question marks). She is also unaware of the robot's inability to use its hands. 
On the other hand, while the robot does not have a complete picture of her mental model, 
it understands that any differences between the models are related to (1) the path from P1 to P9; (2) the path from P4 to P5; (3) its ability to use its hands; and (4) whether it needs its arm to clear rubble.
Thus, from the robot's perspective, the mental model can be one of sixteen possible models
(one of which is the actual one). 
Here, a conformant explanation for the optimal robot plan (blue) is as follows --

% (a video can be viewed at \textcolor{blue}{\url{https://youtu.be/bLqrtffW6Ng}}) --

\begin{verbatim}
Explanation >> remove-known-INIT-has-add-effect-hand_capable
Explanation >> add-annot-clear_passage-has-precondition-hand_capable
Explanation >> remove-annot-INIT-has-add-effect-clear_path P1 P9
\end{verbatim}

\begin{figure}[!t]
\centering
\includegraphics[width=\textwidth]{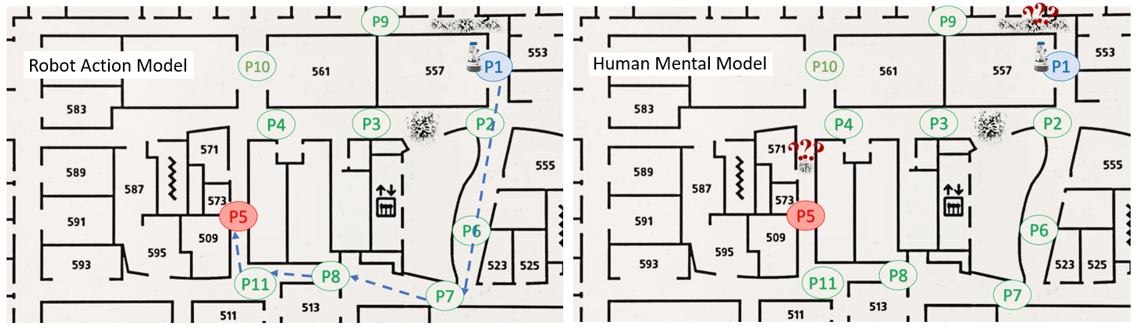}
\caption{Back to our USAR scenario: the robot plan is marked in blue and uncertain parts of the human model is marked with red question marks. 
% A video can be seen at {\protect\url{https://youtu.be/bLqrtffW6Ng}}.
}
\label{usar-uncertain}
\end{figure}

\subsubsection{Model-Space Search for Conformant Explanations}
\index{Conformant Explanation!Algorithm}%
\index{Annotated Model!Min Model}%
\index{Annotated Model!Max Model}%

As we discussed before, we cannot launch an MCE-search for each possible mental model
separately, both for issues of complexity and consistency of the solutions.
However, in the following discussion, we will see how we can reuse the model space
search from the previous section with a compilation trick. 

We begin by defining two models -- the most relaxed model possible $\mathcal{M}_{max}$  and the least relaxed one $\mathcal{M}_{min}$.
The former is the model where all the possible add effects and none of the possible preconditions and deletes hold, the state has all the possible conditions set to true, and the goal is the smallest one possible; while in the latter all the possible preconditions and deletes and none of the possible adds are realized and with the minimal start state and the maximal goal.
This means that, if a plan is executable in $\mathcal{M}_{min}$ it will be executable in all the possible models. Also, if this plan is optimal in $\mathcal{M}_{max}$, then it must be optimal throughout the set. Of course, such a plan may not exist, and we are not trying to find one either. Instead, we are trying to find a set of model updates which when applied to the annotated model, produce a new set of models where a given plan is optimal. 
In providing these model updates, we are in effect reducing the set of possible models to a smaller set. The new set need not be a subset of the original set of models but will be equal or smaller in size to the original set. For any given annotated model, such an explanation always exists (entire model difference in the worst case), and the goal here becomes to find the smallest one. 
$\mathbb{M}^R_h$ thus affords the following two models --

\vspace{10pt}
$\mathcal{M}_{max}  = \langle  F, A, I, G \rangle$  

{\small
\begin{itemize}
\item[-] initial state $I \leftarrow I^0 \cup I^+$; given $\mathbb{I}$
\item[-] goal state $G \leftarrow G^0$; given $\mathbb{G}$
\item[-] $\forall a \in A$ 
\begin{itemize}
\item[-] $\textrm{pre}(a) \leftarrow \textrm{pre}(a); ~a \in~\mathbb{A}$
\item[-] $\textrm{adds}(a) \leftarrow \textrm{adds}(a) \cup \widetilde{\textrm{adds}}(a); ~a \in~\mathbb{A}$
\item[-] $\textrm{dels}(a) \leftarrow \textrm{dels}(a); ~a \in~\mathbb{A}$
\end{itemize}
\end{itemize}
}

\vspace{5pt}
$\mathcal{M}_{min} = \langle  F, A, I, G \rangle$ 

{\small
\begin{itemize}
\item[-] initial state $I \leftarrow I^0$; given $\mathbb{I}$
\item[-] goal state $G \leftarrow G^0 \cup G^+$; given $\mathbb{G}$
\item[-] $\forall a \in A$ 
\begin{itemize}
\item[-] $\textrm{pre}(a) \leftarrow \textrm{pre}(a) \cup \widetilde{\textrm{pre}}(a); ~a \in~\mathbb{A}$
\item[-] $\textrm{adds}(a) \leftarrow \textrm{adds}(a); ~a \in~\mathbb{A}$
\item[-] $\textrm{dels} \leftarrow \textrm{dels} \cup \widetilde{\textrm{dels}}; ~a \in~\mathbb{A}$
\end{itemize}
\end{itemize}
}

\vspace{5pt}
As explained before, $\mathcal{M}_{max}$ is a model where all the add effects hold and it is easiest to achieve the goal, and similarly $\mathcal{M}_{min}$ is the model where it is the hardest to achieve the goal.
These definitions might end up creating inconsistencies (e.g. in an annotated \texttt{BlocksWorld} domain, the \texttt{unstack} action may have add effects to make the block both \texttt{holding} and \texttt{ontable} at the same time), but the model reconciliation process will take care of these. 

\begin{proposition}
\label{prop-both}
For a given MRP $\Psi = \langle \pi, \langle \mathcal{M}^R,  \mathbb{M}^R_h \rangle \rangle$, if the plan $\pi$ is optimal in $\mathcal{M}_{max}$ and executable in $\mathcal{M}_{min}$, then the plan is optimal for all $i$.
\end{proposition}

\vspace{5pt}
This now becomes the new criterion to satisfy in the course of search for an MCE for a set of models.
We again reuse the state representation in Chapter \ref{ch05} (generated by $\Gamma$ as described in Section \ref{ch05:model-rec-sec}).
We start the conformant search by first creating the corresponding $\mathcal{M}_{max}$ and $\mathcal{M}_{min}$ model for the given annotated model $\mathbb{M}^{R}_h$. 
While the goal test for the original MCE only included an optimality test, here we need to both check the optimality of the plan in  $\mathcal{M}_{max}$ and verify the correctness of the plan in $\mathcal{M}_{min}$. As stated in Proposition \ref{prop-both}, the plan is only optimal in the entire set of possible models if it satisfies both tests. Since the correctness of a given plan can be verified in polynomial time with respect to the plan size, this is a relatively easy test to perform.

The other important point of difference between the algorithm mentioned above and the original MCE is how the applicable model updates are calculated. Here we consider the superset of model differences between the robot model and $\mathcal{M}_{min}$ and the differences between the robot model and $\mathcal{M}_{max}$. This could potentially mean that the search might end up applying a model update that is already satisfied in one of the models but not in the other. Since all the model update actions are formulated as set operations, the original MRP formulation can handle this without any further changes. The models obtained by applying the model update to $\mathcal{M}_{min}$ and $\mathcal{M}_{max}$ are then pushed to the open queue.

\begin{proposition}
$\mathcal{M}_{max}$ and $\mathcal{M}_{min}$ only need to be computed once
% before the search 
-- i.e. with a model update $\mathcal{E}$ to $\mathbb{M}$:
%completion set $\{\mathcal{M}^R_{h_i}\}_{i}$,
$\mathcal{M}_{max} \leftarrow \mathcal{M}_{max} + \mathcal{E}$ and $\mathcal{M}_{min} \leftarrow \mathcal{M}_{min} + \mathcal{E}$. % for the new model set. 
\end{proposition}
 
These models form the new $\mathcal{M}_{min}$ and $\mathcal{M}_{max}$ models for the set of models obtained by applying the current set of model updates to the original annotated model. This proposition ensures that we no longer have to keep track of the current list of models or recalculate $\mathcal{M}_{min}$ and $\mathcal{M}_{max}$ for the new set.

\subsubsection{Conditional Explanations}
\index{Conditional Explanations}%
Conformant explanations can contain superfluous information -- i.e. asking the human to remove non-existent conditions or add existing ones.
In the previous example, the second explanation (regarding the need of the hand to clear rubble) was already known to the human and was thus superfluous information.
Such redundant information can be annoying and may end up reducing the human's trust in the robot. This can be avoided by --

\begin{itemize}
\item[-] Increasing the cost of model updates involving uncertain conditions relative to those involving known preconditions or effects.
This ensures that the search prefers explanations that contain known conditions. 
By definition, such explanations will not have superfluous information. 
\item[-] However, sometimes such explanations may not exist. 
Instead, we can convert conformant explanations into {\em conditional} ones. 
This can be achieved by turning each model update for an annotated condition into a question and only provide an explanation if the human's response warrants it -- e.g. instead of asking the human to update the precondition of $\texttt{clear\_passage}$, 
the robot can first ask if the human thinks that action has a precondition $\texttt{hand\_usable}$. 
Thus, one way of removing superfluous explanations is to reduce the size of the completion set by gathering information from the human.
\end{itemize}

\unsure{By using information gathering actions, we can do even better than just remove redundant information from an already generated conformant explanation.
For example, consider the following exchange in the USAR scenario} -- 

\begin{verbatim}
R : Are you aware that the path from P1 to P4 has collapsed?
H : Yes.
> R realizes the plan is optimal in all possible models. 
> It does not need to explain further. 
\end{verbatim}

If the robot knew that the human thought that the path from P1 to P4 was collapsed, 
it would know that the robot's plan is already optimal in the human mental model 
and hence be required to provide no further explanation.
This form of explanations can thus clearly be used to cut down on the size of 
conformant explanations by reducing the size of the completion set.

% \note{TC: Such explanations are only possible in the case of MCE and do not apply for MME -- Proof?}

\begin{definition}
A {\bf conditional explanation}
is represented by a policy that maps the annotated model (represented by a $\mathcal{M}_{min}$ and $\mathcal{M}_{max}$ model pair) to either a question regarding the existence of a condition in the human ground model or a model update request. The resultant annotated model is produced, by either applying the model update directly into the current model or by updating the model to conform to human's answer regarding the existence of the condition.

\end{definition}

In asking questions such as these, the robot is trying to exploit the human's (lack of) knowledge of the problem in order to provide more concise explanations. 
This can be construed as a case of lying by omission and can raise interesting ethical considerations (we will look at the role of lies in team in more detail in Chapter \ref{ch08}, Section \ref{lies}). Humans, during an explanation process, tend to undergo this same ``selection'' process as well in determining which of the many reasons that
could explain an event is worth highlighting. 

\paragraph{Modified $AO^{*}$-search to find Conditional Explanations}
\index{Conditional Explanations!Algorithm}%
We can generate conditional explanations by either performing post-processing on conformant explanations or by performing an AND-OR graph search with $AO^{*}$  \cite{nilsson2014principles}. 
Here each model update related to a known condition forms an OR successor node while each {\em possible} condition can be applied on the current state to produce a pair of AND successors, where the first node reflects a node where the annotated condition holds while the second one represents the state where it does not. So the number of possible conditions reduces by one in each one of these AND successor nodes. This AND successor relates to the answers the human could potentially provide when asked about the existence of that particular possible condition. Note that this AND-OR graph will not contain any cycles as we only provide model updates that are consistent with the robot model and hence we can directly use the $AO^{*}$ search here.

Unfortunately, if we used the standard $AO^{*}$ search, it will not produce a conditional explanation that contains this ``less robust'' explanation as one of the potential branches in the conditional explanation. 
This is because, in the above example, if the human had said that the path was free, the robot would need to revert to the original conformant explanation. Thus the cost of the subtree containing this solution will be higher than the one that only includes the original conformant explanation.

To overcome this shortcoming, we can use a discounted version of the $AO^{*}$ search where the cost contributed by a pair of AND successors is calculated as --

\vspace{-15pt}
\begin{equation*}
min(\textit{node1.h\_val, node2.h\_val}) + \gamma * max(\textit{node1.h\_val, node2.h\_val})
\end{equation*}

\noindent where node1 and node2 are the successor nodes and node1.\textit{h\_val}, node2.\textit{h\_val} are their respective $h$-values. Here $\gamma$ represents the discount fact and controls how much the search values short paths in its solution subtree. When $\gamma = 1$, the search becomes standard $AO^{*}$ search and when  $\gamma = 0$, the search myopically optimizes for short branches (at the cost of the depth of the solution subtree). The rest of the algorithm stays the same as the standard $AO^{*}$ search. 

\subsubsection{Anytime Explanations}

Both the algorithms discussed above can be computationally expensive, in spite of the compilation trick to reduce the set of 
possible models to two representative models. 
However, as we did previously with MCEs, we can also aim for an approximate 
solution by relaxing the minimality requirement of explanation to achieve much shorter explanation generation time when required. 
For this we use an anytime depth first explanation generation algorithm. Here, for each state, the successor states include all the nodes that can be generated by applying the model edit actions on all the known predicates and two possible successors for each possible condition -- one where the condition holds and one where it does not. 
Once the search reaches a goal state (a new model where the target plan is optimal throughout its completion set), it queries the human to see if the assumptions it has made regarding possible conditions hold in the human mental model (the list of model updates made related to possible conditions). 
If all the assumptions hold in the human model, then we return the current solution as the final explanation (or use the answers to look for smaller explanations), else continue the search after pruning the search space using the answers provided by the human. 
Such approaches may also be able to facilitate iterative presentation of
model reconciliation explanations to the human.
\section{Model-Free Explanations}
\index{Model-Reconciliation!Model-Free Explanations}%

\unsure{The previous section focused on cases where there still exists an estimate (incomplete though) of the human's knowledge about the robot. The conciseness of the explanation generated using this method would still rely on how accurate and complete that estimate is. In many cases the robot may start with no such estimate. While technically one could start with an incomplete model that will cover the space of all possible models representable using the current set of action and fluent set, it is unlikely that such a general human mental model would lead to concise explanations}. One possible alternative is to consider a plausible assumption about the human model (such as assuming its an abstraction of the robot model or even that the human is completely ignorant of the robot model). Such assumptions are more effective when the user is ready to interact with the system while receiving the explanation. A slightly different possibility might be to learn the human's model. This approach may be particularly well suited when it is possible to assume that it is possible for the system to collect data from a set of people about their ability to make sense of system decisions and the final end user would have similar mental model to the set of users. Note that this assumption is similar to the one made in model-free explicable plan generation in Chapter \ref{ch03} (Section \ref{section:model-free}) and we will make use of a similar approach to learn simple labeling models that are able to predict whether a human would find a given explanation satisfactory.

Unlike the setting in learning for explicability, we will assume the agent now also has access to a set of explanatory messages $\mu$. Note that such messages only contain information about the original robot model and are independent of what the human may or may not know. So assuming a STRIPS model of representation for explanation, this could contain messages about what propositions are or aren't part of the robot model. As such $\mu$ could be derived from $\Gamma$ (i.e the mapping from the model to a set of propositions) discussed in earlier chapter. %now explicitly including information on what is missing from the model or including costly composite messages that convey the full list of fluents that appear in a certain action definition components. For example, there could be a message that says the only preconditions of action $a_i$ are $p_1,..,p_k$. Such composite messages could help reduce the space of possible messages while at the same time increasing the cost of individual explanations.

Once we have access to such a set of messages our goal is now quite similar to the one described in chapter \ref{ch03}, as we would want to learn a labelling model over the steps in the plan. To keep the discussion simple, we will assume the labeling model here merely captures the fact whether the user find the plan step explicable or not (and will not worry about the exact task label they may associate with the step). So the labeling function we use in this context will look something like

\[ \mathcal{L}_{\pi^R}: \mathbb{F} \times 2^{\mu} \rightarrow \{0, 1\}\]

\unsure{Where $\mathbb{F}$ is the space of features for the plan steps. Note that unlike previous learned labeling model, here the function actually captures the fact that the human's expectation could shift once she is given addition information about the robot model. Now we can again take individual labels of the steps and calculate an explicability score for the entire plan using a function $f_{exp}$, where $f_{exp}$ returns 1 if all the steps are labeled as explicable, which means, according to the labeling model, the plan is perfectly explicable.
Thus the learned model labeling model allows us to evaluate the explicability of the plan at each modified model. This means that we can now run the model space search even without having access to the original human mental model $\mathcal{M}^R_h$. Here the possible model space is represented by the space of possible message sets the robot can provide. Each possible modified model in the original scenario should map to at least one message set in this new setting. Instead of running an optimality check, we can check the labeling model to see if for a given set of messages every step in the plan would be labeled as being explicable. In this new setting, one could calculate an approximation of MCE for a plan $\pi_R$ as follows}

\[ \argmin_{m \subseteq \mu} C(m) ~\textrm{such that} ~f_{exp}(\pi_R, m) = 1\]

\unsure{Where $C$ gives the cost associated with the set of messages $m$. The formulation looks for the cheapest set of messages that would render the entire plan explicable. One could similarly use such labeling models to also try extracting explanations that meet the requirements of other forms of explanations. Of course, the effectiveness of this method relies on the ability to learn good labeling models.  Instead of making the explicability score a hard constraint, it can also be made part of the objective function.} 
%\todo{Property - Condition under which they are equivalent}.\\

\section{Assuming Prototypical Models}
\label{helm}
\index{Model-Reconciliation!Abstraction}%
\index{Model-Reconciliation!Assuming Prototypical Models}%
\index{Explanation!Abstraction}%
All the previous discussion focused on cases where we handled lack of information about human model, by either trying to use uncertain model or learning simple model alternatives. But in many cases such uncertain models may not be available or the system may not have previous data or be able to interact with the human for a long enough time to learn useful models. An alternative possibility might be to assume a simple prototypical model for the human and then use it to generate the explanation. In this section, we will consider such a possibility, specifically we will look at cases where we can perform model reconciliation by assuming human model is some abstraction of the robot model (which includes the possibility that its just an empty model). We will look at the kind of explanatory queries such assumptions can support and how we could extend them to cover all the cases.

In particular, we will consider cases where the human model may be considered as a state abstraction of the original models. We will describe abstraction operations in terms of transition systems of models. Formally a transition system $\mathcal{T}$ corresponding to a model $\mathcal{M}$ can be represented by a tuple of the form $\mathcal{T} = \langle S, L, T, s_o,S_g\rangle$, where $S$ is the set of possible states in $\mathcal{M}$, $L$ is the set of transition labels (corresponding to the action that induce that transition), $T$ is the set of possible labeled transitions, $s_0$ is the initial state and $S_g$ is the set of states that satisfies the goal specified by $\mathcal{M}$.

\begin{definition}
%{\em{
\index{Propositional Abstraction Function}%
A \textbf{propositional abstraction function} $f_{\Lambda}$ for a set of propositions $\Lambda$ and state space $S$, defines a surjective mapping of the form $f_{\Lambda}: S \rightarrow X$, where $X$ is a projection of $S$, such that for every state $s \in S$, there exists a state $f_{\Lambda}(s) \in X$ where $ f_{\Lambda}(s) = s \setminus \Lambda$.
%}
\end{definition}

\begin{definition}
\label{model_abs}
\index{State Abstraction}%
%{\em %\textbf{Abstract Model}\\
For a planning model $\mathcal{M} = \langle F, A, I, G\rangle$ with a corresponding transition system $\mathcal{T}$, a model $\mathcal{M}' = \langle F', A', I', G'\rangle$ with a transition system $\mathcal{T}'$ is considered an \textbf{abstraction of $\pmb{\mathcal{M}}$} for a set of propositions $\Lambda$, if for every transition $s_1 \xrightarrow{a} s_2$ in $\mathcal{T}$ corresponding to an action $a$, there exists an equivalent transition  $f_{\Lambda}(s_1) \xrightarrow{a'} f_{\Lambda}(s_2)$ in $\mathcal{T}'$, where $a'$ is part of the new action set $A'$. %\\
%}
\end{definition}

We will slightly abuse notation and extend the abstraction functions to models and actions, i.e in the above case, we will have $\mathcal{M'} \in f_{\Lambda}(\mathcal{M})$ (where $f_{\Lambda}(\mathcal{M})$ is the set of all models that satisfy the above definition for the set of fluents $\Lambda$) and similarly we will have $a' \in f_{\Lambda}(a)$.
As per Definition \ref{model_abs}, the abstract model is {\em complete} in the sense that all plans that were valid in the original model will have an equivalent plan in this new model. We will use the operator $\sqsubset$ to capture the fact that a model $\mathcal{M}$ is less abstract than the model $\mathcal{M}'$, i.e if $\mathcal{M} \sqsubset \mathcal{M}'$ then there exist a set of propositions $\Lambda$ such that $\mathcal{M}' \in f_{\Lambda}(\mathcal{M})$.

In particular, we will focus on abstractions formed by projecting out a subset of propositional fluents. Where for a given model $\mathcal{M} = \langle F, A, I, G\rangle$ and a set of propositions $\Lambda$,  we define the abstract model to be  $\mathcal{M}' = \langle F',  A', I', G'\rangle$, where  $F' = F - \Lambda$, $I' = f_{\Lambda}(I)$, $G' = f_{\Lambda}(G)$ and for every $a \in A$ (where $a = \langle\textrm{pre}(a),\textrm{adds}(a),\textrm{dels}(a)\rangle$) there exists $a' \in A'$, such that $a' = \langle \textrm{pre}(a)\setminus\Lambda, \textrm{pre}(a)\setminus\Lambda,\textrm{adds}(a)\setminus\Lambda,\textrm{dels}(a)\setminus\Lambda\rangle$.

Thus if the system has information on some subset of fluents that the user may be unaware of, then we can approximate the user model by using the abstraction of the model by setting $\Lambda$ to be these fluents. As we will see even if the approximation is more pessimistic than the true model, we will still be able to generate the required explanation. Notice also that here if we set $\Lambda$ to $P$, we essentially get an empty model that can be thought of as representing the model of a user who isn't aware of any of the details about the task, which effectively means the user thinks all possible plans are feasible.

Now going back to the question of explanation, let us consider cases where the user is raising explicit contrastive queries of the form ``\textit{Why P and not Q?}", where `P' is the current plan being proposed by the robot and `Q' is the alternate plan (or set of plans) being proposed by the human. If the assumption holds that the human's model is a complete abstraction then we don't need to provide any information to justify P since it should be valid in the human model. So our focus would be to address the foil plans raised by the user. Now consider the cases where the foils raised by the human are in fact infeasible in the robot model. Then as part of the explanation should reveal parts of the robot model that will help show the infeasibility of the human plan. Thus explanation here can be characterized by the set of fluents that was previously missing from the human model and whose introduction into the human model will help her understand the planning problem better, and thus model reconciliation would leave the human model less abstract than it was before.

Now what happens if the assumption about the human's model being an abstraction of the robot model is not met? Let us assume the system is aware of the set of fluents that the human may not be completely aware of, but now instead of them simply missing the human may hold incorrect beliefs about them. Interestingly, we can still use the same approach discussed above to refute the foils. That is you can still identify the set of fluents to refute the foils as if the user's model was a complete abstraction. But now we need additional mechanisms to provide explanations about the validity of the plan. One simple approach would be to allow users to ask specific questions about the steps in the plan they assume are invalid and provide explanation for those queries.

%\subsection{Addressing Inferential Differences}
\unsure{In addition to allowing for a reasonable assumption about the user's model, the use of abstractions can also provide us with the tools to address the human's limited inferential capabilities. Note that here the explanations can be provided at the minimal level of abstraction needed to explain the foils, i.e., only introduce the minimal set of fluents needed to address the user queries. Thus the inferential overhead on the human's end to make sense of the explanation is minimized. 
Furthermore, if the human's confusion about the plan is coming from limited inferential capabilities as opposed to a difference in knowledge regarding the task, we can still provide explanations by following the steps described in this section. In this case, the user is exposed to an abstract version of the model and asked to re-evaluate the plan and foils with respect to this simplified model.}
\section{Bibliographic Remarks}
The first work to consider model reconciliation in the presence of uncertain models was \cite{sreedharan2018handling}, and the form of the uncertain models used by the work was adapted from earlier work in robust planning \cite{nguyen2017robust}. These annotated models can also be used to provide explanations for multiple users, by combining the mental model of each users into a single uncertain model. 
The $AO^*$ search described in the chapter is adapted from \cite{nilsson2014principles} introduced to solve planning problems where the solution takes the conditional plans that can be represented as acyclic graphs.
The method for model-free model reconciliation was introduced in \cite{modelfree} for MDP settings, for the discussion in this chapter we adapted it to classical planning problems. The abstraction based explanation framework called HELM was first proposed in \cite{sreedharan2018hierarchical}, and then later adapted to partial foil specifications and explaining unsolvability in \cite{sreedharan2019can}. \todo{\cite{sreedharan2021using} builds on these earlier works and provides a clearer discussion on connections between model reconciliation and HELM style explanation. 
This paper also presents a user study testing the effectiveness of abstraction to reduce cognitive load on the explainee}. Instead of assuming a set of fluents to build abstractions on, these works generally assume access to a lattice of abstract models, which might also contain multiple maximally abstract models. This means in these setting they also have to perform an additional localization operation to identify possible models that could correspond to the human. These lattices could also be used to encode information about the complexity of concepts there by also allowing them to provide simplest explanations for a given query. 
\clearpage
                % bibliography using Author-Year

\chapter{Balancing Communication and Behavior}
\label{ch_balance}
In the previous sections, we considered how in human robot teaming scenarios, the robot behavior influences and is influenced by the human's mental model of the robot. We have been quantifying some of the interaction between the behavior and human's model in terms of three interpretability scores, each of which corresponds to some desirable property one would expect the robot behavior to satisfy under cooperative scenarios. With these measures defined one of the strategies the robot can adopt is to specifically generate behavior that optimizes these measures. Though for one of those measures, we also investigated an alternate strategy, namely to use communication to alter the human's mental model to improve explicability of the current plan. 
Interestingly these strategies are not mutually exclusive, they can in fact be combined to capitalize on the individual strengths of each to create behaviors that are unique. In this chapter, we will start by focusing on how one can combine explicable and explanation generation and will look at a compilation based method to generate such plans. 
\unsure{In general, communication is a strategy we can use for the other two measures as well. As such, we will end this chapter with a brief discussion on how we could use the idea of combining communication and selective behavior generation to improve legibility and predictability scores.}
%We will end this chapter by discussing how one could extend these ideas for other behavioral measures.

\section{Modified USAR Domain}
\label{usar}
\index{USAR}%
\index{Modified USAR}
\begin{figure*}%[!th]
\centering
\includegraphics[width=0.95\textwidth]{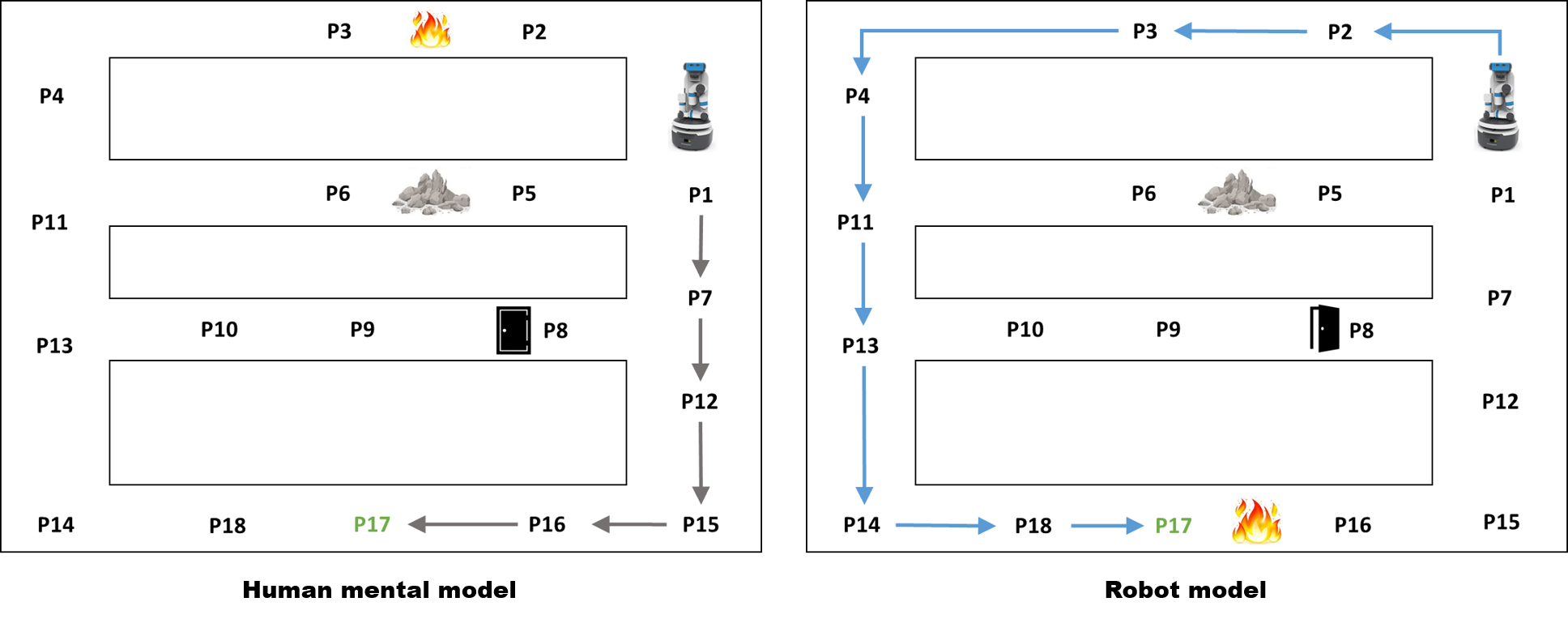}
\caption{
Illustration of the robot model and the corresponding mental model of the human. 
The robot starts at P1 and needs to go to P17. The human incorrectly believes that the path from P16 to P17 is clear and the one from P2 to P3 is blocked due to fire. Both agents know that there is movable rubble between P5 and P6 which can be cleared with a costly \texttt{clear\_passage} action. Finally, in the mental model, the door at P8 is locked while it is unlocked in the model for the robot which cannot open unlocked doors.}
\label{ch_bal:fig:1}
%\vspace{-10pt}
\end{figure*}

A typical Urban Search and Rescue (USAR) 
scenario consists of an autonomous robot deployed to a disaster scene with an external commander who is monitoring its activities. 
Both agents start with the same model of the world (i.e the map of the building before the disaster) but the models diverge over time since the robot, being internal to the scene, has access to updated information about the building.
This model divergence could lead to the commander incorrectly evaluating valid plans from the robot as sub-optimal or even unsafe. One way to satisfy the commander would be to communicate or explain changes to the model that led the robot to come up with those plans in the first place.

% presents an ideal testbed for research on \exact~(EA) planning as it looks at cases where the decision to follow suboptimal or in-executable plan can be potentially disastrous, yet limitations in communications capability could prevent the agents from providing detailed explanations. 

Figure \ref{ch_bal:fig:1} illustrates a scenario where the robot needs to travel from P1 to its goal at P17. 
The optimal plan expected by the commander is highlighted in grey in their map
and involves the robot moving through waypoint P7 and follow that corridor to go to P15 and then finally to P16. The robot knows that it should in fact be moving to P2 -- its optimal plan is highlighted in blue. 
This disagreement rises from the fact that the human incorrectly believes that the path from P16 to P17 is clear while that from P2 to P3 is blocked.

If the robot were to follow the explanation scheme established in Chapter \ref{ch05}, it would stick to its own plan and provide the following explanation:

\begin{lstlisting}[mathescape]
> remove$\textrm{-}$(clear p16 p17)$\textrm{-}$from$\textrm{-}$I
    (i.e. Path from P16 to P17 is blocked)
> add$\textrm{-}$(clear p2 p3)$\textrm{-}$to$\textrm{-}$I 
    (i.e. Path from P2 to P3 is clear)
\end{lstlisting}

If the robot were to stick to a purely explicable plan then it can choose to use the passage through P5 and P6 after performing a costly \texttt{clear\_passage} action (this plan is not optimal in either of the models).
\section{Balancing Explanation and Explicable Behavior Generation}
\index{Balanced Planning}%
\index{Balanced Plan}%
Throughout this chapter, our focus would be mainly on how one could {\em balance} selective behavior generation with communication in the context of explicability. In the case of explicable plan generation, we focus on generating behavior that aligns with observer expectations. 
\unsure{This process is inherently limited by what the robot can perform and one can not always guarantee that the robot can generate plans with high explicability score,  let alone perfectly explicable ones.} 
Moreover under this scheme, the robot is forced to perform suboptimal plans that may be quite far from the best course of action prescribed by the robot's model for the given task.

On the other hand, under explanation, the agent tries to identify the appropriate model information that would help convince the observer of the optimality of given robot plan, and thus help improve the explicability of the plan in the updated human model. Under this scheme, the robot is free to choose a plan that is optimal under its model but now needs to meet the overhead of communicating the explanation, which depending on the scenario could be quite substantial.

The goal of this section is to discuss how one could effectively combine the strengths of these two disparate methods. We will achieve this by folding in the how difficult it is to explain a plan into the plan generation process. Thereby allowing the robot to explicitly reason about the gains made by the plan in terms of explicability with the cost of the plan and the cost of communicating the explanations. We will call such plans as {\em Balanced Plans}.

In the case of explicability, balanced planning consists of optimizing over the following objectives
\begin{enumerate}
    \item \textbf{Plan Cost $C(\pi)$}: The first objective is the cost of the plan that the robot is going to follow.
    \item \textbf{Communication Cost $\mathcal{C}(\mathcal{E})$}: The next objective is the cost of communicating the explanation. Ideally this cost should reflect both the cost accrued at the robot's end (corresponding to the cost of actions that need to be performed to communicate the information) and a cost relating to the difficulty faced by the user to understand the explanation.
    \item \textbf{Penalty of Inexplicability ($C_{\mathcal{IE}}(\pi, \mathcal{M}^R_h + \mathcal{E})$)}: This corresponds to the cost related to the inexplicability experienced by the human for the current plan in their updated mental model. \unsure{We will assume this to be proportional to the inexplicability score of the plan after receiving the explanations $\mathcal{E}$}.\index{Penalty of Inexplicability}%
\end{enumerate}
While in the most general case generating a balanced plan would be a multi-objective optimization, for simplicity, we will assume that we have weight parameters that allow us to combine the individual costs into a single cost. Thus the cost of a balanced plan (which includes both an explanation and a plan), would be given as
\[C((\mathcal{E}, \pi)) = \alpha * C(\pi) + \beta * \mathcal{C}(\mathcal{E}) + \gamma * C_{\mathcal{IE}}(\pi, \mathcal{M}^R_h + \mathcal{E})\]

Thus balanced planning would ideally consist of selecting the plan that minimizes this cost.
In the USAR scenario considered in this chapter, assuming all the component weights are one and all action except \texttt{clear\_passage} is twice the cost, and approximating the inexplicability score $\mathcal{IE}$ to be just exponential of the cost difference. If we set the cost of communicating a single piece of model information to be high (say 50), then the optimal balanced plan would be the most explicable plan, namely the one involving the agent moving through the rubble filled path. Depending on the weights and the exact cost involved the same formulation could give rise to wildly different behaviors.
\unsure{We will refer to the plan that is optimal with respect to the above objective function as \textbf{optimal balanced plans}.\index{Balanced Planning!Optimal Balanced Plans}%
}
Though in some scenarios, we may want to focus on more constrained versions of the problem. In particular, two versions that would be interesting to consider would be

\begin{enumerate}
    \item \textbf{Perfectly Explicable Plans:} 
    \index{Balanced Planning!Perfectly Explicable Plans}%
The first special case is one where we restrict ourselves to cases where the plan is perfectly explicable, i.e., optimal, in the resultant model (i.e. after receiving the explanation). Therefore in our objective we can ignore the inexplicability penalty term and the optimization objective becomes.
    \[\textrm{min}~ \alpha * C(\pi) + \beta * \mathcal{C}(\mathcal{E})\]
    \[\textrm{subject to}~ C_{\mathcal{IE}}(\pi, \mathcal{M}^R_h + \mathcal{E}) = 0\]
    %Within the literature this was the first type of balanced plan studied.
    Note that while the plan may be optimal in the human's updated model, the plan need not be optimal for the robot. %Which brings us to the next group of behavior. 
In the example considered, an example for perfectly explicable plan would be again the plan through the rubble, but now includes an explanation that the path from P16 to P17 is blocked due to fire.
    \item \textbf{Perfectly Explicable Optimal Plans:} 
    \index{Balanced Planning!Perfectly Explicable Optimal Plans}%   
\unsure{In this case, we constrain ourselves to identifying not only plans and explanations that will ensure perfect explicability, but we also try to ensure that the plans are optimal in the robot model.
   Such methods can be particularly effective when the cost of communicating the explanation is much smaller than the cost of the plan itself. 
   } 
%    Note that ensuring the plan  carries an inherent bias that robot's task level actions are always more expensive than communication, which need not be true). 
    Thus the objective in this case becomes just 
    \[\textrm{min}~ \mathcal{C}(\mathcal{E})\]
    \[\textrm{subject to}~ C_{\mathcal{IE}}(\pi, \mathcal{M}^R_h + \mathcal{E}) = 0\]
    \[\textrm{and}~ C(\pi, \mathcal{M}^R) \in C^*_{\mathcal{M}^R}\]
    Interestingly, this involves identifying the optimal plan in the robot's model with the cheapest MCE (in terms of the explanation cost). In the USAR scenario, an example for perfectly explicable optimal plan would be the plan through P4, but now you need to explain that the paths from P16 to P17 and from P2 to P3 are actually clear.
\end{enumerate}
In the following sections, we will see how one could generate these balanced plans.
\subsection{Generating Balanced Plans}
    \index{Balanced Planning!Model Space Search}%
    \index{Balanced Planning!Explanatory Actions}%
    \index{Balanced Planning!Epistemic Effects}%
    \index{Self Explaining Plans}%
Now to actually generate such plans, one could use modifed versions of the model space search discussed in chapter \ref{ch06}. Specifically if we are limiting ourselves to plans that are optimal in the human mental model, then for each possible model in the search space we could consider the space of possible optimal plans and then select the one that best meets our requirement (which would be executability in the robot model). 
Since this is an optimization problem one would have to keep repeating this step till the stopping criteria is met.
In this case, we need to keep repeating this till we find a model where there is at least one plan that is both optimal for the updated model and the original robot model. 
\unsure{Alternatively, one could also consider converting balancing into a single a planning problem.
%This involve generating plans that contain explanatory actions as part of the 
%planning model. 
%One of the main challenges of compiling an balanced planning problem to a traditional planning problems is to allow for a way to handle the identification of model updates and to account for the effect of these model updates on the user's expectation. 
One way to go about this would be by assuming that the robot has access to a set of communicative actions called {\em explanatory actions} through which it can deliver the required explanations to the human.} 
These actions can, in fact, be seen as actions with epistemic effects in as much as they are aimed towards modifying the human mental model (knowledge state). 
This means that a solution to a balanced planning problem can be seen as {\em self-explaining plans}, in the sense that some of the actions in the plan are aimed at helping people better understand the rest of it.

Inclusion of explanatory actions puts balanced planning squarely in the purview of epistemic planning, but the additional constraints enforced by the setting allow us to leverage relatively efficient methods to solve the problem at hand. These constraints include facts such as: all epistemic actions are public, modal depth is restricted to one, modal operators only applied to literals, for any literal the observer believes it to be true or false and the robot is fully aware of all of the observer beliefs. 

Model updates in the form of epistemic effects of communication actions also open up the possibility of other actions having epistemic {\em side effects} as well. The definition of balanced planning makes no claims as to how the model update information is delivered. It is quite possible that actions that the agent is performing to achieve the goal (henceforth referred to as task-level actions to differentiate it from primarily communication actions) itself could have epistemic side-effects. This is something people leverage to simplify communication in day to day lives -- e.g. one might avoid providing prior description of some skill they are about to use when they can simply demonstrate it. The compilation of balanced planning into a single planning problem allows us to capture such epistemic side effects. 
By the same token the compilation can also capture task level constraints that may be imposed on the communication actions.

\subsubsection{Compilation to Classical Planning}
\label{compilation}
    \index{Balanced Planning!Compilation}%

To support such self-explaining plans, we can adopt a formulation that is similar to ones popular in epistemic planning literature to compile it into classical planning formalisms. In our setting, each explanatory action can be viewed as an action with epistemic effects. One interesting distinction to make here is that the mental model now not only includes the human's belief about the task state but also their belief about the robot's model. This means that the planning model will need to separately keep track of (1) the current robot state, (2) the human's belief regarding the current state, (3) how actions would affect each of these (as humans may have differing expectations about the effects of each action) and (4) how those expectations change with explanations.

Given the human and robot model, $\mathcal{M}^R$ and $\mathcal{M}^R_h$, we will generate a new planning model $\mathcal{M}_{\Psi} = \langle F^{\Psi},A^{\Psi},I^{\Psi},G^{\Psi}, C_{\Psi} \rangle$ as follows $F^{\Psi} = F \cup F_{\mathcal{B}} \cup F_{\mu} \cup \{ \mathcal{G}, \mathcal{I}\} $, where $F_\mathcal{B}$ is a set of new fluents that will be used to capture the human's belief about the task state and $F_{\mu}$ is a set of meta fluents that we will use to capture the effects of explanatory actions and $\mathcal{G}$ and $\mathcal{I}$ are special goal and initial state propositions. 
We will use the notation $\mathcal{B}(p)$ to capture the human's belief about the fluent $p$. We are able to use a single fluent to capture the human belief for each (as opposed to introducing two new fluents $\mathcal{B}(p)$ and $\mathcal{B}(\neg p)$) as we are specifically dealing with a scenario where the human's belief about the robot model is fully known and human either believes each of the fluent to be true or false. 

%\note{About observability of believed model}\\
% $F_{\mathcal{B}}$ would contain a new propositional fluent for each element of the original set $F$, i.e, for any $p \in F$, we will have a $\mathcal{B}(p) \in F_{\mathcal{B}}$. 

$F_{\mu}$ will contain an element for every part of the human model that can be changed by the robot through explanations. A meta fluent corresponding to a literal $\phi$ from the precondition of an action $a$ takes the form of $\mu^{+}({\phi}^{\textrm{pre}(a)})$, where the superscript $+$ refers to the fact that the clause $\phi$ is part the precondition of the action $a$ in the robot model (for cases where the fluent represents an incorrect human belief we will be using the superscript ``$-$'').

For every action $a^R = \langle \textrm{pre}(a^R), \textrm{adds}
(a^R), \textrm{dels}(a^R)\rangle \in A^R$ and its human counterpart $a^R_h = \langle \textrm{pre}(a^R_h), \textrm{adds}(a^R_h), \textrm{dels}(a^R_h) \rangle \in A^R_h$, we define a new action 
\[a^{\Psi} = \langle \textrm{pre}(a^{\Psi}), \textrm{adds}(a^{\Psi}), \textrm{dels}(a^{\Psi})\rangle \in A^{\Psi} \in \mathcal{M}_{\Psi}\] 
whose precondition is given as:
{
\small
\begin{multline*}
\textrm{pre}(a^{\Psi}) = \textrm{pre}(a^R_h) \cup \{\mu^{+}({\phi}^{\textrm{pre}(a^R)}) \rightarrow \mathcal{B}(\phi) | \phi \in \textrm{pre}(a^R)\setminus \textrm{pre}(a^R_h)\} \\ \cup 
\{\mu^{-}({\phi}^{\textrm{pre}(a^R)}) \rightarrow \mathcal{B}(\phi) | \phi \in \textrm{pre}(a^R_h)\setminus \textrm{pre}(a^R)\}
\cup \{\mathcal{B}(\phi) | \phi \in \textrm{pre}(a^R_h)\cap \textrm{pre}(a^R)\}
\end{multline*}
}%

The important point to note here is that at any given state, an action in the augmented model is only applicable if the action is executable in robot model and the human believes the action to be executable. Unlike the executability of the action in the robot model (captured through unconditional preconditions) the human's beliefs about the action executability can be manipulated by turning the meta fluents on and off.
The effects of these actions can also be defined similarly by conditioning them on the relevant meta fluent.
In addition to these task level actions (represented by the set $A_{\tau}$), we can also define explanatory actions ($A_{\mu}$) that either add $\mu^+(*)$ fluents or delete $\mu^-(*)$. 
Special actions $a_{0}$ and $a_{\infty}$ that are responsible for setting all the initial state conditions true and checking the goal conditions are also added into the domain model. $a_{0}$ has a single precondition that checks for $\mathcal{I}$ and has the following effects:
\[
\textrm{adds}(a_0) = \{\top\rightarrow p\mid p \in I^R\} \cup \{\top\rightarrow \mathcal{B}(p)\mid p \in I^R_h\} \cup \{\top\rightarrow p \mid p \in F_{\mu^{-}} \}
\]
\[
\textrm{dels}(a_0) = \{\mathcal{I}\}
\]

\noindent where $F_{\mu^{-}}$ is the subset of $F_{\mu}$ that consists of all the fluents of the form $\mu^{-}(*)$. Similarly, the precondition of action $a_{\infty}$ is set using the original goal and adds the special goal proposition $\mathcal{G}$.
\begin{multline*}
\textrm{pre}^(a_{\infty}) = G^R \cup \{\mu^{+}({p}^{G}) \rightarrow \mathcal{B}(p) \mid p \in G^R\setminus G^R_h\} \cup \\
\{ \mu^{-}(p^G) \rightarrow \mathcal{B}(p) \mid p \in G^R_h\setminus G^R\} \cup
 \{\mathcal{B}(p) \mid G^R_h\cap G^R\}
\end{multline*}

Finally the new initial state and the goal specification becomes $I_{\mathcal{E}} = \{\mathcal{I}\}$ and $G_{\mathcal{E}} = \{\mathcal{G}\}$ respectively. To see how such a compilation would look in practice, consider an action $\texttt{(move\_from p1 p2)}$ that allows the robot to move from $\texttt{p1}$ to $\texttt{p2}$ only if the path is clear. The action is defined as follows in the robot model:
\begin{lstlisting}[mathescape]
(:action move_from_p1_p2
    :precondition (and (at_p1) (clear_p1_p2))
    :effect (and (not (at_p1)) (at_p2) ))
\end{lstlisting}

Let us assume the human is aware of this action but does not care about the status of the path (as they assume the robot can move through any debris filled path). In this case, the corresponding action in the augmented model and the relevant explanatory action will be: 
\begin{lstlisting}[mathescape]
(:action move_from_p1_p2
  :precondition (and (at_p1) ($\mathcal{B}$((at_p1))) (clear_p1_p2)
                     (implies ($\mu^{+}_{pre}$(move_from_p1_p2,
                                   (clear_p1_p2))) 
                              ($\mathcal{B}$((clear_p1_p2)))))
  :effect (and (not (at_p1)) (at_p2) 
               (not $\mathcal{B}$(at_p1)) $\mathcal{B}$(at_p2))))
       
(:action explain_$\mu^{+}_{pre}$_move_from_clear 
  :precondition (and)
  :effect (and $\mu^{+}_{pre}$(move_from_p1_p2, (clear_p1_p2))))
\end{lstlisting}

Finally $C_{\Psi}$ captures the cost of all explanatory and task level actions. For now, we will assume that the cost of task-level actions are set to the original action cost in either robot or human model and the explanatory action costs are set according to $C_E$. Later, we will discuss how we can adjust the explanatory action costs to generate desired behavior.

We will refer to an augmented model that contains an explanatory action for each possible model updates and has no actions with effects on both the human's mental model and the task level states as the {\em canonical augmented model}. 
Given an augmented model, let $\pi_{\mathcal{E}}$ be a plan that is valid for this model ($\pi_{\mathcal{E}}(I^{\Psi}) \subseteq G^{\Psi}$). 
From $\pi_{\mathcal{E}}$, we extract two types of information -- 
the model updates induced by the actions in the plan (represented as $\mathcal{E}(\pi_{\mathcal{E}})$) and the sequence of actions that have some effect on the task state (represented as $\mathbb{T}(\pi_{\mathcal{E}})$).
We refer to the output of $\mathbb{T}$ as the task level fragment of the original plan $\pi_{\mathcal{E}}$. 
$\mathcal{E}(\pi_{\mathcal{E}})$ may also contain effects from action in $\mathbb{T}(\pi_{\mathcal{E}})$.
\subsection{Stage of Interaction and Epistemic Side Effects}
    \index{Balanced Planning!Side Effects}%
One of the important parameters of the problem setting that we have yet to discuss is whether the explanation is meant for a plan that is proposed by the system (i.e the system presents a sequence of actions to the human) or are we explaining some plan that is being executed either in the real world or some simulation that the human observer has access to. Even though the above formulation can be directly used for both scenarios, we can use the fact that the human is observing the execution of the plans to simplify the explanatory behavior by leveraging the fact that many of these actions may have epistemic side effects. This allows us to not explain any of the effects of the actions that the human can observe (for those effects we can directly update the believed value of the corresponding state fluent and even the meta-fluent provided the human doesn't hypothesize a conditional effect).\footnote{This means that when the plan is being executed, the problem definition should include the observation model of the human (which we assume to be deterministic). To keep the formulation simple, we ignore this for now. Including this additional consideration is straightforward for deterministic sensor models.}
This is beyond the capability of any of the existing algorithms in this
space of the explicability-explanation dichotomy.

This consideration also allows for the incorporation of more complicated epistemic side-effects wherein the human may infer facts about the task that may not be directly tied to the effects of actions. 
Such effects may be specified by domain experts or generated using heuristics.
Once identified, adding them to the model is relatively straightforward as we can directly add the corresponding meta fluent into the effects of the relevant action. An example for a simple heuristic would be to assume that the firing of a conditional effect results in the human believing the condition to be true, provided the observer can't directly observe the fluents it was conditioned on. For example, if we assume that the robot had an action  $\texttt{(open\_door\_d1\_p3)}$ that had a conditional effect: 

 \begin{lstlisting}[mathescape]
(when (and (unlocked_d1)) (open_d1))
\end{lstlisting}
which says the door will open if it was unlocked.
Then in the compiled model, we can add a new effect to this conditional effect:

 \begin{lstlisting}[mathescape]
(when (and (unlocked_d1)) 
    (and $\mathcal{B}$(open_d1) $\mathcal{B}$(unlocked_d1)))
\end{lstlisting}
which basically says, that if the conditional effect executes the human will believe that the door was unlocked.
Even in this simple case, it may be useful to restrict the rule to cases where the effect is conditioned on previously unused fluents so the robot does not expect the observer to be capable of regressing over the entire plan.

\subsection{Optimizing for Explicability of the Plan}
    \index{Balanced Planning!Optimizing for Explicability}%
    \index{Balanced Planning!Optimality Delta}%
The balancing formulations discussed in this chapter require more than just generating plans that are valid in both robot and updated human model. They require us to either optimize for or ensure complete explicability of plans. For the {\em Perfectly Explicable Plans} and {\em Perfectly Explicable Optimal Plans} formulation, where plans need to be optimal in the human model we will need to enforce this as an additional goal test for the planner while for the most general formulation the inexplicability score could be added as additional penalty to the last action of the plan.

Though to allow for {\em Perfectly Explicable Optimal Plans}, we need to restrict the plans to only ones that are optimal in the robot model. We can generate such robot optimal plans by setting lower explanatory action costs. Before we formally state the bounds for explanatory costs, 
let us define the concept of {\em optimality delta} 
%(\note{Am I reinventing a concept that already exists in planning?})
(denoted as $\Delta\pi_{\mathcal{M}}$) for a planning model, which captures the cost difference between the optimal plan and the second most optimal plan. More formally $\Delta\pi_{\mathcal{M}}$ can be specified as:
%{\small
\begin{multline*}
\Delta\pi_{\mathcal{M}} = \textrm{min}\{ v \mid v \in \mathbb{R}~\wedge\\ \not \exists \pi_1,\pi_2((0 < (C(\pi_1) - C(\pi_2)) < v)\\ \wedge \pi_1(I_{\mathcal{M}}) \models_{\mathcal{M}} G_{\mathcal{M}}  \wedge \pi_2(I_{\mathcal{M}}) \in \Pi^{*}_{\mathcal{M}}\}%\models_{\mathcal{M}} G_{\mathcal{M}}   \}
\end{multline*}
%}
\begin{proposition}
\label{MCE_THEOR}
In a canonical augmented model $\mathcal{M}_{\Psi}$ for the human and robot model tuple $\Psi = \langle \mathcal{M}^R,  \mathcal{M}^R_h\rangle$, if the sum of costs of all explanatory actions is $\leq \Delta\pi_{\mathcal{M}^R}$ and if $\pi$ is the cheapest valid plan for $\mathcal{M}_{\Psi}$ such that $\mathbb{T}(\pi) \in \Pi^{*}_{\mathcal{M}_{\Psi} + \mathcal{E}(\pi)}$, then:

\begin{itemize}
\item[(1)]
$\mathbb{T}(\pi)$ is optimal for $\mathcal{M}^R$
\item[(2)]
$\mathcal{E}(\pi)$ is the MCE for $\mathbb{T}(\pi)$
\item[(3)]
There exists no plan $\hat{\pi} \in \Pi^*_R$ such that MCE for $\mathbb{T}(\hat{\pi})$ is cheaper than $\mathcal{E}(\pi)$, 
i.e. the search will find an the plan with the smallest MCE.  
\end{itemize}
\end{proposition}

First off, we can see that there exists no valid plan $\pi'$ for the augmented model ($\mathcal{M}_{\Psi}$) with a cost lower than that of $\pi$ and where the task level fragment ($\mathbb{T}(\pi')$) is optimal for the human model. 
Let's assume $\mathbb{T}(\pi) \not\in \Pi^*_{\mathcal{R}}$ (i.e current plan's task-level fragment is not optimal in robot model) and let $\hat{\pi} \in \Pi^*_{\mathcal{R}}$. Now let's consider a plan $\hat{\pi}_{\mathcal{E}}$ for augmented model that corresponds to the plan $\hat{\pi}$, i.e,  $\mathcal{E}(\hat{\pi}_{\mathcal{E}})$ is the MCE for the plan $\hat{\pi}$ and $\mathbb{T}(\hat{\pi}_{\mathcal{E}}) = \hat{\pi}$. 
Then the given augmented plan $\hat{\pi}_{\mathcal{E}}$ is a valid solution for our augmented planning problem $\mathcal{M}_{\Psi}$ (since the $\hat{\pi}_{\mathcal{E}}$ consists of the MCE for $\hat{\pi}$, the plan must be valid and optimal in the human model), moreover the cost of $\hat{\pi}_{\mathcal{E}}$ must be lower than $\pi$. This contradicts our earlier assumption hence we can show that $\mathbb{T}(\pi)$ is in fact optimal for the robot model.

Using a similar approach we can also show that no cheaper explanation exists for $\pi_{\mathcal{E}}$ and there exists no other plan with a cheaper explanation.

Note that while it is hard to find the exact value of the optimality $\Delta\pi_{\mathcal{M}}$, it is guaranteed to be $\geq 1$ for domains with only unit cost actions or $\geq (C_2 - C_1)$, where $C_1$ is the cost of the cheapest action and $C_2$ is the cost of the second cheapest action, i.e. $\forall a (C_{\mathcal{M}}(a) < C_2 \rightarrow  C_{\mathcal{M}}(a) = C_1)$. Thus allowing us to easily scale the cost of the explanatory actions to meet this criteria.

\section{Balancing Communication and Behavior For other Measures}
\index{Communication for Legibility}%
\index{Communication for Predictability}%
\todo{
While the majority of the chapter is focused on balancing communication and behavior selection to improve explicability, it should be easy to see that such schemes should be extendable to other behavioral metrics. }

\todo{In the case of legibility, since the original framework itself is focused on communication, it is easy to replace specialized behavior with communication. In particular, for a formulation that dealt exclusively with goal uncertainty, you are effectively trading off a single message (namely the goal of the robot) with the additional effort that needs to be taken by the robot to communicate its eventual goal implicitly. 
Though one could easily consider the more general version of legibility where the uncertainty is over entire models or different model components. In such cases, one could foresee a combination of communication and selective behavior to effectively communicate the actual parameters. In this case, the communication takes a form quite similar to the explanation since we will be communicating model information. The planning objective for a setting that tries to maximize legibility within $k$ steps, would be}

\[C(\langle\mathcal{E}, \pi\rangle) = \alpha * C(\pi) + \beta * \mathcal{C}(\mathcal{E}) + \gamma * P(\theta| \mathbb{M}^R_h + \mathcal{E}, \hat{\pi}^k)\]

Where $\theta$ is the target model parameter being communicated (which can be the full model), $\mathbb{M}^R_h$ the hypothesis set maintained by the human observer, $\hat{\pi}^k$ is the $k$ step prefix of the plan $\pi$. Thus the formulation tries to optimize for the probability that the human would associate with the target model parameter. Similar to the explicability case, one could also formulate the more constrained cases.

Now moving onto predictability, one main difference would be the communication strategy to be used. Predictability, as originally conceived, was meant for cases where the human and robot shares the same mental model, but the human observer is unsure about the future behavior the robot will exhibit. So even if the human has incorrect beliefs or uncertainty about the robot model, the communication strategy here needs to include more than just model information. A useful strategy here could be for the robot to communicate potential constraints on the behavior that the robot could actually follow. A possibility here could be to communicate a partially ordered set of action or state facts, which basically conveys a commitment on the part of the robot that no matter the exact behavior the robot ends up choosing it would be one that involves either performing the actions or achieving facts in the order specified, thereby constraining the possible completions it can follow and thus allowing the rest of the behavior. 
\todo{Now using $\mathcal{E}$ again as the stand-in for the information being conveyed, $\mathcal{M}^R_h +\mathcal{E}$ as the updated human model and applying the constraints specified in the communication, then the objective for planning for balanced plans that maximize predictability in $k$ steps is given as}
\[C(\langle\mathcal{E}, \pi\rangle) = \alpha * C(\pi) + \beta * \mathcal{C}(\mathcal{E}) + \gamma * P(\pi| \mathbb{M}^R_h + \mathcal{E}, \hat{\pi}^k)\]
Where $P(\pi| \mathbb{M}^R_h + \mathcal{E}, \hat{\pi}^k)$ is the probability the human observer associates with the actual plan after observing the prefix $ P(\pi| \mathbb{M}^R_h + \mathcal{E}, \hat{\pi}^k)$. Similar to the earlier formulations, we can also look for constrained versions of the objective, that requires perfect predictability and optimality in robot model.

\todo{Moreover, in both these cases, we can adapt the compilation method discussed earlier to fit these new objectives. In particular, we just need to change the goal-test in the case of constrained version and or the additional penalty calculated from $\mathcal{IE}$ to the respective interpretability measures.}
\section{Bibliographic Remarks}
\index{Epistemic Planning}%
The first work to consider balanced planning was \cite{balancing}, that considers the generation of the {\em Perfectly Explicable Plans}. In particular, they consider a model space search based solution to identify the plan. Unfortunately, in the case of the model space search, the only way to guarantee the explanation generated as part of {\em Perfectly Explicable Plans} is minimal would be to iterate over all the optimal plans in the given model. Thus the method they ended up operationalizing was an approximate version, where they guaranteed that while the plan generated is perfectly explicable, the explanation may be larger than required. The paper also presents some initial user studies to validate the usefulness of such balanced plans.
The classical planning compilation for balanced planning was \cite{exact}. This was also the first work to connect epistemic planning to model reconciliation and the central compilation method was derived from an earlier work to compiling restricted forms of epistemic planning to classical planning \cite{muise2015planning}.
The paper also discusses how all three types of balanced plans (i.e  {\em OPtimal Balanced Plans},  {\em Perfectly Explicable Plans} and,  {\em Perfectly Explicable Optimal Plans}) can be generated.
In terms of balancing for other interpretability measures, while we are unaware of works that consider them in the general form, works like \cite{ppap} have looked at applying these ideas in more specific context. In particular, \cite{ppap} looks at using mixed reality based visualization to improve, explicability, predictability and legibility. The kind of scenarios they considered included a block stacking scenario where they considered use of mixed reality cues to highlight potential blocks that might be used in the plan and thereby improving the identification of the plan and eventual goal of the agent.

\chapter{Explaining in the Presence of Vocabulary Mismatch}
\label{ch07}
\unsure{All previous discussions on model-reconciliation explanations implicitly assume that the robot can  communicate information about the model to the user. This suggests that the human and the robot share a common vocabulary that can be used to describe the model. However, this cannot be guaranteed unless the robots are using models that are specified by an expert. Since many of the modern AI systems rely on learned models, they may use representational schemes that would be inscrutable to most users. So in this chapter, we will focus on the question of generating model reconciliation explanations in the absence of shared vocabulary. Specifically, we will see how one could identify relevant model information in terms of concepts specified by a user. We will describe how the agent could use classifiers learned over such concepts to identify fragments of symbolic models specified in terms of such concepts that are sufficient to provide explanations. We will also discuss how one could measure the probability of the generated model fragments being true (especially when the learned classifiers may be noisy) and also how to identify cases where the user specified concepts may be insufficient.}
\section{Representation of Robot Model}
\index{Vocabulary Asymmetry}%
%One of the key aspects of model reconciliation methods we have covered so far is the assumption that the system can communicate information regarding the model to the end user. 
\unsure{Note that the assumption that the agent can communicate model information in itself entails two separate parts, (a) the model follows a representation scheme that is intuitive or understandable to the end users and  (b) the model is actually represented using factors that makes sense to its end user. 
Model-reconciliation, does not technically require explanation to be carried out in the same terms as the representation scheme used by the agent. 
Instead we could always map the given robot model information into representation schemes that are easier for people to understand. 
In the previous chapters, we assumed that the robot model is already represented using STRIPS like description languages and in this chapter we will look at mapping the original model of the robot (regardless of its current representation scheme) into a STRIPS model.
Such representation schemes are not only expressive (any deterministic goal directed planning problem with finite states and actions can be captured as STRIPS planning problem), but they also have their roots in folk psychology \citep{miller2019explanation}. Thus model information represented in such representation schemes are relatively easier for people to make sense of. 
%This is a natural choice given the roots of the model in folk theory and the fact that, if the planning problem is goal directed, deterministic and the state space is finite, then there must exist a STRIPS representation for the problem. 
%Studies have also shown concepts like precondition/effects of actions and goal are intuitive to people.
}  

A more pressing question is what should be the state factors or action labels that we should use to represent the model. The fact that the model is represented using STRIPS in and of itself wouldn't facilitate effective explanation, if the information about an action's precondition is represented using concepts alien to the end user. We will use the term {\em vocabulary used by the model} to refer to the state fluents  and the action names used by it. Thus access to a shared vocabulary is a prerequisite to facilitating model reconciliation explanations. In this chapter, we will look at cases where this is not a given. We will look at one possible way of collecting vocabulary concepts from the end users and discuss how we could use the collected vocabulary to generate the model or at least identify parts of the model relevant to the explanation. We will look at (a) how to handle cases where the vocabulary items we gathered are insufficient (b) how to ensure that models learned this way actually reflect true model and (c) how to handle cases where the mappings we learn from human concepts to task states may be noisy. \unsure{We will also see how this process of mapping the model into a secondary representation in user defined terms could also give an opportunity to approximate and simplify the model and by extension the explanation itself.}

In this chapter we will focus on acquiring human understandable state variables/factors and assume the action labels are given beforehand. We do this because if the human and the robot have a one to one correspondence between what constitutes an atomic action, it is usually easier to establish their labels/names. Acquiring action labels could become challenging if the agent and the human are viewing the action space at varying levels of abstraction, but we will ignore this possibility for now. 
%Note that the problem of handling vocabulary mismatch in explanation is a relatively new research problem and the discussion presented here only represents 

\section{Setting}
\index{Blackbox Simulators@Blackbox Models}%
\index{Vocabulary Asymmetry! Human Specified Concepts}%
\index{Symbolic Approximations}%

\unsure{For the purposes of discussion, we will look at a slightly different scenario where the robot has access to a deterministic simulator of the task of the form $\mathcal{M}_{\textrm{sim}} = \langle S, A, T, \mathcal{C}\rangle$ (where this simulator effectively acts as the robot model), where $S$ represents the set of possible world states, $A$ the set of actions and $T$ the transition function that specifies the problem dynamics.} The state space $S$ here need not be defined over any state factors and may just be atomic states. The transition function is defined as $T: S \times A \rightarrow S \cup \{\bot\}$, where $\bot$ corresponds to an invalid absorber-state generated by the execution of an infeasible action.
Invalid state could be used to capture failure states that could occur when the agent violates hard constraints like safety constraints. 
Finally, $\mathcal{C}: A \rightarrow \mathbb{R}$ captures the cost function of executing an action in any state where the action is feasible.
We will overload the transition function $T$ to also work on action sequence, i.e., $T(s, \langle a_1,...,a_k\rangle) = T(...T(T(s,a_1),a_2),...,a_k)$.
We will assume the problem is still goal-directed in the sense that the robot needs to come up with the plan $\pi$,
%(henceforth referred to as a plan) 
that will drive the state of the world to a goal state. 
In general we will use the tuple $\Pi_{\textrm{sim}} = \langle I, \mathbb{G}, \mathcal{M}_{\textrm{sim}}\rangle$ to represent the decision making problem, where $I$ is the initial state and $\mathbb{G}$ the set of goal states. 
Similar to the earlier settings, a plan is optimal if it achieves the goal and there exists no cheaper plan that can achieve the goal (where $\mathcal{C}(I, \pi)$ is the total cost of executing the plan $\pi$).

\unsure{Now the challenge before us is to map this blackbox simulator model into a STRIPS model defined using vocabulary that makes sense to the human. 
In particular we want to first collect the set of propositional state fluents that the human uses to define the model $\mathcal{M}^R_h$. 
Before we discuss how to collect such state fluents information, let us take a quick look at the state fluents themselves. For one, how does one connect a model defined over atomic space to a set of propositional fluents? 
The basic convention we will follow here is that each proposition state fluent correspond to a fact about the task, that is either true or false in all states of the given simulator $\mathcal{M}_{sim}$.
For a given set of state fluents, each state in $\mathcal{M}_{sim}$ is mapped to the set of propositional fluents that are true in that state and each propositional fluent can be for our purposes by the subset of states $S$ where that fluent is true.
}

Now lets assume the human observer is ready to list all these fluents (captured by the set $\mathbb{C}$) they believe are relevant to the problem at hand. We can hardly expect the human to be able to list or point out all possible states corresponding to each factor. Instead we will employ a different method to learn the mapping between atomic states and fluents. In this chapter, we will assume that each factor can be approximated by a classifier. Now we can learn such a classifier for a fluent (or concept) in the set $\mathbb{C}$ by asking the user for a set of positive and negative examples for the concept. This means that the explanatory system should have some method of exposing the simulator states to the user. 
A common way to satisfy this requirement would be by having access to visual representations for the states. 
The simulator state itself doesn't need to be an image as long as we have a way to visualize it. The concept list can also be mined from qualitative descriptions of the domain and we can crowdsource the collection of example states for each concept.
Once we have learned such a set of classifiers, we can then use these factors to learn a model of the task and start using it for the explanations.

\subsection{Local Approximation of Planning Model}
\label{ch07:local_approx}
\index{Local Model Approximation}%

In most of the previous discussions in the book the implicit assumption was that the explanation was given with respect to a representation of the entire robot model. This could mean the model being used for explanation is quite complex and large even after the application of simple abstraction methods. \unsure{A strategy quite popular in explaining classification decisions is to explain a given decision with respect to an approximation of the original model that only captures the behavior of the model for a subset of the input space. We can rely on a similar strategy for explanation in sequential decision-making scenarios. Namely, we can focus on a symbolic representation of the task that aims to capture the dynamics of the underlying domain model only for a subset of states.}

More formally, consider a STRIPS model defined over a set of concepts $\mathbb{C}$, $\mathcal{M}^{\mathbb{C}} = \langle \mathbb{C}, A^{\mathbb{C}}, \mathbb{C}(I), \mathbb{C}(\mathbb{G}), \mathcal{C}^{\mathbb{C}}_{\mathcal{S}}\rangle$, where $\mathbb{C}(\mathbb{G}) = \bigcap_{s_g \in \mathbb{G}} \mathbb{C}(s_g)$. We will call this model to be a local approximation of the simulation $\mathcal{M}_{\textrm{sim}} = \langle S, T,A, \mathcal{C}\rangle$ for regions of interest $\hat{S} \subseteq S$, if $\forall s \in \hat{S}$ and $\forall a \in A$, we have an equivalent action $a^\mathbb{C} \in A^\mathbb{C}_{\mathcal{S}}$, such that $a^\mathbb{C}(\mathbb{C}(s)) = \mathbb{C}(T(s,a))$ (assuming $\mathbb{C}(\bot) = \bot$) and  $\mathcal{C}^{\mathbb{C}}_{\mathcal{S}}(a) = \mathcal{C}(a)$.  

Now one of the important aspects of this formulation is the subset of states over which we will define the approximation. A popular strategy used in machine learning approaches is to select a set of input points close to the data point being explained, where closeness of two data points is defined via some pre-specified distance function. This strategy can also be adapted to sequential decision-making settings. In this case, the data point would correspond to the plan (and in some cases also the foils) and the state space of interest can be limited to states close to the ones that appear on the plan or foil execution trace. The distance function can be defined over some representation of the underlying state or use reachability measures to define closeness (i.e. whether the states can be reached within some predefined number of steps). 
\subsection{Montezuma's Revenge}
\index{Montezuma's Revenge}%
\index{Atari}%
\unsure{For the purposes of the discussion} 
%we will use a non-robotic example.} Specifically, 
we will consider the game Montezuma's revenge as the example scenario. The game was first launched in 1984 for Atari gaming systems, and featured a character that needs to navigate various levels filled with traps and enemies while collecting treasures. The game has recently received a lot of attention as a benchmark for RL algorithms and is a domain that requires long term planning to solve. 
\unsure{The simulator here consists of the game engine, which stores the current state and predicts what happens when a specific action is executed}.
Two popular state representation schemes for the game that is used by RL algorithms either involves the use of images of the game state or the RAM state of the game controller corresponding to the game state. Neither of which is particularly easy for a naive user to understand.
\unsure{The human observer could be a user of this game playing system, and as such we will sometimes refer to the human as the user in the text.}
\begin{figure}
\centering
\includegraphics[scale=0.5]{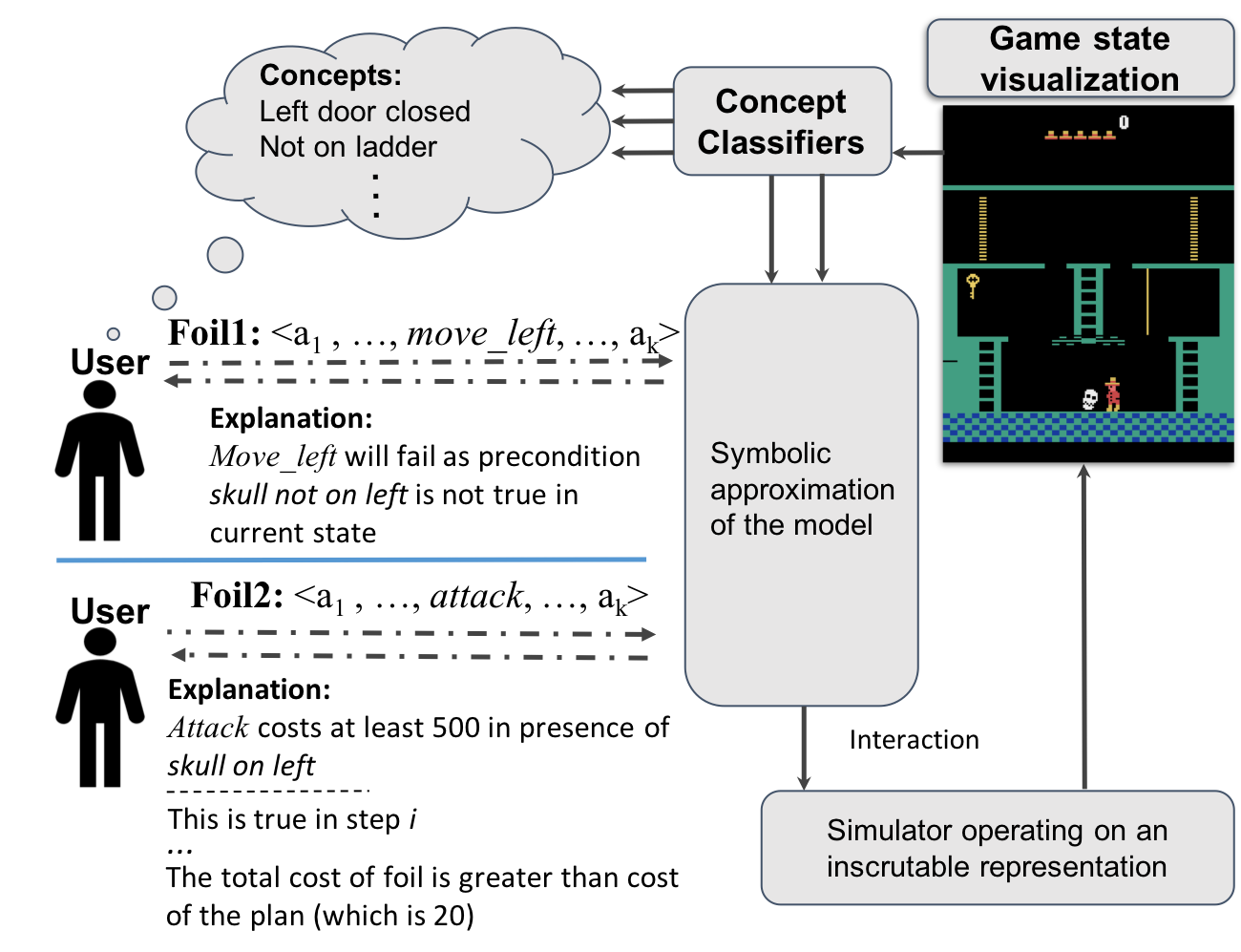}
\caption{An overview of the overall system described in the chapter. The screenshot is from a  game called Montezuma's revenge. In this particular instance, the agent is tasked with making its way to the key on the left side of the screen while avoiding the skull on the ground.}
\label{ch07:picpic}
\end{figure}
In this chapter we will mostly focus on the first level of the game that is pictured in \ref{ch07:picpic}. 
In particular, we will focus on the task of getting the agent from the point visualized in the image to the key it needs to pick up. A successful plan for this task would involve the agent making its way to the lowest platform and them moving towards the skull, jumping over it once it close enough and then climbing up the ladder before going for the key. Now if we were to come up with a set of concepts to describe the agent, many of them would involve the actual position of the agent (like if the agent is on the leftmost ladder or if the agent is on the topmost platform etc.), whether the agent has possession of the key, where the agent is with respect to the skull (especially if the agent is right next to it) etc. One can train a classifier that takes either the image or the RAM state of a specific game state and determines whether the concept is absent or present in that state. Now the goal of many of the methods discussed in this chapter would be to build models that can approximate the dynamics of the game in this level using such concepts. 

\section{Acquiring Interpretable Models}
\index{Learning Symbolic Approximations}%
Now to provide explanations in this setting, the first approach would be to try constructing the model $\mathcal{M}^{\mathbb{C}}$, particularly the action definition as the initial states and goal specification are easier to construct. For now we will assume that we have a perfect classifier for each of the concept listed in the list $\mathbb{C}$. This means that for any state sampled from the set we know the exact list of concepts present in that state. Now we will use the simulator to observe how each action transforms the state and convert it into symbolic representation through the classifiers.

More specifically, we will sample a set of states from the simulator state space (where the samples are limited to ones that meet the desired distance requirement discussed in previous section). For each state sampled, we test whether the actions are executable (i.e it doesn't lead to an invalid state), and for each action we can update the definition of the action $a$ as follows (assuming the sample state is $s_i$ and the state resulting from executing $a$ is denoted as $a(s_i)$)
\begin{equation*}
\begin{split}
    \textrm{pre}(a^c) = \textrm{pre}(a^c) \cap \mathbb{C}(s_i)\\
    \textrm{adds}(a^c) =  \mathbb{C}(a(s_i)) \setminus \mathbb{C}(s_i)\\
    \textrm{dels}(a^c) =  \mathbb{C}(s_i) \setminus \mathbb{C}(a(s_i))\\    
\end{split}
\end{equation*}
Where the original estimate of the action preconditions ($\textrm{pre}(a^c)$) is set to be equal to $\mathbb{C}$. Note that learning using this method makes the specific assumption that each action in the exact local symbolic approximation only contains preconditions that are conjunctions of positive concepts and are free of conditional effects. If the model in fact does not meet these assumptions, we may need to rely on more complex model learning techniques. \footnote{Though if the only assumption not met is the form of the precondition then it can be mapped to a case of missing concepts that could be handled by assuming the exact approximation model meets the same representational assumptions but is defined over a larger set of concepts.}

\unsure{Once such a model is learned we could then leverage many of the explanatory methods discussed in the book to generate explanations (especially the ones making simplifying assumptions about the human model), though one obvious issue that could occur is that the original concept list provided to the explanatory system may be incomplete. That is the fluent list of the exact local approximation (we will represent this as $\mathcal{M}^*$) may be a superset of $\mathbb{C}$. Let the learned model be $\hat{\mathcal{M}}^{\mathbb{C}}$, and for now let us assume that this model is the fix-point of the learning process described above. Then it is easy to see that if the assumptions about $\mathcal{M}^*$ hold (i.e., it is conditional effect free and has only conjunction of positive literals as precondition) then $\hat{\mathcal{M}}^{\mathbb{C}}$, must be a complete abstraction of $\mathcal{M}^*$, in so far as any plan viable in $\mathcal{M}^*$ must also be valid in $\hat{\mathcal{M}}^{\mathbb{C}}$. This is because the above learning method will ensure that the learned preconditions and effects are a subset of the original action precondition and effects.} \unsure{As we have already seen in Chapter \ref{ch06} (Section \ref{model_abs}), this means that even though this is not the exact model, this model suffices to address many specific user queries}. Whenever the user raises a foil that cannot be refuted through $\hat{\mathcal{M}}^{\mathbb{C}}$, then the system can use that as a signal that the current vocabulary is incomplete. 

\section{Query Specific Model Acquisition}
\index{Query Specific Model Acquisition}%
\index{Contrastive Query}%
\index{Foils}%
The previous section talked about how we can acquire a representation of the entire model that can then be used in the standard model reconciliation process. However, this could be a pretty expensive process, especially if the action space is quite large. In cases where the agents are raising specific explanatory queries we may not require the entire model, but rather only learn the parts of the model necessary to answer the specific user query. To illustrate this process, let us take a look at the simplest contrastive explanation setting. Here the human explanatory query consists of asking why the system chose to follow the current plan instead of a set of alternative plans. Specifically in this case, for the  decision-making problem specified by the tuple $\Pi_{sim} = \langle I, \mathbb{G}, \mathcal{M}_{\textrm{sim}}\rangle$ the system identifies a plan $\pi$. 
When presented with the plan, the user of the system responds by raising an alternative plan $\pi_f$ ({\em the foil}) that they believe should be followed instead. 
Now the system would need to explain why the plan $\pi$ is preferred over the foil $\pi_{f}$ in question. The only two possibilities here are that either the foil is inexecutable and hence cannot be followed or it is costlier than the plan in question.

\begin{definition}
The plan $\pi$ is said to be preferred over a foil $\pi_f$ for a problem $\Pi_{sim} = \langle I, \mathbb{G}, \mathcal{M}_{\textrm{sim}}\rangle$, if either of the following conditions are met, i.e.,
\begin{enumerate}
    \item $\pi_f$ {\em is inexecutable}, which means, either (a) $T(I, \pi_f) \not \in \mathbb{G}$, i.e the action sequence doesn't lead to a possible goal state, or (b) the execution of the plan leads to an invalid state, i.e., $T(I, \pi_f) = \bot$.
    %exists a prefix $\hat{\pi}_f$ of $\pi_f$, such that, $T(I, \hat{\pi}_f) = \bot$, i.e the execution of it would lead to an invalid state.
    \item Or $\pi_f$ is {\em costlier} than $\pi$, i.e., $\mathcal{C}(I, \pi) < \mathcal{C}(I, \pi_f)$
\end{enumerate}
\end{definition}

So going back to our montezuma example. Let's assume the plan involves the agent starting from the highest platform, and the goal is to get to the key. The specified plan $\pi$ may require the agent to make its way to the lowest level, jump over the skull, and then go to the key with a total cost of 20. Let us consider a case where the user raises two possible foils that are quite similar to $\pi$, but, (a) in the first foil, instead of jumping the agent just moves left (as in it tries to move through the skull) and (b) in the second, instead of jumping over the skull, the agent performs the \texttt{attack} action (not part of the original game, but added here for illustrative purposes) and then moves on to the key. Now using the simulator, the system could tell that in the first case, moving left would lead to an invalid state and in the second case, the foil is more expensive. \unsure{It may however struggle to explain to the user what particular aspects of the state or state sequence lead to the invalidity or suboptimality as it cannot directly expose relevant model information. Instead, as in earlier parts it would need to map it to a specific symbolic model and expose information about that model. However in this case, rather than first learning a full symbolic model and then identifying relevant model components to provide, we can directly try to learn the model parts.}
%Even efforts to localize parts of its own internal state representation for possible reasons by comparing the foil with similar states where actions are executable or cheaper may be futile, as even what constitutes similar states as per the simulator may be conceptually quite confusing for the user. This scenario thus necessitates the use of methods that are able to express possible explanations in terms that the user may understand. 
To capture scenarios like the one mentioned above, we will allow for a class of richer symbolic, namely one that allows for cost functions that are conditioned on the state in addition to the action. This means the cost function will take the form $\mathcal{C}_{\mathcal{S}}: 2^{F} \times A_{\mathcal{S}} \rightarrow \mathbb{R}$ to capture the cost of valid action executions in a specific state. Internally, such state models may be represented using conditional cost models. In this discussion, we won't try to reconstruct the exact cost function but will rather try to estimate an abstract version of the cost function.\index{State-dependent Cost Functions}%

\unsure{As a quick aside, most of the discussion in this section focuses on generating explanation to refute the alternate plan and not really on explaining why the current plan works or has the cost assigned to it. Since the human can observe the robot and we had previously assumed that there exists mechanism to visualize the simulator states, in theory the robot could just demonstrate the outcome of executing the plan. One could also adapt the methods we discuss for refuting the foil to answer more specific question the human user might have about the robot, for example by learning an abstract cost function for the actions in the plan or by identifying whether a concept the human had in mind is in fact a precondition or not.}

\subsection{Explanation Generation:}

For establishing the invalidity of $\pi_f$, we just need to focus on explaining the failure of the first failing action $a_i$, i.e., the last action in the shortest prefix that would lead to an invalid state (which in our running example is the move-left action in the state presented in Figure \ref{ch07:picpic} for the first foil). We can do so by informing the user that the failing action has an unmet precondition, as per the symbolic model, in the state it was executed in.
Formally
\begin{definition}
\label{exp-inv}
For a failing action $a_i$ for the foil $\pi_f = \langle a_1, ..,a_i,.., a_n\rangle$, $c_i \in \mathbb{C}$ is considered an explanation for failure if $c_i \in pre(a_i)\setminus \mathbb{C}(s_i)$, where $s_i$ is the state where $a_i$ is meant to be executed (i.e $s_i = T(I, \langle a_1,..,a_{i-1}\rangle)$).
%\vspace{-5pt}
\end{definition}

In our example for the invalid foil, a possible explanation 
%could be to inform the user that the agent can only perform move-left 
would be to inform the user that move-left can only be executed in states for which the concept  \texttt{skull-not-on-left} is true; and the concept is false in the given state. 
This formulation is enough to capture both conditions for foil inexecutability by appending an additional goal action at the end of each sequence. The goal action causes the state to transition to an end state and it fails for all states except the ones in $\mathbb{G}$. Our approach to identifying the minimal information needed to explain specific query follows from studies in social sciences that have shown that selectivity or minimality is an essential property of effective explanations.

%Note that for the specific query raised, we don't need to provide the entire approximate symbolic model or even all of the preconditions of the action. 
%Rather we just need to point to a single precondition that was unmet. Thus ensuring that the explanations are selective, which as many studies from social sciences have shown to be an essential property of effective explanations \cite{miller}.
For explaining the suboptimality of the foil, we have to inform the user about $\mathcal{C}_{\mathcal{S}}^{\mathbb{C}}$.
To ensure minimality of explanations, rather than generating the entire cost function or even trying to figure out individual conditional components of the function, we will instead try to learn an abstraction of the cost function $\mathcal{C}_s^{abs}$, defined as follows

\begin{definition}
For the symbolic model $\mathcal{M}^{\mathbb{C}}_{\mathcal{S}} = \langle \mathbb{C}, A^\mathbb{C}_{\mathcal{S}}, \mathbb{C}(I), \mathbb{C}(\mathbb{G}), \mathcal{C}^{\mathbb{C}}_{\mathcal{S}}\rangle$, an abstract cost function $\mathcal{C}_{\mathcal{S}}^{abs}: 2^{\mathbb{C}} \times A^\mathbb{C}_{\mathcal{S}} \rightarrow \mathbb{R}$ is specified as follows
{\small $\mathcal{C}_{\mathcal{S}}^{abs}(\{c_1,..,c_k\}, a) = min\{ \mathcal{C}^{\mathbb{C}}_{\mathcal{S}}(s, a)| s \in S_{\mathcal{M}^{\mathbb{C}}_{\mathcal{S}}} \wedge \{c_1,..,c_k\} \subseteq s\}$]}.
%\vspace{-6pt}
\end{definition}
%\vspace{-5pt}

Intuitively,  {\small $\mathcal{C}_{\mathcal{S}}^{abs}(\{c_1,..,c_k\}, a) = k$} can be understood as stating that {\em executing the action $a$, in the presence of concepts $\{c_1,..,c_k\}$ costs at least k}. 
We can use $\mathcal{C}_{\mathcal{S}}^{abs}$ in an explanation of the form
%Given this specification, we can now describe an explanation of the form

\begin{definition}
For a valid foil $\pi_f = \langle a_1,..,a_k\rangle$, a plan $\pi$ and a problem $\Pi_{sim} = \langle I, \mathbb{G}, \mathcal{M}_{sim} \rangle$, the sequence of concept sets of the form $\mathbb{C}_{\pi_f} = \langle \hat{\mathbb{C}}_1, ...,\hat{\mathbb{C}}_k\rangle$ along with $\mathcal{C}_s^{abs}$ is considered a valid explanation for relative suboptimality of the foil (denoted as $\mathcal{C}_{\mathcal{S}}^{abs}(\mathbb{C}_{\pi_f}, \pi_f) > \mathbb{C}(I, \pi)$), if $\forall \hat{\mathbb{C}}_i \in \mathbb{C}_{\pi_f}$, $ \hat{\mathbb{C}}_i$ is a subset of concepts presents in the corresponding state (where state is $I$ for $i=1$ and $T(I, \langle a_1,...,a_{i-1}\rangle$) for $i > 1$).
% \begin{equation*}
% \begin{split}
%  \hat{\mathbb{C}}_i \subseteq \begin{cases}
%  \mathbb{C}(T(I, \langle a_1,...,a_{i-1}\rangle)), \textrm{If i } > 1\\
%  \mathbb{C}(I)
%  \end{cases}   
% \end{split}
% \end{equation*}
and {\small $\Sigma_{i=\{1..k\}} \mathcal{C}_{\mathcal{S}}^{abs}(\hat{\mathbb{C}}_i, a_i) > \mathcal{C}(I, \pi)$}
\end{definition}

In the earlier example, the explanation would include the fact that executing the action \texttt{attack} in the presence of the concept \texttt{skull-on-left}, will cost at least 500 (as opposed to original plan cost of 20).
%\vspace{-3pt}
\subsection{Identifying Explanations through Sample-Based Trials}
\label{ident}
%Now to the question of identifying the required model parts for the explanatory query while having access to the concept classifiers. 
%In either case, 
For identifying the model parts for explanatory query, we can rely on the agent's ability to interact with the simulator to build estimates. Given the fact that we can separate the two cases at the simulator level, we will keep the discussion of identifying each explanation type separate and only focus on identifying the model parts once we know the failure type.

%. This allows us to keep the discussion focused on how to generate the explanatory content and not worry about details about figuring out the specific failure type.\\
%\note{talk about them separately...}
\subsubsection{Identifying failing precondition:}
%First we will discuss how to explain why the foil is invalid by identifying the missing precondition. 
To identify the missing preconditions, we can rely on the simple intuition that while successful execution of an action $a$ in the state $s_j$ with a concept $C_i$ doesn't necessarily establish that $C_i$ is a precondition, we can guarantee that any concept false in that state cannot be a precondition of that action. 
%This is obvious from the semantics of the models we are considering, where an action $a$ is executable in state $s_j$ only if  $prec_{a} \subseteq \mathbb{C}(s_j)$. 
This is a common line of reasoning exploited by many of the model learning methods.
So we start with the set of concepts that are absent in the the state ($s_{\textrm{fail}}$) where the failing action ($a_{\textrm{fail}}$) was executed, i.e., poss\_pre\_set = $\mathbb{C} \setminus \mathbb{C}(s_{\textrm{fail}})$. We then randomly sample for states where $a_{\textrm{fail}}$ is executable. Each new sampled state $s_i$ where the action is executable can then be used to update the possible precondition set as $\textrm{poss\_pre\_set = poss\_pre\_set}~\cap~ \mathbb{C}(s_i)$.
That is, if a state is identified where the action is executable but a concept is absent then it can't be part of the precondition. We will keep repeating this sampling step until the sampling budget is exhausted or if one of the following exit conditions is met.
(a) In cases where we are guaranteed that the concept list is exhaustive, we can quit as soon as the set of possibilities reduce to one (since there has to be a missing precondition at the failure state). (b) The search results in an empty list. The list of concepts left at the end of exhausting the sampling budget represents the most likely candidates for preconditions. {\em An empty list here signifies the fact that whatever concept is required to differentiate the failure state from the executable one is not present in the initial concept list} $\mathbb{C}$. This can be taken as evidence to query the user for more task-related concepts.

\paragraph{Identifying cost function:} 
\index{Learning Abstract Cost Function}%
Now we will employ a similar sampling based method to identify the right cost function abstraction. Unlike the precondition failure case, there is no single action we can choose but rather we need to choose a level of abstraction for each action in the foil (though it may be possible in many cases to explain the suboptimality of foil by only referrring to a subset of actions in the foil). Our approach here would be to find the most abstract representation of the cost function at each step such that of the total cost of the foil becomes greater than that of the specified plan. Thus for a foil $\pi_f = \langle a_1,...,a_k\rangle$ our objective become
{\small \begin{equation*}
   \begin{split}
       \textrm{min}_{\hat{\mathbb{C}}_1, ...,\hat{\mathbb{C}}_k} \Sigma_{i=1..k} \|\hat{\mathbb{C}}_i\|~ %\Sigma_{i= 1.. k} ||\\
       \textrm{subject to}
       ~\mathcal{C}_s^{abs}(\mathbb{C}_{\pi_f}, \pi_f) > \mathbb{C}(I, \pi)\\
   \end{split} 
\end{equation*}}
For any given $\hat{\mathbb{C}}_i$, $\mathcal{C}_s^{abs}(\hat{\mathbb{C}}_i, a_i)$ can be approximated by sampling states randomly and finding the minimum cost of executing the action $a_i$ in states containing the concepts $\hat{\mathbb{C}}_i$. 
%We can again rely on a sampling budget to decide how many samples to check and we can also choose a sampler that enforces the required locality. 
We can again rely on a sampling budget to decide how many samples to check and enforce required locality within sampler.

\section{Explanation Confidence}
\index{Explanation!Confidence Value}%
The methods discussed above (both for learning full model and model component) are guaranteed to identify the exact model in the limit, the accuracy of the methods is still limited by practical sampling budgets we can employ. 
So this means it is important that we are able to establish some level of confidence in the solutions identified. 
In case of learning full model, there are some loose PAC learning guarantees we could employ, but for learning model components we will strive for a more accurate measure.
To assess confidence, we will follow the probabilistic relationship between the random variables as captured by Figure \ref{graph-model} (A) for precondition identification and Figure \ref{graph-model} (B) for cost calculation. Where the various random variables captures the following facts:
$O_{a}^{s}$ - indicates that action $a$ can be executed in state $s$, $c_i \in p_{a}$ - concept $c_i$ is a precondition of $a$, $O_{c_i}^{s}$ - the concept $c_i$ is present in state $s$, $\mathcal{C}_s^{abs}(\{c_i\}, a) \geq k$ - the abstract cost function is guaranteed to be higher than or equal to k and finally  $O_{\mathcal{C}(s, a) \geq k}$ - stands for the fact that the action execution in the state resulted in cost higher than or equal to $k$.
We will allow for inference over these models, by relying on the following simplifying assumptions -  (1) the distribution of all non-precondition concepts in states where the action is executable is the same as their overall distribution across the problem states (which can be empirically estimated), (3) cost distribution of an action over states corresponding to a concept that does not affect the cost function is identical to the overall distribution of cost for the action (which can again be empirically estimated). 

The first assumption implies that you are as likely to see a non-precondition concept in a sampled state where the action is executable as the concept was likely to appear at any sampled state with the same set of concepts (this distribution is denoted as $P(O_{c_i}^{s}|O_{\mathbb{C}\setminus c_i^{s}})$, where $O_{\mathbb{C}\setminus c_i^{s}}$ represents the other observed concepts). While the second one implies that for a concept that has no bearing on the cost function for an action, the likelihood that executing the action in a state where the concept is present will result in a cost greater than $k$ will be the same as that of the action execution resulting in cost greater than $k$ for a randomly sampled state ({\small $p_{\mathcal{C}(.,a) \geq k}$}). 

For a single sample, the posterior probability of explanations for each case can be expressed as follows: for precondition estimation, updated posterior probability for a positive observation can be computed as  
{\small $P(c_i \in p_{a} | O_{c_i}^{s} \wedge O_{\mathbb{C}\setminus c_i^{s}}\wedge O_{a}^{s}) = (1 - P(c_i \not \in p_{a} | O_{c_i}^{s} \wedge O_{\mathbb{C}\setminus c_i^{s}}\wedge  O_{a}^{s}))$}, where \\
{\small\begin{align*}
 & P(c_i \not \in p_{a} | O_{c_i}^{s} \wedge O_{a}^{s}) = \frac{P(O_{c_i}^{s}|O_{\mathbb{C}\setminus c_i^{s}}) \times P(c_i \not \in p_{a})}{P (O_{c_i}^{s}| O_{a}^{s})} 
 \end{align*}}
 and for the case of cost function approximation
 {\small\begin{align*}
 & P(\mathcal{C}_s^{abs}(\{c_i\}, a) \geq k | O_{c_i}^{s} \wedge O_{\mathbb{C}\setminus c_i^{s}}\wedge  O_{\mathcal{C}(s, a) \geq k}) =\\ &\frac{P(\mathcal{C}_s^{abs}(\{c_i\}, a) \geq k)}{P(\mathcal{C}_s^{abs}(\{c_i\}, a) \geq k)) + p_{\mathcal{C}(.,a) \geq k} * P(\neg \mathcal{C}_s^{abs}(\{c_i\}, a) \geq k))} 
 \end{align*}}
 The distribution used in the cost explanation, can either be limited to distribution over states where action $a_i$ is executable or allow for the cost of executing an action in a state where it is not executable to be infinite. The full derivation of these formulas can be found in the paper \cite{blackbox}.

\section{Handling Uncertainty in Concept Mapping}
\index{Vocabulary Asymmetry! Noisy Concepts}%
\index{Noisy Concepts}%

\begin{figure*}
\centering
\includegraphics[scale=0.4]{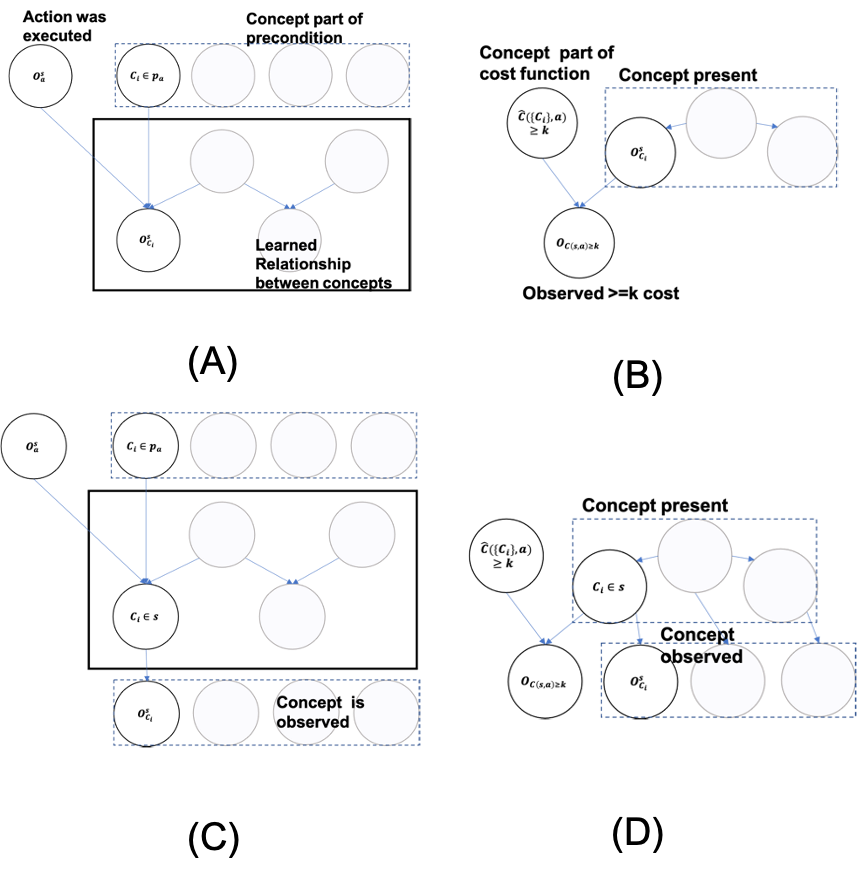}
\caption{\small{A simplified probabilistic graphical models for explanation inference, Subfigure (A) and (B) assumes classifiers to be completely correct, while (C) and (D) presents cases with noisy classifier.}}
\label{graph-model}
%\vspace{-5pt}
\end{figure*}
%\label{prob-class}
%The methods discussed in the previous Section are built around the assumption that we have an exact mapping from states to concepts. Since these mapping would, in the end, be required to be realized through classifiers, 
Given how unlikely it is to have access to a perfect classifier for any concept,
a more practical assumption to adopt could be that we have access to a noisy classifier. However, we assume that we also have access to a probabilistic model for its prediction. 
That is, we have access to a function $P_{\mathbb{C}}: \mathbb{C}\rightarrow [0,1]$ that gives the probability that the concept predicted by the classifier is actually associated with the state. Such probability functions could be learned from the test set used for learning the classifier.
Allowing for the possibility of noisy observation generally has a more significant impact on the precondition calculation than the cost function approximation. 
Since we are relying on just generating a lower bound for the cost function, we can be on the safer side by under-approximating the cost observations received (though this could lead to larger than required explanation). In the case of precondition estimation, we can no longer use a single failure (execution of an action in a state where the concept is absent) as evidence for discarding the concept. Though we can still use it as an evidence to update the probability of the given concept being a precondition. 
We can remove a particular possible precondition from consideration once the probability of it not being a precondition crosses a specified threshold.

To see how we can incorporate these probabilistic observations into our confidence calculation, consider the updated relationships presented in Figure \ref{graph-model} (C) and (D) for precondition and cost function approximation. Note that in previous sections, we made no distinction between the concept being part of the state and actually observing the concept. 
Now we will differentiate between the classifier saying that a concept is present ($O_{c_i}^{s}$) from the fact that the concept is part of the state ($c_i \in \mathbb{C}(S)$).
%We will assume that the probability of the classifier returning the concept being present is given by the probabilistic confidence provided by the classifier. Of course, this still assumes the classifier's model of its prediction is accurate. However, since it is the only measure we have access to, we will treat it as being correct. 
Now we can use this updated model for calculating the confidence. We can update the posterior of a concept not being a precondition given a negative observation ($O_{\neg c_i}^s$) using the formula
{\small\begin{align*}
\begin{split}
& P(c_i \not \in p_{a} | O_{\neg c_i}^s \wedge O_{a}^{s} \wedge O_{\mathbb{C}\setminus c_i^{s}}) =  \frac{P(O_{\neg c_i}^s|c_i \not \in p_{a} \wedge O_{a}^{s} \wedge O_{\mathbb{C}\setminus c_i^{s}})* P(c_i \not \in p_{a})}{P(O_{\neg c_i}| O_{a}^{s})}        
\end{split}
\end{align*}}
%Finally, we can remove a concept from the possible precondition list when it's probability dips below a specific threshold for posterior.
Similarly we can modify the update for a positive observation to include the observation model and also do the same for the cost explanation. For calculation of cost confidence, we will now need to calculate {\small $P( O_{\mathcal{C}(s, a) \geq k}| c_i \not\in \mathbb{C}(s), \mathcal{C}_s^{abs}(\{c_i\}, a) \geq k)$}. This can either be empirically calculated from samples with true label or we can assume that this value is going to be approximately equal to the overall distribution of the cost for the action.

%The derivations for all of the above expressions and formulas for the other cases can be found in the supplementary files. %\note{Mention the additional term in cost calculation}.

\section{Acquiring New Vocabulary}
\index{Vocabulary Asymmetry! Acquiring New Vocabulary}%
\todo{A core capability of the method discussed in the earlier sections is the ability to detect scenarios where the user specified concept list may not suffice to represent the required  model component. This is important as it is unlikely that the user would always be able to provide the required concepts, either because they are unaware of it or they may just have overlooked some of the concepts when coming up with the list provided to the system. The former case could happen when the user is not an expert or if the simulator is modeling novel phenomena and is itself learned from experience. Systems capable of handling such scenarios would require the capability of teaching the user new concepts and even coming with intuitive labels for such concepts. We are unaware any methods that can handle this in a general enough manner.}

\todo{Handling the latter case would generally be easier, since now the system would only need to query the human for a concept they may know but did not specify. One possible way such concepts can be queried is by presenting state pairs and asking the user to provide concepts absent in one but present in another. 
To see how we could perform such queries , consider the case of identifying failing preconditions. Let us denote the state where the foil fails as $s_f$. Now once the system has established that the current set of concepts cannot reliably explain the failing state, we can sample one of the states where action succeeds ($s'$ and we will refer to such states as positive states) and present the two states to the user and ask them to specify a concept that is absent in state $s_f$ but is present in $s'$. Now the system can keep showing more and more positive states and confirm that identified concept is still present in the state. The process can continue till either the system has enough samples to reliably establish the confidence of the fact that the concept is a precondition or the user notes a state where the concept is absent. If the latter happens you can ask the user to pick another concept that is absent in the original failing state but present in all the positive states. 
We can try to reduce cognitive load on the user's end on choosing the right concept by sampling positive states that are similar to $s_f$ (which would generally mean less number of conceptual difference).  Moreover, in cases where the user is being presented a visual representation of the state, then the system could also provide saliency maps showing areas in the state that are important to the decision maker (which should generally cover patches on the image corresponding to precondition concept).
Though this would mean restricting positive states to the ones where the decision maker would have used the failing action.}
%(and not any state where the action is executable).

\section{Bibliographic Remarks}
Many of the specific methods discussed in the chapter were first published in \cite{blackbox}.
As mentioned in the chapter many of the methods also have parallels in explainable Machine Learning. 
\index{LIME}%
\index{TCAV}%
For example, \cite{ribeiro2016should} is a popular explanation method for classifiers that relies on local approximation of models and 
\cite{tcav} generates concept based explanations, where individual concepts are captured through classifiers. Concept as classifier has also been utilized by \cite{hayes} to provide policy summaries and \cite{waa2018contrastive} to provide contrastive explanations for MDPs. Though in the case of \cite{waa2018contrastive}, they require designer to also assign positive and negative outcomes to each action that is used to explain why one policy is preferred over another and their method also do not really allow for explanation of infeasible foils. Outside of explanations, the idea of using classifiers to capture state fluents have also been utilized by \cite{konidaris} to learn post-hoc symbolic models from low-level continuous models. One possibility alluded to in the chapter though not discussed in details is that of using the methods to establish the cost of the plan itself. As one can guess here the process would be nearly identical to that for establishing the cost of the foil, but instead of trying to find max cost bounds, we will be looking for min ones. \unsure{Also the utility of establishing a representation of the agent in symbolic terms the user can understand goes beyond just providing helpful explanation. The paper by \cite{kambhampatisymbols} makes a case to develop a symbolic middle layer expressed in human-understandable terms as a basis for all human-AI interaction.}
\clearpage
                % bibliography using Author-Year

\chapter{Obfuscatory Behavior and Deceptive Communication}
\label{ch08}

In this chapter, we will focus the discussion on some of the behavioral and communication strategies that a robot can employ in adversarial environments. So far in this book, we have looked at how the robot can be interpretable to the human in the loop while it is interacting with her either through its behavior or through explicit communication. However, in the real world not all of the robot's interactions may be of purely cooperative nature. The robot may come across entities of adversarial nature while it is completing its tasks in the environment. In such cases, the robot may have some secondary objectives like privacy preservation, minimizing information leakage, etc. in addition to its primary objective of achieving the task. Further, in adversarial settings, it is not only essential to minimize information leakage but also to ensure that this process of minimizing information leakage is secure. Since, an adversarial observer may use diagnosis to infer the internal information and then use that information to interfere with the robot's objectives. 

To prevent leakage of sensitive information, the robot should be capable of generating behaviors that can obfuscate the sensitive information about its goals or plans. However, the execution of such obfuscating behaviors may be expensive and in such cases it may be essential to balance the amount of obfuscation needed with the resources available. Further, there may be complex real-world settings that involve both cooperative as well as adversarial observers. In such mixed settings, the robot has to balance the amount of obfuscation desired for adversarial entities with the amount of legibility required for the cooperative entities. 
\todo{Additionally, in this chapter we will look at how communication and the model-reconciliation techniques could be used for deception, in particularly to generate {\em lies}. As we will see, white lies could even be used to improve the overall team performance, although the ethics of such an undertaking has to be carefully considered.}
Now let's look at different types of adversarial behaviors and communication strategies that are available to the robot. 

\section{Obfuscation}
\index{Obfuscation}%
A robot can hide its true objective from an adversarial observer by making several objectives seem plausible at the same time. Thus the true objective stays obfuscated from the observer. The robot can choose to obfuscate its activities as well apart from its objectives. In this section, we will see different types of obfuscation approaches that a robot can employ to hide sensitive information in adversarial situations. 

\subsection{Goal Obfuscation}
\index{Obfuscation!Goal Obfuscation}%

\begin{figure}[t!]
\centering 
\includegraphics[width=\columnwidth]{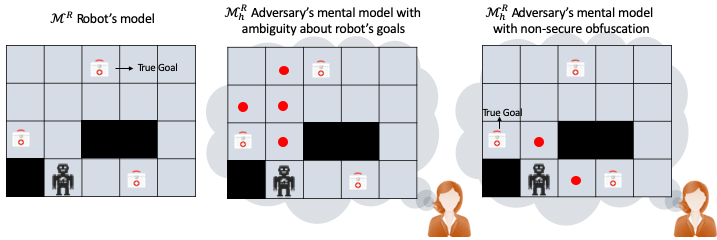}
\caption{Illustration of impact of goal obfuscation and secure goal obfuscation on human's mental model.}
\label{fig:8.1}
\end{figure}

The problem of goal obfuscation arises when the robot wants to obfuscate its true goal from the adversary. This is possible in settings where the adversarial entity is unaware of the robot's true goal but is aware of the possible goals that the robot may try to achieve in the given domain.
\unsure{In this book, we will mostly focus on scenarios where the adversary is a human, but the methods we talk about could also be applied in scenarios where we have another autonomous agent as our adversary.}
This problem setting is also part of the controlled observability planning problem, introduced in Chapter \ref{ch04}. 
In fact, it is similar to the setting seen in the offline goal legibility problem. 
However, here because of the existence of an adversarial observer, instead of being legible with respect to its true goal, the robot obfuscates its true goal. 
\unsure{In order to obfuscate effectively, the robot needs to have access to the adversaries models, incluging both their belief about the robot model ($\mathcal{M}^R_h$) and the observer's sensor model ($\mathcal{O}$).} 
%Such a sensor model defines the type of observations the adversary may get, i.e., it defines the parts of the state or set of actions that may be fully observable, or the set of states/actions that the adversary cannot differentiate between leading to partial observability. 
\unsure{As in the case of Chapter \ref{ch04}, the adversaries may not be able to differentiate between certain states/actions leading to partial observability.}
The adversary may use these observations resulting from the robot's behavior to infer some sensitive information about the robot's goals. Depending on the granularity of the observations available to the adversary, the problem of goal obfuscation becomes easy or complex. 
That is with fine-grained observations, goal obfuscation might be harder to achieve or may even be infeasible in some cases, whereas with coarse-grained observations, it might be easier for the robot to obfuscate its true goal. Further, the adversary may be able to use the information gleaned from observations to interfere with or hamper the robot's activities. Therefore, in this setting, the robot is tasked with the additional objective of preventing the adversary from learning sensitive information about its goal. \unsure{Similar to Chapter \ref{ch04}, we will focus on scenarios where the adversaries expectations about the robot model aligns with the robot's true model, i.e., $\mathcal{M}^R_h$ is the same as $\mathcal{M}^R$.}

\begin{definition} 
\index{Obfuscation!Goal Obfuscation!Planning problem}%
\unsure{A \textbf{goal obfuscation planning problem}, is a controlled observability planning problem, where, $\mathcal{G} = \{G^R \cup G_1 \cup \ldots \cup G_{n-1} \}$, is the set of $n$ goals where $G^R$ is the true goal of the robot, and $G_1, \ldots, G_{n-1}$ are the decoy goals.} 
\end{definition}

A solution to a goal obfuscation planning problem is a \textit{k-ambiguous} plan. Recall that, since the observer receives observations of the robot's behavior, each plan is associated with a corresponding observations sequence. Therefore, the objective of a \textit{k-ambiguous} plan is to make the observation sequence consistent with at least $k$ goals, out of which $k-1$ are decoy goals, such that, $k \leq n$. This is done by taking actions towards robot's true goal, such that, the corresponding observation sequence exploits the observer's belief space in order to be consistent with multiple goals. That is, a \textit{k-ambiguous} plan produces a final belief state consistent with its observation sequence where at least $k$ goals are possible. The $k-1$ decoy goals can be chosen specifically to maximize the obfuscation. Figure \ref{fig:8.1} illustrates goal obfuscation for $k=2$ in the middle image. Here the true goal of the robot to pickup the medkit marked in the leftmost image. The adversarial observer has partial observability of the direction the robot moves in (say, the observer receives coarse observations from the GPS sensor). In order to obfuscate its true goal from the adversarial observer, the robot performs a 4-step k-ambiguous plan (for k=2), as shown in the middle image. 

To define a k-ambiguous plan, we use the notation for belief state, introduced in Chapter \ref{ch04}.

\begin{definition} 
\label{def:obf}
\index{K-ambiguous Plan}%
A plan, $\pi_k$, is a \textbf{k-ambiguous plan}, if $\delta(I, \pi_k) \models G^R$ and the final belief state, $b_n \in BS(\pi_k, I)$, satisfies the following,  $|{G \in \mathcal{G}: \exists s \in b_n, s \models G}| \geqslant k$, where $1 \leqslant k \leqslant n$.
\end{definition}

\begin{definition} 
\index{K-ambiguous Observation Sequence}%
An observation sequence $O_k = \langle o_1, \ldots, o_n \rangle$ is \textbf{k-ambiguous observation sequence} if it is an observation sequence emitted by a k-ambiguous plan. 
\end{definition}

The underlying algorithm for the synthesis of the solution is the same as the one described for the offline legibility planning problem. However, the goal test and heuristic function are customized to ensure that the goal obfuscation constraints are met. 

%solution plan does not end unless at least $k$ goals occur in the belief state, and the heuristic function involves computing the heuristic cost to the true goal from the robot's actual state, and computing the heuristic cost to $k-1$ goals from the belief induced by the emitted observation. The heuristic value of a state is computed as the total of these two values.

\subsubsection{Goal test} The goal test ensures that the solution plan does not end unless all $k$ goals occur in the observer's belief state. Specifically, it validates whether the true goal has been achieved in the robot's true state and also whether the $k-1$ goals have been achieved in the belief state. If there doesn't exist a plan that can achieve the true goal with the given $k$ goals, we restart the search with next combination of $k$ goals (inclusive of true goal) among the total $n$ candidate goals. In the worst case, the value of $k$ can be reduced until a solution is found.

\subsubsection{Heuristic function} 
Here the heuristic strategy is to generate a plan where the observer's final belief state satisfies \textit{k} goals, making it a \textit{k-ambiguous} plan. We use two heuristic functions to estimate a node's cost: 
\begin{math}
h(s) =  h_{G_A}(s) + h_{\mathcal{G}_{k-1}}(b), 
\end{math}
where the first part computes the heuristic cost to the true goal from the robot's actual state, and the second part computes the heuristic cost to $k-1$ goals from the belief induced by the emitted observation. This is computed by finding the \textit{max} of heuristic cost to each of the $k-1$ goals. The heuristic value of a state is computed as the total of these two values. Thus the heuristic helps to guide the search towards at least $k$ goals in the observer's belief. 

\subsection{Secure Goal Obfuscation}
\index{Obfuscation!Secure Goal Obfuscation}%

The above approach to synthesizing goal obfuscating behavior is not secure. In the sense that it reveals information when the following assumptions are not satisfied: \unsure{(1) adversary has access to the algorithm and has the means to simulate the algorithm for a given input. This means the adversary could identify the potential states generated by the algorithm. (2) the adversary can run the algorithm with different inputs. For example, the adversary can set different goals as true goal to check the variability of the output. Figure \ref{fig:8.1}, rightmost subfigure illustrates how an adversary can rerun the algorithm with different input and glean information about the true goal. By setting the true goal to some other candidate goal, the adversary can run the algorithm and find out the robot's true goal which gave the original observation sequence.}

%Therefore, in order to synthesize a secure goal obfuscating behavior, we can introduce some randomization in our solution synthesis process, that would prevent leakage of additional information. 
\unsure{ One way to synthesize a secure goal obfuscating behavior would be to introduce some randomization in the solution synthesis process, that would prevent leakage of additional information. 
To that end, one could randomly choose a set of $k$ goals (inclusive of true goal) from $n$ candidate goals. Then randomly pick a goal from this set as a ``true goal'' and run the algorithm. Let's call this goal as the decoy true goal. We have to maintain the plans to each of the $k$ goals that are consistent with the observation sequence (which includes the robot goal $G^R$). Then by performing a post processing step, we can retrieve the plan that reaches the robot's true goal and is consistent with the observation sequence produced by the plan to the decoy true goal. This process ensures that even if the adversary runs the algorithm again with either the same or different input (different goal as true goal, or different combination of $k$ goals), the identity of the true goal remains \textit{k-ambiguous} even with the observer being aware of the randomization.
Note that unlike the other sections here we are not just considering an observer who is reasoning on their own based on the observations, but rather we are considering a sophisticated adversary who is able to use more sophisticated computing resources to simulate the algorithm. 
}

\unsure{In this setting, a perfectly rational adversary would be one that has unbounded computational resources, and can use them to make perfect inferences. In other words, that is the worst case scenario for this algorithm. The following proposition shows that our algorithm provides \textit{k-ambiguous} guarantee even in the worst case, such that a perfectly rational adversary remains uncertain about the true goal, provided the goals are independent.}

\begin{proposition}
Let $\mathcal{H}$ be a perfectly rational adversary. Given a goal obfuscation planning problem for the robot $\mathcal{R}$. Plans computed using Algorithm \ref{4-procedure:COPP-algo} with \textit{k-ambiguous} goal test and heuristic function are such that, $\mathcal{H}$ will be able to infer the true goal with probability  $\leqslant 1/k$ provided the goals are independent.
\end{proposition}

Let's say goal $G_i$ from the set $\mathcal{G}_k$ is chosen randomly to be the decoy true goal. The observation sequence $O_i$ is obtained by running the algorithm with $G_A = G_i$. The adversary can run the algorithm with each goal from $\mathcal{G}_k$ to search for $G_i$. But $G_i$ can only be the robot's true goal with probability $1/k$. Therefore the adversary can detect the true goal with probability $\leqslant 1/k$.

Note that the goals can be specified using sub-formulae. This makes the input specification much smaller depending on how many states are in the domain that are consistent with the sub-formulae. In the least constrained domain, that may be exponential in the number of propositions not used. 

%\section{Resource Bounded Goal Obfuscation}

%We now discuss an approach where the AI agent can obfuscate its true goal for as long as the observer's model allows in a cost-effective manner. Here the original problem of goal obfuscation remains as is, however, now the algorithm explores a  

\subsection{Plan Obfuscation}
\index{Obfuscation!Plan Obfuscation}%

\begin{figure}[t!]
\centering 
\includegraphics[width=\columnwidth]{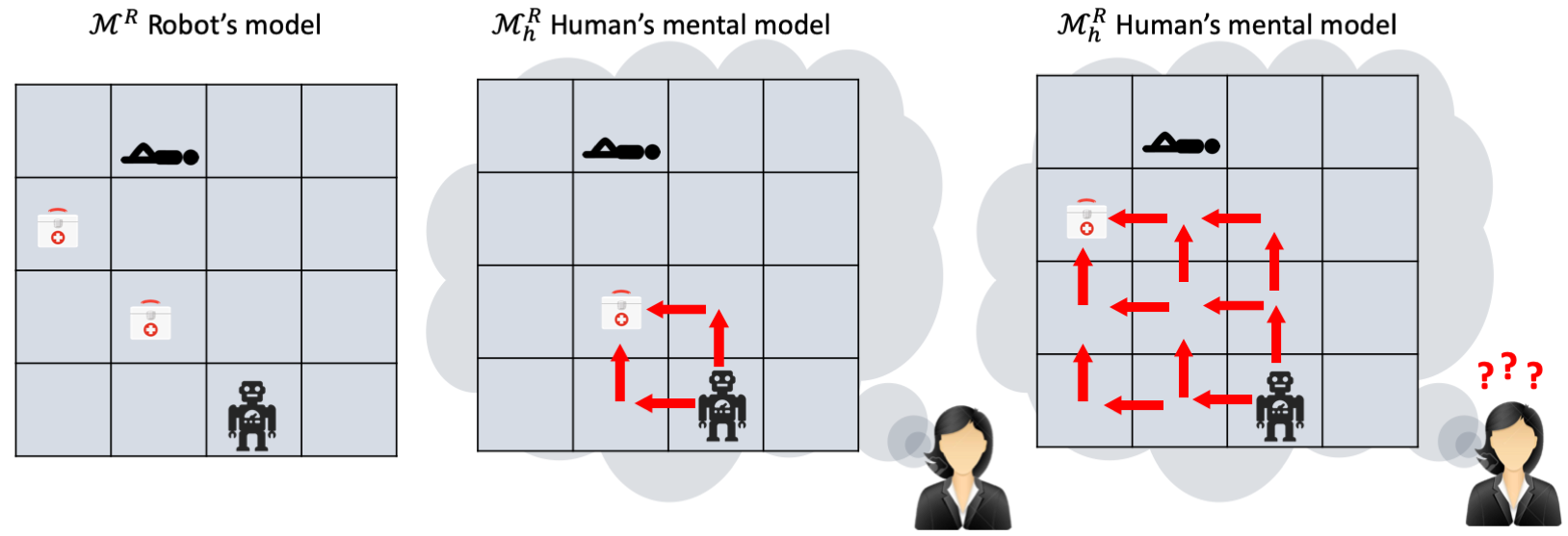}
\caption{Illustration of the impact of plan obfuscation on human's mental model.}
\label{fig:8.2}
\end{figure}

In this problem setting, the adversarial observer is aware of the robot's goal. The robot, however, obfuscates its activities towards its goal thereby hiding sensitive information about its activities. We refer to this behavior as plan obfuscatory behavior. The robot can achieve a plan obfuscatory behavior by computing a plan whose observation sequence conforms to a set of plans that are diverse in terms of their actions, thereby making it hard for the adversarial observer to guess the actions executed by the robot. For instance, in Figure \ref{fig:8.2}, the goal of the robot is to pick up a medkit and treat the victim on the disaster site. It has two choices of medkits to pick from. If it chooses the medkit as shown in the rightmost sub-figure, it can maximize the diversity with respect to its actions. Here, the multiple different paths to the medkit produce the same observation sequence (due to the underlying coarse GPS sensor model that the observer is using), thus hiding the robot's true actions from the adversarial observer.  

\begin{definition}
\index{Obfuscation!Plan Obfuscation!Planning Problem}%
A \textbf{plan obfuscation planning problem} is a controlled observability planning problem, where, $|\mathcal{G}| = 1$, and therefore the true goal is known to the observer.
\end{definition}

The solution to a plan obfuscation problem is an $\ell$-diverse plan. An $\ell$-diverse plan is a plan whose observation sequence is consistent with at least $\ell$ diverse plans to the goal, such that, the diverse plans are at least $d$ distance away from each other. In order to compute an \emph{$\ell$-diverse} plan, we need to keep track of the plans that are consistent with the observation sequence and reach the goal. To compute the diversity between all the pairs of plans consistent with the observation sequence, a plan distance measure like action distance, causal link distance, state sequence distance (introduced in Chapter \ref{ch03}) can be used. In general, the approach can use any valid plan distance. We now define an $\ell$\emph{-diverse} plan.  

We refer to the set of plans that are consistent with the belief sequence of a given plan, as a belief plan set (BPS), introduced in Chapter \ref{ch04}.

\begin{definition} 
\index{Minimally d-distant Belief Plan Set}%
A BPS induced by plan p starting at $s_0$ is \textbf{minimally d-distant}, $d_{min}(BPS(p, s_0))$, if $\displaystyle d = \min_{p1, p2 \in BPS(p, s_0)} \mathcal{D}(p1, p2)$.
\end{definition}

\begin{definition} 
\index{l-diverse Plan}%
A plan, $\pi_l$, is an \textbf{$\bm{\ell}$-diverse plan}, if for a given value of d and distance function $\delta$, $d_{min}(BPS(\pi_l, I)) \geq d$, $|BPS(\pi_l, I)| \geq \ell$, where $\ell \geq 2$ and every plan in $BPS(\pi_l, I)$ achieves the goal in $\mathcal{G}$.
\end{definition}

\unsure{We can again leverage Algorithm \ref{4-procedure:COPP-algo} to generate the plans for obfuscation problem. The goal test and heuristic function are customized to ensure that there are at least $\ell$ diverse plans to the true goal that are consistent with the observation sequence and the minimum distance between these plans is at least $d$.}

\subsubsection{Goal test} 
To ensure the plans in $BPS$, induced by an $\ell$\textit{-diverse} plan, can achieve the goal in $\mathcal{G}$, we can check whether at least $\ell$ plans are reaching the goal or not and whether the minimum distance between plans in $BPS$ is at least $d$. Also in order to ensure termination of the algorithm, there is a cost-bound given as input to the algorithm.

\subsubsection{Heuristic function}
Apart from minimizing the heuristic cost to the goal, the customized heuristic given below also maximizes the $d$ of $d_{min}(BPS(p, s_0))$ induced by plan $p$ starting at $s_0$. This increases the minimum distance between the plan pairs. This distance is computed using a plan distance measure.
\begin{align}
h(s) = h_{G_A}(s) - d_{min}(BPS(p, s_0))
\end{align}

\subsection{Deception}
\index{Deception}%
\index{Decoy Goal}%

In addition to obfuscation, the robot can also employ deceptive behaviors to hide sensitive information. While the concept of obfuscation deals with the notion of confusing the adversary with several decoy candidates, the concept of deception involves making one of the decoy candidates more likely than the robot's true objective/activities. Thus, in order to deceive an adversarial observer, it is crucial to have access to their goal or plan recognition models. 
By incorporating the human's reasoning ability (in regards to identifying the goal or the plan), the robot can synthesize behavior that deceives the adversarial observer into believing that the decoy candidate is the true candidate. 
In the prior literature, synthesis of deceptive behaviors with respect to robot's goals has been studied in path planning scenarios where the observer has  full observability of robot's activities. 
In order to successfully deceive an adversarial observer, the robot's plan has to end when a decoy goal is achieved. However, in reality, the robot has a primary objective of achieving its true goal. Therefore, in cases where the observer has full observability of the robot's activities, deceptive behavior may be hard to synthesize. To that end, the notion of the radius of maximum probability with respect to a goal has been studied in the literature on deceptive planning. This is a radius around a goal location, within which that goal itself becomes the most likely goal. So deception can be maintained until the robot reaches the area within the radius of maximum probability for its true goal. 

\section{Multi-Observer Simultaneous obfuscation and legibility}

Obfuscatory strategies are also crucial in scenarios where there are multiple different types of observers, such that some are of adversarial nature while some others are of cooperative nature. In such cases, the robot has to ensure that its behavior is simultaneously legible to cooperative observers and obfuscatory to adversarial ones. For instance, in soccer, a player may perform feinting trick to confuse an opponent while signaling a teammate. 
Synthesizing a single behavior that is legible and obfuscatory to different observers presents significant technical challenges.
In particular, the robot may have to be deliberately less legible to its friends so that it can be effectively more obfuscatory to its adversaries. 
This problem gives rise to a novel optimization space that involves trading-off the amount of obfuscation desired for adversaries with the amount of legibility desired for friends. 

\subsection{Mixed-Observer Controlled Observability Planning Problem}
\index{MO-COPP}%

We now discuss a problem framework called \emph{mixed-observer controlled observability planning problem}, MO-COPP, that allows a robot to simultaneously control information yielded to both cooperative and adversarial observers while achieving its goal. 
This framework models and exploits situations where different observers have differing sensing capabilities, which result in different ``perceptions'' of the same information.
Typically, different observers in an environment may have different ``sensors'' (perception capabilities) due to differences in prior communication, background knowledge, and \todo{innate differences in sensing capabilities}. 

\begin{figure}[!t]
\begin{subfigure}{\columnwidth} 
\centering
\caption{The robot's goal is to deliver two packages to the delivery area:}
\includegraphics[width=0.5\columnwidth]{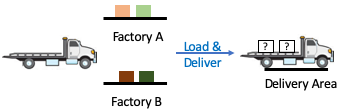}
\label{fig:ex1}
\end{subfigure} 
\begin{subfigure}{\columnwidth}
\centering
\caption{Plan-1 - Robot delivers 1 package from factory A and 1 from B:}
\includegraphics[width=0.8\columnwidth]{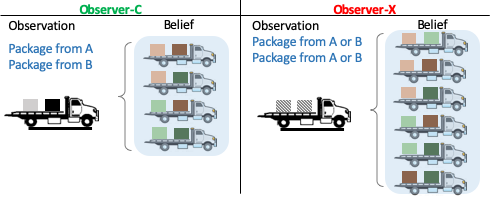}
\label{fig:ex2}
\end{subfigure}
\begin{subfigure}{\columnwidth}
\centering
\caption{Plan-2 - Robot delivers both the packages from factory A, helping observer-C and thwarting observer-X:}
\includegraphics[width=0.8\columnwidth]{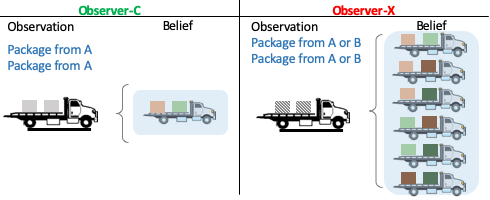}
\label{fig:ex3}
\end{subfigure}
\caption{Different belief updates induced by the same plan for two observers; observer-C is cooperative, observer-X is adversarial.}
\label{fig:example2}
%\vspace{-0.2em}
\vspace{-8pt}
\end{figure}

\subsubsection{Example}
Consider the example in Figure \ref{fig:example2}, where an automated truck has to deliver two packages. As shown in Figure \ref{fig:ex1}, there are two packages at factory A and two at factory B. 
%Therefore, there are 6 candidate goals in total (i.e., all combinations of cardinality 2). 
Let there be two observers who know the initial inventory at each factory, but do not know which two packages will be delivered by the truck. Each observer has different sensors (or sources of information) about the truck's activities. Observer-C (cooperative observer) has access to a database of barcodes which allows it to determine the identity of the factory that the package was loaded from (but not the identity of the individual packages), whereas observer-X (adversarial observer) does not have any additional information and cannot differentiate between any of the packages. The truck wants to convey its true goal to observer-C but obfuscate it to observer-X. 
%The observers do not know the true goal of the truck. 
When the truck loads one package each from A and B (Figure \ref{fig:ex2}), observer-C cannot distinguish the packages. However, when it loads both the packages from A (Figure \ref{fig:ex3}), observer-C can infer the identity of the loaded packages, while observer-X still does not know the identity of the packages. Therefore, plan-2 makes the true goal legible to observer-C while obfuscating the true goal with all the possible goals for observer-X. 

MO-COPP setting involves an actor ($\mathbf{R}$) (i.e., the robot) and two observers, where one is adversarial observer ($\mathbf{X}$) while the other is cooperative ($\mathbf{C}$). This framework assumes that the robot has full observability of its own activities and knows the sensor models used by the observers, and that the observers have different sensor models. When the robot takes an action and reaches a new state, each observer receives an observation. After obtaining the observations, each observer updates its belief, similar to the belief update in COPP formulation. The robot leverages the known limits in the observers' sensors to control the observability of multiple observers in the environment simultaneously. 
Given a set of candidate goals, the objective of the robot is to convey information about its goal to the cooperative observer and to hide it from the adversarial observer. 
This involves increasing the number of candidate goals possible in the adversary's belief, while decreasing the number of candidate goals possible in cooperative observer's belief. 

\begin{definition}
\index{MO-COPP!Planning Problem}%
A \textbf{mixed-observer controlled observability planning problem} is a tuple, \\
$MO-COPP = \langle \Lambda, \mathcal{M}^R, \mathcal{G}, \{\Omega_i\}_{i \in \Lambda} , \{\mathcal{O}_i\}_{i \in \Lambda},$ $\{\mathcal{B}^i_0\}_{i \in \{\mathbf{X}, \mathbf{C}\}} \rangle $, where,
\begin{itemize}
\item $\Lambda = \{ \mathbf{R},\mathbf{C}, \mathbf{X}\}$ is the set of agents. 
\item $\mathcal{M}^R = \langle  F^R, A^R, I^R, G^R \rangle$ is the robot model
%\item $\mathcal{P} = \langle  F, A, I, G^R \rangle $ is $\mathbf{R}$'s task captured as a classical planning problem, where $F$ is the set of fluents, $A$ is the set of actions, $I$ is the initial state and goal $G^R$ is a subset of fluents. 
%Each state $s \in \mathcal{S}$ is an instantiation of all the fluents. 
Also, for $a \in A$, $pre(a)$, $adds(a)$, $dels(a)$ are each a subset of fluents representing precondition, add effects and delete effects of $a$.
\item $\mathcal{G} = \{G_1, G_2, \ldots, G_{n-1}, G^R\}$ is the set of candidate goals, where $G^R$ is the true goal of $\mathbf{R}$, which is not known to both $\mathbf{C}$ and $\mathbf{X}$.
\item $\Omega_i$ is the set of observation symbols for agent $i$, which are emitted when $\mathbf{R}$ takes an action and reaches a new state. Further, $\Omega_R = \{o^{\mathbf{R}}_{a,s}|  a \in A, s \in \mathcal{S}\}$.
%such that, $\Omega_i \cap \Omega_j = \emptyset,~where,~i \neq j~\text{and}~i, j \in \{1, \ldots, |\Lambda|\}$.
\item $\mathcal{O}_i: A \times \mathcal{S} \rightarrow \Omega_i$ is agent $i$'s deterministic sensor model. $\mathcal{S}$ is the set of states, where each state is instantiation of all fluents. Further, $\mathcal{O}_\mathbf{R}$ 
%is an accurate sensor model that 
maps each action-state pair to a unique observation, $\forall~ a, a' \in A,~ s, s' \in \mathcal{S}, a \neq a' \land s \neq s': \mathcal{O}_\mathbf{R}(a, s) \neq \mathcal{O}_\mathbf{R}(a', s')$, while $\mathcal{O}_\mathbf{X}$ and $\mathcal{O}_\mathbf{C}$ are noisy sensor models that map multiple action-state pairs to the same observation symbol.
\item $\mathcal{B}^i_0$ is the initial belief of an observer, $i \in \{\mathbf{X}, \mathbf{C}\}$. The initial belief is a set of states inclusive of $I$.
%$\Omega_i \neq \Omega_j,~where,~i \neq j~\text{and}~i, j \in \{1, \ldots, |\Lambda|\}$.
\label{dfn:mo-copp-def}
\end{itemize} 
\end{definition}

Every time the robot acts, each $i \in \Lambda$ receives an observation consistent with its sensor model. The sensor model of an observer $i \in \{\mathbf{X}, \mathbf{C}\}$  supports many-to-one mapping of $\langle a,s \rangle$ pairs to observation symbols, i.e., $\exists a, a' \in A, s, s' \in \mathcal{S}, a \neq a' \land s \neq s': \mathcal{O}_i(a, s) = \mathcal{O}_i(a', s')$.
For an agent $i$, the inverse of sensor model gives the set of $\langle a,s \rangle$ pairs consistent with an observation symbol $o^i \in \Omega_i$, i.e., $\mathcal{O}_i^{-1}(o^i) = \{ \langle a, s \rangle | \forall a \in A, s \in \mathcal{S}, \mathcal{O}_i(a, s) = o^i\}$. \unsure{Again, we assume the observer is aware of the true robot model.}

Each observer $i \in \{\mathbf{X}, \mathbf{C}\}$ maintains its own belief, which is a set of states. $\delta(\cdot)$ is a transition function, such that, $\delta(s, a) = \bot$ if $s \not\models pre(a)$; else $\delta(s, a) = s \setminus dels(a) \cup adds(a) $. Now we can define the belief update: (1) at time step $t = 0$, the initial belief of observer $i$ is given by  $\mathcal{B}^i_0$, (2) at time step $t \in \{1, \ldots, \mathcal{T}\}$, $\mathcal{B}^i_{t} = \{ \hat{s}~|~ \exists \hat{a},\bar{s}~~ \delta(\bar{s},  \hat{a}) =  \hat{s}; \bar{s} \in \mathcal{B}^i_{t-1}; \mathcal{O}_i(\hat{a}, \hat{s}) = o^i_t\}$, where $\mathcal{T}$ is the time horizon.   
That is, the belief is updated using the previous belief and the observation received. 
A plan $\pi$ (or sequence of actions) is associated with a sequence of observations, $ObsSeq_i(\pi) = \{o^i_1, \ldots, o^i_\mathcal{T}\}$ for each observer $i$. The belief of $\mathbf{R}$ is a single state due to full observability. 

\subsubsection{MO-COPP Solutions} 
\index{MO-COPP!Solution}%

A plan $\pi$ solves a $MO-COPP = \langle \Lambda, \mathcal{M}^R,$ $\mathcal{G}, \{\Omega_i\}_{i \in \Lambda} , \{\mathcal{O}_i\}_{i \in \Lambda}, \{\mathcal{B}^i_0\}_{i \in \{\mathbf{X}, \mathbf{C}\}} \rangle $, iff: $~\delta(I, \pi) \models G_A$. In other words, any solution to $\mathcal{M}^R$ is a solution to MO-COPP. We measure the quality of a MO-COPP solution in terms of the extent to which the plan is legible to $\mathbf{C}$ and obfuscating to $\mathbf{X}$. Given a solution plan, $\pi$, let $\mathcal{G}^\pi_\mathbf{X} = \{G \in \mathcal{G}: \exists s \in \mathcal{B}^\mathbf{X}_\mathcal{T},~ s \models G \}$, denote the goals present in $\mathbf{X}$'s final belief and $\mathcal{G}^\pi_\mathbf{C} = \{G \in \mathcal{G}: \exists s \in \mathcal{B}^\mathbf{C}_\mathcal{T},~ s \models G \}$, denote the goals present in $\mathbf{C}$'s final belief. Also, we can define the inverse of its observation sequence for observer $i$, as follows: $\forall i \in \{\mathbf{X}, \mathbf{C}\}, ObsSeq^{-1}_i(\pi) = \{ \hat{\pi} | \exists s^i_0 \in \mathcal{B}^i_0, G \in \mathcal{G}^\pi_i, \delta(s^i_0, \hat{\pi}) \models G, ObsSeq_i(\pi) =  ObsSeq_i(\hat{\pi}) \}$.  

Here $|\mathcal{G}^\pi_\mathbf{X}|$ represents the amount of \textbf{goal obfuscation for $\mathbf{X}$}, and by increasing the number of goals in $\mathcal{G}^\pi_\mathbf{X}$, we can increase the goal obfuscation. Similarly, $|\mathcal{G}^\pi_\mathbf{C}|$ represents the amount of \textbf{goal legibility for $\mathbf{C}$}, and by decreasing the number of goals in $\mathcal{G}^\pi_\mathbf{C}$, we can increase the goal legibility. We now introduce a goal difference metric to estimate the quality of a solution plan to MO-COPP.

\begin{definition}
\index{MO-COPP!Goal Difference}%
Given a solution plan $\pi$ that solves MO-COPP, where $|\mathcal{G}| = n$, the \textbf{goal difference}, of $\pi$ is given by:
\begin{math}
GD(\pi) = \frac{|\mathcal{G}^\pi_\mathbf{X}| - |\mathcal{G}^\pi_\mathbf{C}|}{n-1}
\end{math}
where the denominator represents the difference between ideal values of $\mathcal{G}^\pi_\mathbf{X}$ and $\mathcal{G}^\pi_\mathbf{C}$.
\end{definition}

An optimal solution to MO-COPP maximizes the trade-off between amount of goal obfuscation and goal legibility. That is, it maximizes the difference between the number of goals in $\mathcal{G}^\pi_\mathbf{X}$ and $\mathcal{G}^\pi_\mathbf{C}$. 
%as the number of goals in $\mathcal{G}^\pi_\mathbf{X}$ have to be higher in number while those in $\mathcal{G}^\pi_\mathbf{C}$ have to be lower in number. This is exactly what equation \ref{dfn:gd} does. Higher the goal difference, more the goal obfuscation and goal legibility. 
Equivalently, closer the $GD(\pi)$ value to 1, better is the plan quality. A solution plan with $GD(\pi) = 1$ is an optimal plan. E.g., in Figure \ref{fig:ex3}, plan-2 is an optimal plan with $GD(plan$-$2) = \frac{6-1}{6-1} = 1 $. The denominator is essential for comparing the quality of plans across different problems, with varying number of candidate goals. 

\begin{proposition}
Given a solution plan, $\pi$, to MO-COPP, if $|\mathcal{G}^\pi_\mathbf{C}|=1$, then $G^R \in \mathcal{G}^\pi_\mathbf{C}$.
\end{proposition}

The above proposition states that when maximum goal legibility is achieved, only one goal $G_\mathbf{A}$ is present in $\mathbf{C}$'s final belief. This is because $I \in \mathcal{B}^\mathbf{C}_0$, and while performing belief update at $t$, $\langle a_t, s_t \rangle \in \mathcal{O}_{\mathbf{C}}^{-1}(o^\mathbf{C}_t$). That is the true state is part of each belief update. Therefore, $G^R \in \mathcal{G}^\pi_\mathbf{C}$.

\subsection{Plan Computation}
\index{MO-COPP!Plan Generation}%
We now present two solution approaches to solve MO-COPP. 
In the first approach, MO-COPP can be formulated as a constraint optimization problem and provide an IP encoding to solve it in $\mathcal{T}$ steps. The IP encoding provides an optimal solution for the given horizon by maximizing the trade-off between the amount of obfuscation and legibility. The second approach builds up on the search algorithm presented in Chapter \ref{ch04} to solve COPP problems. The search algorithm generates solutions that satisfy a prespecified lower bound on the amount of goal obfuscation and goal legibility.

\subsubsection{MO-COPP as Integer Program} 
\index{MO-COPP!IP Encoding}%
\todo{The IP encoding includes variables for an action being executed at a time, each state being true at a time, for a specific observation to emitted at a give time and for specific state and goal being in the belief state. The objective function is essentially the numerator of $GD(\cdot)$ metric, i.e.,}

\begin{align}
max~~ \sum_{G \in \mathcal{G}}~ g^{\mathbf{X}}_{G, \mathcal{T}} -~ \sum_{G \in \mathcal{G}}~ g^{\mathbf{C}}_{G, \mathcal{T}}
\label{eq:1}
\end{align}

Where $g^{i}_{G, \mathcal{T}}$ is an indicator variable capturing the fact that goal $G$ is present in the belief state for the agent $i$ at the time step $\mathcal{T}$. The denominator of the $GD$ metric is skipped, as it is a constant and does not contribute to the optimization. 
%In the first term, we maximize the goal obfuscation with respect to $\mathbf{X}$ and in the second term we maximize the goal legibility with respect to $\mathbf{C}$. 
Maximizing for this objective provides a single solution that achieves the maximum difference between the number of goals possible for the two observers.
\unsure{One could also formalize the problem as a multi-objective optimization problem, and look at various pareto optimal solutions that trade off the various goals that each observer believes will be achieved.}
%Note that, it would make sense to get the Pareto optimal solutions if we wanted to explore all the combinations of goals achieved for the two observers. However, that is not our objective.

The constraints provided to the problem, captures the exact evolution of the state and belief (for each agent) given the robot action, the observation generated by them and the fact that the robot needs to achieve its goal within the specified horizon limit. The exact IP encoding can be found in the paper \cite{kulkarni2019signaling}. Also note that, the current objective function trades off goal obfuscation with goal legibility for the observers. However, the robot can ensure a predefined level of goal obfuscation (say obfuscate with \emph{at least} $k$ candidate goals) by adding an additional constraint that enforces a bound for goal obfuscation and goal legibility. 
% \begin{align}
% \sum_{G \in \mathcal{G}} g^{\mathbf{X}}_{G, \mathcal{T}} \geqslant k, ~~~s.t.~ 1 \leqslant k \leqslant |\mathcal{G}| 
% \end{align} 
% and goal legibility (say legible with \emph{at most} $j$ goals) by 
% \begin{align}  
% \sum_{G \in \mathcal{G}} g^{\mathbf{C}}_{G, \mathcal{T}} \leqslant j, ~~~s.t.~ 1 \leqslant j \leqslant |\mathcal{G}| 
% \end{align}. 

\subsubsection{Search Algorithm}
\index{MO-COPP!Search}%

It is also possible to leverage search techniques that address goal obfuscation and goal legibility in isolation to solve MO-COPP. The search algorithm in Chapter \ref{ch04} is adapted to address goal obfuscation and goal legibility simultaneously to two different observers. The bounds on the amount of goal obfuscation and goal legibility desired can be specified, similar to the ones seen in the IP: obfuscate with at least $k$ goals, make it legible with at most $j$ goals. These bounds, $\Phi = \langle \Phi_\mathbf{X}, \Phi_\mathbf{C}\rangle$, are given as input to the search algorithm. 

\begin{algorithm}
    \caption{Heuristic-Guided Search}
    \label{alg:algorithm1}
    \begin{algorithmic}[1]
    \STATE Initialize $\textit{open}$, $\textit{closed}$ and $\textit{temp}$ lists; $\Delta = 1$ 
    \STATE $\langle b^X_{\Delta}, b^C_{\Delta} \rangle \gets$ approx$(I, \mathcal{B}^X_0, \mathcal{B}^C_0)$ 
    \STATE $\textit{open}$.push$(I, \langle b^X_{\Delta}, b^C_{\Delta} \rangle, \langle \mathcal{B}^X_0, \mathcal{B}^C_0 \rangle, priority = 0)$
    \WHILE{$\Delta \leqslant |\mathcal{S}|$}
        \WHILE{$\textit{open} \neq \emptyset$}
        \STATE $s, \langle b^X_{\Delta}, b^C_{\Delta} \rangle, \langle \mathcal{B}^X, \mathcal{B}^C \rangle, h_{node} \gets \textit{open}.\textrm{pop}() $
            \IF{$ |b^X_{\Delta}| > \Delta$ or $|b^C_{\Delta}| > \Delta$}
                \STATE $\textit{temp}.$push$(s, \langle b^X_{\Delta}, b^C_{\Delta} \rangle, \langle \mathcal{B}^X, \mathcal{B}^C \rangle,  h_{node})$
                \STATE continue
            \ENDIF
            
            \STATE add $\langle b^X_{\Delta}, b^C_{\Delta} \rangle$ to $\textit{closed}$
            \IF{$s \models G_A$ and $\mathcal{B}^X \models \Phi_X$ and $\mathcal{B}^C \models \Phi_C$}
            \STATE $\textbf{return}~\pi, ObsSeq_{X}(\pi), ObsSeq_{C}(\pi)$
            \ENDIF
            
            \FOR{$s^{\prime} \in $successors$(s)$}
            \STATE $o^X \gets \mathcal{O}_X(a, s^{\prime})$; $o^C \gets \mathcal{O}_C(a, s^{\prime})$ 
            \STATE $\widehat{\mathcal{B}^X}=$ Update($\mathcal{B}^X, o^X);  \widehat{\mathcal{B}^C} =$ Update($\mathcal{B}^C, o^C)$
            \STATE $\langle \widehat{b^X_{\Delta}}, \widehat{b^C_{\Delta}} \rangle \gets$ approx$(s', \widehat{\mathcal{B}^X}, \widehat{\mathcal{B}^C)}$ 
            \STATE $h_{node} \gets h_{G_A}(s') + h_{\mathcal{G}_{k-1}}(\widehat{\mathcal{B}^X}) - h_{\mathcal{G}_{\mathcal{G}-j}}(\widehat{\mathcal{B}^C})$ 
            \STATE add new node to $open$ if $\langle \widehat{b^X_{\Delta}}, \widehat{b^C_{\Delta}} \rangle$ not in $\textit{closed}$\\
            \ENDFOR
            
        \ENDWHILE
        \STATE increment $\Delta$; copy items from $\textit{temp}$ to $\textit{open}$; empty $\textit{temp}$ \\
    \ENDWHILE
\end{algorithmic}
\end{algorithm}

%We present the aspects of \citet{implicitHRC2018}'s approach that are essential to the novel extensions developed in this paper. Further details can be found in their paper. 
Each search node maintains the associated beliefs for both observers. The $approx$ function generates an approximate belief, $b^i_\Delta$, of size $\Delta$ (i.e. cardinality of $b^i_\Delta$ is $\Delta$). $b^i_\Delta$ is always inclusive of the true state of the robot, this is because the robot can only take actions that are consistent with its true state. If all such $\Delta-$sized beliefs (or smaller) are explored then $b^i_\Delta$ of $\Delta+1$ size is computed, and this node gets put in the temporary list and is explored in the next outer iteration when $\Delta$ has been incremented. For each $\Delta$, all $\Delta$-sized unique combinations of belief (that include the actual state of the robot) are explored. This allows systematic and complete exploration of multiple paths to a given search node. 
%The outer loop incrementally increases the cardinality of the belief representation (number of states in the belief) $\Delta$ in search nodes for each observer separately. 
The inner iteration performs heuristic guided forward search (we use greedy best first search) to find a plan while tracking at most $\Delta$ states in each $b^i_\Delta$. In the inner loop, the node expansion is guided by (1) customized heuristic function, which computes value of the node based on true goal and belief constraints given by $\Phi$ for the observers, and (2) goal test, which checks for satisfaction of true goal and satisfaction of the belief constraints given by $\Phi$. 
%The outer loop increments the value of $\Delta$ from $1, \ldots, |\mathcal{S}|$. 
The algorithm stops either when a solution is found or when all the $\Delta$ iterations have been explored.

\noindent \textbf{Property} \ In both the solution approaches, we can assert a lower bound on the extent of goal obfuscation and goal legibility for a MO-COPP solution plan. 
In IP, we can specify the aforementioned goal constraints to assert this minimum value, while in the search, the goal tests allow us to assert it. By setting $k$, $j$ to desired values, we can eliminate solutions with low $GD$ score. This affords the following guarantee:

\begin{proposition}
\label{prop:min}
Let $\mathbf{X}$ and $\mathbf{C}$ be perfectly rational adversarial and cooperative observers respectively. Given a MO-COPP = $\langle \Lambda, \mathcal{M}^R, \mathcal{G},$ $\{\Omega_i\}_{i \in \Lambda}, \{\mathcal{O}_i\}_{i \in \Lambda}, \{\mathcal{B}^i_0\}_{i \in \{\mathbf{X}, \mathbf{C}\}} \rangle$ with \unsure{equally probable and independent goals}, $|\mathcal{G}|=n$, and goal constraints of at least $k$ goal obfuscation for $\mathbf{X}$ and at most $j$ goal legibility for $\mathbf{C}$, then a solution plan, $\pi$, 
%computed using either the IP encoding listed in Section 3.1 with constraints \eqref{eq:18} and \eqref{eq:19} or the search algorithm with goal test and heuristic function presented in Section 3.2, 
gives the following guarantees:
\begin{enumerate}
\item $\mathbf{X}$ can infer $G_\mathbf{A}$ with probability $\leqslant 1/k$,
%\vspace{-2pt}
\item $\mathbf{C}$ can infer $G_\mathbf{A}$ with probability $\geqslant 1/j$, and
%\vspace{-2pt}
\item Goal difference metric, $GD(\pi) \geqslant \frac{k-j}{n-1}$
%\vspace{-2pt}
\end{enumerate} 
\end{proposition}

\iffalse
\begin{proof}
From the definition of observation sequence associated with a solution plan, we have, $\forall i \in \{\mathbf{X}, \mathbf{C}\}, G \in \mathcal{G}^\pi_i, \exists s^i_0 \in \mathcal{B}^i_0,~ \hat{\pi} \in ObsSeq^{-1}_i(\pi): \Gamma(s^i_0, \hat{\pi}) \models G$. That is, there exists a valid plan for each goal in $\mathcal{G}^\pi_\mathbf{X}$ and $\mathcal{G}^\pi_\mathbf{C}$. Given $|\mathcal{G}^\pi_\mathbf{X}| \geqslant k$, $X$ can infer $G_\mathbf{A}$ with probability $\leqslant 1/k$. Similarly, given $|\mathcal{G}^\pi_\mathbf{C}| \leqslant j$, $C$ can infer $G_\mathbf{A}$ with probability $\geqslant 1/j$. Also, similarly, $GD(\pi) \geqslant \frac{k-j}{n-1}$.
\end{proof}
\fi

Given $|\mathcal{G}^\pi_\mathbf{X}| \geqslant k$, $X$ can infer $G_\mathbf{A}$ with probability $\leqslant 1/k$. Similarly, given $|\mathcal{G}^\pi_\mathbf{C}| \leqslant j$, $C$ can infer $G_\mathbf{A}$ with probability $\geqslant 1/j$. Also, similarly, $GD(\pi) \geqslant \frac{k-j}{n-1}$. The above proposition states that, based on the observation equivalence there is no additional information revealed about the robot's true goal. Therefore, we can derive goal detection upper bound for $\mathbf{X}$ and lower bound for $\mathbf{C}$. Also this allows us to derive a lower bound on the plan quality. 

\section{Lies}
\label{lies}
\index{Lies}%

Most accounts of explanations studied in this book consider cases where an automated agent shares information about its decision-making model, to accurately justify decisions it has made. It is of course possible that if a system would desire so, it could hijack such a model reconciliation communication process to shape the observer's expectation to be divergent from the true model that is being used by the robot. In this section, we will refer to such manipulation of the model reconciliation process as {\em lies}. As with the implicit communication/behavioral manifestations of deception, lies, and explanations are inexorably connected. One could argue that any poct hoc reconstructive explanations discussed in Chapter \ref{ch07} constitute a form of lies (insofar as the model being communicated was never explicitly used in its decision-making process), or by relying on minimal explanations, the agent is engaging in lies by omission (insofar that they may be relying on the human's false beliefs to simplify explanations). To simplify the discussion, when we refer to lies we are referring to scenarios wherein the agent is communicating information it knows to be false.

\subsection{When Should an Agent Lie?}
To start with, the first question we might want to ask is, why would we ever want an agent that lies? Outside purely adversarial scenarios, are there cases where we might want to employ such systems? Rather than look at scenarios where the agent is taking an explicitly adversarial stance, we will consider cases where the lies are used in service of improving the team utility in some capacity.
%, i.e., our focus would be on scenarios where the robot may need to lie in service of greater good. 
There is enough evidence from social sciences that lies do play a (potentially positive) role in many real world human-human interaction scenarios. 
\unsure{For example consider the interaction between a doctor and a patient. 
There are few scenarios where the role of lies have been thoroughly debated than in doctor-patient interactions. Whether it relates to the prescription of placebo or withholding certain information from the patient, there are many scenarios within a doctor-patient relationship wherein the people have argued for potential usefulness of deception or lies.
%in service of the greater good. 
In the case of human-machine teaming scenario, some potential use cases for such lies include}
\index{Lies!Belief Shaping}%
\index{Lies!White Lies}%
\begin{enumerate}
    \item Belief Shaping: Here the system could choose to shape the beliefs of their potential teammate, through explicit communication, to persuade your teammate to engage behavior the autonomous agent believes would result in higher utility.
    \item White Lies: This case could involve cases where the exact explanation may be too expensive to communicate, and the robot could choose to communicate a simpler `explanation', which while technically contains untrue information is used to justify the optimal behavior. Such white lies could be particularly helpful when it may be expensive computationally or time consuming for the human teammate to process the original explanations.
\end{enumerate}
% \todo{Note that in this book, we are not making a value-judgement or condoning a robot potentially deceiving it' }.
\subsection{How can an Agent Lie?}
Now the question would be how one could generate lies. In general one could leverage the same model space search, and instead of restricting the model updates to only those constrained by the true robot model, in this case the search is free to make any update on the model. If the agent is not allowed to introduce new fluents or new actions, this is still a finite search space for classical planning domains. Though this would constitute a much larger search space than the one considered in the explanation, which brings up an interesting point that {\em at least in computational terms telling the truth might be an easier strategy to following}.

One way to constrain the search space would be to limit the kind of changes that are allowed under the lies. One natural choice may be to limit lies to those that remove factors from the propositional representation of the model (i.e $\Gamma(\mathcal{M}^R_h)$) (Section \ref{ch05:model-rec-sec}). Such lies have been sometimes referred to as {\em lies by omission} in the literature. 
Another possibility may be to leverage theories of model evolution or drift when such information may be available to consider believable changes to models. Or if the lies need to include new possibilities (including new actions or fluents), consider leveraging large language models (like \cite{brown2020language}) or knowledge bases like WordNet \citep{fellbaum2010wordnet} to introduce them in a meaningful way. That is consider existing actions and fluents and use the knowledge base to look for related concepts and try to introduce them into the model.
 
\subsection{Implications of Lies}
\index{Lies!Implications}%
User studies (as reported in \cite{when}) have shown that at the very least lay users seem to be open to the idea of an agent engaging in such deceptive behavior, when it is guaranteed to lead to higher team utility. 
\todo{In this book, we will not investigate the moral and ethical quandaries raised by designing an agent capable of lying, but rather look at a much simpler question --
if we are allowing for deceptive behavior in the hope of higher utility, how does one even guarantee that such behavior would in fact lead to higher team utility?} In some cases, the agent could very well be aware of the fact that it has access to more information than the other humans in the environment (owing to more sophisticated sensors or its location in the environment) and it may be confident in its ability to process and reason with the information it has access to. Though in many scenarios there is a small probability that the information it has access to is incorrect (owing to a faulty sensor or just change in information) or it may have overlooked some factor in its computation. As such the lie could lead to unanticipated negative outcomes. In such cases, it is quite possible that the human teammate could be a lot more critical of its automated teammate that chose to lie as compared to a case where the robot made an unintentional mistake. To us, this speaks for the need for significant further work to be done to not only understand under what conditions an automated system could confidently make such calls, but also better tools to model trust of teammates.

\section{Bibliographical Remarks}

The different obfuscatory behaviors discussed in this chapter have been formulated within the controlled observability planning framework introduced by \cite{unified-anagha}. The work on deceptive planning was introduced by \cite{masters2017deceptive}. The generalized extension of controlled observability framework (referred to as \textsc{mo-copp}) used for balancing goal obfuscation for adversaries with the goal legibility for cooperative observers was introduced by \cite{kulkarni2019signaling}.
In terms of lies, the discussion provided is based on the papers, \cite{when} and \cite{how}, where the works respectively considered when and why the agents should consider generating lies and the computational mechanisms that could be employed towards generating such explanations. Outside of these specific papers, there is a lot of work investigating the utility of lies in various teaming scenarios. There exists a particular extensive literature on the role of lies in decision-making in medical literature \cite{palmieri2009lies}. In particular, many works have argued for the importance of a doctor withholding information from the patient (cf. \cite{korsch1998intelligent} and \cite{holmes1895medical}).

\clearpage
                % bibliography using Author-Year
\chapter{Applications}
\label{ch09}
In this section, we will look at four different applications that leverage the ideas discussed in this book. In particular, all the systems discussed in this chapter will explicitly model the human's mental model of the task and among other things use it to generate explanations. In particular, we will look at two broad application domains. One where the systems are designed for collaborative decision-making, i.e systems designed to help user come up with decisions for a specific task and another system designed for helping users specify a declarative model of task (specifically in the context of dialogue planning for an enterprise chat agent).

\section{Collaborative Decision-Making}
\index{Collaborative Decision-Making}%
Our first set of applications will be centered around systems that are designed to help end-users make decisions.
\section{Humans as Actors}
\index{Humans as Actors}%
\begin{figure}[t!]
\centering 
\includegraphics[width=\columnwidth]{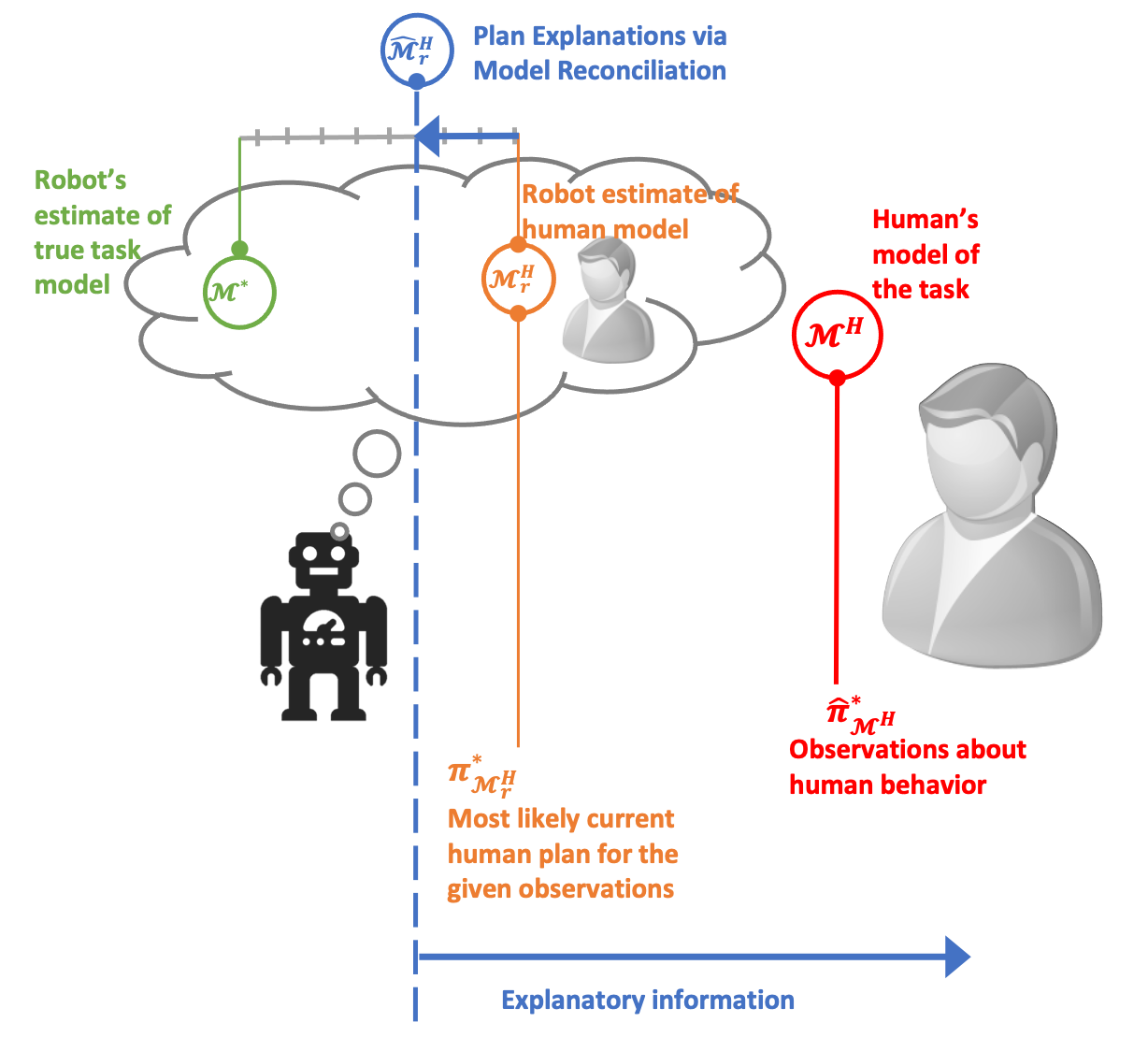}
\caption{Illustration of the scenarios where human is the primary actor and robot the observer.}
\label{ch_10:human_centric}
\end{figure}
All the theoretical and formal discussions in the book until now have focused on scenarios, where the robot is the one acting in the world and the human is just an observer. The specific decision-support scenarios, we will look at in this section, requires us to consider a different scenario, one where the human is the primary actor and the robot is either an observer or an assistant to the robot.
The setting is illustrated in figure \ref{ch_10:human_centric}. 
In this scenario, we have a human with a model $\mathcal{M}^H$ who is acting in the world, and a robot observing the actor. The robot has access to an approximation of the human model $\mathcal{M}^H_r$ and may also have access to a model $\mathcal{M}^*$ which they believe is the true model of the task that the human is pursuing. In addition to the decision-support settings, these settings are also present in settings where the AI system may be trying to teach the human (as in ITS systems \cite{kurt-its1}) and even cases where a robot may be trying to assist the human achieve their goal \cite{chakraborti2015planning}.

Model-reconciliation explanation in this setting would consist of the robot communicating information present in $\mathcal{M}^{*}$, but may be absent in $\mathcal{M}^{H}$ (or the robot believes it to be absent based on its estimate $\mathcal{M}^{H}_r$). One could also formalize a notion of explicable plan in this setting, particularly for cases where the robot may be suggesting plans to the human. In this case, the explicable plan consists of solving for the following objective
\\
\[\textrm{Find:}~ \pi\]
\[\max_{\pi \in \Pi^{\mathcal{M}^H_r}}~    E(\pi, \mathcal{M}^H_r) \]
\[\textrm{Such that}~ \pi ~ \textrm{is executable in both}~\mathcal{M}^H_r~\textrm{and}~\mathcal{M}^*\]\\
That is the goal here becomes to suggest a plan that is explicable with respect to the model $\mathcal{M}^H_r$ in that it is close to the plan expected under the model $\mathcal{M}^H_r$, but is executable in both $\mathcal{M}^H_r$ and $\mathcal{M}^*$. Additionally, we may also want to choose plans whose cost in the model $\mathcal{M}^*$ is low (could be captured by adding a term $-1\times C^*(\pi)$ to the objective).

\subsection{RADAR}
\index{RADAR}%
\index{Decision-Support Systems}%
\index{Proactive Decision Support@PDS}%
The first example, we will consider is a Proactive Decision Support system called RADAR \citep{radar}. Proactive Decision Support (PDS) aims at improving the decision making experience of human decision makers by enhancing both the quality of the decisions and the ease of making them. 
RADAR leverages techniques from automated planning community that aid the human decision maker in constructing plans. Specifically, the system focuses on expert humans in the loop who share a detailed, if not complete, model of the domain with the assistant, but may still be unable to compute plans due to cognitive overload.

\begin{figure*}[t]
\centering
\includegraphics[width=\textwidth]{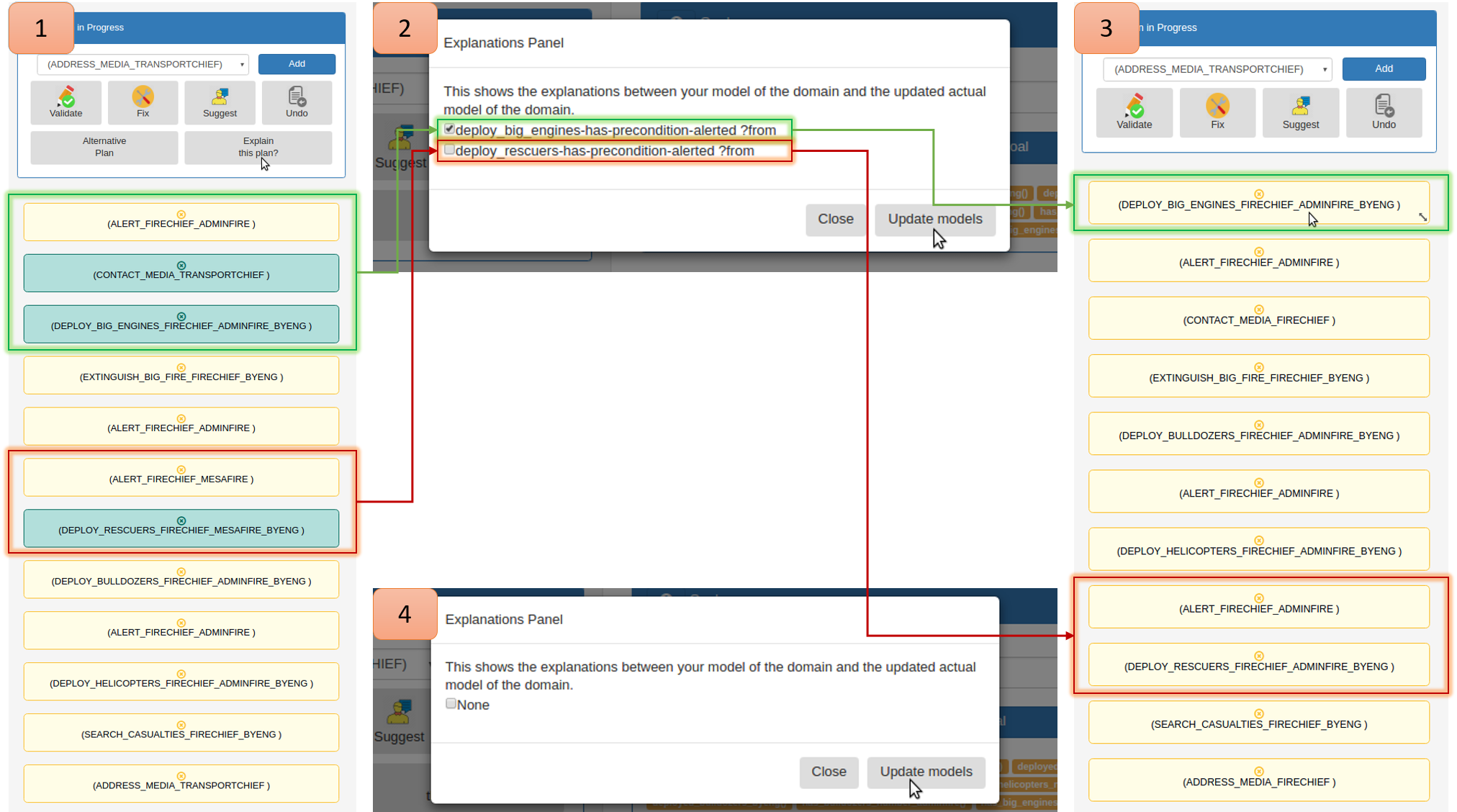}
\caption{RADAR system being applied to a firefighting domain, where the goal is to extinguish a number of fires. (1) RADAR knows that in the environment, the commander needs to inform the fire station's fire chief before deploying big engines and rescuers. In green, Adminfire's fire chief is alerted to deploy big engines from Admin Fire Station. In red, Mesa fire stations' fire chief is alerted to deploy rescuers from Mesa fire station. (2) The human's model believes that there is no need to inform fire chiefs and questions RADAR to explain his plan.  RADAR finds these differences in the domain model and reports it to the human. The human acknowledges that before deploying rescuers one might need to alert the fire chief and rejects the update the fire chief needs to be alerted before deploying big engines. (3) In the alternative plan suggested by RADAR, it takes into account the humans knowledge and plans with the updated model.  (4) Clicking on `Explain This Plan' generates no explanations as there are none (with respect to the current plan) after the models were updated.}
\label{ch12:radar}
\end{figure*}

The system provides a number of useful features like plan suggestion, completion, validation and summarization to the users of the system. The entire system is built to allow users to perform naturalistic decision-making, whereby the system ensures the human is in control of the decision-making process. This means the proactive decision support system focuses on aiding and alerting the human in the loop with his/her decisions rather than generating a static plan that may not work in the dynamic worlds that the plan has to execute in. Figure \ref{ch12:radar}, presents a screenshot of the RADAR system.

The component that is of particular interest to discussions in this book is the ability to support explanations. The system adapts model reconciliation explanations to address possible differences in the planner's model of the domain and the human expectation of it. Such differences could occur if the system may be automatically collecting information from external sources and thus the current estimate of the model diverges from the original specification provided/approved by the user. Here the system performs a model-space search to come up with Minimally Complete Explanation (Chapter \ref{ch05} Section \ref{ch05:explanation_types}) to explain the plan being suggested.
An important distinction in the use of explanations from previous chapter here is the fact that the human has the power to veto the model update if she believes that the planner's model is the one which is faulty, by choosing to approve or not approve individual parts of the explanation.

For example, consider the scenario highlighted in Figure \ref{ch12:radar}, which presents a firefighting scenario where a commander is trying to come up with a plan to extinguish a fire in the city of Tempe. In this scenario, the RADAR system presents a plan (which itself is a completion of some action suggestions made by the commander), which contains unexpected steps to notify the Fire chief at multiple points. When the commander asks for an explanation, the system responds by pointing out that actions for both deploying big fire engines and the deploying rescuers has a precondition that the fire chief should be alerted in its model (which according to the system's model of the commander is missing from the commander's model). The commander responds by agreeing with the system on the fact that the fire chief needs to be informed before deploying rescuers, but also informs the system that fire chief doesn't need to be informed before deploying big fire engines. The system uses this new information to update its own models about the task and generates a new plan.

% built around naturalistic decision making
% Particularly of interest here is explanation
% Unlike the prev cases we don't assume its true
% Suggestion

\subsection{MA-RADAR}
\index{MA-RADAR}%
\index{Multiple Decision Makers}%
\index{Augmented Reality Devices}%
\begin{figure*}[t]
\centering
\includegraphics[width=\textwidth]{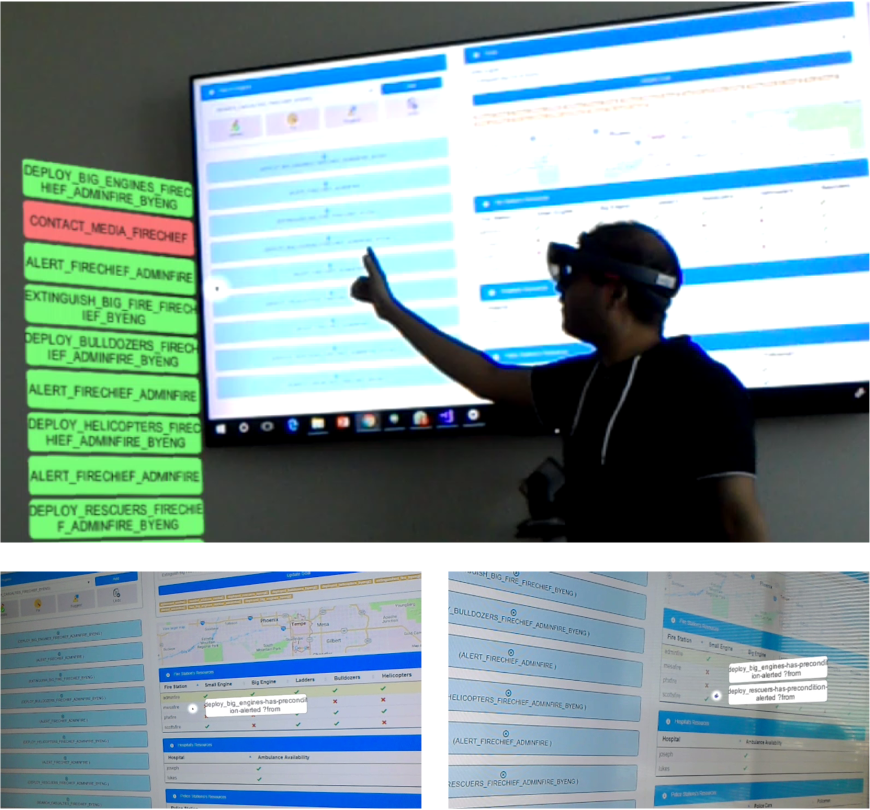}
\caption{A user wearing an augmented reality (AR) device to use the MA-RADAR interface. The figures show, how in addition to the common information shown on the screen the AR device allow users to view information private to each user.}
\label{ch12:ma-radar}
\end{figure*}
Next we will consider yet another extension of RADAR, but one that is focused on supporting multiple user types. In particular, MA-RADAR \citep{ma-radar} considers a scenario where there are multiple users of possibly different backgrounds working together to construct a single solution. Such decision-making scenarios are further complicated when the different users may have differing levels of access and may have privacy concerns limiting the sharing some of the information with other decision-makers in the loop. As such, one of the focuses of  MA-RADAR was to use augmented reality (AR) to allow the different users to view differing views of the task while working together on the same interface. The fact that they are using AR techniques means that in addition to the public information (i.e. the information that all the users can access), each user can view their specific private information which may not be accessible to others. For example, revisiting the firefighting domain, let us consider the case where there are two commanders working together to come up with a plan for controlling the fires. Here in addition to each commanders personal understanding of the task which may be inconsistent with each other and even what the system knows (for example status of some resources etc.), but may have knowledge about certain fluents and actions that are private to each commander and may not necessarily want the other commander to know. In terms of the explanations, each user could have different background knowledge and the system should strive to establish a common ground whenever possible. As such, the system uses the multi-model explanation method discussed in Chapter \ref{ch06} (Section \ref{ch06:model_unc}). Specifically, MA-RADAR combines the individual models into a single annotated model and explanations are generated with respect to this annotated model. The generated explanation is then filtered with respect to each user, so they only view information relevant to them (i.e the information is not redundant as per their model and  is not private to any of the other users). Each user can then use their AR interface to view their specific explanations. Figure \ref{ch12:ma-radar}, presents the augmented reality interface that is shown to the user of the system.

\subsection{RADAR-X}
\index{RADAR-X}%
\index{Preference Elicitation}%
\index{Conflict Set}%
\index{Nearest Feasible Plan}%
\index{Plaussible Subset}%

Next we will consider a variant of the basic RADAR system was extended to support explanatory dialogue and iterative planning. The system, named RADAR-X \citep{radar-x}, assumes that the user may have latent preferences that may not be specified to the system upfront or need not be completely realizable. The explanatory dialogue, particularly contrastive questions (i.e the user asks {\em Why this plan P as opposed to plan Q?}) thus becomes a way for the user to expose their underlying preferences. This system assumes that as a response to a specific plan, the users responds with a partial specification of the plans they prefer. In particular the system expects users to provide partial plans that can be defined by a tuple of the form $\hat{\pi} = \langle \hat{A}, \prec \rangle$, where $\hat{A}$ specifies a multi-set of actions the user expects to see in the plan and $\prec$ specifies the set of ordering constraints defined over $\hat{A}$ that the user expects to be satisfied. A sequential plan is said to satisfy a given partial plan $\hat{\pi}$ if the plan contains each action specified in $\hat{A}$ and they satisfy all the ordering constraints specified in $\prec$.

\unsure{If the partially specified plan is feasible, the system suggests one of the possible instantiations of the partial plan (i.e. a plan that satisfies the given specification) to the decision-maker. The user could possibly further refine their choice by adding more information into the partial specification, until she is happy with the choice made by the system. If the partial specification is not feasible, then the system responds by first providing an explanation as to why the foil is not feasible. This is similar to the explanation generation method specified in Section \ref{contr}, but with the end condition being the case of identifying the updated model where the foils are infeasible.}

Once the explanation is provided, the system tries to generate plans that are closer to the specified foil. The system currently tries to find a subset of the original partial plan that the user specified that can be realized by the system, where for a given partial plan $\hat{\pi}=\langle \hat{A}, \prec\rangle$ a partial plan $\hat{\pi}' = \langle \hat{A}, \prec\rangle$ is considered a subset if $\hat{A}' \subseteq \hat{A}$, and for every $a_1, a_2 \in \hat{A'}$, $a_1 \prec' a_2$ if and only in  $a_1 \prec a_2$. The system currently considers three strategies for coming up with such subsets.
\begin{enumerate}
    \item Present the closest plan: Among the possible maximal subsets the system chooses one of them and presents a plan that corresponds to this subset.
    \item Present plausible subsets: The user is presented with all possible maximal subsets and they can select the one that most closely represents their preferences.
    \item Present conflict sets: The user with minimal conflict sets. That is minimal subset of actions and their corresponding ordering constraints that cannot be achieved together. So the user is asked to make a choice to remove one of the conflicting action from the set.
\end{enumerate}
The system also looks at possible approximations that could be used to speed up the calculations.
% \begin{figure}
% \centering
% \includegraphics[scale=0.3]{figures/chapter-9/fresco.png}
% \caption{The various visualizations provided by the Fresco subsystem.}
% \label{fresco}
% \end{figure}
% \subsection{Mr. Jones}

% \subsection{Fresco}
% Fresco was a subsystem developed as part of a project to developed as part of a proactive decision support system embodied in a smart room. A part of the functionality they provided was a suite of visualization tools that are aimed at externalizing the brain to allow the user being helped to better understand the decision-making process being followed by the assistant. Fresco was one such tool, designed to allow the users to understand the plan in question. In particular, they allow for the system to visualize top-k plans (i.e. K valid plans for the problem with lowest possible cost) and a model based visualization of individual plans that shows how the individual action in the plan interact. The model-based visualization is of particular interest in the context of this book, since it relies on
% \begin{figure}
% \centering
% \includegraphics[scale=0.3]{figures/chapter-9/radar-x.png}
% \caption{A screenshot of the Radar-X system, including the foil specification, the interface for the conflict set and plausible sets.}
% \label{picpic}
% \end{figure}
% \subsubsection{Model Based Visualization}
% \subsubsection{Visualization as Model Reconciliation}
%\subsection{iPASS}
\section{Model Transcription Assistants}
In this section, we will consider the problem of systems designed to allow domain experts to transcribe the models in some declarative form. Here the difference in the mental models actually comes from possible mistakes that the user may make while writing the declarative model. Clearly in this case, the system doesn't have access to the human's mental model and the direction of reconciliation is to get closer to the user's mental model.
So here the process involves generating explanations for specific user queries by assuming that the human's model is an abstract version of the current model. Thus exposing relevant fragments of models that correspond to the relevant behavior and thus letting user directly fix any inconsistencies in the exposed fragment.
\subsection{D3WA+}
\index{D3WA+}%
\index{Goal Directed Conversation Agents}%
\unsure{The tool D3WA+ \cite{d3wa+} implemented this idea in the context of transcribing declarative models for dialogue planning for an automated chat agent. It was an extension of a previous system called D3WA \cite{d3wa} developed by IBM. D3WA allowed dialogue editors to encode their knowledge about the dialogue tree in the form of a non-deterministic planning domain. D3WA+ took this base system and extended by providing debug tools to the domain designers that allowed them to query the system to better understand why the system was generating the current dialogue tree. Figure \ref{ch09:d3wa+} presents a screenshot of the D3WA+ interface.}

In particular, the system focused on providing the domain writer with the ability to raise two types of queries:

\begin{enumerate}
    \item Why are there no solutions?
    \item Why does the generated dialogue tree not conform to their expectation?
\end{enumerate}
The first question is expected to be raised when there exists no possible trace from initial state to goal under the current model specification, and the latter when the domain writer was expecting the dialogue tree to take a particular form that is not satisfied by the one generated by the system. Thus the system would need to explain why the current problem is unsolvable. In the second case, the domain writer has to specify their expected dialogue flow on the storyboard panel. A dialogue flow in this case would consist of possible questions the chat agent could raise and possible outcomes. 
Note the specified flow doesn't need to contain a complete dialogue sequence (which starts at the beginning, ends with the end-user, i.e. the one who is expected to use the chat bot, getting the desired outcome and contains every possible intermediate steps) but could very well be a partial specification of the dialogue flow that highlights some part of the dialogue. Assuming that the expected flow can not be supported by the current model specification, the system would need to explain why this particular flow isn't possible. If the user had specified a complete dialogue sequence, then the system can merely test the sequence in the model and provide the failure point. Though this will no longer be possible if the user only specified a partial foil. \unsure{In such cases, we would need to respond to why any possible sequence that satisfies the partial specification provided by the domain writer will be impossible. This can now be mapped into explaining the unsolvability of a modified planning problem, one that constrains the original problem to only allow solutions that align with the specified flow.}

Thus the answer to both these questions maps into explanations of problem unsolvability. The system here leverages approaches discussed in \cite{sreedharan2019can}, to find the minimal subset of fluents for which the problem is unsolvable. A minimal abstraction over this set of fluents (with others projected out) is then presented to the domain writer as an unsolvable core of the problem that they can then try to fix. In addition to the model abstraction, two additional debugging information is provided to the user. An unreachable landmark and the failure information for an example trace. The unsolvable landmark is extracted by considering the delete relaxation of the unsolvable abstract model. The example trace is generated from the most concrete solvable model abstraction for the original model, so that more detailed plans are provided to the user. Such concrete models are generated by searching for the minimum number of fluents to be projected out to make the problem solvable.
\begin{figure*}
\centering
\includegraphics[scale=0.5]{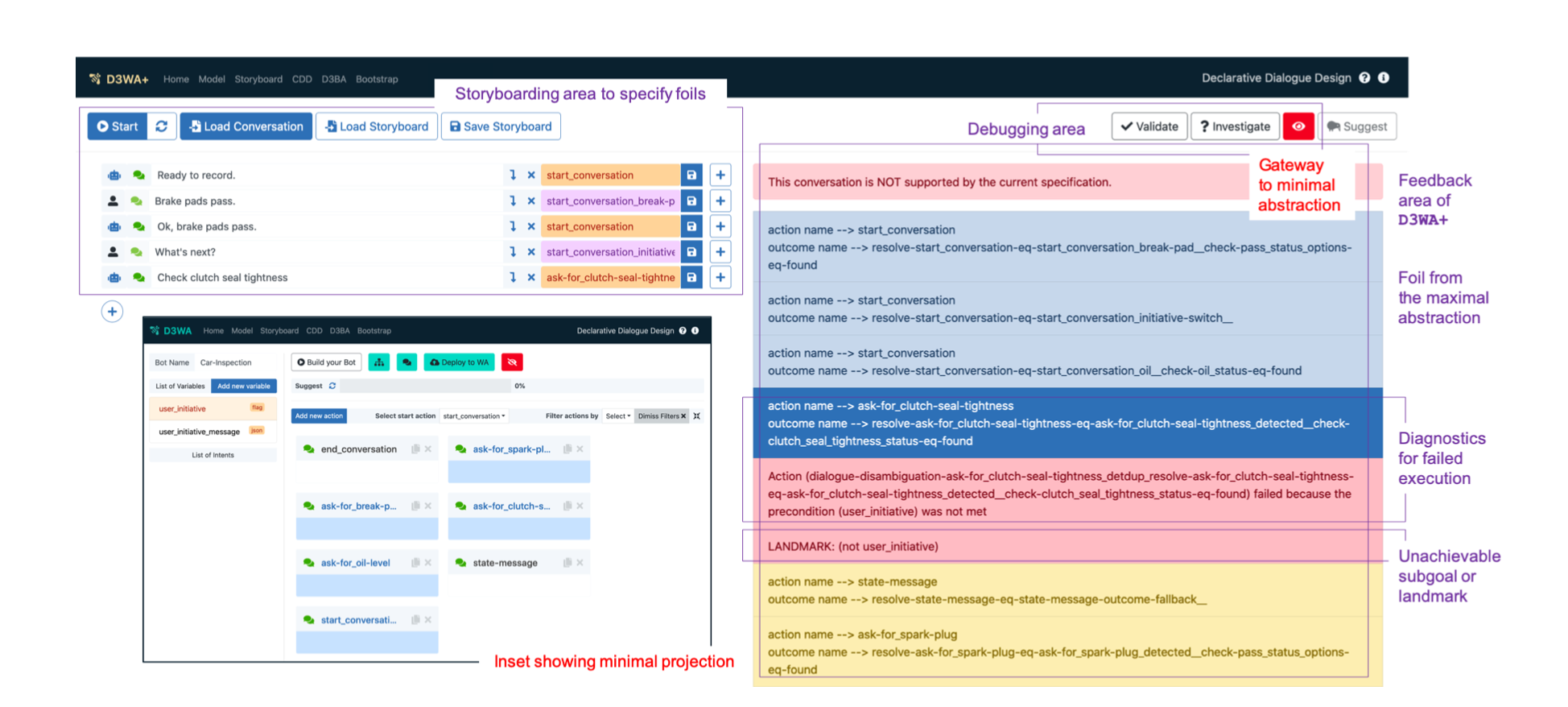}
\caption{An overview of the interface provided by D3WA+ to its end users. Which includes an option to specify foils or expected behavior, a panel that shows an abstract version of the domain and also additional debugging information provided to the user that includes information like: example plan failure, unachievable landmarks etc.}
\label{ch09:d3wa+}
\end{figure*}

\section{Bibliographic Remarks}
\index{FRESCO}%
\unsure{Many of the applications we consider in this chapter have been demonstrated at various conferences including ICAPS and AAAI. RADAR system was first described in \cite{radar}. A version of the system that focused on plan of study generation was described in  \cite{grover2020model}. The paper also presented a user study which validated various components that were part of the system. The multi agent version MA-RADAR was presented in \cite{ma-radar} and was demonstrated in ICAPS 2018 demo track. The contrastive version, i.e., RADAR-X was presented in \cite{radar-x} and was also presented as a demo at AAAI-21. The D3WA+ system was introduced in \cite{d3wa+} , the system was presented as a demo in ICAPS-20 and was awarded the best demo award. Apart from the ones listed in this chapter the ideas presented in the book have also been used in other applications. One prominent one is the FRESCO system \citep{fresco} that was part of an IBM project for a smart room. The system made use of a variant of the model reconciliation for selective plan visualization. They try to model plan visualization as a process of model reconciliation against an empty user model. The process involved identifying the minimum number of components of the actions in the current plan (i.e., the preconditions and effects) that needed to be communicated to the user for the plan to make sense (i.e the plan is valid and optimal). As mentioned earlier, many of the decision-support applications look at settings where the human is the actor and the system has access to a model $\mathcal{M}^H_r$ of the human. All the specific applications discussed assume that $\mathcal{M}^H_r$ is accurate in that it is close to $\mathcal{M}^H$, but this may not be true in general. While this is quite similar to the asymmetry between $\mathcal{M}^R$ and $\mathcal{M}^R_h$. There is a major difference in that unlike the earlier case, the agent with access to the true model, in this case, the human, may not be invested in performing the reconciliation even if she is aware of the asymmetry. In this scenario, the robot may have to initiate the process of reconciliation, one possible way may be to make use of questions to identify information about the human model \cite{grover2020model}.}  
% Image of the interface
% questions
% Explanations for each question

\clearpage
                % bibliography using Author-Year
\chapter{Conclusion}

\label{ch10}
This book presents a concise introduction to recent research on human-aware decision-making, particularly ones focused on the generation of behavior that a human would find explainable or deceptive. Human-aware AI or HAAI techniques are characterized by the acknowledgment that for automated agents to successfully interact with humans, they need to explicitly take into account the human's expectations about the agent.
%Consider a simple scenario with a human and robot, with the robot acting in the world and the human observing robot behavior.
% HAAI prescribes for the robot to succeed
In particular, we look at how for the robot to successfully work with humans, it needs to not only take into account its model $\mathcal{M}^R$, which encodes the robot's beliefs about the task and their capabilities but also take into account the human's expectation of the robot model $\mathcal{M}^R_h$, which captures what the human believes the task to be and what the robot is capable of. 
It is the model  $\mathcal{M}^R_h$ that determines what the human would expect the robot to do and as such if the robot expects to adhere to or influence the human expectations, it needs to take into account this model.
Additionally, this book also introduces three classes of interpretability measures, which capture certain desirable properties of robot behavior. Specifically, we introduce the measures {\em Explicability}, {\em Legibility}, and {\em Predictability}.

% Interpretability metrics
\unsure{In the book, we mostly focused on developing and discussing methods to address the first two of these interpretability measures. We looked at specific algorithms that allow us to generate behaviors that boost explicability and legibility.
We also looked at the problem of generating explanations for a given plan, and how it could be viewed as the use of communication to boost the explicability score of a selected plan by updating human expectations. Additionally, we looked at variations of this basic explanation framework under differing settings, including cases where the human models of the robot may be unknown or where the robot's decision-making model may be expressed in terms that the human doesn't understand. We also saw a plan generation method that is able to incorporate reasoning about the overhead of explanation into the plan selection process, thereby allowing for a method that is able to combine the benefits of purely explicable plan generation methods and those that identify explanations after the plan has been selected.}

\unsure{Additionally, we also saw that modeling human expectations not only allows us to create interpretable behaviors but also provides us with the tools needed to generate deceptive and adversarial behaviors. In particular, we saw how one could leverage the modeling of the other agent to create behaviors that obfuscate certain agent model information from an adversarial observer or even deceive them about the model component. We also saw how model reconciliation methods could be molded to create lies, which even in non-adversarial scenarios could help the agent to achieve higher team utility at the cost of providing some white lies to the human.}

% In the concluding chapter.. a quick overview of some topics that have been overlooked
\unsure{
While the problem of generating explanations for AI decisions or developing deceptive AI agents has been studied for a while, the framing of these problems within the context of human-aware AI settings is a relatively recent effort. As such, there exists a number of exciting future directions to be explored and technical challenges to overcome, before we can have truly human-aware systems. Some of these challenges are relatively straightforward (at least in their conception), as in the case of scaling up the methods and applying them within more complex scenarios that are more faithful to real-world scenarios. Then there are challenges that may require us to rethink our current strategies and whose formalization itself presents a significant challenge. Here we would like to take a quick look at a few of these challenges.}

\paragraph{Creating a Symbolic Middle Layer}
\index{Symbolic Middle Layer}%
% Symbols -- do we need em
The necessity of symbols in intelligent decision-making has been a topic that has been widely debated within the field of AI for a very long time. Regardless of whether symbols are necessary for intelligence, it remains a fact that people tend to communicate in terms of symbols and concepts. As such, if we want to create systems that can interact with people effectively they should be capable of communicating using symbols and concepts people understand. In chapter \ref{ch07}, we have already seen an example, where an agent translates information about its model into terms that are easier for a human to understand, but if we want to create successful systems that are truly able to collaborate with humans then we need to go beyond just creating explanation generation systems. They need to be able to take input from the human in symbolic terms even when the model may not be represented in those terms. Such advice could include information like instructions the agent should follow, possible domain, and preference information. Interestingly the advice provided by the human would be colored by what they believe the agent model to be, and as such correct interpretation of the specified information may require analyzing the human input in the light of their expectation about the agent model. Some preliminary works in this direction include \cite{guan2020explanation} and for a discussion on the overall direction, the readers can refer to \cite{kambhampatisymbols}.
\paragraph{Trust and Longitudinal Interaction}
\index{Trust}%
\index{Longitudinal Interaction}%
Most of the interactions discussed in this book are single-step interactions, in so far that they focus on the agent proposing a plan for a single task and potentially handling any interaction requirements related to that plan. But we are not focused on creating single-use robots. Rather we want to create automated agents that we can cohabit with and work with on a day-to-day basis. This means the robot's choice of actions can no longer be made in isolation, rather it should consider the implications of choosing a certain course of action on future interaction with humans. We saw some flavors of such consideration in methods like Minimally Monotonic Explanations (MME) (Chapter \ref{ch05}, Section \ref{ch05:explanation_types}) and the discounting of explicability over a time horizon (Chapter \ref{ch03}, Section \ref{Ch03:Design}), though these are still limited cases. As we move forward, a pressing requirement is for the robots to be capable of modeling the level of trust the human holds for the robot and how the robot's actions may influence the human trust.
Such trust-level modeling is of particular importance in longitudinal settings, as it would be the human trust on the robot that would determine whether or not the human would choose to work with it on future tasks.
Thus the impact of the robot action on human trust should be a metric it should consider while coming up with its actions. At the same time, reasoning about trust levels also brings up the question of trust manipulation and how to design agents that are guaranteed to not induce undeserved trust in its capabilities. Some preliminary works in this direction are presented in \cite{zahedi2021trust}.
\paragraph{Learning Human Models}
\index{Learning Human Models}%
The defining feature of many of the techniques discussed in this book is the inclusion of the human's expectations, sensory capabilities, and in general their models into the reasoning process. In many scenarios, such models may not be directly available to the robot and it may be required to learn such models either by observing the human or by receiving feedback from the human. We have already seen examples of learning specific model proxies that are sufficient to generate specific classes of behaviors. But these are specialized models meant for specific applications. More general problems may require the use of additional information and even the complete model. While the model $\mathcal{M}^H$ could be learned by observing the human behavior, the model $\mathcal{M}^R_h$ is usually more expensive to learn as it requires the human to either provide feedback on the robot behavior or provide the full plan they are expecting from the robot. As such learning a fully personalized model for a person may require too much information to be provided by a single person. A more scalable strategy may be to learn approximate models for different user types from previously collected data. As and when the robot comes into contact with a new human, they can use the set of learned models to identify the closest possible model. This model can act as a starting point for the interaction between the system and the user and can be refined over time as the robot interacts with the human. Some initial work in this direction can be found in \cite{soni2021not}.
\paragraph{Super Human AI and Safety}
% unintentional harm
Currently, most successful AI systems are generally less competent than humans overall but may have an edge over people on specific narrow tasks. There is no particular reason to believe that this condition should always persist. In fact, one could easily imagine a future where AI systems outpace humans in almost all tasks. It may be worth considering how the nature of human-robot interaction may change in such a world. For one thing, the nature and goal of explanation may no longer be about helping humans understand the exact reason for selecting a decision but rather about giving a general sense of why the decisions make sense. For example, it may be enough to establish why the decision is better than any alternative the human could come up with. Going back to the problem of vocabulary mismatch introduced in Chapter \ref{ch07}, it may very well be the case that there may be concepts that the system makes use of that have no equivalent term in human vocabulary and as such explanation might require the system teaching new concepts to the human. The introduction of the ability to reason and model the human mental model also raises additional safety and ethical questions in the context of such  superhuman human-aware AI. For one, the questions of white lies and deception for improving team utility (Chapter \ref{ch08}, Section \ref{lies}) takes on a whole new dimension and could potentially head into the realm of wireheading. It is very much an open question as to how one can build robust methods and safeguards to control for and avoid such potential safety concerns that may arise from the deployment of such systems. As we see more AI systems embrace human-aware AI principles, it becomes even more important to study and try ameliorate unique safety concerns that arise in systems capable of modeling and influencing human's mental models and beliefs.

\clearpage                % bibliography using Author-Year
    \renewcommand{\sectionmark}[1]{\markright{#1}}
    \addcontentsline{toc}{chapter}{Bibliography}
    \bibliographystyle{plainnat}
    \bibliography{bib}
    \cleardoublepage
    
%blankpage

\chapter*{Authors' Biographies}
\markboth{AUTHORS' BIOGRAPHIES}{AUTHORS' BIOGRAPHYIES}
\addcontentsline{toc}{chapter}{\protect\numberline{}{Authors' Biographies}}

\section*{Sarath Sreedharan}
\begin{wrapfigure}{L}{0.2\textwidth}
    \includegraphics[width=0.2\textwidth]{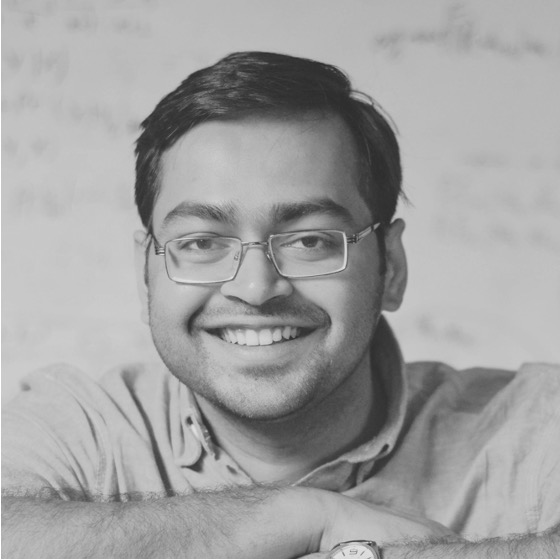}
\end{wrapfigure}
\textbf{Sarath Sreedharan} is a Ph.D. student at Arizona State University working with Prof. Subbarao Kambhampati. His primary research interests lie in the area of human-aware and explainable AI, with a focus on sequential-decision making problems. Sarath's research has been featured in various premier research conferences, including IJCAI, AAAI, AAMAS, ICAPS, ICRA, IROS, etc, and journals like AIJ. He was also the recipient of Outstanding Program Committee Member Award at AAAI-2020.
\\
\section*{Anagha Kulkarni}
\begin{wrapfigure}{L}{0.2\textwidth}
    \includegraphics[width=0.2\textwidth]{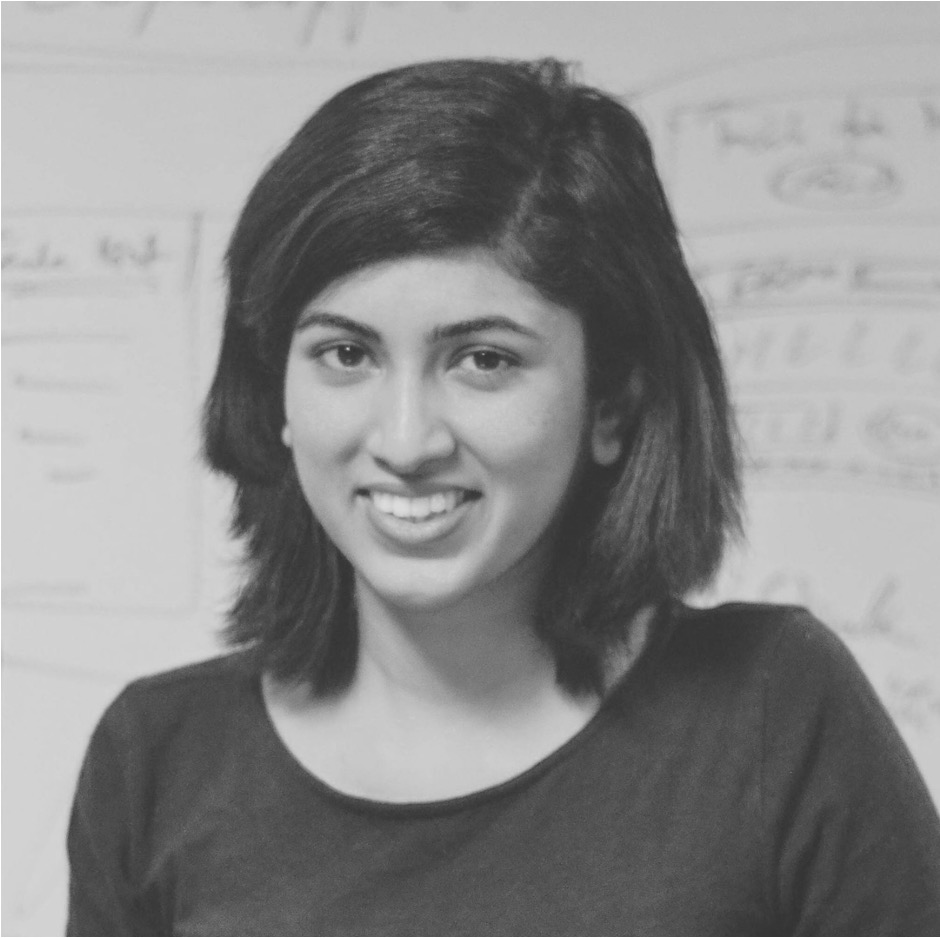}
\end{wrapfigure}
\textbf{Anagha Kulkarni} is an AI Research Scientist at Invitae. Before that, she received her Ph.D. in Computer Science from Arizona State University. Her Ph.D. thesis was in the area of human-aware AI and automated planning. Anagha's research has featured in various premier conferences like AAAI, IJCAI, ICAPS, AAMAS, ICRA and IROS.   
\\
\section*{Subbarao Kambhampati}
\begin{wrapfigure}{L}{0.2\textwidth}
    \includegraphics[width=0.2\textwidth]{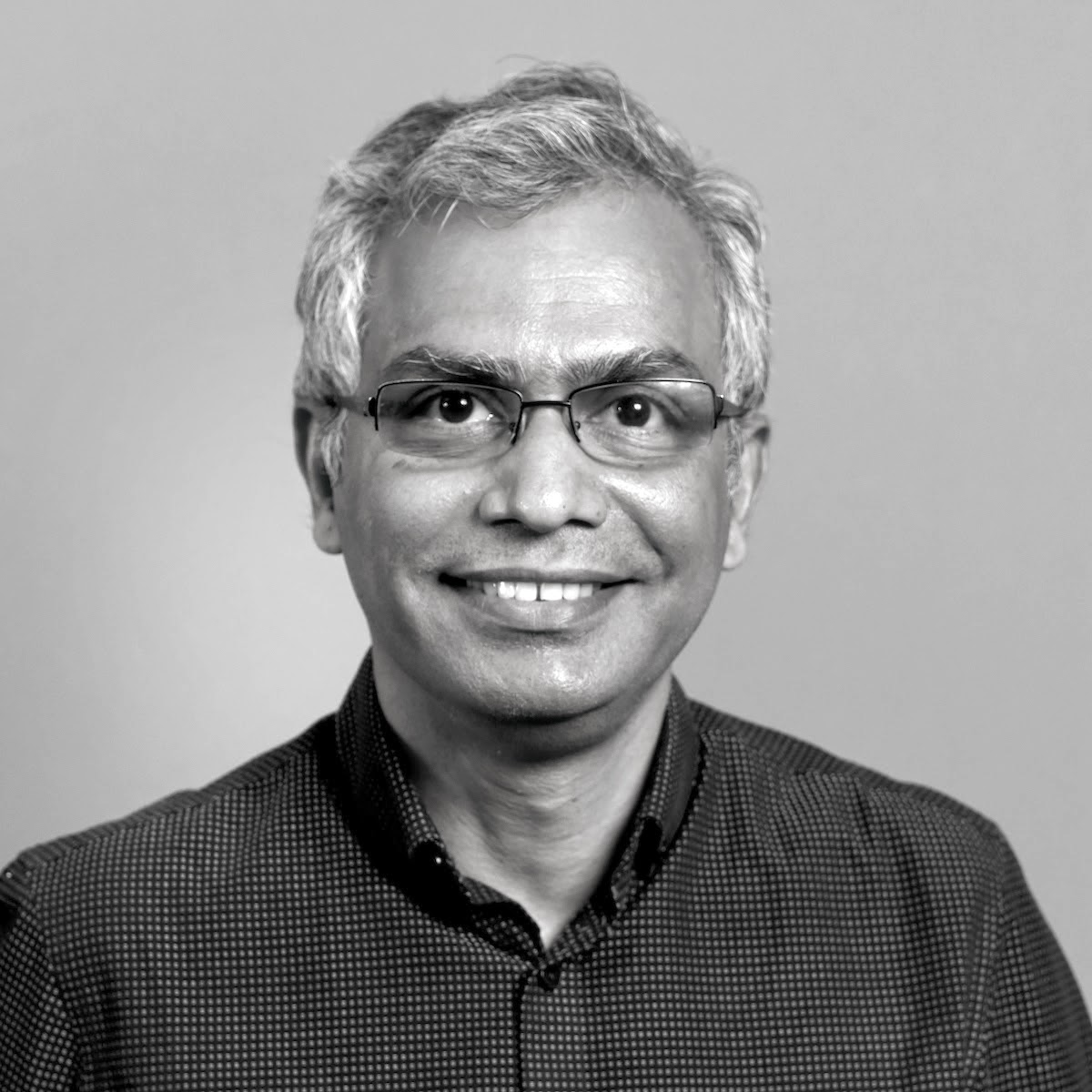}
\end{wrapfigure}
\textbf{Subbarao Kambhampati} is a professor in the School of Computing \& AI at Arizona State University. Kambhampati studies fundamental problems in planning and decision making, motivated in particular by the challenges of human-aware AI systems. He is a fellow of Association for the Advancement of Artificial Intelligence, American Association for the Advancement of Science,  and Association for Computing machinery, and was an NSF Young Investigator. He was the president of the Association for the Advancement of Artificial Intelligence, trustee of International Joint Conference on Artificial Intelligence, and a founding board member of Partnership on AI. Kambhampati’s research as well as his views on the progress and societal impacts of AI have been featured in multiple national and international media outlets. 
%He writes a column on the societal and policy implications of the advances in Artificial Intelligence for The Hill. He can be followed on Twitter @rao2z.  

%\textbf{FirstName LastName} ...

    % \appendix
    % \pretocmd{\chapter}{\pagenumbering{arabic}
	   %                        \renewcommand*{\thepage}{\thechapter.\arabic{page}}
    %                        }{}{} 
    % \include{appendix1}
\end{document}